\documentclass[ twoside,openright,titlepage,fleqn, 
headinclude,BCOR5mm,%
a4paper,footinclude,abstractoff, 
 listsseparated,chapterprefix,headsepline,titlepage
]{scrreprt}

\KOMAoptions{%
 paper=a4,%
cleardoublepage=empty,%
numbers=noenddot,%
parskip=half,%
footinclude=true%
}

\usepackage[usenames,dvipsnames]{color}
\usepackage{classicthesis-ldpkg} 

\hypersetup{%
    linktocpage=true, pdfstartpage=1, pdfstartview=FitV,%
   breaklinks=true, pdfpagemode=UseNone, pageanchor=true, pdfpagemode=UseOutlines,%
   plainpages=false, bookmarksnumbered, bookmarksopen=true, bookmarksopenlevel=1,%
    hypertexnames=true, 
   colorlinks=true, 
  linkcolor=RoyalBlue, citecolor=Green, 
   urlcolor=Black, linkcolor=Black, citecolor=Black, 
}

\usepackage[T1]{fontenc}
\usepackage{mathpazo}

\usepackage[utf8]{inputenc}
\usepackage[american]{babel}
\usepackage{amsmath}
\usepackage[shortlabels]{enumitem}
\usepackage{lscape}
\usepackage{ifpdf}
\usepackage{amssymb}
\usepackage{amsthm}
\usepackage{listings}
\usepackage{multicol}
\usepackage{ifthen}

\newtheorem{thm}{Theorem}[section]
\newtheorem{cor}[thm]{Corollary}

\newcounter{axm}
\newtheorem{axiom}[axm]{Axiom}
\theoremstyle{definition} 
\newtheorem{defn}[thm]{Definition}

\usepackage{hyperref}
\usepackage{graphicx}
\setkomafont{disposition}{\normalcolor\bfseries}

	   \areaset[5mm]{400pt}{753pt} 
\RequirePackage[automark]{scrpage2} 

\clearscrheadings
    \lehead{\thepage \hfill \headmark}
    \rohead{\headmark \hfill \thepage}

\newcommand{\titel}{General Riemannian SOM\xspace} 

\newcommand{\ich}{Jascha Schewtschenko\xspace}
\newcommand{\betreuera}{Prof. Dr. Helge Ritter\xspace}
\newcommand{\betreuerb}{Dr. J\"{o}rg Ontrup\xspace}
\newcommand{\fakultaet}{Technische Fakult\"{a}t\xspace}
\newcommand{\ag}{AG Neuroinformatik\xspace}
\newcommand{\uni}{\protect{Universit\"{a}t Bielefeld}\xspace}

\newcommand{\abgabe}{M\"{a}rz 2009\xspace}
\newcommand{\typ}{Diplomarbeit an der Technischen Fakult\"at\xspace}

\newcommand{\set}[1]{\left\{#1\right\}}
\newcommand{\setgen}[2]{\left\{ #1 | #2 \right\} }

\newcommand{\ind}[1]{\ensuremath{\stackrel{(#1)}=}}
\newcommand{\pdiff}[2]{\ensuremath{\frac{\partial #1} {\partial #2}}}
\newcommand{\Landau}{\mathcal{O}}
\newcommand{\Tr}{\mathrm{Tr}}
\newcommand{\expval}[1]{\langle #1 \rangle}
\newcommand{\norm}[1]{\ensuremath{\left\|#1 \right\|}}

\newcommand{\matrixzwei}[1]{\ensuremath{\left(\begin{array}{cc}#1\end{array}\right)}}
\newcommand{\matrixdrei}[1]{\ensuremath{\left(\begin{array}{ccc}#1\end{array}\right)}}

\newcommand{\setR}{\ensuremath{\mathbb{R}}}
\newcommand{\setZ}{\ensuremath{\mathbb{Z}}}
\newcommand{\setN}{\ensuremath{\mathbb{N}}}
\newcommand{\setC}{\ensuremath{\mathbb{C}}}

\newcommand{\acosh}{\ensuremath{\mathrm{acosh}}}
\newcommand{\asinh}{\ensuremath{\mathrm{asinh}}}
\newcommand{\atanh}{\ensuremath{\mathrm{atanh}}}

\newcommand{\atantwo}{\ensuremath{\mathrm{atan2}}}

\newcommand{\eps}{\ensuremath{\varepsilon}}
\newcommand{\ho}{\ensuremath{h_{rs}^0}}
\newcommand{\hop}{\ensuremath{h_{r's}^0}}

\newcommand{\err}[1]{\tiny{(\ensuremath{\pm} #1)}}
\usepackage{setspace} 
\usepackage{blindtext}
\usepackage{wrapfig}
\usepackage{verbatim} 
\usepackage{enumitem}
\usepackage{subfig}
\usepackage{url}
\usepackage{listings}
\usepackage{placeins} 

\usepackage[rightcaption]{sidecap} 

\renewcommand\url{\begingroup \def\UrlLeft{\ensuremath{<}}%
                                   \def\UrlRight{\ensuremath{>}}%
                                   \urlstyle{tt}\small \Url}

\usepackage{chngcntr}
\counterwithout{footnote}{chapter}

\makeatletter
\newcommand\wrapfill{\par
\ifx\parshape\WF@fudgeparshape
\nobreak
\vskip-\baselineskip
\vskip\c@WF@wrappedlines\baselineskip
\allowbreak
\WFclear
\fi
}
\makeatother

\usepackage{float}
\definecolor{BoxGrey}{rgb}{0.91, 0.91, 0.91}

\newcommand{\elaboxname}{Kasten}
\newfloat{elabox}{box}{\elaboxname}
\makeatletter%
{\begin{elabox}[#1]{}\begin{lrbox}{\@tempboxa}
\refstepcounter{boxes}
\addcontentsline{box}{boxes}
{\protect\numberline{\thechapter.\arabic{boxes}}#2}\par
\begin{minipage}[c]{0.98\columnwidth}%
\textsc{\elaboxname~\arabic{boxes}:  #2}\newline\noindent}
{\end{minipage}\vspace{-0.5cm}\end{lrbox}\colorbox{BoxGrey}{\usebox{\@tempboxa}}\end{elabox}}\makeatother

\usepackage{chngpage}
{\begin{adjustwidth}{42 pt}{1cm} \vspace{0.4cm} \begin{small}}
{\end{small}  \end{adjustwidth} 
}

\newcounter{Lcount}
{\begin{list}{(\alph{Lcount})}
    {\usecounter{Lcount}
  \setlength{\rightmargin}{\leftmargin}}
  \setlength{\itemindent}{0pt} 
 \setlength{\labelsep}{0.5em} 
  \setlength{\itemsep}{0pt} 
}
{\end{list}}

\usepackage{caption} 
\newcommand{\tocsection}[1]{\phantomsection \addcontentsline{toc}{chapter}{#1}}

\newcommand{\zb}[1]{beispielsweise }

\newcommand{\unterschriftsfeld} {\vspace*{1.5cm}
\hspace{0.55\textwidth } Bielefeld, ~\abgabe
}

\newcommand{\abschliessendeerklaerung}[1]{
\cleardoublepage
\null
\vfill

\begingroup
\let\clearpage\relax
\let\cleardoublepage\relax
\let\cleardoublepage\relax
\chapter*{Versicherung} 
Versicherung gem\"{a}\ss{} Paragraph 20, Absatz 9 der Diplompr\"{u}fungsordnung f\"{u}r den Studiengang Naturwissenschaftliche Informatik an der Technischen Fakult\"{a}t der Universit\"{a}t Bielefeld vom 1. April 2003.\bigskip\\ 
Hiermit versichere ich, da\ss{} ich die vorliegende Diplomarbeit selbst\"{a}ndig erarbeitet und keine anderen als die angegebenen Quellen und Hilfsmittel benutzt sowie Zitate kenntlich gemacht habe. \\ 

\unterschriftsfeld
\endgroup
\thispagestyle{empty}

}

\newcounter{axiom:line}
\newcounter{axiom:metric} 
\newcounter{axiom:param} 
\newcounter{axiom:HP}
\newcounter{axiom:angles} 
\newcounter{axiom:SAS}
\newcounter{axiom:parallel}

\begin{document}
\frenchspacing
\raggedbottom


\pagenumbering{roman}
\pagestyle{plain}
\begin{titlepage}
\setcounter{page}{-1}
\setlength\oddsidemargin{0.95cm}
    \begin{center}
        \begingroup
        	\vspace*{1cm}\large  
            \color{Maroon}\LARGE\titel 	\\ \bigskip
        \endgroup

        \vfill

        \typ \\ der \uni \bigskip \\
        \abgabe

        \vfill             
		{\scshape \ich}        
		\vfill         

		Betreuer~/~Pr\"{u}fer\\  
		\betreuera ~\\ 
		\betreuerb ~\\ 
		~ \\ 
		\uni \\ 
		\fakultaet \\ 
		\ag \\ 
		Universit\"{a}tsstra\ss e 25\\ 
		33615 Bielefeld \\		

		\vspace*{1.32cm}       
		\includegraphics[height=2cm]{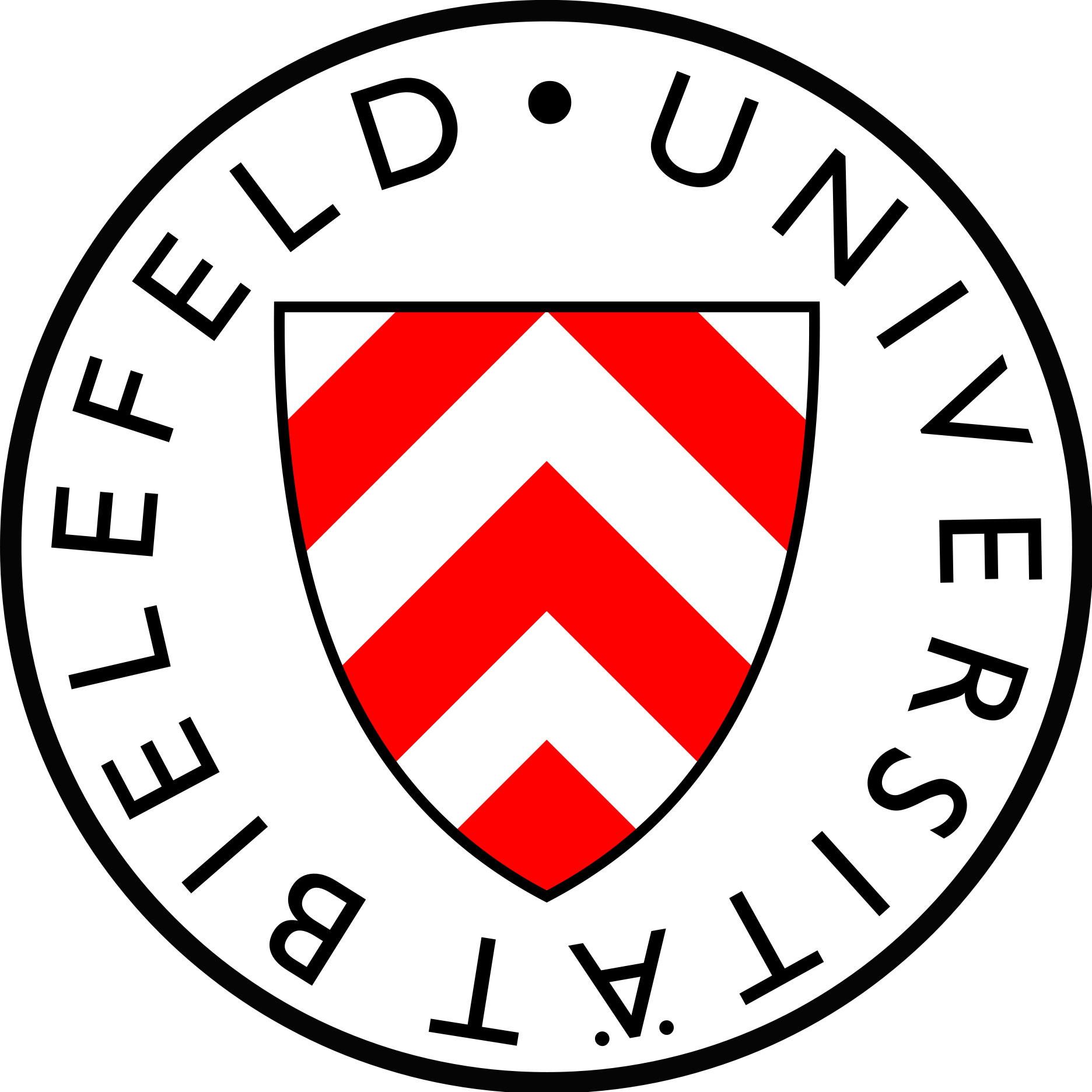}
    \end{center}        
\end{titlepage}

\cleardoublepage
\begin{titlepage}

\hfill
\vfill
\begin{figure}[H]
\includegraphics[width=\linewidth]{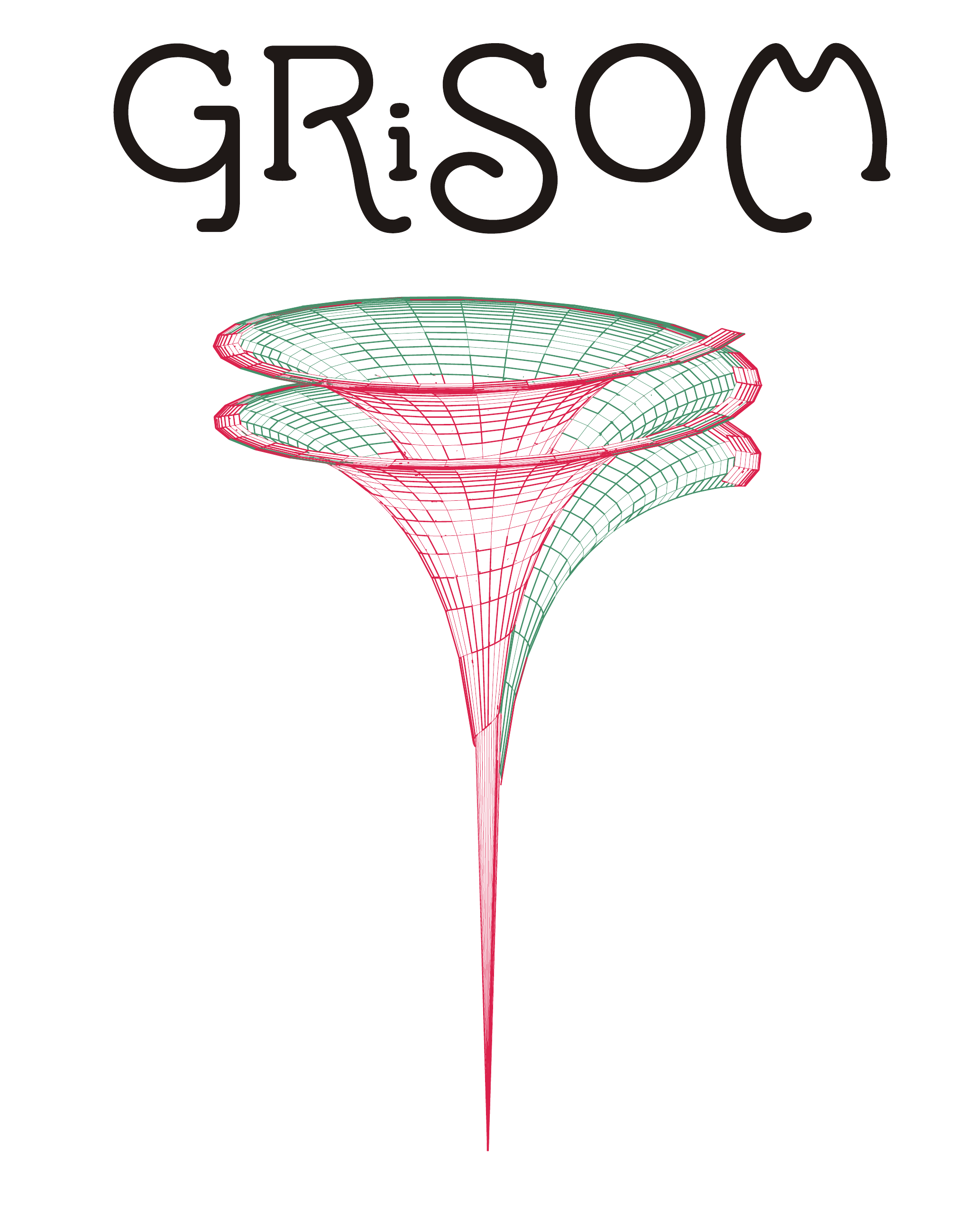}
\end{figure}

\vfill

\noindent\ich: \textit{\titel,} \typ der \uni, 
\textcopyright{}~\abgabe\  --- \LaTeXe

\end{titlepage}

\setcounter{page}{0}

%
%
%
%
%

\cleardoublepage
\setcounter{page}{1}            



\null
\vfill
\tocsection{Acknowledgements}
\begin{flushright}{\slshape    
	Problems worthy of attack prove their worth by fighting back.
	} \\ \medskip
         --- Paul Erdos 
\end{flushright}

\bigskip

\begingroup
\let\clearpage\relax
\let\cleardoublepage\relax
\let\cleardoublepage\relax
\chapter*{Acknowledgments}
An dieser Stelle m\"{o}chte ich all denjenigen Personen danken, die die Entstehung dieser Diplomarbeit in irgend einer Form positiv beeinflusst haben. Meinem Betreuer Prof. Dr. Helge Ritter gilt daher der besondere Dank f\"{u}r die aufgebrachte Geduld und die Bereitschaft zu den Diskussionen, bei denen sich meist zu viele inspirierende Anregungen ergaben, dass diesen gar nicht allen im Rahmen dieser Arbeit nachgegangen werden konnte.\\
\\
Einen weiteren ganz besonderen Dank m\"{o}chte ich zudem meinen Eltern zukommen lassen. Ihr stets bedingungsloser Beistand w\"ahrend meiner gesamten schulischen und akademischen Laufbahn hat es erst erm\"{o}glicht, dass ich dieses Studium heute in dieser Form abschlie\ss en kann.\\
\\
Nicht unerw\"{a}hnt lassen m\"ochte ich auch die unsch\"atzbare Hilfe durch meinen Freunden und B\"urokollegen, allen voran Leena Hanisch, Janina de Jong, Sven Kanies und Stefan Vitz, die durch ihre moralische Unterst\"utzung, Motivation und lektorischen F\"ahigkeiten in der englischen Sprache ihren eigenen Beitrag zum erfolgreichen Abschluss dieser Arbeit geleistet haben. Ihre Geduld, sich meine Arbeitsergebnisse auch meist ohne eigenes Vorwissen auf dem Gebiet anzuh\"oren und geschickt zu hinterfragen, half mir des weiteren oft, so manche H\"urde aus neuen Blickwinkeln zu betrachten und so erst zu l\"osen. Ihnen allen sei daher an dieser Stelle ausdr\"ucklich gedankt.\\
\\
Mein abschliessender Dank gilt zu guter Letzt den unz\"ahligen fleissigen und selbstlosen Mitmenschen, die ihren Beitrag zu der Entwicklung all der freier Software geleistet haben, die ich im Zuge dieser Arbeit genutzt habe.

\unterschriftsfeld

\endgroup

\cleardoublepage



\cleardoublepage

\tableofcontents
\markboth{\contentsname}{\contentsname} 

\pagestyle{scrheadings}
\cleardoublepage

\part{Preamble}
\addcontentsline{toc}{chapter}{Motivation}
\chapter*{Motivation}

Since its invention in 1982 by Teuvo Kohonen, the study of self-organzing maps has become a vast field of research in computer science. In more than 7000 publications based on the SOM methods (cf.\cite{biblio1},\cite{biblio2}) many properties as well as practical applications of the algorithm have been examined and presented and the progress still goes on. This thesis contributes to it by focusing on one already encountered as well as on one new aspect of the SOM that shall be motivated in the following sections.

\section*{The three ``travelling ruler problems''}

The possibility to tackle the famous travelling salesman problem (TSP) by using Kohonen's maps was for the first time discovered by Durbin and Willshaw \cite{durbin}. They used a closed ring of neurons to represent the (shortest) path for the salesman. By adapting the Self-organizing map based on the position of the cities, they obtained a good approximation of the solution of the problem even if it is not guaranteed that the global minimum for the path length is always found. However this method has been always restricted to the case where the cities lie in an Euclidean plane. Thus the question arose for us whether the classical SOM can be generalized in a way to also solve the case where he cities may lie on more general surfaces or spaces. The following three examples named the ``travelling ruler problems'' shall briefly illustrate what is meant by that.

\begin{figure}[!ht]
\begin{center}
\includegraphics[width=0.32\linewidth]{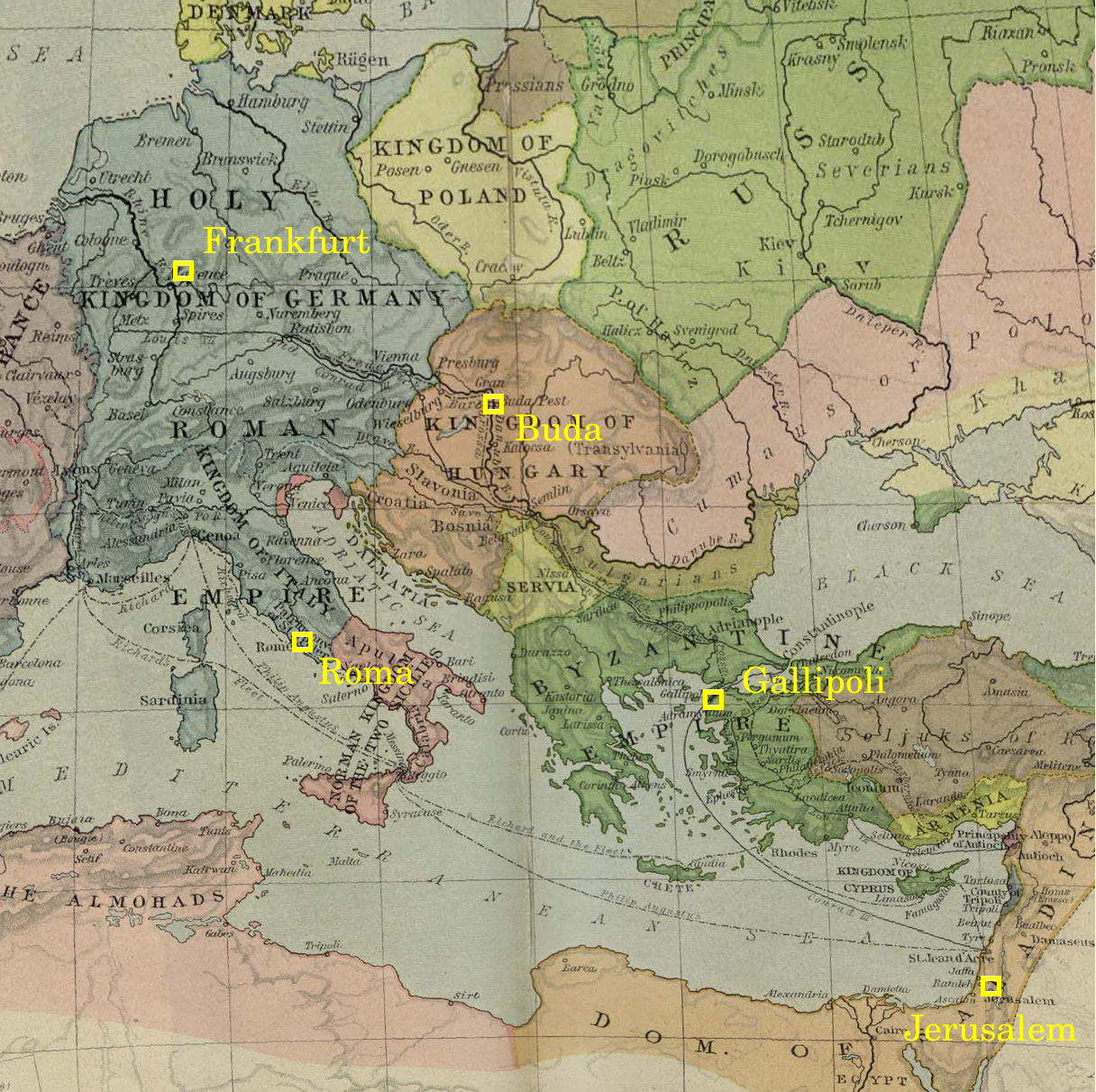}
\includegraphics[width=0.31\linewidth]{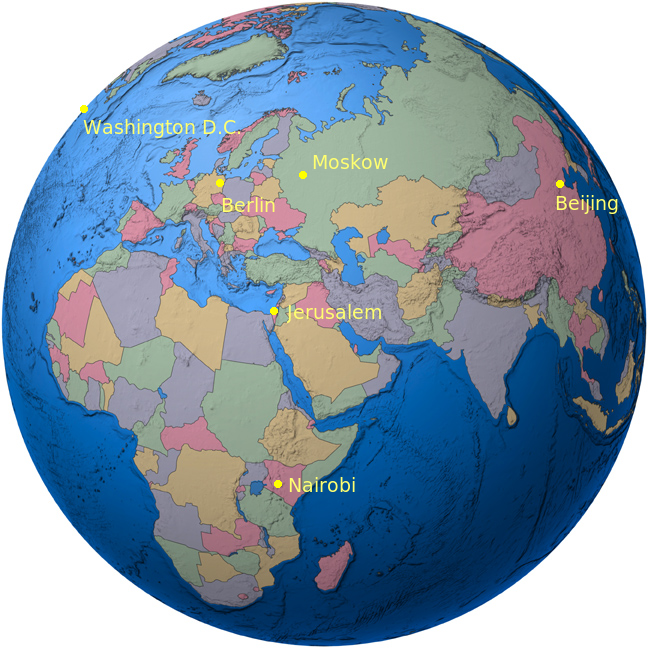}
\includegraphics[width=0.33\linewidth]{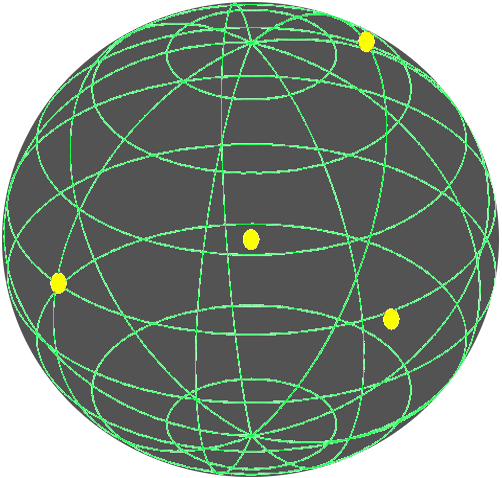}
\end{center}
\caption{Traveling ruler problems: (left) flat (middle) spherical (right) hyperbolic space}
\label{fig:intro:trp}
\end{figure}

\subsubsection*{Emperor Frederick I Barbarossa}
In the year of our lord 1189, Frederick I Barbarossa is about to prepare his crusade to the holy land. He plans to visit at least several cities on his way. They are been shown in the left map in Fig.\ref{fig:intro:trp}. Starting in Frankfurt where he has been elected and crowned many years before he wants to meet the pope in Rome, accompany his son to Hungary where he will get engaged with a local princess, cross the strait between Europe and Asia at Gallipoli and finally ``liberate'' the holy city of Jerusalem. According to the doctrine of the church, this medieval world is flat. So, Frederick could just plan his travel by using a map and a ruler (at least if he wants to take the straight route ignoring any obstacles).

\subsubsection*{Pres. Barack Obama} 
More than 800 years later, the newly-elected US-President Barack Obama is about to make his first official trip abroad. He wants to meet his European allies as well as inspect several ``conflict zones'' of the past and present, namely Russia, China and the Middle East. Furthermore he plans to take the opportunity to also visit some relatives in Nairobi, Xenia. Unlike the medieval world of Barbarossa his world is no longer flat because in the meantime an Italian guy named Gallilei (re)discovered that the surface of the world is actually curved like a sphere. Thus in order to find now the shortest flight route for the ``Air Force One'' this has to be taken into account.

\subsubsection*{President Kse'nu} 
Again a few hundred years later the (fictive) President Kse'nu of the Galactic Confederacy encounters his ``travelling spaceman problem'' as he wants to visit his dominions that are located in several galaxies. Since scientists have finally proven today's conjecture that the space is in fact also globally curved, but this time hyperbolically, this non-Euclidean geometry affects the route that he has to pick to minimize the flight path for his pan-galactic cruiser.
\\ \\
While the problem of Barbarossa in the Euclidean world matches the case formerly discussed by Durbin and Willshaw, the classical SOM does not provide a way to approach the latter two. Our first goal shall therefore be to extend the SOM algorithm, in order to handle non-Euclidean or even more general feature spaces, and to present an implementation of this \emph{General Riemannian SOM} (GRiSOM) model.

\section*{Stability analysis}

A second field of interest of the work concerning this thesis was to verify and extend the analytic and numerical analysis of the SOM that was formerly done by Ritter and Schulten \cite{schulten}. For the numerical case we are mostly interested how the stability limits are affected if we use other regular maps. In the numerical approach we then want to verify the analytic results and in addition to that examine the stability properties of several configurations using a HSOM and even the newly-defined GRiSOM. The latter case is, in particular, very interesting, as recently published papers like one of Tran and Vu\cite{tran_vu} discussed the ability to apply dimensionality reduction techniques in Hyperbolic data spaces and the corresponding reconstruction of the data and the GRiSOM naturally provide a way to perform this task under the given circumstances.

\addcontentsline{toc}{chapter}{Overview / Reader's Guide}
\chapter*{Overview / Reader's Guide}

Based on the given motivation this thesis focuses on two topics. First of all, we want to extend the definition of the classic SOM to obtain the GRiSOM and, secondly, we are going to focus on the analytic and numerical analysis of the stability limits of several SOM configurations. The first goal can be achieved quite straightforwardly and briefly as the specification is kept as abstract as possible and therefore not much previous knowledge is needed. The second goal we are striving for, however, needs much more preparatory work since we have to specify concretely all the components required for the simulation of the GRiSOM. This thesis will hereby try to give a self-contained view on all the subjects we have to deal with. Even with having banished the lesser instructive calculations and results to the appendix, the complexity of the mainpart is therefore still rather high since e.g. some longish derivations have been kept therein as they provide a deeper, essential insight into the currently handled matter. This shall enable the more interested readers to comprehend the detailed train of thoughts and furthermore provide detailed guidelines on how to tackle the problems we faced there for future research in this field.\\
\\
While some readers, who are interested in every detail of the work, are benefited by this style of presentation, the ``readers in a hurry'' on the other hand may be hindered as it is harder to distinguish the less important deviations and the more important intermediate and final results when skim-reading is used on this large volume of text. To support as well this type of readers without tearing the structure of the text apart (which would result in clouding the thread for the more interested reader) we asterisked those sections of the text, which can be skipped as they present only less important deviations or intermediate results that are mostly of importance in the local scope or just provide detailed insights that are not needed to follow the main thread. If, nonetheless, some formulas of these parts have to be used in more important sections they are cross-referenced and thus easily retrievable.\\
\\
Furthermore does each chapter begin with a brief introductory paragraph where the content and its structure are motivated and which provides a more detailed orientation for both more interested and skim readers.
And, finally, to give an overall view, a brief presentation and motivation of the chosen structure of the whole thesis is offered in the following.
\\
\\
As indicated above, the thesis starts with a chapter wherein we introduce the newly-defined \emph{General Riemannian Self-organizing map}. We therefore discuss briefly the classic SOM to have then a look at the intrinsic prerequisite, that the SOM algorithm imposes on the used map and feature spaces. This will then enable us to find a natural modification of the SOM algorithm which finally lead to the specification of the GRiSOM.
\\
As the definition of the GRiSOM in chapter \ref{ch:gsom} is kept quite abstract, we have now to concern ourselves with some preparatory work to specify concrete GRiSOM models before we can perform any real simulations e.g. to examine the stability.\\
\\
First of all, we have to take a look at the map and feature spaces that we want to use. Ch.\ref{ch:geometry} therefore focus on the geometrical properties of certain Euclidean and non-Euclidean spaces. We will thereby have a brief look at the axioms defining these spaces in general and are going to introduce subsequently several concrete models of these spaces that we later want to use in our simulations and calculations. Several structures, such as the geodesics and isometries will also be derived in detail.\\
\\
Given a certain map space, we will have then to decide how to embed the neurons, which form the map, in it. We will therefor study in chapter \ref{ch:tess} the possible regular tessellation of the hyperbolic and Euclidean map space and have to occupy ourselves with the question on how to generate these tilings.\\
\\
Chapter \ref{ch:distr} now raises the issue of the generation of appropriate sample sets for the stability analysis. We will first loosely define the needed distributions and then focus in the second part on deriving mathematically the algorithms needed to generate samples that realize these previously defined distributions.\\
\\
The preparatory work is concluded by presenting a concrete implementation of the GRiSOM, the derived spaces, maps and distributions. We will briefly discuss the design and most important features of the software bundle that has been written in the course of this thesis.

Having finished the discussion of the mathematical background and the implementation needed to construct concrete SOMs, we will be finally able to concern ourselves with the dynamics and, in particular, the stability of certain equilibriums states of the SOM. This will be done by using both analytic and numerical approaches.
\\
First we will focus on the analytical analysis, confining the calculations in chapter \ref{ch:analytic} to regular Euclidean maps and Euclidean feature spaces. We will deduce the Fokker-Planck equation for the corresponding stochastic process and use it then to determine the stability limit for the particular SOM configurations.\\
\\
After having tackled the analytic approach, we are going to verify and furthermore extend these results with numerical means i.e. the Monte Carlo method by using the implemented GRiSOM in combination with the sample sets whose distributions we will have discussed in chapter \ref{ch:distr}. We will start by discussing how the stability limits can be determined by evaluating the data produced by this simulation. We will then focus on verifying the results obtained by the analytic approach using an classic SOM setting with regular Euclidean maps. We will then continue with the numerical stability analysis of a HSOM, which has a hyperbolic map space, and finally are going to try to determine some aspects of the stability behavior of a GRiSOM with both hyperbolic map and feature space. The numerical simulations and therefore the whole analysis part will be concluded by returning to the ``Travelling ruler problems'' that we used as our primary motivation that resulted in the GRiSOM. We will thereby present the results which are obtained when performing the numerical simulation using the GRiSOM with the proper feature spaces.\\
\\
The final chapter of the thesis is reserved for a conclusion of our work. We will briefly take a final look at the results we obtained in the stability analysis, discuss problems we have encountered and finally outline some possible goal for future work on this subject.\\
\\
As noted above, the appendix finally contains proofs, calculations, data and figures that were omitted before as they would have bloated the chapters and/or distract the reader from more important calculations and results. This includes e.g. the elementary geometric calculations and the derivation of several equations needed in the analytic stability analysis. Beside this, an installation guide for the software bundle can also be found there.

\cleardoublepage
\pagenumbering{arabic} \setcounter{page}{1}
\part{GRiSOM}
\chapter{A Generalized View on Self Organizing Maps} \label{ch:gsom}

This first chapter is dedicated to the definition of the \emph{General Riemannian SOM} which is a modification of the classic \emph{self-organizing maps}, formerly introduced by Teuvo Kohonen \cite{kohonen}, to obtain a unifying model for SOMs for both more general map and input spaces. We therefore briefly present the original definition of the SOM followed by a detailed discussion of modifications we made that lead to GRiSOM.

\section{Self-organizing map}

In general, the self-organizing maps (SOM) are a class of artificial neural networks, which is based on \emph{competitive} learning on a fully-connected set of artificial neurons (cf. Fig.\ref{fig:gsom:kfm}). 

\begin{figure}[!ht]
\begin{center}
\includegraphics[width=0.42\linewidth]{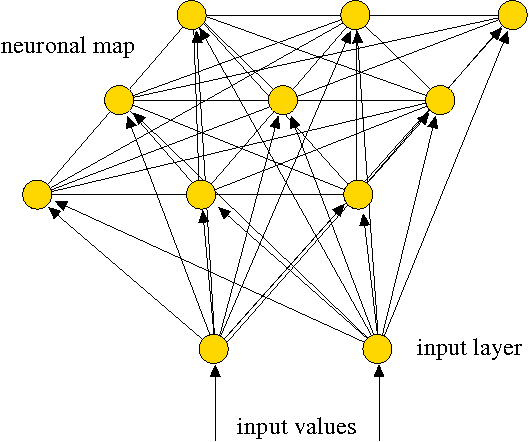}
\end{center}
\caption{formal Kohonen feature map with 2-dimensional input and 3x3 map}
\label{fig:gsom:kfm}
\end{figure}

It represents a topology-preserving generalization of \emph{Vector Quantization}(VQ) and is used for many purposes including visualization of high-dimensional data, classification/clustering and automated feature extraction.\\
\\
The basic idea is to place the artificial neurons at the vertices of a more dimensional lattice which defines a topological structure on the neurons. By taking their \emph{synaptic weight vectors}, they are furthermore mapped to specific locations in the \emph{data}, or as we will call it henceforth, \emph{feature space} $F$. The (output) neurons now compete among themselves and only the \emph{winning neuron} is fired for any input/sample (``winner-takes-all principle''). In other words, the input patterns in the feature space (of arbitrary dimension) are transformed in a (lesser dimensional) discrete map ($\rightarrow$ dimension reduction) as every input is mapped to a unique winning neuron, but, in contrast to VQ, in a topological ordered fashion.\\
\\
The whole process of this transformation can be reduced to the following sub-processes:

\begin{description}
\item[Competition:] As mentioned above, the neurons compete among themselves. For this purpose, each neuron compute a value using a discriminant function which calculates the response of the particular neuron to the input. The neuron belonging to the largest response is then the winner.\\
\\
To be more precise: Given a SOM with a set of neurons $A$ and a (n-dimensional) input sample $\vec v$ and let furthermore $\vec w_j$ be the synaptic weight vector (and therefore its location in the feature space) of neuron $j$ and $disc(\cdot)$ a discriminant function, then the winning neurons $i(\vec v)$ belonging to input $\vec v$ is determined by:
\begin{equation*}
 i(\vec v) = \arg\max_{j\in A}(disc(\vec v,\vec w_j))
\end{equation*}
Due to the neurobiological motivation that led to the SOMs, an inner product $\vec v \cdot \vec w_j$ is typically used to compute the neural response i.e. as the discriminant function. The resulting maximum problem can then be substituted by finding the minimal Euclidean distance:
\begin{equation*}
 i(\vec v) = \arg\min_{j\in A}\norm{\vec v - \vec w_j}
\end{equation*}
where $\norm{\cdot}$ is the canonical Euclidean norm.
As each possible point in the feature space is mapped to the neuron which is located next to $v$ in $F$, we get a \emph{Voronoi tessellation} of the feature space according to its metric (cf. chapter \ref{ch:tess}). Each neuron is then a representative of a whole set of points, namely all the points lying in the corresponding voronoi cell.

\item[Cooperation:] Unlike in VQ, the winning neuron is cooperating with other neurons lying in a certain topological neighborhood i.e. it is enlarging or inhibiting the response of these neighbors. Given two neurons $i$ and $j$, the degree of cooperation is thereby determined by a (unimodal) \emph{neighborhood function} $h_{i,j}$, which typically depends only on the distance between the two neurons in the lattice of the map. There exists many different choices for $h_{i,j}$ and we are going to discuss three (classes of) possible functions in the last section of this chapter.

\item[Synaptic Adaption:] All neurons in the topological neighborhood of the winner neuron are adapted, such that the response of them to similar input pattern is enhanced. This is achieved by applying a simplified form of Hebbian learning, which uses in general the adaption rule for the location of neuron $j$ in the input space:
\begin{equation*}
 \Delta w_j = \eta y_j \vec v - f(y_j) \vec w_j 
\end{equation*} where $\eta$ is an adaption parameter called \emph{learning rate}, $y_j$ is the response of neuron $j$ and $f(y_j) \vec w_j$ a \emph{forgetting term}. We simplify this equation by substituting the response by the neighborhood function and the forgetting term by $\eta h_{i(\vec v),j}$. Hence we obtain ($\vec v$ input and $i(\vec v)$ again winner neuron)
\begin{equation*}
 \Delta w_j = \eta h_{i(\vec v),j} (\vec v - \vec w_j)
\end{equation*} which is now used for the adaption of the neurons in the SOM. Thus, an adaption result in a move of the winner neuron and (at a lesser rate) of its neighbors (along an Euclidean line) in the feature space towards the input. The (partial) move of the neighbors as well ensures thereby that the topology defined by the lattice of the map tends to be preserved.
\end{description}

\subsection{The SOM Algorithm} \label{sec:som:algorithm}

Combing the three processes discussed above, we can formulate the SOM algorithm:

\begin{enumerate}[Step 1:]

\item (Initialization) Given a set of neurons $A$, fix the position of these neurons in the map and choose suitable initial locations for them in the feature space $F$ (e.g. by random).

\item (Sampling) Draw an input sample $\vec v$ in $F$ according to sensory input or a prespecified probability distribution.
\label{theo:SOM:alg_step2}

\item (Matching) Determine winner neuron $s :=i(\vec v)$ which is the neuron that matches $\vec v$ best i.e. has the nearest representatives to the location of the sample in the feature space:
  \begin{equation*} 
    \norm{\vec w_s - \vec v} = \min_{r \in A} \norm{\vec w_r - \vec v}
  \end{equation*}

\item (Updating) Adjust locations of neurons in feature space according to rule found above for the process of the Synaptic Adaption:
   \begin{equation}
     \vec w_r^{\mathrm{new}} = \vec w_r^{\mathrm{old}} + \epsilon \cdot h_{r,s} \cdot (\vec v- \vec w_r^{\mathrm{old}}) \quad \mathrm{for\ all\ } r\in A 
     \label{eq:som:adapt}
   \end{equation}

\item return to Step \ref{theo:SOM:alg_step2} until any predefined break conditions are met (e.g. reaching maximal number of adaption steps or falling below a certain representation error)
\end{enumerate}

\subsection{Dimensionality reduction/conflict}

 We conclude the discussion of the classic SOM by briefly discussing on one of its fields of application, namely the dimension reduction. As mentioned above, the SOM tries to embed itself into a given input set while preserving the topologic structure of the set of nodes of the map space also in the feature space. We thus obtain a mapping of the input onto the discrete lattice of a certain dimension \emph{reducing the dimension} of the feature space to the dimension of the map space. If the intrinsic dimension of the input embedded in the feature space is lower or equal than the map dimension, this reduction works smoothly, but if this intrinsic dimension exceeds the map dimension, we face the so-called \emph{dimension conflict}, i.e. the problem to represent a higher dimensional space with a lesser dimensional one. The SOM tries to still provide a good representation of the input by folding itself into the space, but is unable to take the topological structure of the input into account as even closely neighboring input often belong to Voronoi cells of map nodes, which are located on different foldings and therefore maybe far apart from each other in the map lattice.

\section{General Riemannian SOM}

Many different modifications of the SOM have been proposed over the last decades. Some of them like the Hyperbolic SOM \cite{ritter} focused on replacing the cooperation process by fixing the nodes in more complex lattices in a hyperbolic map space instead of the usual Euclidean square lattice above, which solved the dimensional conflict for at least Euclidean input space.
Others, as the \emph{Batch SOM}, changed the adaption process to improve the behavior of the SOM for non-stationary environments \cite{batch} and again others combined the SOM with kernel methods and thus adapted the competition procedure by transforming the input space \cite{kernel}. \\
\\
Our modification called \emph{General Riemannian SOM} will now be aimed at allowing the SOM to work on a whole range of different feature spaces as well as maps by complicating the SOM algorithm as little as possible. To achieve this goal, we will first have a look at the intrinsic properties of both the map and the feature space of the classic SOM and then try to generalize them to meet our needs.

\subsection{Map space} \label{sec:som:mapspace}

Although the map is above just defined as a discrete lattice, we can consider the nodes of the map to be embedded in a certain space. Obviously, there are hardly any demands to the embedding space. That is, that we only have to be able to quantify the cooperation between two neurons, i.e. a kind of distance between two nodes in the map which is then used in the neighborhood function\footnote{This also determines the topological structure that is inherent in the lattice}. Hence a suitable choice is, to embed the map in a \emph{metric space}, such that we can then naturally obtain the needed distance function on the nodes of the map lattice by transferring the distance function defined in the metric space.\\
\\
In the numerical experiments we will mostly use Riemannian manifolds like the \emph{Euclidean Space} or the \emph{Hyperbolic Space}, but in general every set of points which has defined a metric on it can be used. The most simple one is the \emph{trivial space} i.e. a set of points associated with the discrete metric
\begin{eqnarray*}
 && d(x,x) = 0\\
 && d(x,y) = 1\quad \mathrm{if\ } x\neq y
\end{eqnarray*}
In this case, however, the topology is also the discrete topology and therefore each neighborhood. So its preservation is trivial. Choosing a suitable neighborhood function, we just get a simple \emph{vector quantization}.

\subsection{Feature space}

For the feature space, we have to consider that both the competition and the adaption process take place there. As we have seen above in the discussion of the classic SOM, the search for the neuron with the maximal response to given input $v$, defined by the inner product of weight vector and input position, can be substituted by the search for the neuron with the least Euclidean distance to the input in the feature space. While the definition of the response based on the inner product does not make sense in more general spaces (e.g. non-pre-Hilbert spaces), distance functions represent a more general structure\footnote{Each pre-Hilbert space has a naturally defined metric which turns it into a metric space, but there exists many metric spaces which are not even vector spaces}. Thus, in respect to the competition process, it suffices to have a general metric space.\\
\\
The second task, we have to be able to perform in the feature space, is the adaption process. In order to update the representatives, we need to determine how to move a point in direction to another one, or to be more exact, move the $w^{\mathrm{old}}_r$ towards the input sample $v$. In the classic adaption process, this is done, as seen above, by calculating the difference of the position vectors of the input and the neuron that has to be moved and adding it partially to the current position of this neuron. Obviously, this update method can be transferred to every other (finite-dimensional) vector space. But for more general spaces, like the Dini's surface seen on the titlepage, which is not a vector space, we have to find a more appropriated definition of the update rule. As we mentioned above, this vector operations can be interpreted as a move of the neuron along an Euclidean straight line, which connects both the neuron and the input. The concept of a line is now not restricted to the Euclidean space, but exists at least in many more general metric spaces. The so-called \emph{geodesic} line is thereby, according to (Differential) Geometry (cf.e.g.\cite{aubin},\cite{abraham}), defined as a local length minimizing path in respect to the metric of the particular space. Thus, a suitable new definition of the update rule would be, to define the update in general as a move along these lines while the response of the neuron on the input determines, how far the neuron is moved towards the location of the input.\\
\\
We have, nevertheless, to consider, that in general metric spaces the shortest path has not to be neither existent nor unique, but both existence and uniqueness are essential for our modified update rule. So, in order to satisfy this requirements, we have to put a further restriction onto the used feature spaces. Fortunately this condition is always (at least) locally fulfilled by (pseudo-)Riemannian manifolds (cf.\cite[Thm.5.14]{aubin}). Thus to ensure that our modified adaption process that we will define below is always well-defined we restrict the possible choice of the feature space to those manifolds.

\subsection{The GRiSOM Algorithm}

The conclusion of the two preceding sections is, that it is possible, by just having a more generalized view on the intrinsic prerequisites of the three processes, to generalize the classic SOM in a quite natural manner. This generalized model provides a huge range of possible configurations of the SOM in respect to the map and feature space, which will, from now on, improve our ability to choose the right spaces which reflect the structure of the feature space (which is e.g. Hyperbolic or Spherical as in the three ``travelling ruler problems'') and the expected intrinsic structure of the input (e.g. directional or hierarchical) most adequately to avoid, for example, dimensionality conflicts.\\
\\
By taking the modification made to the competition, cooperation and adaption process now into account, we can finally obtain the following modified SOM algorithm:

\begin{enumerate}[Step 1:]
\item (Initialization) Given a set of neurons $A$, fix the position of these neurons in the map space $M$ and choose suitable initial locations for them in the feature space $F$ (e.g. by random).

\item (Sampling) Draw an input sample $ v$ in $F$ according to sensory input or a prespecified probability distribution.
\label{theo:GRiSOM:alg_step2}

\item (Matching) Determine winner neuron $s :=i(v)$ which is the neuron that matches $v$ best i.e. has the nearest representatives to the location of the sample in the feature space according to the metric/distance function $dist$ in this feature space:
  \begin{equation*} 
    dist(w_s, v) = \min_{r \in A} dist(w_r,v)
  \end{equation*}

\item (Updating) Adjust locations of neurons in feature space according to rule found above for the process of the Synaptic Adaption:
   \begin{equation*}
 	 w_r^{\mathrm{new}} = p
    \end{equation*}
 with $p \in F $ satisfying the following conditions
 \begin{equation*}
	d(w_r^{\mathrm{old}},p) + d(p,v) = d(w_r^{\mathrm{old}}, v), \quad \frac {d(p,v)}  {d(p, w_r^{\mathrm{old}})} 
	= \eps \cdot h(r,s)
\end{equation*}
that means that the representative is moved along a geodesic between $ w_r^{\mathrm{old}}$ and $v$ by the relative distance according to the result of the neighborhood function and the step size. This adaption algorithm is well-defined as long as the geodesics are unique. 

\item return to Step \ref{theo:GRiSOM:alg_step2} 
\end{enumerate}

According to the section above, this will be the case for the right choice of the feature space. We will take a detailed look on this aspect with regard to the concrete spaces used in the simulations right in the next chapter. But before, we will briefly have a look at the (classes of) neighborhood functions that will be used in the SOM.

\section{Neighborhood functions} \label{sec:gsom:neighb_f}

As mentioned above, the SOM allows winner neurons to interact with its neighborhood. This cooperation is then implemented by the neighborhood function, which depends usually only on distances in the mapspace and determines how extensive the adaption process is. We will distinguish here between three classes of possible functions:

\begin{description}

\item[Long range / Gaussian]

  One representative of the \emph{long-range} neighborhood functions is the \emph{Gaussian function} (with \emph{mean value} $\mu = 0$)

\begin{equation*}
  h^{Gauss}(x) = \frac 1 {\sqrt{2 \pi} \sigma} \exp(-\frac{x^2}{2 \sigma^2})
\end{equation*}
or in the case of a discrete grid of neurons
\begin{equation} \label{eqn:gsom:gauss}
  h^{Gauss}_{rr'} = \sum_s \delta_{r+s,r'} \exp(-\frac{s^2}{2 \sigma^2})
\end{equation}

where $\sigma^2$ is the variance. With this class of functions in use, the excitation is laterally inhibited (i.e. $h$ is strictly decreasing), but has no upper limit for the range and thus affects the whole set of neurons. This results in an enormous need of computing time since in every step, each neuron has to be adapted.

\item[Short range / NN]

To evade the problem of the excessive needs of the long-ranged functions in respect to computing time, the short-ranged functions possess an (bounded) compact support and therefore the excitation is locally restricted. Thus only a small fraction of the neurons has to be updated which accelerate the whole algorithm at the cost of losing interactions beyond the neighborhood defined by support. An example for this family of functions is the following function:

\begin{equation} \label{eqn:gsom:NN}
  h^{NN}(x) = \Theta(x-d_{NN}) = \left\{ \begin{array}{ll} 1 & \mathrm{if\ neurons\ are\ nearest\ neighbors} \\  0 & \mathrm{otherwise} \end{array} \right. 
\end{equation}
where $d_{NN}$ is the distance of the nearest neighbors. We will consider this function besides the Gaussian one when working both numerically and analytically.

\item[Ultra-short range / VQ]

In the limit of small ranges and small variances both cases above lose their capability to assure the preservation of the neighborhood in the feature space. Only winner neurons are adapted, i.e. we have a case of \emph{vector quantization}.

\begin{equation}\label{eqn:gsom:vq}
  h^{VQ}(x) = \delta(x)
\end{equation}

\end{description}

\part{Preparatory work}
\chapter{Certain Aspects of Riemannian manifolds}
\label{ch:geometry}

In the preceding chapter we have defined the GRiSOM using Riemannian manifolds as feature (and map) spaces. Since this offers a field of study that is too broad for the scope of this thesis, we will focus on the two- and three-dimensional manifolds with constant sectional curvatures -1, 0 and 1. These are (in same order) the Hyperbolic, Euclidean and Spherical space (or surface) corresponding to the three feature spaces that we have encountered in the three modified travelling salesman problems. Furthermore, have \emph{Hyperbolic} or \emph{Hierarchical Self-Organizing Maps} (HSOM) been already considered in present works, e.g. Ritter et al. (cf.\cite{ritter}, \cite{ontrup}, \cite{ontrup2}) to improve the clustering and dimension reduction for hierarchical data set. These types of SOMs therefore have maps embedded in the Hyperbolic space which take the hierarchical structure into account by providing an exponential growth of the neighborhoods in the map space. So it stands to reason to use this configuration as well in the present work to study the properties of SOMs with Non-Euclidean feature spaces and extend these studies to non-euclidean feature spaces.
\\
\\
The most common way to introduce non-euclidean spaces is to present concrete models and define things like geodesics/lines and distance measures in them. Some of them offer a simple way to visualize e.g. the hyperbolic or spherical plane in a way, that can be easily understood. Most popular in the hyperbolic case are the \emph{hyperboloid model} which is an isometric embedding of the hyperbolic plane in the Minkowski space and the \emph{Poincare Disk/Sphere model}. It is not wrong to decide to work only in a particular model since the hyperbolic space is \emph{categorical} i.e. its axioms define the geometry completely and uniquely and thus all representations are isomorphic, but it may mislead readers to associate every definition in this space with only a single particular model, which may complicate the comprehension of some structures we will have to derive. Thus we will take another, more fundamental approach following \cite{Ramsay}. The euclidean, spherical and hyperbolic geometry will be introduced by their axioms, not just to get a feel for the fundamental structures, but also see, how similar the three geometries are and where are their differences. First this will be done in two dimensions and then be generalized to three dimensions for the hyperbolic case. After these first steps, the later used analytic models and representations will be introduced. The ``readers in hurry'' may focus on these sections.

 Therefore we present in the last part of this chapter the formulas concerning properties and structures, which are needed for the further studies, with an emphasis on hyperbolic geometry.

\newpage

\section{Euclidean Geometry}

Although geometry is one of the oldest sciences and is found throughout many different human cultures for millennia like the famous studies of Euclid in Ancient Greece (ca. 300 BC) or the less known work of the followers of Mozi (ca. 330 BC, China), it had, for long time, a flawed structure with rather informal axioms. This was fixed by mathematicians at the end of the 19th century, especially by David Hilbert who reworked the whole structure by strict axiomatic reasoning and presented 1899 with his ``Grundlagen der Geometrie'' \cite{Hilbert} the modern foundations of geometry.

\subsection{*Hilbert's Axioms of the Euclidean Plane}

Given a plane $\mathbb{P}$ i.e. a set of a priori undefined things called \emph{points} and subsets of it isomorphic to $\setR$ called \emph{lines}, the axioms of Hilbert provide a structure on it and we get the \emph{Euclidean plane} $\mathbb E^2$:

 \setcounter{axiom:line}{\value{axm}}
\begin{axiom}[line] \label{axiom:geom:line}
If $A$ and $B$ are distinct points, i.e. the symbols A and B denote \emph{different} points, then there is one and only one line that passes through them both. This line is denoted by $\overline {AB}$.
\end{axiom}

\setcounter{axiom:metric}{\value{axm}}
\begin{axiom}[metric space] \label{axiom:geom:metric} 
 There is a function $|PQ|$ defined for all pairs of points $P,Q$, such that
\begin{enumerate}[(i)]
 \item $|PQ| \geq 0$ (non-negativity)
\item $|PQ| = |QP|$ (symmetry)
\item $|PQ| \leq |PR| + |RQ|$ for all arbitrary points $R$ (triangle inequality)
\end{enumerate}
\end{axiom}

\setcounter{axiom:param}{\value{axm}}
\begin{axiom}[Isometric Isomorphism of line and $\setR$] \label{axiom:geom:param}
For each line $l$, there exists a bijection $x$ called \emph{coordinatization} from the set $l$ to the set \setR, such that if $A$ and $B$ are any points on $l$, then
$$ |AB| = |x(A)-x(B)|$$
\end{axiom}

\begin{defn}[ray]
Given any line $l$, $x$ a coordinatization of $l$ and any point $A$ on this line, then the subset $\vec r \subset \setgen{Z\in l}{x(A) > x(Z)}$ is a \emph{ray} with origin $A$. The set $l \cap \vec r^c$ is also a ray (called \emph{opposite ray}) with the same origin and will be denoted $- \vec r$.
\end{defn}

\begin{defn}[segment]
Given any line $l$ and two points $A$ and $B$ on $l$, then the \emph{segment} $\overline{AB}$ is the subset of $l$ defined by $$\setgen{Z\in l}{x(A) \leq x(Z) \leq x(B) \lor x(B) \leq x(Z) \leq x(A)}$$ and $|AB|$ is called its length.
\end{defn}

 \setcounter{axiom:HP}{\value{axm}}
\begin{axiom}[Half-planes] \label{axiom:geom:HP}
If $l$ is any line, it defines a partition $\{HP_1,HP_2\}$ of the plane, where the ``parts'' are the corresponding \emph{half-plane}, such that
\begin{enumerate}[(i)]
 \item if $P$ and $Q$ are in the same ``part'', the segment $PQ$ contains no point of $l$
\item if $P$ and $Q$ are in opposite ``parts'', the segment $PQ$ contains a single point of $l$
\end{enumerate}
\end{axiom}

\setcounter{axiom:angles}{\value{axm}}
\begin{axiom}[Measures of angles]\label{axiom:geom:angles}
It exists an mapping $rad$ any angle $\angle \vec h \vec k$ (where $\vec h$ and $\vec k$ have the same origin) to the interval $[0,\pi]$ called the \emph{(radian) measure of the angle}, such that
\begin{enumerate}[(i)]
 \item $\vec h = \vec k \Rightarrow rad(\angle \vec h \vec k)=0$ and $\vec h = -\vec k \Rightarrow rad(\angle \vec h \vec k)=\pi$
\item the sum of an angle and its supplement is $\pi$
\item if $j$ is in the interior of an angle  $\angle \vec h \vec k$, then rad( $\angle \vec h \vec j$) + rad( $\angle \vec j \vec k$) = rad( $\angle \vec h \vec k$)
\item if a ray $k$ from a point $Z$ lies in a line $l$, then for each half-plane defined by $l$ the measure $rad$ is a bijection from the rays $j$ with origin Z lying in the particular half-plane to the set of real numbers $\alpha$ in $]0,\pi[$
\end{enumerate}
\end{axiom}

\begin{defn}[triangle]
 A \emph{triangle} is a set of three (non-collinear) points $A,B,C$ denoted by $\triangle ABC$ and the segments $\overline{AB}$,$\overline{AC}$ and $\overline{BC}$ are called its \emph{sides}.
\end{defn}

\begin{defn}[congruence of segments and angles]
If two segments have the same length or two angle the same measure, they are called \emph{congruent}.
\end{defn}

 \setcounter{axiom:SAS}{\value{axm}}
\begin{axiom}[congruence of triangles / side-angle-side criterion] \label{axiom:geom:SAS}
Given two triangles. If two sides and the included angle of the first triangle are congruent to two sides and the included angle of the second one, then these triangles are congruent.
\end{axiom}

 \setcounter{axiom:parallel}{\value{axm}}
\begin{axiom}[Euclidean parallel axiom] \label{axiom:geom:eucl_parallel}
Given any line and given any point not on it, there is only one line through the given point that never intersects the given line. 
\end{axiom}

Even though this ``construction'' above is for the 2-dimensional space, it can be easily adjusted to match for higher dimensional spaces. An example how to do this will be given when defining the axioms of the 3-dimensional hyperbolic space in section \ref{theo:spaces:3d_hyp}.

\subsection{Models} \label{sec:geom:euclModels}

The standard (analytic) model of the Euclidean plane/space is well known since algebra lessons in school. A point is represented as an element of $\setR^n$ i.e. by a ordered set / tuple $(x_i)_i $ where the $x_i \in \setR$ are called its \emph{Cartesian Coordinates}. We will show below, that, indeed, lines are defined e.g. in two dimensions by either $\set{\vec x | x_1 =a}$ if they are vertical or  $\set{\vec x | x_2 = mx_1+b}$ otherwise. Furthermore a distance function i.e. $d(\vec x, \vec y) = (\vec x - \vec y)^T \cdot (\vec x - \vec y)$ is defined by the canonical metric.
\\
\begin{minipage}{0.4\linewidth}
\begin{center}
\includegraphics[width=\linewidth]{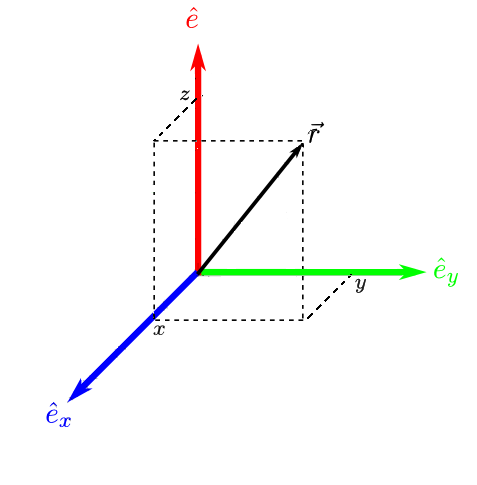}
\captionof{figure}{Cartesian coordinate system in three dimensions \label{fig:geometry:cart_coords}}
\end{center}
\end{minipage}
\begin{minipage}{0.6\linewidth}
\begin{eqnarray*}
  d(x=(x_i)_i,y=(y_i)_i) &=& \sqrt{\sum_i (x_i-y_i)^2}\\
    ds^2 &=& \sum_i dx_i^2 \\
    \Leftrightarrow G &=& \mathbf 1_n
\end{eqnarray*}
where $ds^2$ is the line element and $G$ its corresponding metric tensor.
\end{minipage}\\
\\
Another common model is the \emph{polar} or \emph{(hyper-)spherical model} where each point in the space is identified by its distance from the origin and certain angles.\\
\begin{minipage}{0.4\linewidth}
\begin{center}
\includegraphics[width=\linewidth]{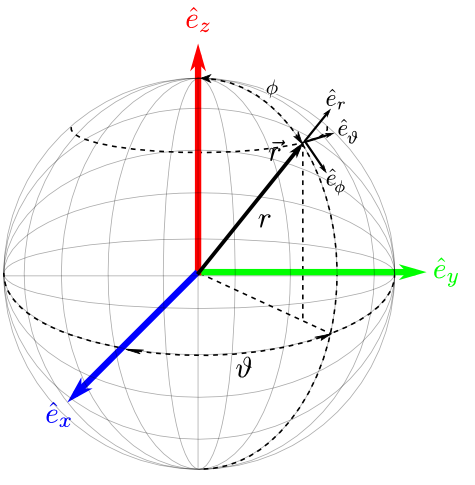}
\captionof{figure}{Spherical coordinate system in three dimensions \cite{sphericalPic} \label{fig:geometry:spherical_coords} }
\end{center}
\end{minipage}
\begin{minipage}{0.6\linewidth}
\begin{eqnarray}
  ds^2 &=& dr^2 + r^2 (\sum_{i=1}^{n-1} \sin^{n-i+1}(\Theta_i) d\Theta_i^2 )\\ 
    \Leftrightarrow G &=& \left(
	\begin{array}{cccc}
	1 & 0 & 0 & \\ 
	0 & r^2 & 0 & \cdots \\	
	0 & 0 & r^2 \sin \Theta_2 & \\
	& \vdots & & \ddots  
        \end{array}\right)
\end{eqnarray}
$r$ is hereby the distance from the origin whereas $\Theta^{(1)} (= \Theta)$ is the azimuth and $\Theta^{(2)} (=\Phi), \Theta^{(3)}, \dots$ are the angles between the 3rd, 4th, \dots - axis.\\
\\
Like fig.\ref{fig:geometry:spherical_coords} suggests, the transformation from spherical to cartesian coordinates is
\begin{eqnarray*}
x_1 &=& r \cos(\Theta_1) \\
x_2 &=& r \sin(\Theta_1) \cos(\Theta_2) \\
&\vdots& \\
x_{n-1} &=& r \sin(\Theta_1) \cdots \sin(\Theta_{n-2}) \cos(\Theta_{n-1}) \\
x_n &=& r \cos(\Theta_1) \cdots \sin(\Theta_{n-2}) \sin(\Theta_{n-1})
\end{eqnarray*}
\end{minipage}\\ \\
\section{Spherical Geometry}
Spherical Geometry has a long history as spheres and ball naturally fascinates mankind since the ancient age due to their high degree of symmetry (which is often seen as a reflection of divine perfection). Besides the pure fascination for this bodies, the most important motivation have been the many applications as in geography, astronomy and navigation. 

\subsection{*Axioms of Spherical space}

Although Spherical Geometry is not so close to the Euclidean Geometry as the Hyperbolic case, that we discuss further down in this chapter, most of the axioms still hold (at least locally) when they are slightly adapted i.e.
\footnote{The reader should, nevertheless, not to be mislead to attach too much value to the similarities between the Spherical and the other geometries. They just enable us to handle the structures on these geometries in analogous ways. In principal, the geometrical attributes of the surface of a sphere (unlike those of the Euclidean and Hyperbolic plane) are not consistent with all the properties of a common Hilbert space.}
\setcounter{axiom:param}{\value{axm}}
\begin{axiom}[Isometric Isomorphism of line and $\setR/\setZ$] 
For each line $l$, there exists a bijection $x$ called \emph{coordinatization} from the set $l$ to the set $\setR/\setZ$, such that if $A$ and $B$ are any points on $l$, then
$$ |AB| = |min(x(A)-x(B),x(B)-x(A))|$$
\end{axiom}
Furthermore, we have to modify the definition of \emph{ray} and have to deal with the fact that we loose the uniqueness of line when encountering antipodal points (cf.\cite[Ch.7]{spherical}). The axiom that obviously fails totally is the parallel postulate. Instead of it we have to use the following:
\setcounter{axm}{\value{axiom:parallel}}
\begin{axiom}[Spherical parallel axiom]
Given a line and and a point not on it, no lines exists through that point that never intersects the given line.\footnote{In other words, two lines always intersect.}
\end{axiom}

\subsection{Models} \label{ch:geom:sph_models}

The most known model of the two-dimensional spherical space is the geographic coordinate system, that is used to map the earth. It is quite the same as the Euclidean spherical coordinates model with a fixed $r$. The only difference is, that the polar angle $\phi$ is replaced by the latitude $\phi_{\mathrm{Lat.}} = \frac \pi 2 - \phi$ whereas the azimuth stays the same and is called \emph{longitude} $\lambda$. The metric is thus given by
\begin{eqnarray*}
  ds^2 &=& d\phi_{\mathrm{Lat.}}^2 + (\cos(\phi_{\mathrm{Lat.}}))^2 d\lambda^2 \\
  G &=& \matrixzwei{1 & 0 \\ 0 & \cos^2(\phi_{\mathrm{Lat.}})}
\end{eqnarray*}
The distance of two points $A$ and $B$ can furthermore be directly computed by using
\begin{equation*}
d(A,B) = \arccos(\cos(\phi_{\mathrm{Lat.}}^A) \cos(\phi_{\mathrm{Lat.}}^B) \cos(\lambda^A - \lambda^B) + \sin(\phi_{\mathrm{Lat.}}^A) \sin(\phi_{\mathrm{Lat.}}^B)
\end{equation*}

\section{Hyperbolic Geometry}
\label{theo:spaces:3d_hyp}

The \emph{Hyperbolic Geometry} has many different roots (cf.\cite{Cannon}). Many of them result from efforts to derive the parallel postulate (Axiom \ref{axiom:geom:eucl_parallel}). Many mathematicians since the ancient times like Proclus (ca. 400 AD) or K\"{a}stner and Kl\"{u}gel in the 18th century tried to prove that it would be possible to deduce this postulate by using the other ones\footnote{They actually didn't use the axioms listed above but based their attempts on older equivalents like the fives axioms of Euclid.}. Due to the many failed attempts, the approach finally changed in the 19th century when mathematicians like Schweikart and Taurinius (1794-1816) or Gauss, Lobachevski and Bolyai no longer tried to find a proof but focused on the consequences of denying this axiom. This resulted in a first grasp of the non-Euclidean geometry but still without an analytic understanding or even model. The fundament for this was then laid by Euler's studies (among others) of curved surfaces and resulted finally in the theory of Riemannian manifolds in the middle of the 19th century and the subject of Differential Geometry. We will also make use of this work later when we will deduce certain structures in the hyperbolic space as it is a simply connected Riemannian manifold with a constant negative (sectional) curvature.\\
\\
Similar to the spherical geometry, the parallel postulate no longer holds here. In the two-dimensional case of the hyperbolic plane it is thus replaced by
\setcounter{axm}{\value{axiom:parallel}}
\begin{axiom}[Hyperbolic parallel axiom]
There exists a line and a point not on it, such that there are at least\footnote{in effect there are infinitely many for any line and any point not on it} two lines through that point that never intersects the given line. 
\end{axiom}
Beside this, all the other axioms and definitions of the Euclidean case are still valid here without any further adaption.\\
\\
Without getting too much into details it is still important to remark that the Hyperbolic (as well as the Euclidean and Spherical) axioms are consistent and categorical. The consistency follows directly from the existence of the models shown below. The proof of the categoricalness is more subtle and will be skipped here. The interested is therefore referred to e.g. \cite[chap.7.7]{Ramsay}. An important consequence of these two properties is, that we can freely choose an arbitrary model to prove anything we want and the results will still hold in any other model.

\subsection{*Axioms of Hyperbolic space}

Before concerning ourselves with concrete models we will have to discuss how to extend the hyperbolic plane given by the axioms above to spaces of higher dimensions. This is necessary as we want to use non-Euclidean spaces as feature spaces. Thus we will take another look on the axioms and try to adapt them in the following for at least the three-dimensional hyperbolic space.\\
\\
The Axioms \ref{axiom:geom:line}, \ref{axiom:geom:metric}, \ref{axiom:geom:param} and \ref{axiom:geom:SAS} are exactly the same in the three dimensional case, so they will be skipped here. Instead, there will be three additional axioms:

\setcounter{axm}{\value{axiom:HP}}
\begin{axiom}[Half-planes]
Give a plane, then any line in this plane separates it into two half-planes as in Axiom \ref{axiom:geom:HP} in the planar case.
\end{axiom}

\setcounter{axm}{\value{axiom:angles}}
\begin{axiom}[Measures of angles]
It exists an angle measure $rad$  such that any angle that lie in a given plane satisfy the Axiom \ref{axiom:geom:angles} in the two dimensional case.
\end{axiom}

\setcounter{axm}{\value{axiom:parallel}}
\begin{axiom}[Hyperbolic parallel axiom]
There exists a plane, a line and a point in that plane such that the point is not on the line and that at least two other lines in the plane can be found that pass the point but don't intersect the first given line.
\end{axiom}

\begin{axiom}\label{axiom:geom:line_plane}
Given two distinct points in a plane, the line defied by them (cf.Axiom \ref{axiom:geom:line}) lies entirely in the plane.
\end{axiom}

\begin{axiom}\label{axiom:geom:points_plane}
Any three noncollinear points lie in a unique plane.
\end{axiom}

\begin{axiom}[Half-space]
Given any plane $P$ there are two so-called \emph{half-spaces $HS_1$ and $HS_2$ bordered by $P$}. That are subsets of $\mathbb H^3$ such that both and $P$ are a partition of the space and thus pair-wise disjoint and if
\begin{enumerate}[1)]
\item if points $A$ and $B$ are in the same half-space, the line segment between them does not intersect $P$
\item if points $A$ and $B$ are in opposite half-spaces, the the line segment between them intersect $P$ (in a single point)
\end{enumerate}
\end{axiom}

\begin{cor} \label{cor:geom:2din3d}
Obviously, the axioms of the two-dimensional case hold in any two-dimensional plane in the three-dimensional (or even higher-dimensional) space.
\end{cor}

\subsection{Models} \label{ch:geom:Hyp_models}

There are several analytical models of the Hyperbolic $n$-space that are commonly used. We will present here the two of them that we will use in this thesis. We will also restrict the discussion to the most basic attributes i.e the metric and the representation of points. All other concrete structures can then be deduced using the axioms above. We will do this in an extra section following below.

\begin{description}

\item[Polar resp. (hyper-)spherical model]

The polar or, in more dimensions, (hyper-)spherical coordinates are defined in the same way they are defined in the Euclidean case as shown in fig. \ref{fig:geometry:spherical_coords} for three dimensions.  Since we use only two and three dimensional spaces, we will restrict the following description to these two cases.

The metric in 2-D (i.e. polar coordinates) is then given by
\begin{eqnarray}
  ds^2 &=& dr^2 + (\sinh(r))^2 d\Theta^2  \label{eqn:geom:HSph_metric}\\
  G &=& \matrixzwei{1 & 0 \\ 0 & \sinh(r)} \nonumber
\end{eqnarray}
In three dimensions this is generalized to
\begin{eqnarray}
  ds^2 &=& dr^2 + (\sinh(r))^2 (\sin^2\phi  d\Theta^2 + d\Phi^2) \label{eqn:geom:line_element_HSph} \\
  G &=& \matrixdrei{1 & 0 & 0 \\ 0 & \sinh(r)^2 \sin^2\phi & 0  \\ 0 & 0& \sinh(r)^2} \label{eqn:geom:metric_tensor_HSph}
\end{eqnarray}

The spherical model may not be as instructive as the two Poincare models following below concerning e.g. geodesic lines and visualization, but it simplifies some calculations because e.g. regarding a point in the origin, $r$ provides directly the distance form this point to any other point without having to compute any ``nasty'' hyperbolic trigonometric functions or complicated transformations. So when facing certain problems in the hyperbolic space, we will often retreat from the given model to this one and then convert the result back into the original model.

\item[Poincare disk/sphere model]

The ($n$-dimensional) Poincare disk/sphere model, often denoted by $\mathbb D^n$, is an (non-isometric) embedding of the hyperbolic space into the $\setR^n$ where geodesic lines are either Euclidean straight lines passing the origin or  Euclidean circles which intersect the \emph{circle at infinity} perpendicularly as shown in fig.\ref{fig:geometry:poincare_disk}.\\ \\ Historically, there are several concrete ways to obtain this models starting at other models. Given the Poincare Half-Space model the Poincare Sphere model can be constructed by ``rolling up'' the bottom \emph{line at infinity} whereas given the hyperboloid model the Poincare sphere can be obtained by stereographic projection. A more detailed view on this subject can be found e.g. in \cite{Cannon}).\\
\begin{figure}[!ht]
\begin{center}
\includegraphics[width=0.4\linewidth]{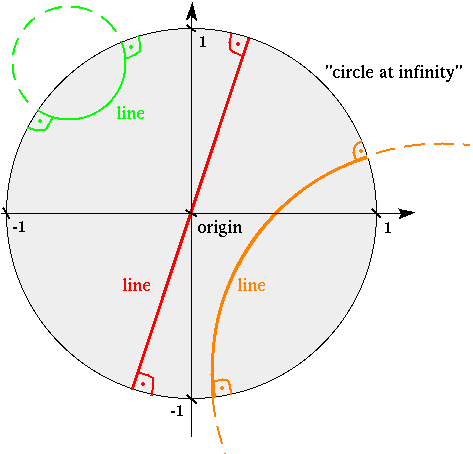}
\caption{2-dim. Poincare disk (grayed-out) with some (colored) geodesic lines}
\label{fig:geometry:poincare_disk}
\end{center}
\end{figure}
The model furthermore is very similar to the spherical model and thus both can be easily converted into each other by transforming the spherical coordinates ($r_PD,\Theta_1^PD, \Theta_2, \dots $) of the hyperbolic space into the spherical coordinates ($r_PD,\Theta_1^PD, \Theta_2, \dots $) of the Euclidean space, in which the model is embedded, and then converting these spherical coordinates into cartesian ones ($x_1,x_2, \dots$). We get similar to the transformation between cartesian and spherical coordinates in the Euclidean space (cf. section \ref{sec:geom:euclModels}) an isomorphism from spherical coordinates to Poincare disk:
\begin{eqnarray} 
u_1 &=& \tanh(\frac r 2) \cos(\Theta_1) \nonumber \\
u_2 &=& \tanh(\frac r 2) \sin(\Theta_1) \cos(\Theta_2) \nonumber \\
&\vdots&  \label{eq:geom:trafo_hsph_pd}\\
u_{n-1} &=& \tanh(\frac r 2) \sin(\Theta_1) \cdots \sin(\Theta_{n-2}) \cos(\Theta_{n-1}) \nonumber  \\
u_n &=& \tanh(\frac r 2) \cos(\Theta_1) \cdots \sin(\Theta_{n-2}) \sin(\Theta_{n-1})\nonumber 
\end{eqnarray}
The metric in the $n$-dimensional case is defined by:
\begin{eqnarray}
  G &=& \frac 2 {(1-\sum_i u_i^2)} \mathbf 1_n \label{eqn:geom:metric_tensor_PD}\\
  ds^2 &=& 4 \frac{\sum_i du_i^2}{(1-\norm{u}^2)^2} \label{eqn:geom:line_element_PD}
\end{eqnarray}
A nice feature of the Poincare model is that the distance between two arbitrary points $P$ and $Q$ can be easily explicitly computed using the following formula:
\begin{equation}
d(P,Q) = \acosh(1+ 2 \frac{\norm{\vec p - \vec q}^2}{(1-\norm{\vec p}^2) (1-\norm{\vec q}^2)})
\label{eqn:geom:distance_PD}
\end{equation} where $\vec p$ and $\vec q$ are the representations of $P$ and $Q$.

\end{description}

\textbf{Remark:} Henceforth, we will work in the Hyperbolic hyperspherical model unless otherwise noted.

\newpage
\section{Various Formulas}

In this section we will examine some important structures and deduce some formulas of the Euclidean and Hyperbolic space that we will need in the next chapters.

\subsection{*Isometries} \label{sec:geom:isom}

One important class of transformations in the metric spaces we defined above are isometries i.e. distance-preserving transformations. 

\subsubsection{*Isometries in Euclidean space}

In the Euclidean case the isometries are elements of the common Euclidean group with subgroups like translations and rotations (about an axis or, in general, simplex). Like in the following hyperbolic case we won't discuss the group in detail but we will focus instead on the subgroup that we will need in the course of this thesis. In the n-dimensional Euclidean space this will be the subgroups already mentioned above combined with reflections in (n-1) hyper-planes. Their corresponding representations in e.g. the cartesian coordinates model are thereby
\begin{description}
\item[translations:] $T_{\vec a}(\vec p) = \vec p + \vec a$ where $\vec a$ is the vector of translation
\item[rotations about the origin:] $R_D(\vec p) = D \cdot \vec p$ with $D \in SO(n)$
\item[(main) reflections:] $M_{i}(\vec p) = \vec \tilde p$ where $\tilde p_j = p_j$ if $j \neq i$ and $-p_i$ otherwise.
\end{description}
Instead of general reflections we defined above the so-called \emph{main reflexions} i.e. the reflexions that invert the orientation of exactly one given axis in the space. They do not form a group by themselves, but combined with the rotations and translations they generate the group of the reflexions about any subspace. To be more exact, any reflection in an arbitrary (n-1) hyper-planes can be obtained by coordinate transformations via translations, rotations and main reflections. Analogously we can decompose each rotation into \emph{main rotations} that are the rotations in the planes spanned by two arbitrary axes. In the cartesian coordinates model the $(i,j)$-th main rotation is then defined by (cf. \cite{aguilera})
\begin{equation}
\mathrm{mainrotation}_{i,j}(\vec p, \vartheta) = D \cdot \vec p
\end{equation} such that the ($i$,$j$)-submatrix of $D$ is element the two-dimensional rotation with rotation angle $\vartheta$ and the rest of $D$ is just the identity matrix.\\
\\
The whole (sub)group of isometries that we will use is then obtained by combining all these subgroups/generators.

\subsubsection{*Isometries in Hyperbolic space}

For the hyperbolic space we will start by describing the complete isometry group of the two-dimensional Poincare Disk model $\mathbb D$. As shown in \cite{Anderson} it is isomorphic to a subgroup of the \emph{Moebius group} i.e. the \emph{projective special linear group} $PGL(2,\setR)$. It can be shown that in the given model, the set of these transformations is defined by
\begin{equation}
 \mathrm{Isom}(\mathbb D) = \setgen{\frac{a g(z)+p}{\bar p a g(z) +1}}{a,p \in \setC, \norm{a}=1, \norm p < 1, g \in \set{\mathrm{Id,Conjugation}} }
\end{equation}
In the case that $p=0$ we get a pure rotation (together with a reflection if $g$ is the conjugation) whereas for $a=1$ we get pure translations. This simplifies our work with these isometries when we need to construct certain isometric transformations. While we won't show that these are in fact already all isometries we will at least briefly prove the following
\begin{thm} \label{thm:geom:hyp_isom}
All elements of the set $\mathrm{Isom}(\mathbb D)$ as defined above are indeed isometries and this set forms a group using the complex multiplication as the group operation and is therefore closed under this operation.
\end{thm}
\begin{proof}
To show that the distances are preserved we make use of Eq.\ref{eqn:geom:distance_PD}. Let $T$ be an element of $Isom(\mathbb D)$ and $v$ and $w$ two arbitrary points in $\mathbb D$. Since the acosh is strictly monotonic increasing it is now sufficient to check that the fraction in Eq.\ref{eqn:geom:distance_PD} is invariant under substitution of $v$ and $w$ by its transformations.
Furthermore, it suffices to show this only for pure translations since a rotation about the origin and reflection in the real axis clearly preserve the norms and therefore the distance:
\begin{eqnarray*}
&&\norm{T_{a,p,g}(v) - T_{a,p,g}(w)}^2 = \norm{\frac{a g(v)+p}{\bar p a g(v) +1} - \frac{a g(w)+p}{\bar p a g(w) +1}}^2 
=  \frac{\norm{v-w}^2(1-\norm{p}^2)^2}{\norm{\bar p v+1}^2\norm{\bar p w+1}^2} \\
&&\land(1-\norm{T_{a,p,g}(v)}^2)(1-\norm{T_{a,p,g}(w)}^2)  = (1-\norm{\frac{a g(v)+p}{\bar p a g(v) +1} }^2) (1-\norm{\frac{a g(v)+p}{\bar p a g(v) +1} }^2) \\
&& =  \frac{(1-\norm{p}^2)^2}{\norm{\bar p v+1}^2\norm{\bar p w+1}^2 (1-v\bar v)(1-w\bar w)} \\ 
&& \Rightarrow\frac {\norm{T_{a,p,g}(v) - T_{a,p,g}(w)}^2} {(1-\norm{T_{a,p,g}(v)}^2)(1-\norm{T_{a,p,g}(w)}^2)} = \frac{\norm{v-w}^2}{(1-\norm{v}^2)(1-\norm{w}^2)}\\ 
&&\Rightarrow \boxed{d(T(v),T(w)) = d(v,w)}
\end{eqnarray*} 
To prove that $\mathrm{Isom}(\mathbb D)$ is a group we have to check the axioms of the group.
\begin{itemize}
 \item Existence of neutral element: $\mathrm{Id}_{\mathrm{Isom}(\mathbb D)} = T_{1,0,\mathrm{Id}}$ 
\item Existence of inverse element: Let $T_{a,p,g}$ be an arbitrary element of $\mathrm{Isom}(\mathbb D)$ and $z$ and arbitrary point in $\mathbb D$. Then we get
\begin{eqnarray*}
 && y := T_{a,p,g}(z) = \frac{a g(z)+p}{\bar p a g(z) +1} \Leftrightarrow y (\bar p a g(z) +1) = a g(z)+p \Leftrightarrow g(z) = \frac {y-p}{a(1-\bar p y)} \\
&&\Rightarrow \boxed{T^{-1}(y) = \frac{g(\bar a) g(y) + g(-p \bar a)}{ g(-\bar p a) g(\bar a) g(y) +1}}
\end{eqnarray*}
This is obviously an element of $\mathrm{Isom}(\mathbb D)$ for any $T$.
\item Closure: Let $T = T_{a,p,g}$ and $\tilde T = T_{\tilde a, \tilde p, \tilde g}$ be two arbitrary elements of $\mathrm{Isom}(\mathbb D)$
\begin{eqnarray*}
 && (T \circ \tilde T)(z) = T_{a,p,g}(T_{\tilde a, \tilde p, \tilde g}(z)) = \frac{a g(\frac{\tilde a \tilde g(z)+ \tilde p}{\bar \tilde p \tilde a \tilde g(z) +1})+p}{\bar p a g(a g(\frac{\tilde a \tilde g(z)+ \tilde p}{\bar \tilde p \tilde a \tilde g(z) +1}) +1}  \\
&&= \frac{a g(\tilde a \tilde g(z)) + a g(p) + a g(\bar \tilde p \tilde a \tilde g(z)) + a}{\bar \tilde p a g( \tilde a \tilde g(z)) + \bar p a g(\tilde p) + g(\bar \tilde p \tilde a \tilde g(z)) +1} 
= \frac{ \frac{a g(\tilde a) + p g(\bar \tilde p \tilde a)}{\bar p a g(\tilde p) + 1} (g \circ \tilde g)(z) + \frac{a g(\tilde p) + p} {\bar p a g(\tilde p) + 1}}{ \frac{ \bar p a g(\tilde a) + g(\bar \tilde p \tilde a) }{\bar p a g(\tilde p) + 1}(g \circ \tilde g)(z) +1}\\
&&= \frac {A \cdot G(v) + P}{\bar P \cdot A \cdot G(v) +1} = T_{A,P,G}(z)
\end{eqnarray*} where $A$,$P$ and $G$ are the parameters of the resulting transformation. $T_{A,P,G}$ is again an element of $\mathrm{Isom}(\mathbb D)$ as $\norm{A}$ = 1 and $G = \mathrm{Id}$ if $g=\tilde g$ and $g = \mathrm{Conj.}$ otherwise.
\item Associativity: With the formula for the result of an group operation obtained above, the associativity of $\mathrm{Isom}(\mathbb D)$ follows directly from the associativity of the complex numbers.
\end{itemize}
Thus, besides calculating some formulas about these isometries that we will need later on, we have shown the claim that $\mathrm{Isom}(\mathbb D)$ is indeed a group of isometries in the Poincare Disk model.
\end{proof}
For the higher dimensional case we will just note that the group generated by the Euclidean main rotations is a subgroup of the isometry group of the Poincare $n$-Hypersphere model. This follows since this model is isotropic i.e. any structure won't change if we rotate the whole space about the origin. This thereby is true as we deduced (or will deduce) any structure from the given metric and the axioms which are obviously all invariant under this kind of transformation.

\subsection{Geodesic lines} \label{sec:geom:geodesics}

Until now we have avoided to discuss what exactly lines are. Taking a close look at Axiom \ref{axiom:geom:metric} and \ref{axiom:geom:param}, we can deduce that, taken two arbitrary points $A$ and $B$, the segment of the line between these points has to be a (local) distance minimizing curve i.e. a so-called \emph{geodesic} or \emph{geodesic lines}. We will now show how we can determine in general the set of points forming a line in any of the two-dimensional analytic models above. The higher-dimensional lines can then be derived analogously.\\
\\
Let $A$ and $B$ two arbitrary points, $\mathcal C_0$ the line segment between them and $a=x(A)$ and $b=x(B)$ the coordinates of these points in respect the coordinatization $x$ of the line. Furthermore let $u$ and $v$ be coordinates and $G(u,v)$ the metric tensor of the model. We can therefore describe the line in terms of the coordinates of the model as follows:
\begin{equation*}
\overline{AB} = \setgen{(u(x),v(x))}{\begin{array}{l} x \in \setR, u_A = u(a), v_A = v(a), u_B = u(b), v_B = v(b)\\ \mathrm{with\ u,v\ continuous\ and\ distance \ minimizing} \end{array}}
\end{equation*}
By Axiom \ref{axiom:geom:param} we then get:
\begin{eqnarray*}
 b - a &=& |AB| = s = \int_{\mathcal C_0} ds = \int_{\mathcal C_0} \sqrt{g_{11}(u,v) du^2 + 2 g_{12}(u,v) du dv + g_{22}(u,v) dv^2} \\
&=& \int_a^b \sqrt{g_{11}(u(x),v(x)) \dot{u}^2 + 2 g_{12}(u(x),v(x)) \dot{u} \dot{v} + g_{22}(u(x),v(x)) \dot{v}^2} dx \\ &=& \int_a^b \sqrt{\Phi(u(x),v(x),\dot{u}(x),\dot{v}(x))} dx
\end{eqnarray*} where $s$ is the arclength of $\mathcal C_0$, $ds^2$ is the line element belonging to the given metric and $\dot{u}$ and $\dot{v}$ denotes the derivative of the coordinates in respect to the parameter $x$. Our goal is now to determine the functions, or to be more exact, functionals $u(x)$ and $v(x)$ that belong to the line i.e. minimize the arclength. This is obviously an extremum problem. A common method to solve it is to use \emph{calculus of variation}.
Thus, let $\mathcal C$ a small variation of $\mathcal C_0$ from $A$ to $B$ determined by function(al)s
\begin{eqnarray}
 \tilde u(x) &=& u(x) + \eps f(x) \label{eqn:geom:tilde_u} \\
 \tilde v(x) &=& u(x) + \eps g(x) \label{eqn:geom:tilde_v}
\end{eqnarray}
where $f$ and $g$ are arbitrary smooth functions vanishing for $a$ and $b$ and $\eps$ is the parameter of the variation. Thus we obtain for the length of $\mathcal C$:
\begin{equation}
\mathrm{length}(\mathcal C) = \mathrm{length}(\eps) = \int_a^b \sqrt{\Phi(\tilde u(x),\tilde v(x),\dot{\tilde u}(x),\dot{\tilde v}(x))} dx
\end{equation}
Since the length has to have a minimum at $\eps = 0$ for any choice of $f$ and $g$, it is required that the derivation in respect to $\eps$ vanishes there
\begin{eqnarray*}
\pdiff{\mathrm{length}(\mathcal C)}{\eps}|_{\eps=0} &=& \pdiff{\mathrm{length}(\eps) }{\eps}  = \left. \pdiff{ }{\eps} \int_a^b \sqrt{\Phi(\tilde u(x),\tilde v(x),\dot{\tilde u}(x),\dot{\tilde v}(x))} dx \right|_{\eps=0} \\
&=& \int_a^b \frac 1 2 \sqrt{\Phi} \pdiff{\Phi}{\eps}|_{\eps=0} dx 
\stackrel{!}=0
\end{eqnarray*}
Eqn.\ref{eqn:geom:tilde_u} and \ref{eqn:geom:tilde_v} and Integration by parts yield
\begin{eqnarray*}
&& \int_a^b \frac 1 2 \sqrt{\Phi} \pdiff{\Phi(u+\eps f,v+\eps g,\pdiff{}{x}(u+\eps f), \pdiff{}{x}(v+ \eps g))}{\eps}\lvert_{\eps=0} dx = 0\\
&\Leftrightarrow& \int_a^b \sqrt{\Phi} \left(\pdiff{\Phi}u f(x) + \pdiff{\Phi}v g(x) \right)\lvert_{\eps=0} dx  + \int_a^b \sqrt{\Phi} \left(\pdiff{\Phi}{\dot{u}} \pdiff f x + \pdiff{\Phi}{\dot{v}} \pdiff g x \right) \lvert_{\eps=0} dx  = 0\\
&\Leftrightarrow& \int_a^b \sqrt{\Phi} \left(\pdiff{\Phi}u f(x) + \pdiff{\Phi}v g(x) \right)\lvert_{\eps=0} dx \\
&& - \int_a^b  \sqrt{\Phi} \left( \left( \pdiff{}{x} \pdiff{\Phi}{\dot{u}} \right) f(x)  + \left( \pdiff{}{x} \pdiff{\Phi}{\dot{v}} \right) g(x) \right) \lvert_{\eps=0} dx = 0 \\
&\Leftrightarrow& \int_a^b  \sqrt{\Phi} \left(\pdiff{\Phi}u - \pdiff{}{x} \pdiff{\Phi}{\dot{u}} \right) f(x) dx \lvert_{\eps=0}  + \int_a^b  \sqrt{\Phi} \left(\pdiff{\Phi}v - \pdiff{}{x} \pdiff{\Phi}{\dot{v}} \right) g(x) dx \lvert_{\eps=0} = 0 
\end{eqnarray*}
As this equation has to hold for any $f$ and $g$, it follows that along $\mathcal C_0$ the following two differential equations called \emph{equations of geodesics} have to be fulfilled
\begin{equation}
 \boxed{\pdiff{\Phi}u - \pdiff{}{x} \pdiff{\Phi}{\dot{u}} = 0}, \qquad \boxed {\pdiff{\Phi}v - \pdiff{}{x} \pdiff{\Phi}{\dot{v}} =0} \qquad (a \leq x \leq b)
\label{eqn:geom:geod_eq}
\end{equation}
We will use this result to apply it on two models.

\subsubsection{Geodesics in Euclidean space}
In the Euclidean $n$-space we will determine the geodesics in the cartesian coordinates model. According to the metric shown in section \ref{sec:geom:euclModels} $\Phi$ is given by
\begin{equation}
\Phi(x_1,\dots,x_n,\dot{x}_1,\dots,\dot{x}_n) = \sum_{i=1}^n \dot{x}_i^2
\end{equation}
Thus we obtain $n$ equations of geodesics:
\begin{eqnarray*}
&& \pdiff{\Phi} {x_i} - \pdiff{}{x} \pdiff{\Phi}{\dot{x}_i} = 0 \qquad i \in \set{1, \dots n}\\
&\Rightarrow& \boxed{\ddot{x}_i = 0}
\end{eqnarray*} 
Hence the geodesics are the known straight lines.

\subsubsection{Geodesics in Hyperbolic space}

As seen in section \ref{sec:geom:euclModels} the line element in the two-dimensional hyperbolic spherical coordinates model is
\begin{equation*}
 ds^2 = dr^2 + \sinh^2(r) d\theta^2 \Rightarrow \Phi(r, \theta, \dot{r}, \dot{\theta}) = \dot{r}^2 + \sinh^2(r) \dot{\theta}^2
\end{equation*}
By using this in Eq.\ref{eqn:geom:geod_eq} we get
\begin{eqnarray*}
&&2 \sinh(r) \cosh(r) - 2 \ddot r = 0  \quad \land \quad
- s \pdiff{}{x} (\sinh^2(r) \dot{\theta}) = 0 
 \end{eqnarray*}
By integrating the right equation and substituting the result in the left one, we obtain:
\begin{eqnarray}
&\Rightarrow& \sinh^2(r) \dot \theta = A  \Leftrightarrow \dot \theta = \frac {C_1} {\sinh^2(r)} \label{eqn:geom:geod_hyp_3} \\
&\Rightarrow& c_1^2 \frac {\cosh(r)}{\sinh^3(r)} = \ddot(r) 
\Rightarrow \pdiff{}{x} (2 \dot r) \frac {C_1^2 \cosh(r)}{\sinh^3(r)} = (2 \dot r) \dot r^2  \nonumber \\
&\Rightarrow& \frac {C_1^2}{\sinh^2(r)} + \dot r^2  = C_2 \label{eqn:geom:geod_hyp_4}
\end{eqnarray} where $C_1$ and $C_2$ are constants. $C_1$ is hereby depending on the two points uniquely defining the geodesic curve whereas $C_2$ is depending on the parameterization of the line. We will now focus on the case that $C_1=0$. Then Eq.\ref{eqn:geom:geod_hyp_3} and \ref{eqn:geom:geod_hyp_4} yield
\begin{eqnarray*}
C_1=0 &\Rightarrow& \dot r = C_2 \quad \land \quad \dot \theta = 0\\
&\Rightarrow& \boxed{r = C_1 \cdot x \quad \land \quad \theta = const.}
\end{eqnarray*}
These are curves that pass through the origin. Transferring them to the Poincare Disk model would result in Euclidean lines in the embedding complex plane.\\
Considering the case of $C_1\neq 0$ would give us the other geodesics (which are not passing the origin and are represented by the segments of certain (Euclidean) circles in the Poincare model). But instead of calculating them we can use isometric transformations of the particular model namely rotations around the origin and translations to convert any possible case in the hyperbolic $n$-space to the one case we already solved above. Indeed, according to Cor.\ref{cor:geom:2din3d} we can handle any problem restricted to an arbitrary plane in the $n$-space as a problem in the two-dimensional hyperbolic space. In the preceding section we already discussed the isometries, namely the group generated by transformations, rotations and reflections in the two-dimensional Poincare disk model and already briefly mentioned the part main rotations are playing in the higher-dimensional spaces.\\
\\
Given two arbitrary points in the n-dimensional Poincare Sphere model. We use first an modification of the \emph{Aguilera-Perez Algorithm} presented in \cite{aguilera} to consecutively rotate both points into the $x_{n-1}$-$x_n$ plane of the embedding $\setR^n$-space as follows (where $\vec a$ and $\vec b$ are the representations of the points $A$ and $B$):

\begin{lstlisting}[captionpos=b,mathescape=true, numbers=none, caption=modified version of Aguilera-Perez Alg. to rotate both $\vec a$ and $\vec b$ into $x_{n-1}$-$x_n$-plane]
let T = Identity
for i := 1 to n-1
	M := MainRotation${}_{i+1,i}$(atan2($a_{i+1}$,$a_i$))
	$\vec a$ := M($\vec a$)
	T = M $\circ$ T

$\vec b$ := T($\vec b$)
for i := 1 to n-2
	M:= MainRotation${}_{i+1,i}$(atan2($b_{i+1}$,$b_i$))
	$\vec b$ := M($\vec b$)
	T = M $\circ$ T
\end{lstlisting}
In the $x_{n-1}$-$x_n$ plane we can now use the already defined two-dimensional translation to move one of the points to the origin. Thus, we are in a case where we already know how the geodesic look like. To get the points on the geodesic back in the $n$-space, we just have to apply the reversal transformations (reverse translation and reverse of rotation $M$ in the algorithm) on the points of the known geodesic in the $x_{n-1}$-$x_n$ plane.\\
\\
By doing so we get for example in the two- and three-dimensional Poincare sphere model the results that are shown in fig.\ref{fig:geom:geodesics}

\begin{figure}[!ht]
\begin{center}
\includegraphics[width=0.48\linewidth]{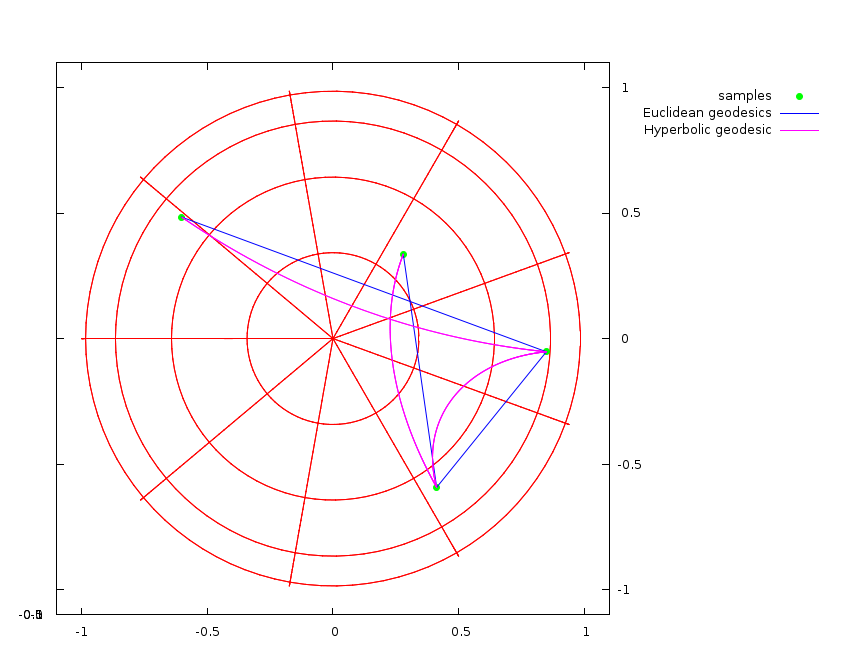} 
\includegraphics[width=0.48\linewidth]{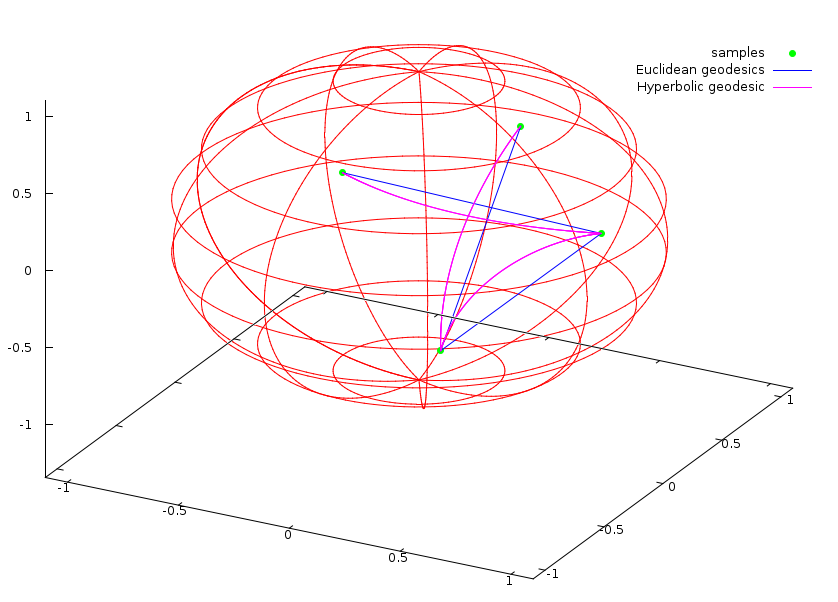} 
\caption{Geodesics in (left) 2-dim (right) 3-dim Euclidean (blue) and Hyperbolic space (purple)}
\label{fig:geom:geodesics}
\end{center}
\end{figure}

\subsubsection{Geodesics in Spherical space}
Although we won't prove it here, it should be briefly noted that the geodesics in spherical space are the so called \emph{Great Circles}. These are the circles around the center if we embed the two-dimensional spherical plane naturally in the three-dimensional case (i.e. sphere).

\subsection{*Trigonometric laws} \label{sec:geom:trig_laws}

Since the Euclidean and Hyperbolic geometries are quite similar, the well-known trigonometric laws of the Euclidean space have their equivalences in the hyperbolic space. Given a triangle with sides $a$,$b$ and $c$ and the corresponding angles $\alpha$, $\beta$ and $\gamma$, these are (proof can be found in e.g. \cite[ch.5]{Anderson}):

\begin{description}
 \item[Law of sines:] 
\begin{equation}
 \frac {\sinh(a)} {\sin(\alpha)} = \frac {\sinh(b)} {\sin(\beta)} = \frac {\sinh(c)} {\sin(\gamma)}
\end{equation}
 \item[Law of cosines I:]
\begin{equation}
 \cosh(c) = \cosh(a)\cosh(b) - \sinh(a)\sinh(b) \cos(\gamma)
\end{equation}
 \item[Law of cosines II:]
\begin{equation}
 \cos(\gamma) = -\cos(\alpha)\cos(\beta) + \sin(\alpha)\sin(\beta)\cosh(c)
\end{equation}
\end{description}
As a property of the hyperbolic space mentioned above, in the limit of infinitesimal small triangles we will get the Euclidean version back.

\subsection{*Equidistant curves} \label{sec:geometry:equidistant}

Concluding this chapter, we will determine the formulas for curves and curved planes that are equidistant to a straight line or flat plane. Assuming wlog. that the representation of the reference plane in the particular models is chosen in such a manner that the coordinates except one are invariant under a projection onto a (n-1)dim subspace of the same geometry, i.e. a plane containing (n-1) axis. This is always possible, because we can transform every flat plane into one that fulfills this condition by using isometries.\\
\\
We will start again with the Euclidean space. In two dimensions, a equidistant curve to a straight line $l$ equals a parallel line of $f$. Thus it is again a (straight) line. \\
The generalization to higher dimensions can be defined recursively. A n-dimensional equidistant plane in the distance $d$ to the plane $P = \set{(x_i)_i | x_1 \dots, x_{n-1} \in \setR, x_n = 0} $ is then determined by $\set{ (x_i)_i | x_1\dots, x_{n-1} \in \setR, x_n = d}$\\
\\
In hyperbolic space equidistants are no longer straight lines and flat planes, but can be again easily derived from the trigonometric laws. 

\begin{figure}[!ht]
\begin{center}
\includegraphics[width=0.6\linewidth]{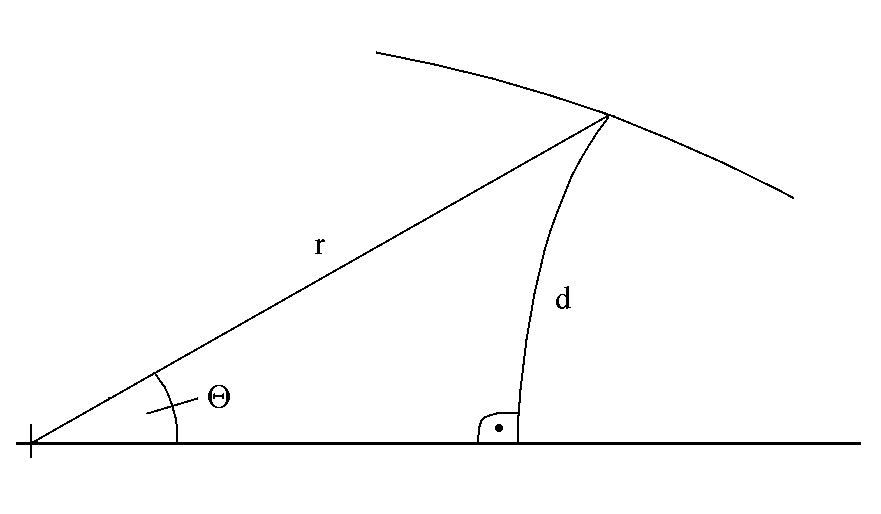} 
\caption{Equidistant curve in 2-dim hyperbolic space}
\label{fig:geometry:skizze_equidistant_hyp}
\end{center}
\end{figure}

Given a right triangle as shown in fig. \ref{fig:geometry:skizze_equidistant_hyp} with on vertex at the origin,$d$ as the distant and one side be part of one of the lines, using the law of sines leads to the following formula defining the equidistant curve.

\begin{equation} \label{eqn:geom:equi_hyp}
  \sinh(r) = \frac {\sinh(d)} {\sin(\Theta)}
\end{equation}

Using $\phi = \pi - \Theta$ instead, we just get here a relation between the radius and the polar angle of the hyperbolic spherical coordinates model that a equidistant has to fulfill. Converting this result e.g. to the two- and three-dimensional Poincare Sphere model we finally get the curves and surfaces that are shown in fig.\ref{fig:geom:equidistants_hyp}.
\begin{figure}[!ht]
\begin{center}
\includegraphics[width=0.49\linewidth]{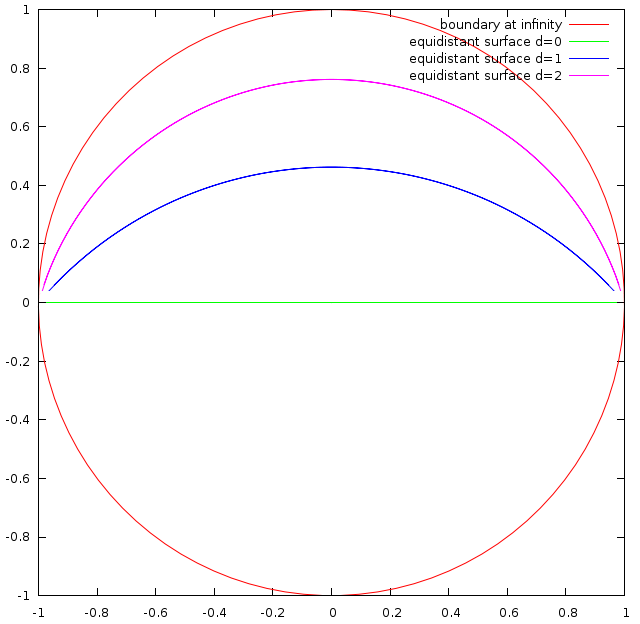}
\includegraphics[width=0.49\linewidth]{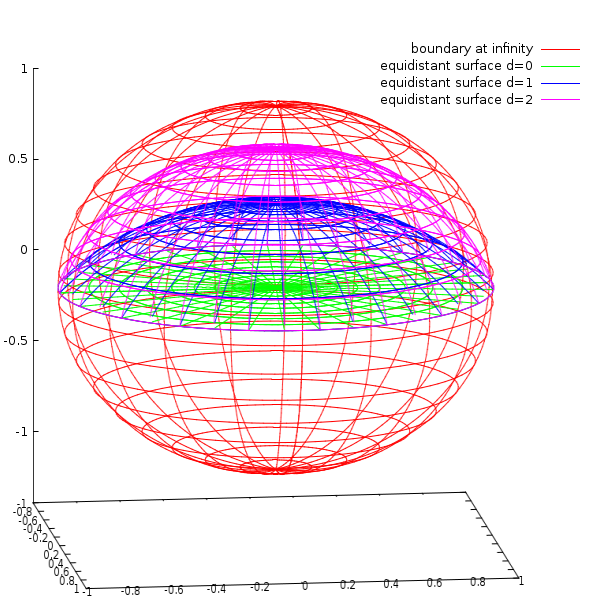}
\caption{Equidistant curves and surfaces in (left) 2-dim (right) 3-dim Hyperbolic Poincare Disk/Sphere model}
\label{fig:geom:equidistants_hyp}
\end{center}
\end{figure}

\chapter{Generating Maps} \label{ch:tess}

As discussed in ch. \ref{ch:gsom}, the neurons form a special kind of structure in the map space. In our case, this structure is chosen in such a way that the characteristics of the space, in which it is embedded, are best represented. A common way to satisfy this criterion is to partition the the space into a set of regularly bounded regions. These form the so called \emph{regular tiling} or \emph{tessellation}. The representatives of the neurons in the map space are then placed at the positions of the centroids $\bar v_i$ of these tiles given by
\begin{equation*}
\bar v_i = \frac{\int_{\mathrm{Tile\ }i} \vec v dV}{\int_{\mathrm{Tile\ }i} dV}
\end{equation*} where $dV$ is the volume element of the particular map space. Thus in return taking this set of nodes, the corresponding \emph{Dirichlet} or \emph{Voronoi tessellation} defined by
\begin{equation*}
 \mathrm{Tile\ }i = \setgen{v}{\norm{v-v_i} \leq \norm {v-w_j} \forall j}
\end{equation*}
which equals the original tiling. The regularity of the tiling ensures a high degree of symmetry like certain rotation and translation invariances that reflect the symmetries of the map space itself.\\
\\
The representatives that are placed in such a manner and the edges between them form itself the \emph{dual tiling}. As it is easier to create the dual tiling than to calculate the centroids of the original tiling, we will use the latter method. To avoid thereby unnecessary confusion about the nomenclature of the maps (i.e. if a triangular map is defined by a triangular Voronoi tessellation of its nodes or by a triangular dual tiling that generates the nodes) we will name henceforth the resulting maps by referring to the shape of the Voronoi cells. Thus, using our example above, a triangular Euclidean map corresponds to the set of vertices of a partition of hexagonal tiles (and vice versa). Furthermore, a map, defined by \emph{Schlaefli symbols} (cf.\cite{polytopes}), will always refer to the respective Voronoi tiling of the map described by these Schlaefli symbols.\\
\\
In the following we will now have a look at the Euclidean and Hyperbolic plane and discuss the possible regular tilings. Then we introduce a method to use these tilings to generate the map.

\newpage
\section{Tilings of the Euclidean plane}

To determine the possible regular tilings for the plane i.e. the tiling by regular polygons, we will examine their \emph{triangulation}, which is a tiling by triangles. Fig. \ref{fig:tesselation:triangl} shows such a triangulation of certain polygons. We will see below that these polygons are all bases for tessellations of the plane and that there exist no others.\\
\\
\begin{figure}[!ht]
\begin{center}
\includegraphics[height=0.25\linewidth]{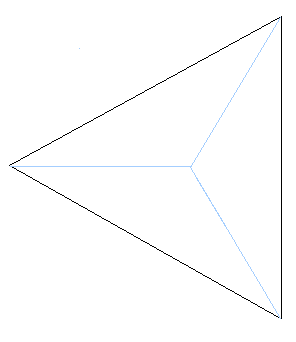}
\hspace{0.1\linewidth}
\includegraphics[height=0.25\linewidth]{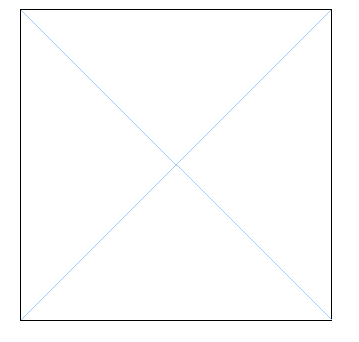}
\hspace{0.1\linewidth}
\includegraphics[height=0.25\linewidth]{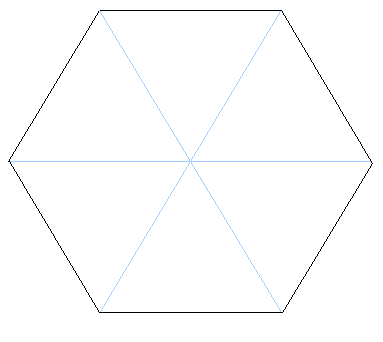}
\end{center}
\caption{triangulation of triangle, square \& hexagon}
\label{fig:tesselation:triangl}
\end{figure}

Each polygon can be divided into triangles by connecting the center with each vertex. So, a n-gon consists of $n$ such triangles. As it was assumed that the n-gons are regular, the triangles are all congruent to each other. The angle $\alpha$ belonging to the vertex in the center of the polygon is therefore $2\pi$ divided by the number of edges or vertices of the n-gon. The other two angles $\beta$ and $\gamma$ are equal. So we get two conditions which the interior angles at the vertices of the polygon have to satisfy. First, if there are $q$ neighbors meeting at each vertex, these angles have to be $ \frac {2 \pi} q$. Secondly, since two of the $p$ triangles meet in each of the vertices, we conclude, that the interior angles are just twice as large as $beta$ and $gamma$. Summarizing all this results in

\begin{eqnarray*}
 \alpha &=& \frac {2\pi} p \\
\beta &=& \gamma = \frac {\pi} q
\end{eqnarray*}

We now use the fact that triangles in Euclidean space have neither an angular defect nor an angular excess i.e. the sum of the three interior angles is always equal to $\pi$ (in contrast to the spherical and hyperbolic case, in which it is strictly larger resp. smaller than $\pi$). Thus we obtain:
\begin{equation*}
 \alpha + \beta + \gamma = \pi \Rightarrow \frac {2\pi}p + \frac {\pi}q + \frac {\pi}q = 1 \Rightarrow \frac 1p + \frac 1q = \frac 12
\end{equation*}
This is fulfilled for only three choices of $p,q \in \setN$:
\begin{center}
\begin{minipage}{0.8\linewidth}
\begin{description}
\item[(triangular)] $p=3, q=6 \rightarrow \alpha = 120^\circ, \beta = \gamma = 30^\circ$
\item[(square)] $p=4, q=4 \rightarrow \alpha = \beta = \gamma = 90^\circ$
\item[(hexagonal)] $p=6, q=3 \rightarrow \alpha = 120^\circ, \beta = \gamma = 30^\circ$
\end{description}
\end{minipage}
\end{center}

So, all valid regular tilings of the Euclidean plane consist of sets of identical tiles which are congruent to one of these three defined above. Even if it seems quite obvious, it is important to mention, that the edge size of this tiles can be chosen at will. Hence arbitrary refinements of the tilings and therefore any density of neurons are possible.

\begin{figure}[!ht]
\begin{center}
\includegraphics[width=0.25\textwidth]{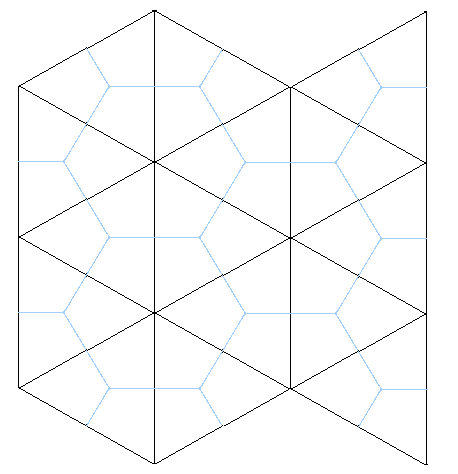}
\hspace{0.07\textwidth}
\includegraphics[width=0.25\textwidth]{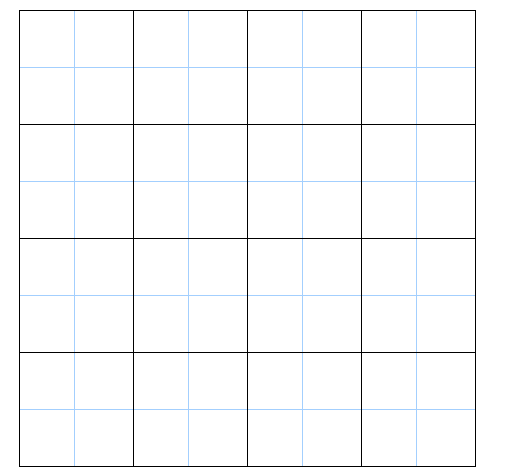}
\hspace{0.07\textwidth}
\includegraphics[width=0.25\textwidth]{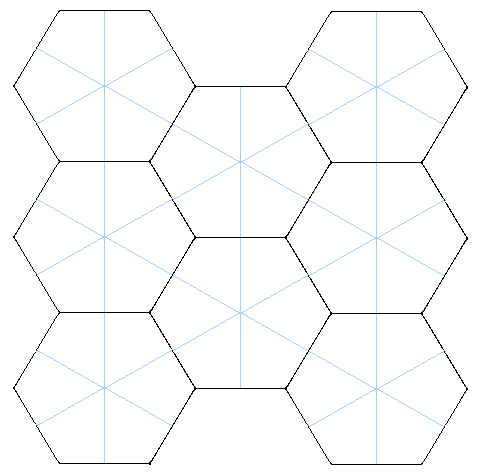}
\end{center}
\caption{Regular tesselation of Euclidean plane to Schlaefli symbols (3,6), (4,4) and (6,3) }
\label{fig:tesselation:triangulation}
\end{figure}

\subsection{Generating tilings} \label{ch:tess:generation}

Each regular tiling belongs to a certain symmetry group. Each member of this group thereby maps tilings onto each other. By using the whole group, it is already possible to generate the whole tiling starting with only one given tile or vertex. Depending on the tiling there are many ways to choose the symmetries. Following a paper of N. Kuiper \cite{kuiper} we will use involutions at the center of the edges, i.e. reflections perpendicular to the edge or halfturns. The needed transformations can be easily composed using the basic isometries of the Euclidean (and hyperbolic) space which are rotations $R(\phi)$ around origin by angle $\phi$, translations $T(z)$ (moving 0 to z) and reflections $M$ along an axis. Given a tiling of $p$-gons with $q$ neighbors at each vertex, the (finite) symmetry group is consisting of
\begin{equation*}
R^{-1}(\frac {2\pi\cdot i} q) \circ T^{-1}(\frac{\mathrm{edge\ length}}2) \circ M \circ T(\frac{\mathrm{edge\ length}}2) \circ R(\frac {2\pi \cdot i} q) \quad i \in \{0 \dots p-1\}
\end{equation*}

Then every vertex can be created by letting a composition of elements of the group acting on a distinguished point $o$ called \emph{origin}, i.e. for every vertex $\nu$ exists a finite sequence $(g_{i_1}, \dots, g_{i_n})$, $g_{i_k} \in G$  such that

\begin{equation} \label{eq:tess:gen_nu}
\nu = (g_{i_1} \circ \dots  \circ g_{i_n})(o)
\end{equation}

Using the involution in a side center as proposed, every $g_{i_k}$ ``transports'' the point at the one end of the side to the other. Fig.\ref{fig:tesselation:generation} illustrates this process using each member of the symmetry group of the (6,3)-tiling once.

\begin{figure}[!ht]
\begin{center}
\includegraphics[width=0.8\textwidth]{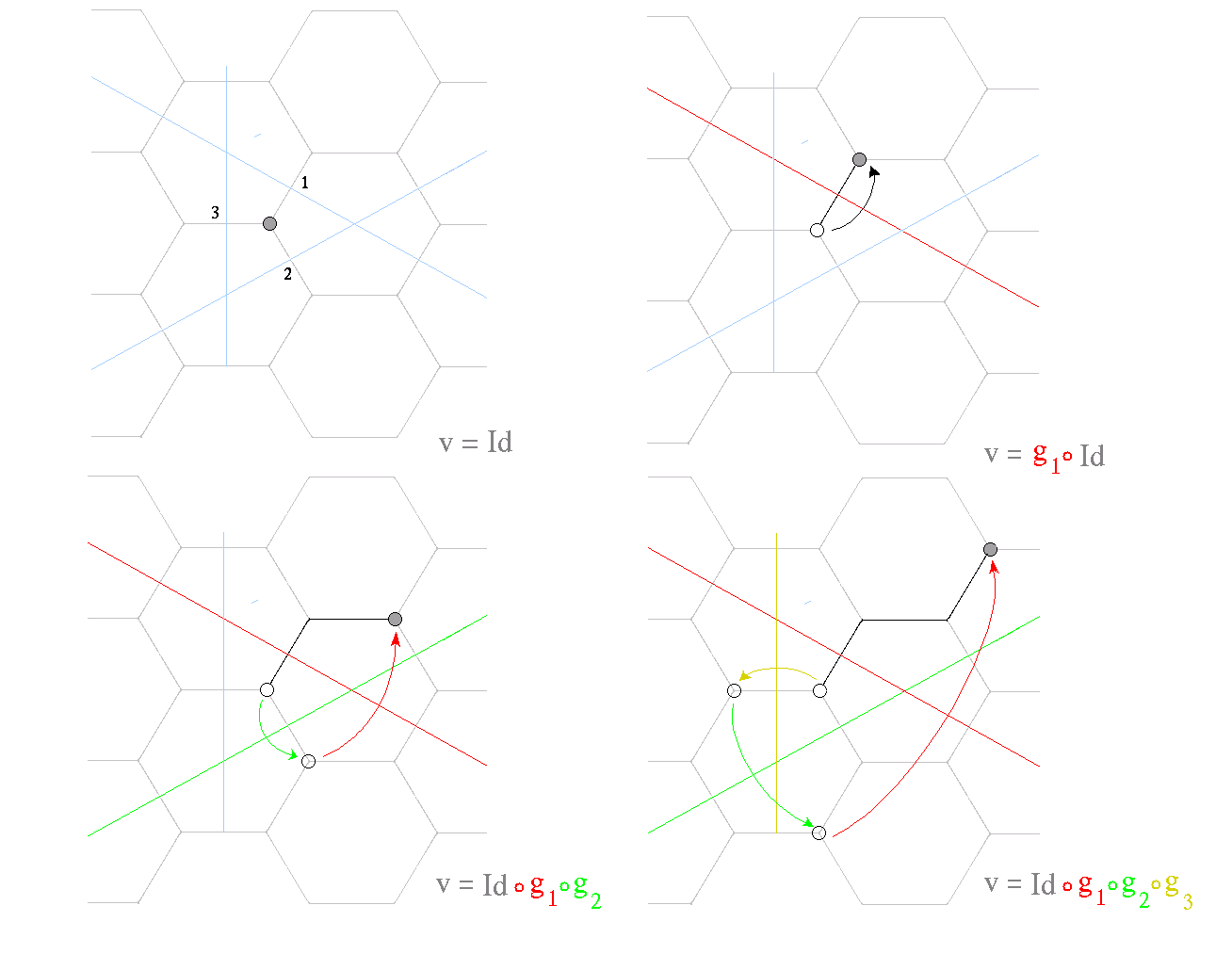}
\end{center}
\caption{} 
\label{fig:tesselation:generation}
\end{figure}

Thus, when applying each of the generator of the group on a certain vertex, we obtain all the neighbors of it, i.e. all vertices connected to the former one by an edge. This will be called \emph{expansion} of the \emph{parent} vertex creating as many \emph{children} as there are generators.\\
\\
It is note-worthy, that the given definition of the generating symmetry group do not use $p$ as a parameter i.e. the tiles are already defined by $q$ and the length of their edges. In the Euclidean plane $p$ depends on $q$ and vice versa. In the hyperbolic case the edge length is also needed. We will see this when discussing hyperbolic tilings below. But beside this fact everything in this paragraph so long holds for both geometries.

\subsection{Periodic boundary conditions}

A useful ``feature'' of the Euclidean space is the possibility to define \emph{periodic boundary conditions} quite easily. Boundary conditions are an important method in numerical computations. Since every simulation has a limited computational time, it is only possible to simulate finite grids. As the neighborhood of points near the edge of these finite grids now differs significantly from the neighborhood of the vertices in the center, every truncation of the map results in boundary effects. A solution to avoid this is to wrap the space, so that each edge is ''glued`` to the opposite one, i.e. if $f(\vec z)$ describes any property of the point $\vec z$ then $f(\vec z) = f(\vec z+\vec r)$ where $|r_i|$ is the size of the grid in the $x_i$-direction. In our case, this means, that distances and adaptions are always calculated in respect to the nearest representative of this class of points. More precisely, in two dimensions we then no longer work in the Euclidean plane but in the quotient space $ \setR^n / \setZ^2 \cong (\setR / \setZ)^2 $ called \emph{flat torus} which has nontheless a zero curvature.\\
Fig. \ref{fig:tesselation:periodic_bound} shows examples of finite square boundaries for each of the three tilings found above. Due to different symmetries of the square and hexagonal tiling the grid size is restricted to even numbers of neurons along each dimension.
\begin{figure}[!ht]
\begin{center}
\includegraphics[width=0.48\textwidth]{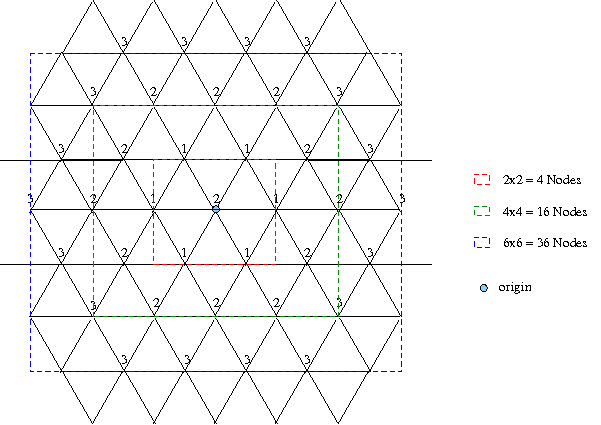}
\includegraphics[width=0.48\textwidth]{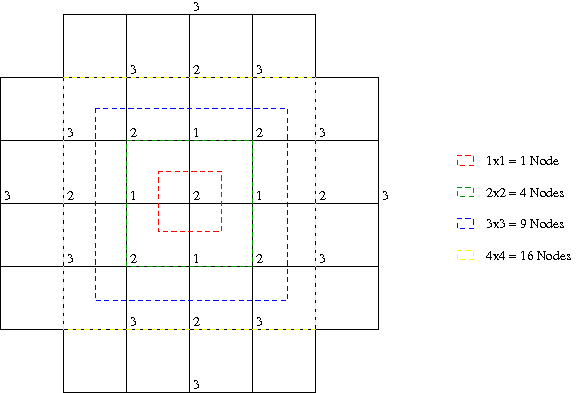}
\includegraphics[width=0.48\textwidth]{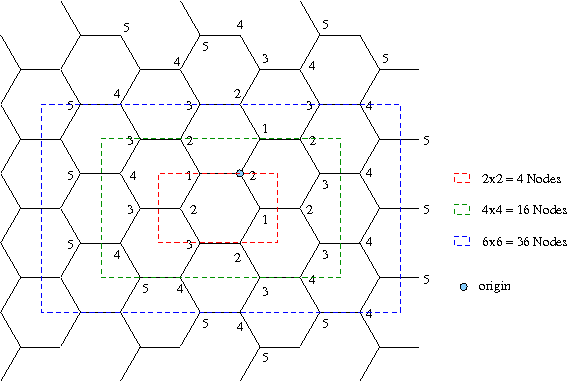}
\end{center}
\caption{Periodic boundaries for tesselation of Schlaefli symbols (3,6), (4,4) and (6,3) }
\label{fig:tesselation:periodic_bound}
\end{figure}

\section{Tilings of the Hyperbolic plane}

The same method, that we used to find an appropriate tiling for the Euclidean plane, can be adapted for the Hyperbolic case. It differs only in the fact that we have now angular defects to take into account. This softens the restrictions to the interior angles, since the sum of the angles of each triangle (of finite size) is now strictly smaller than $\pi$, which allows infinitely many choices for $p$ and $q$. Examples are\footnote{In fact, we will use maps that correspond to these tilings in the numerical analysis.}
\begin{center}
\begin{minipage}{0.8\linewidth}
 \begin{description}
\item[(triangular)] $p=3, q=n \qquad n\geq 7 \Rightarrow \frac 13 + \frac 1n < \frac 12$
\item[(square)] $p=4, q=n \qquad n\geq 5 \Rightarrow \frac 14 + \frac 1n < \frac 12$
\item[(hexagonal)] $p=6, q=n \qquad n\geq 4 \Rightarrow \frac 1 6 + \frac 1n < \frac 12$
\end{description}
\end{minipage}
\end{center}
While therefore the choice of the shape of the tiles and their neighborhood is less strict than in the Euclidean plane, the contrary is true for the third parameter, namely the length of the edges. According to the \emph{theorem of Gauss-Bonnet} (cf.\cite{Anderson},\cite{Ramsay}), the area of a triangle and hence for every polygon is proportional to the angular defect. Since the choice of $p$ and $q$ determines the defect $\Delta$ by
\begin{equation*}
 \Delta = 2\pi - \frac {2\pi} p - 2 \frac {\pi} q > 0
\end{equation*}
the length of the edges is also already uniquely determined by them. To deduce now the formula for this edge length, we will have to remind the hyperbolic law of cosines II (cf.chapter \ref{sec:geom:trig_laws}):
\begin{equation*}
 \cos(\gamma) = -\cos(\alpha)\cos(\beta) + \sin(\alpha)\sin(\beta)\cosh(c)
\end{equation*}
While using the same triangulation of a regular polygon as seen above and regarding one of these resulting triangles, let  $\alpha$, $\beta$ and $\gamma$ be chosen as above and let $c$ be the length of the edge opposite to $gamma$. Then $c$ is equal to the edge length of the polygon. Therefore, we obtain:
\begin{equation*}
c = \acosh(\frac{\cos(\frac {2\pi} p) + \cos^2(\frac \pi q)} {\sin^2(\frac \pi q)})
\end{equation*}
Since we need only the half of the length for defining the symmetry group which generates the tiling (see \ref{ch:tess:generation}) this can be further simplified as follows
\begin{eqnarray*}
(c/2) &=& \frac 12 \acosh\left(\frac{\cos(\frac {2\pi} p) + \cos^2(\frac \pi q)} {\sin^2(\frac \pi q)}\right) 
= \frac 12 \acosh\left(\frac{\cos(\frac {2\pi} p) + 1 - \sin^2(\frac \pi q)} {\sin^2(\frac \pi q)}\right)\\
	   &=& \frac 12 \acosh\left(2 \frac{\cos^2(\frac {\pi} p)} {\sin^2(\frac \pi q)} -1 \right) 
	   = \acosh(\frac {\cos(\frac {\pi} p)}{\sin(\frac \pi q)})
\end{eqnarray*}
Thus, the regular tilings like the three examples shown in fig.\ref{fig:tesselation:tess_hyperbolic} are completely defined.

\begin{figure}[!ht]
\begin{center}
\unitlength 1cm
\includegraphics[width=0.25\textwidth]{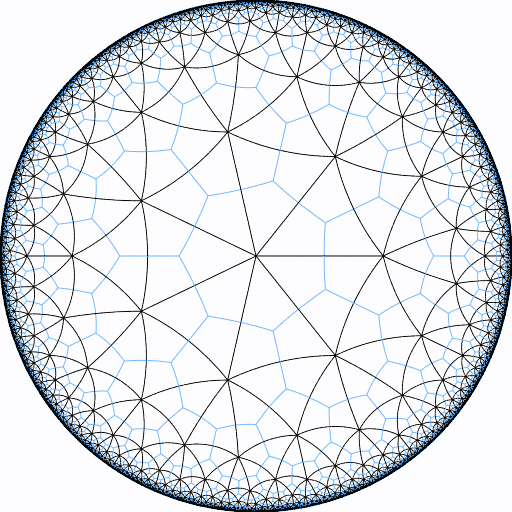}
\hspace{0.07\textwidth}
\includegraphics[width=0.25\textwidth]{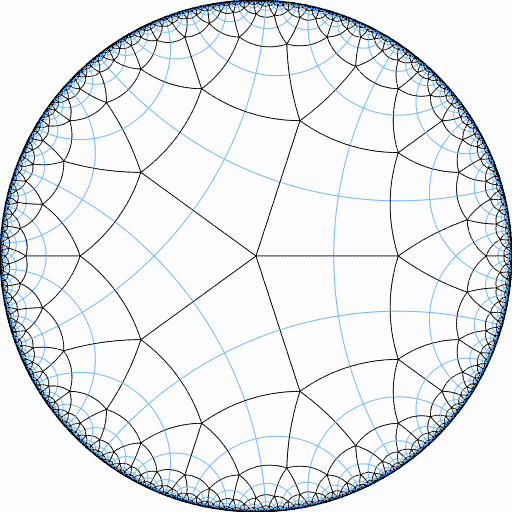}
\hspace{0.07\textwidth}
\includegraphics[width=0.25\textwidth]{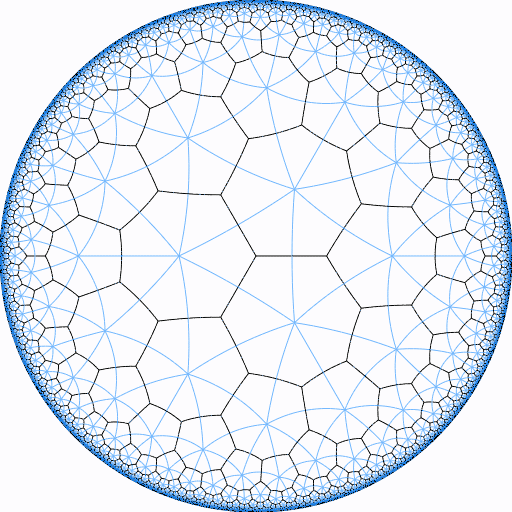}
\end{center}
\caption{Regular tesselation of hyperbolic plane (in Poincare model) to Schlaefli symbols (3,7), (4,5) and (7,3) }
\label{fig:tesselation:tess_hyperbolic}
\end{figure}

\subsection{Generating tilings}

Using a symmetry group defined in the same way as in the Euclidean case, any of the infinitely many regular tessellations of the hyperbolic plane can be generated. But, as we have seen right above, the regular Hyperbolic maps are already completely defined by the Schlaefli number and does not allow a variation of the edge length.\\
\\
Another significant difference to the Euclidean case becomes obvious if we consider the number of nodes lying in a given circular neighborhood $B_r$ in the two-dimensional Euclidean and Hyperbolic plane or, to be more exact, its growth rate in respect to the radius $r$ of the neighborhood. Since all the regular tiles have the same area, the number of nodes multiplied by the area of the faces in the dual tiling is approximately proportional to the area of $B_r$. In the Euclidean case, this area is given by $A^{Eucl.}(r)=\pi r^2$. The area of the Hyperbolic circle on the other side can be analogously determined by the following integration\footnote{using Hyperbolic spherical coordinates}:
 \begin{equation*}
  A^{Hyp.}(r) = \int_0^r \int_0^{2\pi} \Theta \sinh(r) d\Theta dr = 2 \pi (\cosh(r) -1)
 \end{equation*}
Thus, this yields an exponential growth rate in respect to the radius for the Hyperbolic case while the same rate is only polynomial in the Euclidean plane. Tab.\ref{tab:tess:area} lists now the number of nodes for various Euclidean and Hyperbolic maps and different sizes of neighborhoods as an example. It can been seen that the number of nodes lying in a circular neighborhood indeed growths nearly by the same factor as the area of the neighborhood itself, as claimed above.

\begin{table}[!ht]
\begin{tabular}{|l||c|c|c||c|c|c|c|}\hline
& \multicolumn{3}{c||}{Euclidean space} & \multicolumn{4}{c|}{Hyperbolic space}\\
radius & area &(4,4),$d_{NN}=1$&(3,6),$d_{NN}=1$&area&(3,7)&(4,5)&(6,4) \\ \hline
$\sqrt2$&6.28&9&7&7.40&8&6&5\\ \hline
2 &12.57&13&19&17.36&15&11&5\\ \hline
$2\sqrt2$&25.13&25&31&47.05&43&41&21\\ \hline
4 &50.26&49&61&165.30&176&101&81\\ \hline
\end{tabular}
\caption{Number of nodes for Euclidean and Hyperbolic maps and different sizes of circular neighborhoods}
\label{tab:tess:area}
\end{table}

\subsection{Boundaries/Truncation}

Analog to the Euclidean case, we are confronted with the fact that we can only use maps with finite size in numerical simulations, but unlike the Euclidean maps, periodic boundaries can not be defined here in such an easy manner. Furthermore, due to the exponential growth most of the nodes of the Hyperbolic map lie always near the edge of a finite map.
Two sensible methods of truncating the infinite tiling \footnote{not to be mistaken for the ''truncation`` of a tiling!} to get a finite map are:

\begin{itemize}
 \item[-] Truncation at a given maximal distance of the generated nodes to the origin $o$. 
 \item[-] Truncation at a given maximal level of expansion i.e. maximal number of expansions needed to create a node when starting with the origin $o$
\end{itemize}
Given a circular area around $o$, the truncation at the maximal distance ensures the maximal number of nodes in this area. So, this method was for example used to determine the number of nodes in tab.\ref{tab:tess:area} for the various hyperbolic maps.
\\
The second method of truncating at a maximal expansion level, however, results in a shape of the map that is, in general, less circular. Fig.\ref{fig:tesselation:overlap} illustrates this fact by comparing the distances of the nodes of different layers for the (3,7)- and the (6,4)-map. While the nodes of the lower layers are close together, the nodes of higher layers are rather scattered and the layers even overlap. Thus, by truncating at a certain level, many nodes that would lie at the same distances from the origin as the generates vertices would be neglected.

\begin{figure}[!ht]
\begin{center}
\includegraphics[width=0.48\textwidth]{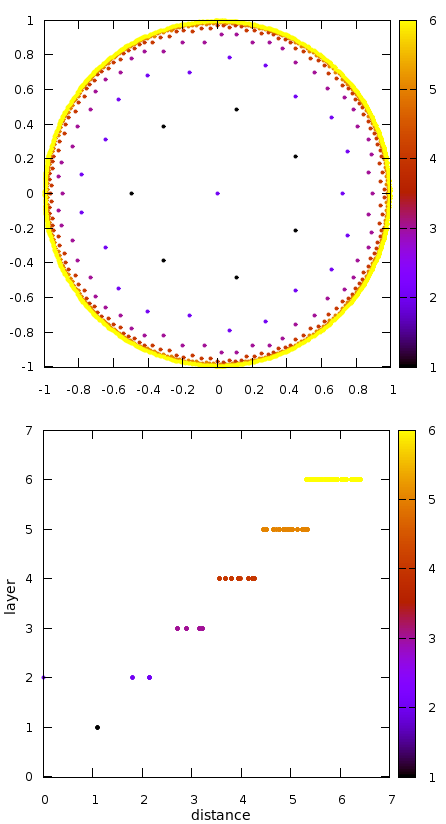}
\includegraphics[width=0.48\textwidth]{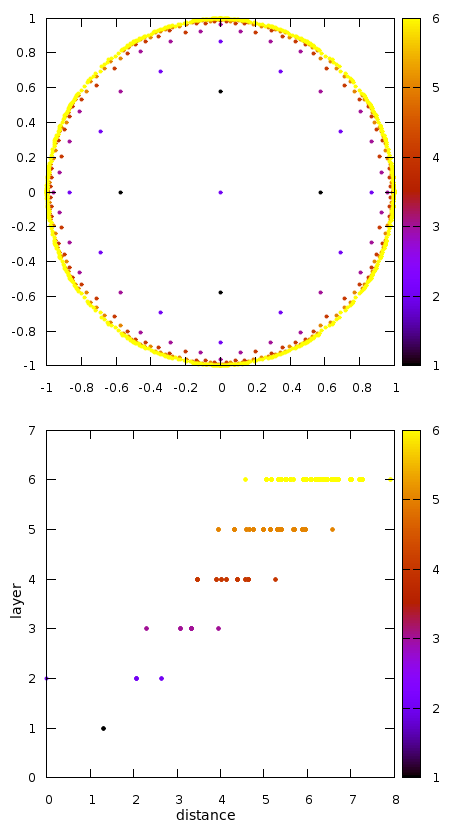}
\end{center}
\caption{Position of nodes/layers in Poincare Disk and distance of nodes to the origin (left) (3,7)-map (right) (6,4)-map}
\label{fig:tesselation:overlap}
\end{figure}

But, on the other hand side, this truncation method ensures, that the map is isotropic in respect to its hierarchical structure, and is used in the papers covering (non-growing) HSOMs (e.g. \cite{ritter},\cite{ontrup2}).

\chapter{Sample distributions} \label{ch:distr}

Normally, the samples in the feature space presented to the SOM have an distribution that reflects the structure of the examined object or networks which is in general very complex. The goal here is to analyze the stability of certain configurations of SOMs concerning an additional dimension (i.e. input in feature space has one more intrinsic dimension than map space. In the following, we will name the dimensions except the additional one \emph{map dimensions} whereas the additional one will be called \emph{extra dimension}) and depending on their type of map and feature space configuration, sample sets with particular constructed probability densities will be used. To adjust the influence of the extra dimension, the spread of the samples will be bounded in this direction and, to ensure the translation invariance of the sample distribution in the dimension except the extra one, two equidistant (curved) hyperplanes are chosen as the boundaries. In the other dimensions there will also be boundaries to limit the sample set volume to keep the probability density normalizable. As described in the section about tessellation, the use of periodic boundaries can then be used to create at least the illusion of an unbounded sample set in all dimensions besides the extra one. Nevertheless should the extra dimension be much smaller than any other to suppress boundary effects.\\
\\
Since we have already collected all necessary geometrical "ingredients" in a former chapter about the geometry of the spaces, we will now concentrate on the definition of these distributions for the different spaces and models fitting best to their inherent structure without imposing any other additional special structures. Thus appropriate distributions are uniform or can be obtained by transferring uniform structures of the map space onto the feature space. They share all the invariances concerning rotations and translations that are intrinsic to the grids in the respective map space. Hence the definition of the probability distribution that are needed here is quite trivial. They will be shortly introduced in the next paragraph. Much more work has to be done to construct and implement these. A detailed view on this will be done further down.

\newpage
\section{Distributions} \label{sec:distr:distr}

Generally speaking, a (continuous) uniform distribution is a class of probability distributions such that all subsets with equal volumes are equally probable i.e. each infinitesimal volume element $dV$ should have the same probability. Knowing the metric and thus also the volume element $dV$ allows us to calculate directly the (joint or multidimensional) probability density function for each particular model of the space we want to use. At this point it should be noted, that the densities that will be presented below, still depend on a normalization constant $N$, that itself depends on the support of the functions, i.e. the volume of the set of possible samples. Since the densities will be used unnormalized, $N$ will not have to be determined in every case.

Before starting to go through the individual cases, a theorem we will make use of has to be mentioned (cf. \cite[Thm.4.2]{Luc}).

\begin{thm}[Transformation of (n-dimensional) random variables] \label{thm:distr:luc4.2}
Let X have a continuous density $f$ on its support set $S \subseteq \setR^d$ and given a transformation which is a $C^1$-diffeomorphism $h:S \stackrel{~}\rightarrow T\subseteq \setR^d$, i.e. continuous and bijective function such that, if $g(y) := h^{-1}(y) : T\stackrel{~}\rightarrow S$ is the inverse of the transformation, its first partial derivatives and therefore its Jacobian matrix $J = (\frac {\partial g_i}{\partial y_i})_{ij}$ exists and the derivatives are continuous on T. Then $Y=h(X)$ has density
$$ f(g(y)) \det(J) \qquad y \in T$$ 
\end{thm}

With this theorem and the transformations defined in ch.\ref{ch:geometry} we can convert densities given in one model to the corresponding densities for another model.

\subsection{Euclidean-Euclidean case} \label{sec:distr:def_EE}

In the Euclidean space we will use the well-known cartesian coordinate system. The corresponding metric tensor is simply the identity. Thus the volume element is given by 
\begin{equation*}
 dV = dx_1 \cdot dx_2 \cdot \dots \cdot dx_n
\end{equation*}
Let $V$ be the volume of the compact subset $s \subset \setR^n$, in which the samples shall be located. Then the joint probability density function of the cartesian coordinates is:
\begin{equation*}
 \varrho(X_1,\dots,X_n) = 1/V =: N
\end{equation*}
where $X_1,\dots,X_n$ are the random variates representing the cartesian coordinates. An illustration of this can be found in fig \ref{distribution:EE}.
\begin{figure}[!ht]
\begin{center}
\includegraphics[width=0.45\linewidth]{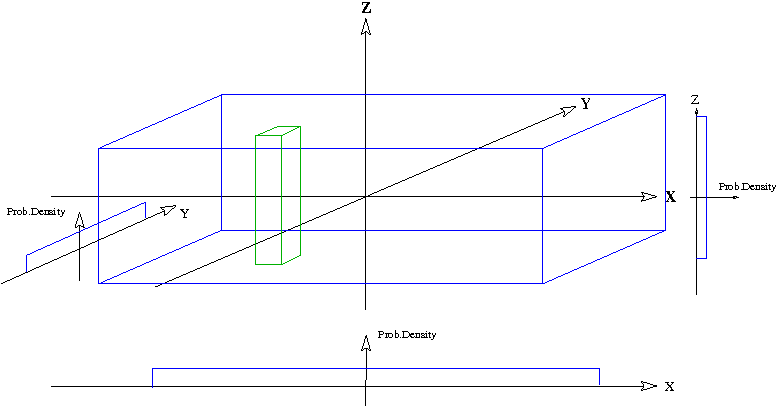}
\includegraphics[width=0.45\linewidth]{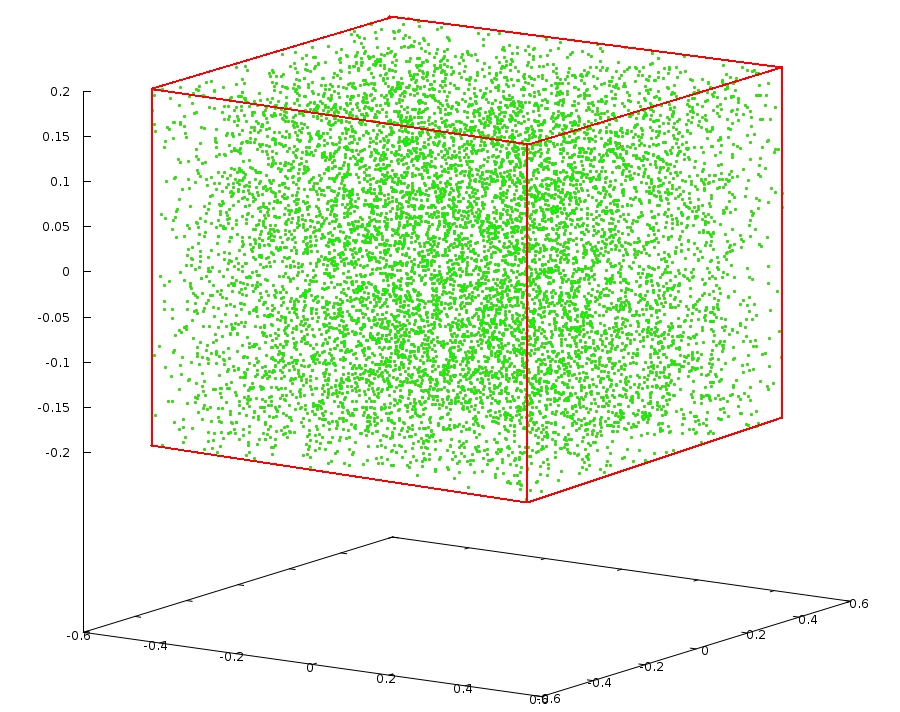}
\end{center}
\caption{Euclidean space: (left) illustration of distribution boundaries, densities and voronoi cell (right) realization of uniform distribution}
\label{distribution:EE}
\end{figure}

Since the equidistant surfaces are planes in the Euclidean space, the shape of the compact subset $S$ will simply be cuboidal (as seen in the figure).

\newpage
\subsection{Hyperbolic-Hyperbolic case} \label{sec:distr:def_HH}

In hyperbolic space (as in any other) the probability density depends on the chosen model. We will focus here on the spherical coordinate and Poincare (hyper-)sphere model, because both will be used in the following. Again since both map and feature space are again equal, a uniform distribution will be used. The way to define it is analog to the Euclidean case above. Only the volume element and therefore the density is more complicated.

\begin{description}
 \item[Hyperbolic Spherical Coordinates model] Due to the metric defined in Eq. \ref{eqn:geom:HSph_metric}, the volume element is given by:
\begin{equation*}
 dV_{HSph} = (\sinh r)^2 \sin \phi dr d\theta d\phi
\end{equation*}
and so we obtain as the joint density function:
\begin{equation*}
\varrho_{HSph}(R,\Phi,\Theta) = N (\sinh r)^2 \sin \phi
\end{equation*}
where $V$ is again the volume of the compact subset $S \subset H^2$, from which the samples are drawn.
 \item [Poincare (hyper-)sphere model] There are now two ways to get the uniform density for this model. We can again have a look at the volume element and deduce the density by using it analogously to the cases above. But since we are going to implement the uniform distribution belonging to this model by just using the spherical equivalent, it is more suitable here to derive the density also by using the coordination transformation between the two models and applying theorem \ref{thm:distr:luc4.2}.
\\
We have already seen in section \ref{ch:geom:Hyp_models} that this transformation $T_{HSph \rightarrow PD}$ is defined by Eq.\ref{eq:geom:trafo_hsph_pd}. Thus the reverse transformation $T_{PD \rightarrow HSph}$ for the desired three-dimensional case is:
\begin{eqnarray*}
r &=& 2 \atanh(\sqrt{x_1^2 + x_2^2 + x_3^2}) \\
\phi &=&  \atantwo(\sqrt{x_1^2 + x_2^2} ,x_3^2) \\
\theta &=& \atantwo(x_2,x_1)
\end{eqnarray*}
where $\atantwo$ is the usual two-argument variant of arctangent. Furthermore is the Jacobian matrix of $T_{Sph \rightarrow PD}$ given by:
\begin{eqnarray*}
 &&J_{HSph \rightarrow PD} = \pdiff{T_{Sph \rightarrow PD}(r,\phi,\theta)}{(r, \phi,\theta)} \\
 &&  = \matrixdrei{ \frac{1- \tanh^2(\frac r2)}2 \sin(\phi) \cos(\theta) & \tanh(\frac r2)  \cos(\phi) \cos(\theta) & - \tanh(\frac r2)  \sin(\phi) \sin(\theta) \\
 \frac{1 - \tanh^2(\frac r2)}2 \sin(\phi) \sin(\theta) & \tanh(\frac r2)  \cos(\phi) \sin(\theta) & \tanh(\frac r2)  \sin(\phi) \cos(\theta) \\
\frac{1 - \tanh^2(\frac r2)}2 \cos(\phi)  & - \tanh(\frac r2)  \sin(\phi) & 0} \\
\end{eqnarray*}
The Jacobian determinant of the inverse transformation, that will be needed to apply the theorem can then simply be obtained by calculating the determinant of $J$ above and then by using the fact that the determinant is a multiplicative map taking the reciprocal\footnote{We skip at this point to specify explicitly the Jacobian of the inverse transformation, but its existence follows directly by the non-vanishing determinant and thus the existence of the reciprocal.}. This yields:
\begin{eqnarray*}
  \det(J_{HSph \rightarrow PD}) &=& \tanh^2(\frac r2) \sin \phi \cdot ( \frac{ 1-\tanh^2(\frac r2)} 2) \\
\Rightarrow \det(J_{PD \rightarrow HSph}) &=& \frac 1 {\tanh^2(\atanh(\sqrt{x_1^2 + x_2^2 + x_3^2})) \sin(\atantwo(\sqrt{x_1^2 + x_2^2} ,x_3^2))}\\ &\cdot& \frac 2 {( 1-\tanh^2(\atanh(\sqrt{x_1^2 + x_2^2 + x_3^2})))} \\
&=& \frac 2 {R^2 \sin(\atantwo(\sqrt{x_1^2 + x_2^2} ,x_3^2)) (1-R^2) }
\end{eqnarray*} 
with $R = \sqrt{x_1^2 + x_2^2 + x_3^2} = \norm{x}$. Now we can finally transfer the spherical probability density $\varrho_{HSph}$ to the Poincare Sphere model:
\begin{eqnarray*}
&&\varrho_{PD}(x_1,x_2,x_3) \stackrel{\mathrm{Thm.}\ref{thm:distr:luc4.2}}= \varrho_{HSph}(T_{PD \rightarrow HSph}(x_1,x_2,x_3)) \cdot \det(J_{PD \rightarrow HSph}) \\
&&= N \frac {2 (\sinh^2(2 \atanh(R))} {R^2 \cdot ( 1-R^2)}
= N\frac { 2 \frac{\tanh^2(2 \atanh(R))}{1-\tanh^2(2 \atanh(R))} }{R^2 \cdot ( 1-R^2)}\\
&&= N\frac { 2 \frac{(2R/(1+R^2))^2}{1-(2R/(1+R^2))^2}}{R^2 \cdot ( 1-R^2)}
= N\frac { 2 \cdot 4 R^2/(1-R^2)^2}{R^2 \cdot ( 1-R^2)} = \frac{\tilde N} {( 1-R^2)^3}\\
&\Rightarrow& \varrho_{PD}(x_1,x_2,x_3) = \frac{\tilde N} {( 1-\norm{x}^2)^3}
\end{eqnarray*}

\end{description}

Unlike in the Euclidean case the shape of the sample set as shown in the Fig.\ref{fig:distribution:HH} is not so trivial, since the equidistant curves and surfaces (cf. section \ref{sec:geometry:equidistant}) are no longer lines and planes in the hyperbolic space. To avoid boundary effects, the ``lateral'' boundary is shaped in a way to preserve the voronoi cells of the outermost nodes, i.e. is defined by the geodesics at each point of the edge of the disk such that they are perpendicular to it. 

\begin{figure}[!ht]
\begin{center}
\unitlength 1cm
\includegraphics[width=0.49\linewidth]{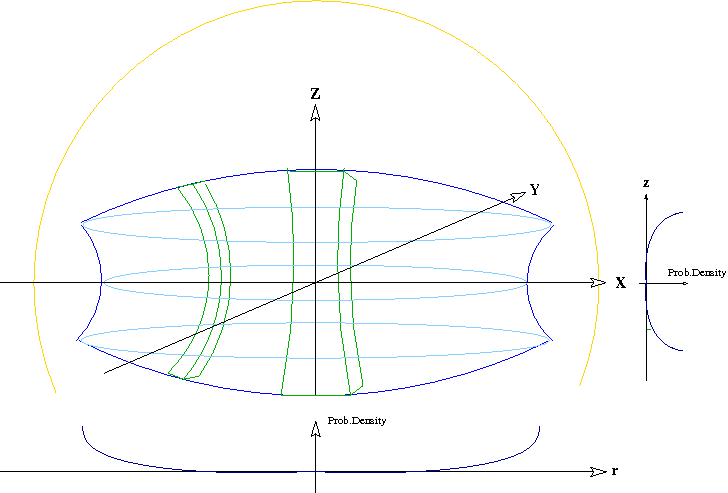}
\includegraphics[width=0.49\linewidth]{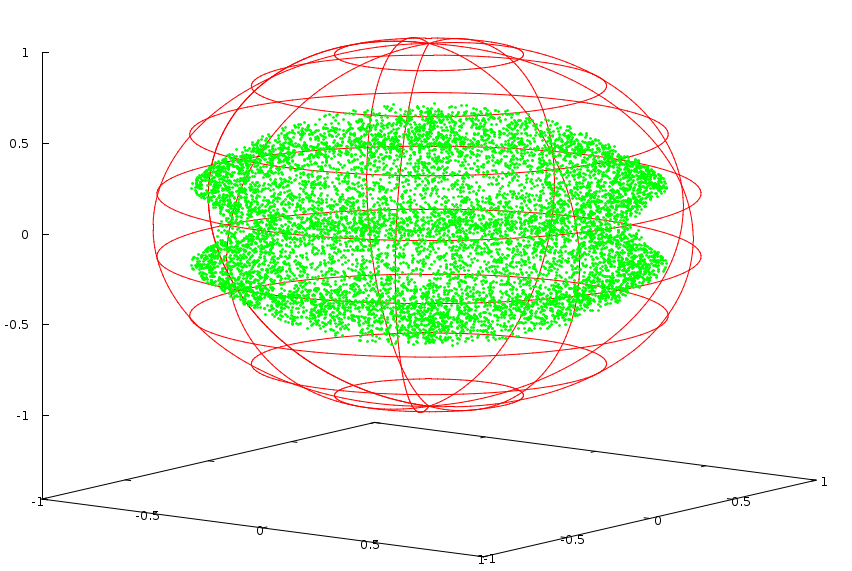}
\end{center}
\caption{Hyperbolic space / Poincare Sphere: (left) illustration of distribution boundaries, densities and voronoi cells (right) realization of uniform distribution}
\label{fig:distribution:HH}
\end{figure}

\subsection{Hyperbolic-Euclidean case} \label{sec:distr:def_HE}

This is the first case where the map space is different than the feature space. It is not longer enough to just use a uniform distribution in the feature space, because this would result in the lost of the hierarchical structure that we have in the map space and therefore in the grid of neurons. To conserve this, we adapt the probability density by using the volume element of the map space. Only for the extra dimension a uniform Euclidean distribution is used. Thus by embedding the Poincare Disk model into our Euclidean feature space we get:
 
\begin{equation*}
dV = dV_{PD} \cdot ds_{Eucl} = (1 - \norm{x}^2)^{-(n-1)} dx_1\dots dx_n
\end{equation*}

In other words, we present the SOM in the map dimensions a hierarchical structure and an uniform Euclidean scattering in the extra dimension as we would obtain by statistical fluctuations around a ``hierarchical object''. The shape is therefore as seen in fig. \ref{fig:distribution:HE} a ``fat'' disk.

\begin{figure}[!ht]
\begin{center}
\unitlength 1cm
\includegraphics[width=0.49\linewidth]{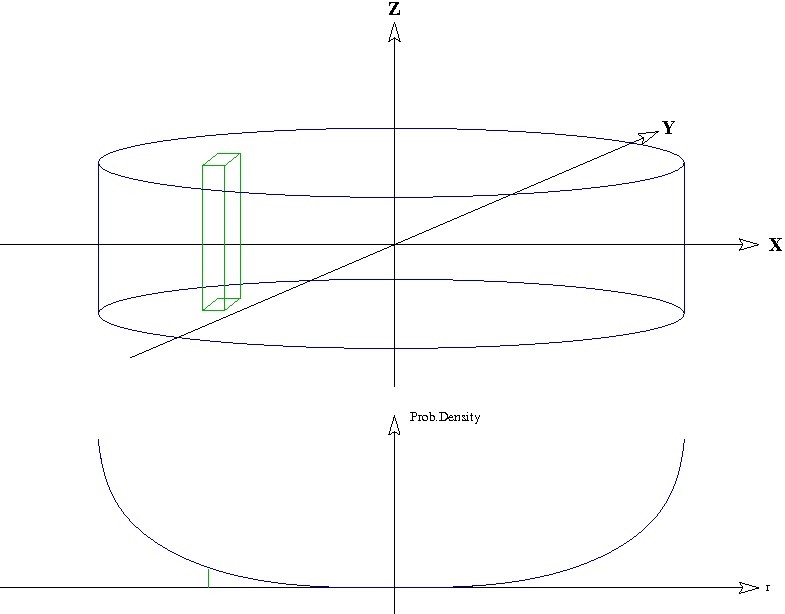}
\includegraphics[width=0.49\linewidth]{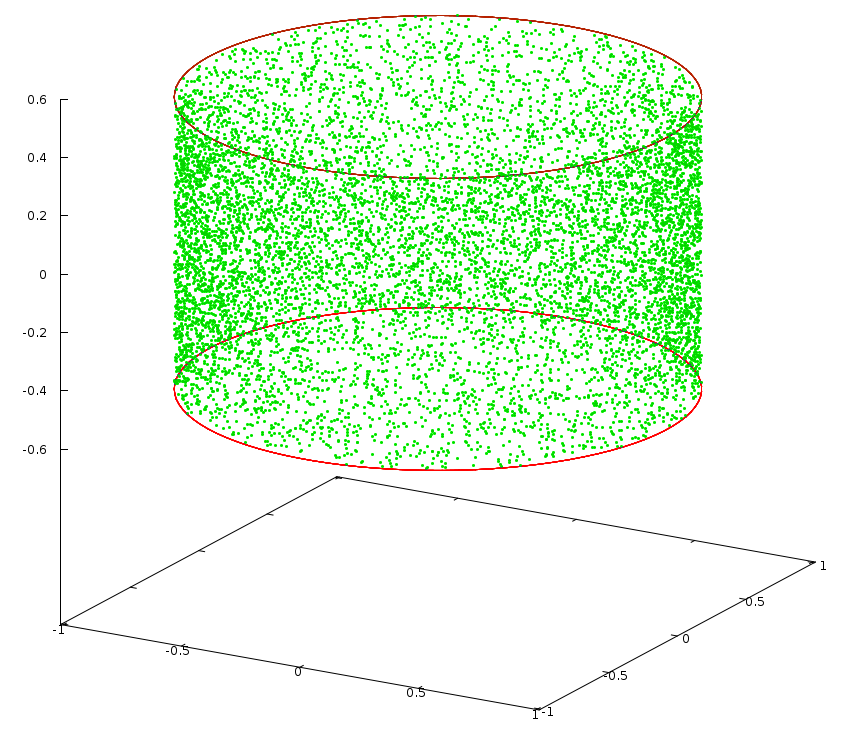}
\end{center}
\caption{Euclidean space: (left) illustration of distribution boundaries, densities and voronoi cells (right) realization of distribution}
\label{fig:distribution:HE}
\end{figure}

\newpage
\section{*Generating random variates} \label{sec:distr:generate}

Now we have to discuss how to generate the random variates defined by a particular density function. Referring to Luc Devroye's ``bible'' for non-uniform number generation \cite{Luc} we will present two important principles. In order to do this as short as possible, we will skip the proofs of the thereby used theorems, but the interested reader may find them in the appendix \ref{ch:app_distr}. We will further assume that the reader has knowledge of the basic vocabulary of probability theory.

\subsection{*Inversion principle}

This method is based upon the property given by \cite[thm.II.2.1]{Luc}

\begin{thm}
Let $F$ be a continuous cumulative distribution function (CDF) on \setR with Inverse $F^{-1}$ defined by
$$ F^{-1}(u) = \inf\{x \in \setR | F(x) = u\} $$
If $U$ is a uniform random variable in $[0,1]$, then $F^{-1}(U)$ has distribution function $F$. Also, if $X$ has distribution function $F$, then $F(X)$ is uniformly distributed on $[0,1]$.
\label{thm:distributions:inversion}
\end{thm}

Using this theorem and given an arbitrary continuous CDF $F$ which has an explicitly known inverse, a random variate with this distribution function can be generated by the following simple algorithm:

\begin{cor}[Inversion method] 
Random variates distributed with CDF $F$ can be obtained as follows::
\begin{itemize}
\item[] \texttt{Generate a uniform [0,1] random variate U}.
\item[] \texttt{RETURN X} $\Rightarrow$ $F^{-1}(U)$
\end{itemize}
\label{thm:distributions:inversion_method} 
\end{cor}

(The proof of this corollary is trivial and will be omitted)

\subsection{*Rejection method}

The requirement of having a CDF with known inverse is often hard to fulfil. So even if the inverse is well-defined, there may often be no analytical solution of $F(x)=u$ and therefore it can difficult to compute the inverse fast and accurate. One way to fix this, is to use a numerical solution of  $F(x)=u$. This would lead to an unavoidable trade-off between computation time and accuracy. Instead we will make use of the \emph{rejection principle}. This method allows to create accurate random variates at the cost of having to reject some generated ones as will be shown below. But first two theorems.

\begin{thm}
 \begin{enumerate}[1)]
\item Let $X=(X_1,\dots,X_d)$ be a collection of $d$ independent and identically-distributed (iid) random variates i.e. a random vector  with joint density $f$ on $\setR^d$ and let $U$ be an independent  uniform [0,1] random variate.
Then $(X, cUf(x))$ is uniformly distributed on $A = \setgen{(x,u)}{x \in \setR^d, 0 \leq u \leq cf(x)}$ where $c > 0$ is an arbitrary constant.
\item Vice versa, if $(X,U)$ is a random vector in $\setR^{d+1}$ uniformly distributed on $A$, then $X$ has density $f$ on $\setR^d$.
\end{enumerate}
\label{thm:distributions:rejection_1}
\end{thm}
\begin{thm}
Let $X_1, X_2, \dots$ be a sequence of iid random vectors taking values in $\setR^d$, and let $A \subseteq \setR^d$ be a Borel set such that $P(X_1 \in A) = p > 0$. Let $Y$ be the first $X_i$ taking values in $A$. Then $Y$ has a distribution that is determined by
$$P(Y\in B) = \frac {P(X_1 \in A \bigcap B)} p $$
where $B$ is another Borel set in $\setR^d$.
In particular, if $X_1$ is uniformly distributed in $A_0$, where $A_0 \supseteq A$, them $Y$ is uniformly distributed in A.
\label{thm:distributions:rejection_2} 
\end{thm}

Combining both theorems we get the following:

\begin{thm}[Rejection method]
Given a density function $f$ let $g$ be a function and  $c \leq 1$ be a constant such that
$$ f(x) \leq c \cdot g(x) \quad \forall x$$.  Then the following algorithm  generates random variates with density $f$:
\begin{itemize} 
\item[] \texttt{REPEAT}
\item[] \texttt{Generate X with density g and U uniformly distributed on [0,1]}
\item[] \texttt{set } $T \leftarrow c \cdot \frac{g(X)}{f(X)}$
\item[] \texttt{UNTIL} $U\cdot T \leq 1$
\item[] \texttt{RETURN X}
\end{itemize}
\label{thm:distributions:rejection_method} 
\end{thm}

In other words, three things are required here:
\begin{enumerate}[(i)]
\item a (good) dominating density $g$
\item a simple (and efficient) method for generating random variates with density $g$
\item knowledge of $c$
\end{enumerate}
Although $g$ and $c$ can be easily found by an analytical analysis of $f$, it's quite important to choose $g$ wisely, because we have just exchanged our problem of generating random variates using the density $f$ with same problem with using $g$ instead. When we are going to construct the generators for the sample sets below, we will thus choose $g$ such that one the hand its CDF is easily invertible, such that the inversion method can be used, and, on the other hand, it will be close enough to $f$ so that the rejection rate is kept low to maintain the efficiency.

\subsection{*Euclidean-Euclidean case}
This case is simple if it is done by using the Euclidean model of cartesian coordinates. First, we assume that the sample set is centered in the origin. Then the equidistant surfaces i.e. the boundaries in the extra dimension are given by $x_3= \pm s$ where $2s$ is the distance of the opposite boundaries. The boundaries in the map dimension are obtained in the same manner. Hence the choice of each coordinate is independent of the choice of the other ones. Thus we get our uniform distributed sample set by combining the realizations of three iid random variates to a vector, i.e. let $X_1:\Omega \rightarrow ]X_1^{min},X_1^{max}]$,$X_2:\Omega \rightarrow ]X_1^{min},X_1^{max}]$ and $X_3:\Omega \rightarrow ]X_1^{min},X_1^{max}]$ be three uniform idd random variates. Then the resulting three-dimensional random variate $\vec X: \Omega \rightarrow S$ is obtained by:
\begin{equation*}
 \vec X(\cdot) = \left(X_1(\cdot), X_2(\cdot), X_3(\cdot) \right)^T
\end{equation*}

\subsection{*Hyperbolic-Hyperbolic case}
The hyperbolic-hyperbolic case is the most complex one of the three. To simplify the problem we will work in the hyperbolic spherical model, because, as we have seen above, the probability density $\varrho_{HSph}$ only depends on one of the three dimensions, namely $r$. Furthermore we will only take care of the case that the sample set is less extended in the extra dimension than in the map dimensions (This is not a severe restriction since we can simply enlarge the map arbitrarily). We then transfer our results to the Poincare Disk by transforming the resulting realizations.
\\
Regardless of the beneficial form of the probability density the choice of the radial component $r$, the polar angle $\phi$ and the $azimuth$ $\theta$ still depend on each other due to the shape of the sample set. We will therefore fix $r$ and have a look at the respective ``shells'' of the sample set. The propability that a sample has a certain distance $r$ from the origin is then given by the ``relative mass'' of the shell where the measure is determined by the probability density and thus the volume element as seen above\footnote{As announced we will hereby omit the normalization constant $N$. Thus actually we do not use the infinitesimal mass of the shell, but the ``relative mass'' which is the (finite) product of the mass of the shell and the total mass of the bounded volume of the sample set and therefore the density and cumulative distribution functions are normalizable but not normalized.}. Let us assume that we have the two-dimensional planar disk defined by $\rho =0$ and $r \leq R$. Then the boundaries in the extra dimension are the equidistant surfaces above and beneath this plane at the distant $s$. As noted before, we will only regard the case where $s \leq R$. The other boundaries are given as already defined in section \ref{sec:distr:def_HH}. 
Fig.\ref{fig:distribution:skizze_cases} shows the three cases that may occur for the shells where $R_{max}$ is the maximal distance between the origin and the points in the volume. It can obtained by using the law of cosines in the right triangles with the side length $R_{max}$, $R$ and $s$:
\begin{equation*}
R_{max} = \arccos(\cosh(s) \cosh(R)) 
\end{equation*}

\begin{figure}[!ht]
\begin{center}
\includegraphics[width=0.7\linewidth]{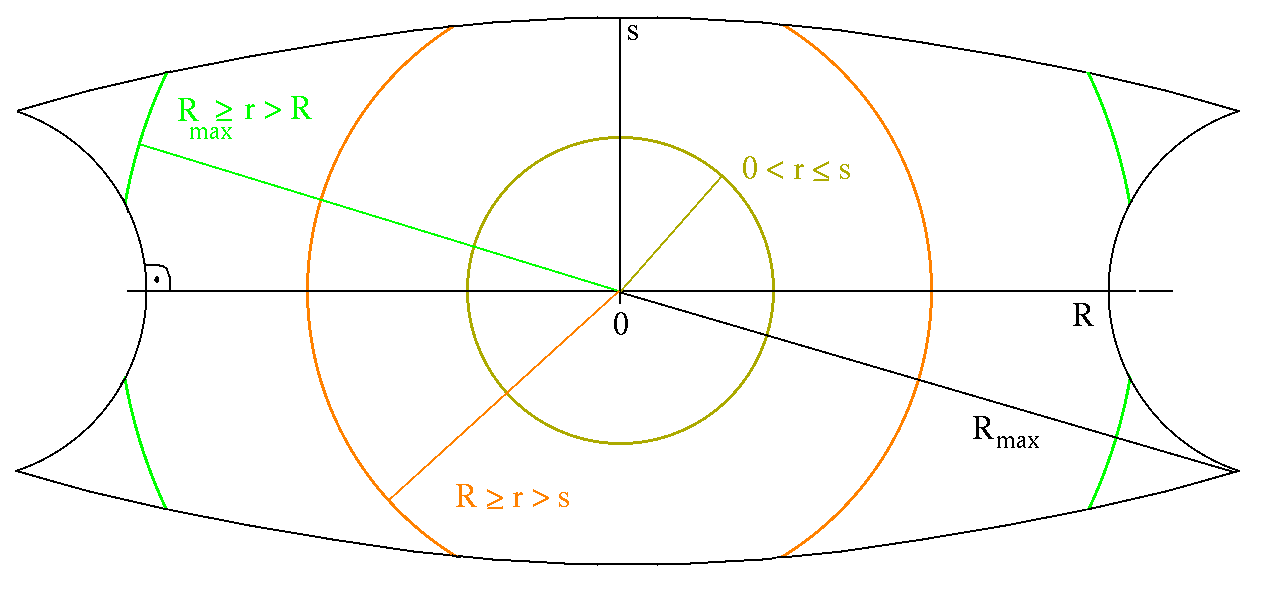}
\end{center}
\caption{3 cases for shells corresponding to fixed $r$ (projected onto plane $\theta = 0 \lor \theta = \pi$)}
\label{fig:distribution:skizze_cases}
\end{figure} 

If $r$ is smaller than $s$ the ``relative mass'' is simply the area of the surface of the three-dimensional sphere, i.e.
\begin{equation*}
\rho_1(r) = \int_0^\pi \int_0^{2\pi} dV =  \int_0^\pi \int_0^{2\pi} (\sinh r)^2 \sin \phi dr d\theta d\phi = 4 \pi \sinh^2 r 
\end{equation*}
The CDF is then obtained by taking the volume of the sphere that is
\begin{equation*}
F(r) = \int_0^r 4 \pi \sinh^2(\tilde r) d \tilde r = 2 \pi ( \sinh(r) \cosh(r) - 1) 
\end{equation*}
The constant summand unfortunately causes a problem if we would try to invert this CDF to apply the inversion method. We will see how to approach this when we are going to deal with the overall density function.\\
\\
When $r$ exceeds $s$ the surface is not longer the one of the whole sphere but restricted to the component between the two equidistant surfaces. Thus we have to determine the bounds of the polar angle $\phi$ corresponding to this restriction. We get these values by using the equation, which we already derived for the equidistant curves (cf. Eq.\ref{eqn:geom:equi_hyp}). The only thing we have to remind is, that we have to substitute the $\theta$ which was originally used by $\pi - \phi$. Hence we get:
\begin{eqnarray*}
 \phi_{min}(r) &=& \frac \pi 2 - \arcsin\left( \frac {\sinh(s)}{\sinh(r)} \right) = \arccos\left( \frac {\sinh(s)}{\sinh(r)} \right) \\
\phi_{max}(r) &=& \frac \pi 2 - \arcsin\left( \frac {\sinh(s)}{\sinh(r)} \right) = \pi - \arccos\left( \frac {\sinh(s)}{\sinh(r)} \right)
\end{eqnarray*}
Thus we obtain for the (non-normalized) density in this case:
\begin{eqnarray*}
\rho_2(r) &=& \int_{\phi_{min}(r)}^{\phi_{max}(r)} \int_0^{2\pi} dV =  \int_{\arccos\left( \frac {\sinh(s)}{\sinh(r)} \right)}^{\pi - \arccos\left( \frac {\sinh(s)}{\sinh(r)} \right)} \int_0^{2\pi} (\sinh r)^2 \sin \phi dr d\theta d\phi\\
 &=& 4 \pi \sinh(s) \sinh(r)
\end{eqnarray*}
In the case that $r$ is now larger than $R$ we get an additional restriction for the polar angle as we can see in Fig.\ref{fig:distribution:skizze_phi}.
\begin{figure}[!ht]
\begin{center}
\includegraphics[width=0.7\linewidth]{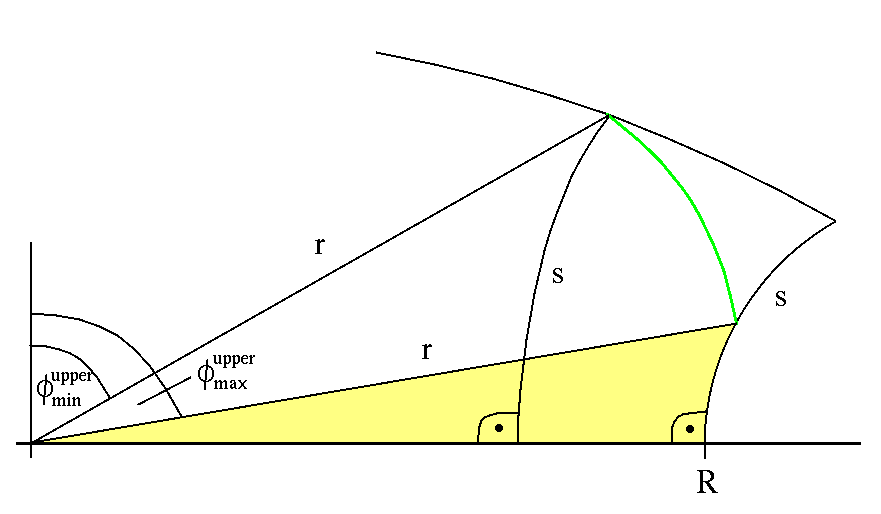}
\end{center}
\caption{Case $r>R$ (projected onto plane $\theta = 0 \lor \theta = \pi$)}
\label{fig:distribution:skizze_phi}
\end{figure}
$\phi_{min}^{upper}$ is equal to $\phi_{min}$ of the preceding case whereas $\phi_{max}^{lower}$ (not shown in diagram) equals $\phi_{max}$. The remaining two angles can be determined by regarding the yellow right triangle. The law of cosines yields
\begin{eqnarray*}
\cosh(r) &=& \cosh(s) \cosh(R) \\ 
\Leftrightarrow s &=& \acosh\left(\frac{\cosh(r)}{\cosh(R)} \right) =\acosh\left(\frac{\cosh(r)}{\cosh(R)} \right)
\end{eqnarray*}
Combining this result with the law of sines we obtain
\begin{eqnarray*}
\sinh(r) &=& \frac {\sinh(s)}{\cos(\phi_{max}^{upper})} \Leftrightarrow \cos(\phi_{max}^{upper}) = \frac {\sinh(s)}{\sinh(r)} \\
\Rightarrow \phi_{max}^{upper} &=& \arccos \left(\frac {\sinh(\acosh\left(\frac{\cosh(r)}{\cosh(R)} \right))}{\sinh(r)} \right)\\
&=& \arccos \left(\frac {\sqrt{ \cosh^2(\acosh\left(\frac{\cosh(r)}{\cosh(R)} \right)) -1 }}{\sinh(r)}\right)
= \arccos \left(\frac {\sqrt{ \frac{\cosh^2(r)}{\cosh^2(R)}-1 }} {\sinh(r)}\right)
\end{eqnarray*}
Analogously follows
\begin{equation*}
 \phi_{min}^{lower} = \pi - \arccos \left(\frac {\sqrt{ \frac{\cosh^2(r)}{\cosh^2(R)}-1 }} {\sinh(r)}\right)
\end{equation*}
Thus we obtain for the density:
\begin{eqnarray*}
\rho_3(r) &=& \int_{\phi^{upper}_{min}(r)}^{\phi^{upper}_{max}(r)} \int_0^{2\pi} dV + \int_{\phi^{lower}_{min}(r)}^{\phi^{lower}_{max}(r)} \int_0^{2\pi} dV \\
&=& 4 \pi \sinh(r) \left( \sinh(d) - \sqrt{ \frac{\cosh^2(r)}{\cosh^2(R)}-1 } \right)
\end{eqnarray*}
This function has a null at $R_{max}$ as can be checked easily. \\
\\
Now we can define the overall density function:
\begin{equation*}
 \rho(r) = \left\{ \begin{array}{ll}
                    \rho_1(r) \qquad & \mathrm{if\ } 0 \leq r \leq s \\
		    \rho_2(r)  & \mathrm{if\ } s < r \leq R \\	
		\rho_3(r)  & \mathrm{if\ } R \leq r \leq R_{max} \\
0 & \mathrm{otherwise}
                   \end{array} \right.
\end{equation*}
It is worth mentioning that by construction $\rho(r)$ is continuous and piece-wise even smooth. Since it has a compact support, the CDF exists. Unfortunately due to the part of the third case it is very hard to determine it analytically, let alone to invert it afterwards. Thus we cannot use the inversion method and we will have to use the rejection method instead. Thereby we first have to find a suitable dominating density function $g$. We will do this for each part individually.\\
\\
In the interval $[0,s]$ the density is convex. This can be seen by differentiating it two times with respect to $r$. Thus, a good dominating approximation is the linear approximation and, by using sufficient many sampling points, the approximation error is small. Furthermore can the corresponding CDF $F_1(r)$ easily be inverted as it is quadratic in each segment between the sampling points.\\
\\
For the second interval i.e. $r \in [s,R]$ the CDF can be obtained by using the original density function:
\begin{eqnarray*}
F_2(r) &=& F_1(s) + \int_s^r 4 \pi \sinh(s) \sinh(\tilde r)d\tilde r \\
&=& F_1(s) + 4 \pi \sinh(s) (\cosh(r) - \cosh(s))
\end{eqnarray*}
The inverse is then given by:
\begin{eqnarray*}
&\Leftrightarrow& \frac {F_2(r) - F_1(s)} {4 \pi \sinh(s)} =  \cosh(r) - \cosh(s)\\
&\Rightarrow& r =: F^{-1}_2(X) = \acosh\left(\frac {X - F_1(s)} {4 \pi \sinh(s)} + \cosh(s) \right)
\end{eqnarray*}

The density function in the third interval has a convex and a concave part. For some choices $R$ and $s$ it may not even be monotonic. We will nevertheless use the same approach as in the first interval, because even if the density may be therefore not dominating for all points the error is sufficiently small, if enough sampling points are chosen, since the original function is continuous. Using the rejection method, this will lead to a modelled distribution whose density is given by $\min(\rho(x),g(x))$. So, if the linear approximation is already good, the error in the modelled distribution is small as well.

Fig.\ref{fig:distribution:dominating_density} shows the graph of the density function $\rho(r)$ using $R=2$ and $s=1$ and a rough approximation using 5 sampling point for each linear approximation.
\begin{figure}[!ht]
\begin{center}
\includegraphics[width=0.8\linewidth]{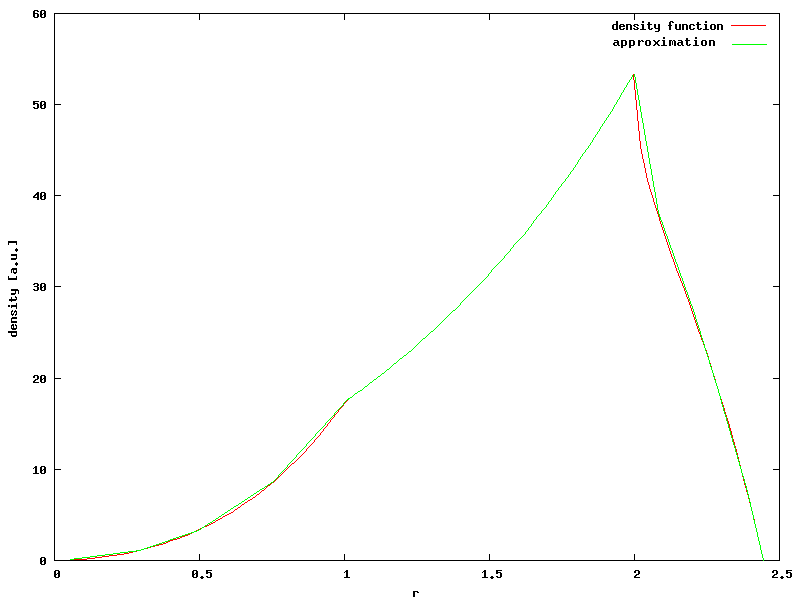}
\end{center}
\caption{example for the original and dominating density $g$}
\label{fig:distribution:dominating_density}
\end{figure}
To get the wanted distribution for the random variate $X_r$, belonging to $r$, we now compute the CDF of the (mostly) dominating density $g$. As already shown, or rather, mentioned above, this CDF can be inverted and thus the inversion method can be applied. The random variate thus obtained can then be used in the rejection method to finally get $X_r$. Due to the isotropy of the shells (restricted to the bounds of $\phi$) the random variates 
\begin{equation*}
X_\phi:\Sigma \rightarrow [\phi_{min}(r),\phi_{max}(r)] \mathrm{\ or\ } [\phi^{upper}_{min}(r),\phi^{upper}_{max}(r)]\cup [\phi^{lower}_{min}(r),\phi^{lower}_{max}(r)] 
\end{equation*}
and
\begin{equation*}
X_\theta:\Sigma \rightarrow [0,2\pi] 
\end{equation*}
belonging to the two angles are independent and uniformly distributed. To retrieve the according distribution in the Poincare Sphere model we just have to convert the realization of the radius according to the known transformation.

\subsection{*Hyperbolic-Euclidean case}

As noted in the definition, this distribution is just a product of the two-dimensional hyperbolic uniform distribution that is embedded in the Euclidean space and a one-dimensional Euclidean uniform one. For the first one, we use the same approach as in the Hyperbolic-Hyperbolic case above. We determine the CDF of the random variate $X_r$:

\begin{equation*}
F(r) = \int_0^r \int_0^{2\pi} \sinh(\tilde r) d\theta d\tilde r = 2 \pi \sinh(r)
\end{equation*}
The inverse is then given by:
\begin{eqnarray*}
F^{-1}(X) = \asinh \left(\frac {X} {2 \pi} \right)
\end{eqnarray*}
Thus we can use the inversion method. The random variate $X_\theta:\Omega\rightarrow [0,2\pi]$ is due to the isotropy again independent and uniformly distributed. \\
\\
For the latter distribution we get the uniform distributed random variate $X_z : \Omega \rightarrow [-s,s]$ where $2s$ is the distance between the boundaries in the extra dimension. The wanted three-dimensional random variate is then obtained by combining these three in cylindrical coordinates and then converting the result to the cartesian coordinate system.

\chapter{Implementation of a General Riemannian SOM} \label{ch:src}

Besides an analytic approach, we intend to perform the stability analysis of the classic SOM and the GRiSOM by using computer-based numerical means. We therefore need a computer program that simulates the GRiSOM and furthermore allows us to study its behavior in respect to the stability. In the course of the work on this thesis a whole software library has been written to meet this needs. Using C++ as the programming language, we took a special emphasis on the modularity, universality and (re)useablity of the implementation.\\
\\
Modularity has been achieved by using interfaces and abstract classes and avoiding cross-dependencies as far as possible. This has been visualized in Fig.\ref{fig:uml:dep} in the Appendix. It illustrates all dependencies between the header files of the project (and some extern ones). As it can be seen, the only cross-dependencies are between the headers defining templates and the corresponding files that implement their specialization. This is unfortunately unavoidable due to the improper implementation of the C++ language standard in most of the compilers.\\
The universality has been realized by using templates and specializations. This allows to design the software in a manner that is close to the mathematical structures. Thus e.g. points and spaces can be handled as abstract objects. Only if we want to use a particular representation of these mathematical objects, we have to specialize the templates to work with the corresponding implementations.\\
The third goal was finally approached, besides the already listed features of the other two points, by adding certain structures for convenience purposes. This includes e.g. wrapping structures of the SOM to allow to handle and even implement analyzers and classes for the automatization of the process of the numerical analysis easily. The software itself has been separated in three parts. Two of them are libraries, namely the \emph{Core} and the \emph{Auxiliary library}, that can be used by linking them statically or dynamically to other projects. The third part consists of the \emph{tester program} that can then be used to obtain the numerical results in the next chapter.\\
\\
We will now take a look at each of these parts, but before we continue, it should be noted that that this chapter does not intend to presents a full API of the software library but rather shall show how all the more interesting aspects discussed in the preceding chapters have been designed and implemented in the software that will be used in the upcoming numerical analysis. Thus the reader should be aware, that also inside the listings, sections as e.g. most of the comments may have been skipped and only but a few class diagrams are presented where they ease the understanding. However, a full documention of the package can be found in an additional document. The chapter \ref{ch:software} in the appendix provides more information about this.

\lstset{language=C++, 
	numbers=left,frame=single,
	stepnumber=2, breaklines=true,
	captionpos=b,
        basicstyle=\tiny,
        keywordstyle=\color{Purple},
        commentstyle=\color{BrickRed}, 
        stringstyle=\color{Maroon},
        showstringspaces=false,
        emph={export,class,bool,int,double,unsigned,char,true,false,void}, emphstyle=\color{Green}
}
\section{Core library [\texttt{libgsom\_core.\{a,so\}}]}

The core library bundles all the classes that are at the least needed to provide the most basic implementation of the generalized SOM.

\subsection{Required basic structures [\texttt{topology.\{h,cpp\}}]}

Back in chapter \ref{ch:gsom} we have presented the generalization of the SOM with more general types of map space and feature space. To take this into account we have define these spaces in our implementation:

\begin{lstlisting}[linerange={174-177,183-185,192-195,199-199,206-207,210-212,219-222,230-231,236-237,244-244,247-248}, caption=topology.h - definition of space]
template<class PointType>
class Space 
{
 public:
  virtual Space<PointType>* clone() const;
};

template<class PointType>
  class Metric_Space : public virtual Space<PointType> 
{ 
 public:
    Metric_Space(const typename metric<PointType>::f& dist);
    virtual Metric_Space<PointType>* clone() const;

    const typename metric<PointType>::f distance; 
};

template<class PointType>
  class Geometric_Space : public Metric_Space<PointType> 
{
 public:
    Geometric_Space(const Metric_Space<PointType>& metric_space, 
		    const typename geodesic_line_segment<PointType>::f& line);
    Geometric_Space(const typename metric<PointType>::f& dist, 
		    const typename geodesic_line_segment<PointType>::f& line);
    virtual Geometric_Space<PointType>* clone() const;
    const typename geodesic_line_segment<PointType>::f line;
};
\end{lstlisting}

According to the axiomatic construction of the spaces, we have not given any further definition of a point. Therefore also in the implementation of the SOM, a point can be anything by strictly using it as a template parameter. But since we want to do calculations in certain models, three possible types of point have been picked.

\begin{lstlisting}[linerange={22-24,30-30,37-39,44-46,50-50,56-56,61-61,65-66, 69-73,84-86,94-94,98-98,100-100,107-107,111-111,117-117,123-123,129-129,133-134,137-140}, caption=topology.h - definition of point types]
class Point 
{
 public:
  virtual Point* clone() const = 0; 
  virtual bool operator==(const Point&) const = 0;
};

class Simple_Point : public Point
{
  public:
  Simple_Point(int _id=0);
  virtual Simple_Point* clone() const;
  virtual bool operator==(const Point& p) const;
  int get_id() const;
  
protected:
  /*! \brief identifier of point */
  int id;
};

class Coordinates : public Point 
{
  public:
  Coordinates(uint dim=0, double* coords=0);
  Coordinates(const Coordinates& c); 
  virtual ~Coordinates();
  virtual bool operator==(const Point& p) const ;
  virtual void operator=(const Coordinates& c);
  virtual Coordinates* clone() const; 
  double operator[](const uint& index) const;
  void set(const uint& index, const double& new_value);
  uint get_dim() const;
  
 protected:
  uint dim;
  double* coords;
};
\end{lstlisting}

When implementing the particular spaces for the simulations, we will confine ourselves to using the (infinite-dimensional) \texttt{Coordinate} model as each of the defined models of the Euclidean, Hyperbolic and Spherical space make use of coordinate systems to describe the location of points. \texttt{Simple\_Point} is mostly used in combination with the trivial metric space to provide a simple map space for the vector quantization. Nonetheless, it would be possible, since we only need to distinguish between a finite number of points in the map spaces (namely the fixed positions of the neurons), to use \texttt{Simple\_Point} in a more complex map space by explicitly defining all the distances between the element of the finite set.

\subsection{Generalized SOM algorithm [\texttt{som.\{h,cpp\}}]}

The next step is to implement a formal neuron. As defined in section \ref{sec:som:mapspace}, its just a pair of two points where one of them lies in the map space, the other one in the feature space.

\begin{lstlisting}[linerange={16-19,24-25,30-30,33-33,39-39,44-45,47-47,49-51}, caption=som.h - definition of neuron]
template<class MapPointType, class FeatPointType>
class Neuron
{
 public:
  Neuron(const MapPointType& map_point,
	 const FeatPointType& init_feat_point);
  Neuron(const Neuron& n);
  ~Neuron();
  void operator=(const Neuron& n);
  void set_feat_point(const FeatPointType& new_feat_point);

  MapPointType map_point;
  FeatPointType feat_point; 
};
\end{lstlisting}

The last missing ``ingredient'' is the implementation of a neighborhood function. To keep it still most general it is simply defined in the SOM class as a function (object) getting a distance represented by a \texttt{<double>} and returning its result in form of another \texttt{<double>}.

\begin{lstlisting}[linerange={95-95}, caption=som.h - definition of neighborhood function]
  typedef function<double(const double&,const double&)> neighborhood_function;
\end{lstlisting}

The concrete implementations of the particular neighborhood functions are thereby not part of the core library as they are not necessarily needed here.\\
\\
This is already everything that we need to implement the SOM now.

\begin{lstlisting}[linerange={82-85,89-92,94-97,99-99,107-111,123-128,134-134,141-141,150-150,161-163,166-166,174-183,192-193}, caption=som.h - definition of SOM]
template<class MapPointType, class FeatPointType>
  class SOM 
{	
  public:	
  typedef typename metric<MapPointType>::f dist_map_type; 
  typedef typename metric<FeatPointType>::f dist_feat_type; 
  typedef typename geodesic_line_segment<FeatPointType>::f geodesic_feat_line_type;
  typedef Neuron<MapPointType,FeatPointType> neuron_type;
  typedef vector<neuron_type> som_net_type;
  typedef function<double(const double&,const double&)> neighborhood_function;
  typedef Metric_Space<MapPointType> map_space_type;
  typedef Geometric_Space<FeatPointType> feat_space_type;
  
  SOM(som_net_type &neurons,
      const map_space_type& map_space,
      const feat_space_type& feat_space,
      const neighborhood_function& h
      );
  SOM(som_net_type &neurons, 
      const dist_map_type& dist_map, 
      const dist_feat_type& dist_feat,
      const geodesic_feat_line_type& line_feat,
      const neighborhood_function& h
      );
  SOM(const SOM<MapPointType,FeatPointType>& som);
  void print_state(ostream& stream = std::cout) const;
  const som_net_type& get_state() const;
  virtual Neuron<MapPointType,FeatPointType> adapt(const FeatPointType& input, 
						   const double& epsilon,
						   const double& sigma);
  virtual ~SOM();
  /*! \brief working set of neurons */
  som_net_type _neurons;
  /*! \brief map space - defining distances between map points */
  const map_space_type* _map_space;
  /*! \brief feature space - defining distances and translations between map points */
  const feat_space_type* _feat_space;
  /*! \brief neighborhood function used in the adaption procedure */
  const neighborhood_function _h;

 private:
   neuron_type* determine_winner(const FeatPointType& input) ;
};
\end{lstlisting}

As we can see here, each instance of SOM has its own map and feature space, neighborhood function and set of neurons to work with. The adaption parameters $\eps$ and $\sigma$ will be then passed for each adaption step. At last we will have a closer look at the two important functions in the algorithm. The first one is to determine the neuron, in whose feature set or voronoi cell the given sample lies. Since we have emulated all the mathematical structures needed in the formal algorithm, the implementation is quite straightforward.

\begin{lstlisting}[linerange={113-118,123-128,130-130,133-136,174-189,192-194}, caption=som\_template.cpp - implementation of adaption and determination of winner neuron]
export template<class MapPointType, class FeatPointType>
Neuron<MapPointType,FeatPointType> SOM<MapPointType, FeatPointType>::adapt(const FeatPointType& input, const double& epsilon, const double& sigma)
{
  neuron_type* winner = determine_winner(input);
 
  FeatPointType winner_pos =  winner->feat_point;
  for (typename som_net_type::iterator it = _neurons.begin(); it!= _neurons.end(); it++)
    {
      // move feat_point towards input
      double coeff = epsilon * _h(_map_space->distance((*it).map_point, (*winner).map_point,0),sigma);
      if (coeff) // only if sth is to do
	{
	  it->set_feat_point(_feat_space->line(it->feat_point, input,coeff));
	}
    }
  return Neuron<MapPointType,FeatPointType>(winner->map_point,winner_pos);
}
export template<class MapPointType, class FeatPointType>
Neuron<MapPointType,FeatPointType>* SOM<MapPointType, FeatPointType>
::determine_winner(const FeatPointType& input)
{
  neuron_type* winner = 0;
  double winning_dist = 0;
	
  for (typename som_net_type::iterator it = _neurons.begin(); it!= _neurons.end(); it++)
    {
      double actual_dist = _feat_space->distance((*it).feat_point, input, winning_dist);
      if( !winner || actual_dist<winning_dist)
	{
	  winning_dist=actual_dist;
	  winner = &(*it);
	}	
    }
  return winner;
}
\end{lstlisting}

\section{Auxiliary Library [\texttt{libgsom\_aux.\{a,so\}}]}

While the Core library, as seen above, only provides the most basic implementation of the GRiSOM, the Auxiliary Library now provides the means needed to work seriously with it. This includes so-called \emph{Factory classes} to create certain spaces (Ch.\ref{sec:src:space_fac}), tessellations (Ch.\ref{sec:src:tess}), distributions (Ch.\ref{sec:src:dist}) and finally the particular SOMs in whole (Ch.\ref{sec:src:som_factory}). Furthermore tools to analyze the behavior of the SOM are made available (Ch.\ref{sec:src:eval}).

\subsection{*Space Factory [\texttt{space\_factory.\{h,cpp\}}]} \label{sec:src:space_fac}

Regarding its dependencies the Space Factory would have to be discussed further down as it specializes structures that are abstractly defined in sections still following below. That means that, following the design decisions mentioned in the introduction of this chapter, e.g. some types of spaces are not defined until in the evaluation section as they will be mainly used there, but will be inherited by the concrete space models that have been implemented in the Space factory. We will nonetheless start with the look on the Space Factory as it reflects better the chronologic order of the approach of the implemention of the GRiSOM (since the implemented models are used to construct the particular SOMs that we want to simulate even without any evaluation of the results).\\
\\
In Chap.\ref{ch:geometry} we discussed in detail the geometrical properties of the particular spaces and their models which we will use in the stability analysis. We now have to transfer them into the implementation. Beside the necessary structures of the metric and geometric space, our implemented spaces inherit three additional spaces namely \texttt{Vector\_Space}, \texttt{Pre\_Hilbert\_Space} and \texttt{Projectable\_Space} which on their part have structures that will be needed for certain Analyzers (cf. Ch.\ref{sec:src:eval}). They provide e.g. a sense of addition, scalar multiplication, (canonical) inner products or projections in some of our models. A further discussion about them will be delayed until we will have to use them in the section about the evaluations. The complete definition of our models are now:

\begin{lstlisting}[linerange={342-348,355-364,649-654,661-669,732-733,735-736,744-748,750-751}, caption=space\_factory.h - definition of several spaces and models]
template<class PointType>
class Euclidean_Cart_Coord_Space : public Geometric_Space<PointType>, 
                                   public Pre_Hilbert_Space<PointType>, 
                                   public Projectable_Space<PointType>
{
 public:
  typedef Isometry<PointType,Euclidean_Cart_Coord_Space<PointType> > Isometry;
  virtual Euclidean_Cart_Coord_Space<PointType>* clone() const;
  
 private:
  friend class Space_Factory<PointType, Euclidean_Cart_Coord_Space<PointType> >;
  Euclidean_Cart_Coord_Space(const Geometric_Space<PointType>& gs,
			     const Pre_Hilbert_Space<PointType>& phs,
			     const Projectable_Space<PointType>& ps
			     );
};

template<class PointType>
class Hyperbolic_Disk_Space : public Geometric_Space<PointType>, 
                              public Projectable_Space<PointType>
{
 public: 
  typedef Isometry<PointType, Hyperbolic_Disk_Space<PointType> > Isometry;
  virtual Hyperbolic_Disk_Space<PointType>* clone() const;
  
 private:
  friend class Space_Factory<PointType,Hyperbolic_Disk_Space<PointType> >;
  Hyperbolic_Disk_Space(const Geometric_Space<PointType>& gs,
			const Projectable_Space<PointType>& ps
			);
};

template<class PointType>
class Spherical_Geo_Space : public Geometric_Space<PointType> 
{
 public:
  virtual Spherical_Geo_Space<PointType>* clone() const;
  
 private:
  friend class Space_Factory<PointType,Spherical_Geo_Space<PointType> >;
  Spherical_Geo_Space(const Geometric_Space<PointType>& gs
		      );
};
\end{lstlisting}

To avoid confusions due to the multiple inheritance, the class diagram in Fig.\ref{fig:uml:space} illustrates the relations between these special spaces and the overall structure. All methods, that can used to get instances of these special spaces, are now bundled in the template \texttt{Space\_Factory} using a specialization of it for each model. The constructors of the space classes itself are defined as private to assure the consistency by restricting the possibility to create instances to the particular specialized factory class. Henceforth we will focus the further discussion to two models, namely the Poincare Disk model embedded in the complex plane and the Poincare (hyper)sphere model embedded in the $\setR^d$. All other spaces and models have been analogously implemented.\\
\\
Since the \texttt{Poincare disk space} is a subclass of both the geometric and the projectable space, it has to provide all their functions i.e. the ``metric'' and the ``geodesic line segment'' of the geometric space and the ``canonical projection'' of the projectable space. The 'projections' thereby map the points in the space onto a pre-defined subset. In our case it is just a less dimensional subspace and the projected point is the point in the subspace that is closest to the one that we project. In our implementation we then obtain the result of the projection by simply reflecting the point in this subspace and then computing the point on the geodesic in the middle between these two points.\\
\\
We get now the following specialization of the space factory template:

\begin{lstlisting}[linerange={521-524,529-529,534-534,539-539,543-549}, caption=space\_factory.h - specialization of space factory]
template<>
class Space_Factory<Coordinates, Hyperbolic_Disk_Space<Coordinates> >
{
 public:
  static metric<Coordinates>::f create_metric();
  static geodesic_line_segment<Coordinates>::f create_geodesic_line_segment();
  static function<Coordinates (const Coordinates&,const uint&)> create_canonical_projection();
  static Hyperbolic_Disk_Space<Coordinates> create_space();

 private:
  static double metrics(const Coordinates& c1, const Coordinates& c2, const double& cutoff=0);
  static Coordinates geodesic_line_segments(const Coordinates& c1,const Coordinates& c2,const double& param);
  static Coordinates canonical_projection(const Coordinates& c, const uint& dim);
};
\end{lstlisting}

The private methods define the needed functions or function objects while the corresponding public creation method bind the passed parameters to these functions before returning them. This allows to have e.g. a boundary condition as an inherent attribute of the metric or geodesic. The implementation of the distance function thereby simply uses Eq.\ref{eqn:geom:distance_PD}:

\begin{lstlisting}[linerange={1118-1134,1136-1137}, caption=space\_factory.cpp - implementation of creating distance function]
double Space_Factory<Coordinates,Hyperbolic_Disk_Space<Coordinates> >::metrics(const Coordinates& c1, 
									       const Coordinates& c2,
									       const double& cutoff)
{  
  double norm_1sq =0;
  double norm_2sq =0;
  double norm_diffsq =0;
  uint dim = max(c1.get_dim(),c2.get_dim());
  
  for (uint j=0; j<dim; j++) 
    {
      double diff=c2[j]-c1[j];

      norm_1sq += pow(c1[j],2);
      norm_2sq += pow(c2[j],2);
      norm_diffsq += pow(diff,2);
    }
  return acosh(1 + 2*norm_diffsq/((1-norm_1sq)*(1-norm_2sq)));
}
\end{lstlisting}

For the implementation of the geodesics (and therefore also for the projection) we will have to discuss another structure, we make use of, namely the isometries.

\subsubsection{*Isometries [\texttt{isometry.h}]}

The definition and implementation of isometric transformations in the library serve several purposes. First of all, as concluded in chapter \ref{ch:tess:generation}, we can use a subgroup of them to generate the regular tessellations of the space. Furthermore, referring to section \ref{sec:geom:geodesics}, they provide a more comfortable way to compute certain geometric structures like the geodesics. In Listing \ref{lst:src:isom_def} we see the generic definition of this concept in the source code.

\begin{lstlisting}[linerange={16-19,25-25,30-30,37-37,44-44,48-48,54-55,59-60}, label={lst:src:isom_def}, caption=isometry.h - generic definition of isometries]
template<class PointType, class SpaceType>
class Isometry 
{
 public:
  virtual PointType apply(const PointType& p, double* error=0) const;
  virtual Isometry<PointType, SpaceType> operator*(const Isometry<PointType, SpaceType>& T)const;
  virtual bool operator==(const Isometry<PointType,SpaceType>& T) const;
  virtual bool apply_equal(const Isometry<PointType,SpaceType>& T, const PointType& p) const;
  virtual Isometry<PointType,SpaceType> inverse() const;
  virtual double get_error(const PointType& p) const;

  static Isometry<PointType,SpaceType> Identity();
};
\end{lstlisting}

Besides the group operation and a way to compare, invert and apply the isometries, we added two more functions. The first one, \texttt{apply\_equal}(\dots), provides a way to determine, if the isometry object and another given one map a given point to the same new point. \texttt{get\_error()} however allows to return the intern numerical error of the isometry. This error occurs in any numerical computation and determines the precision of the results. Thus, when we get two points which are results of transformations, they may represent the same analytical result if they differ by at least by this error margin. This has been taken into consideration when generating the tessellations. Nonetheless does this generic definition and its implementation provide the structure of the trivial isometry group i.e. the one only containing the Identity.\\ \\
As an example for a specialization of the isometries, we will once again take the Hyperbolic Poincare (hyper-)sphere space with complex numbers as representations of points. We have already done all the needed preparatory work to determine the formulas for the operations, as for inverting or combining, in the process of proving Thm.\ref{thm:geom:hyp_isom}. So again the implementation is straightforward.

\begin{lstlisting}[linerange={1459-1460,1464-1469,1473-1476,1478-1480,1492-1497}, caption=space\_factory.cpp - specializations of isometry (implementation)]
Complex Isometry<Complex, Hyperbolic_Disk_Space<Complex> >::apply(const Complex& z, double* error) const
{
  Prec_Complex result((_p + (_r*(_inv?conj(z):z)))/((conj(_p)*_r*(_inv?conj(z):z)) + 1)); // (r*z + p)/(conj(p)*r*z +1) resp. (r*conj(z) + p)/(conj(p)*r*conj(z) +1)
  return Complex(real(result),imag(result));
}

Isometry<Complex,Hyperbolic_Disk_Space<Complex> > Isometry<Complex, Hyperbolic_Disk_Space<Complex> >::operator*(const Isometry<Complex, Hyperbolic_Disk_Space<Complex> >& T) const
{
  Prec_Complex denom = _r*(_inv?conj(T._p):T._p)*conj(_p) +  1;
  Prec_Complex R = (_r*(_inv?conj(T._r):T._r)+_p*(_inv?(conj(T._r)*T._p):(T._r*conj(T._p))))/denom;
  Prec_Complex P = (_r*(_inv?conj(T._p):T._p)+_p)/denom;
  R = normalize(R);
  return Isometry<Complex,Hyperbolic_Disk_Space<Complex> >(R,P,!(_inv ==T._inv) ) ;
}

Isometry<Complex,Hyperbolic_Disk_Space<Complex> > Isometry<Complex, Hyperbolic_Disk_Space<Complex> >::inverse()
{
  return (Isometry<Complex,Hyperbolic_Disk_Space<Complex> >(Complex(1),0,_inv) 
	  * (Isometry<Complex,Hyperbolic_Disk_Space<Complex> >(conj(_r)) 
	     * Isometry<Complex,Hyperbolic_Disk_Space<Complex> >(1,_p*-1)));
}
\end{lstlisting}

All the specializations for the other spaces and models can be again obtained in an analogous manner. For this case we can now finally transfer our results of section \ref{sec:geom:geodesics} to implement the function of the geodesics.

\begin{lstlisting}[linerange={1144-1209,1224-1225}, caption=space\_factory.cpp - geodesic function of \texttt{Hyperbolic\_Disk\_Space<Coordinates>})]
Coordinates Space_Factory<Coordinates,Hyperbolic_Disk_Space<Coordinates> >
::geodesic_line_segments(const Coordinates& c1, const Coordinates& c2, 
			 const double& param)
{     		
  if(param==0) return Coordinates(c1);
  if(param==1) return Coordinates(c2);

  uint dim = max(c1.get_dim(),c2.get_dim());

  // using variation of Alguiera-Perez Algorithm i.e. rotating into a 2 dim disk and then working there

  Isometry<Coordinates, Hyperbolic_Disk_Space<Coordinates> > T = Isometry<Coordinates, Hyperbolic_Disk_Space<Coordinates> >::Identity();

// rotate c1 into X_n axis

  Coordinates p1 = c1;

  for (uint i=0; i<dim-1;i++)
    {
      Isometry<Coordinates, Hyperbolic_Disk_Space<Coordinates> > M 
	= Isometry<Coordinates,  Hyperbolic_Disk_Space<Coordinates> >::Pure_Main_Rotation(atan2(p1[i],p1[i+1]), i, i+1); 
      p1 = M.apply(p1);
      T = M*T;
    }

  // rotate c2 into X_n-1 - X_n - Plane

  Coordinates p2 = c2;
  p2 = T.apply(p2);

  for (uint i=0; i<dim-2;i++)
    {
      Isometry<Coordinates, Hyperbolic_Disk_Space<Coordinates> > M 
	= Isometry<Coordinates,  Hyperbolic_Disk_Space<Coordinates> >::Pure_Main_Rotation(atan2(p2[i],p2[i+1]), i, i+1);
      p2 = M.apply(p2);
      T = M*T;
    }

  // now move point using Translation in 2D Poincare Disk

  double dist = metrics(p1,p2);
  
  // need implementation of translations in isometry group. will use workaround instead by doing it manually
  //   Isometry<Coordinates, Hyperbolic_Disk_Space<Coordinates> > M = Isometry<Coordinates, Hyperbolic_Disk_Space<Coordinates> >::Pure_Translation(p1);
  //    T= M*T;

  Complex pp1(p1[dim-2],p1[dim-1]);
  Complex pp2(p2[dim-2],p2[dim-1]);
  pp2 = (pp2 - pp1)/((conj(pp1*(-1.))*pp2)+1.);
  
  double dist_PD = abs(pp2);

  // and now let's rescale

  double new_dist_PD = sqrt((cosh(dist*param) - 1)/(cosh(dist*param) +1 ));
  double ratio = new_dist_PD/dist_PD;

  pp2 *= ratio;

  // now rewind all trafos

  pp2 = (pp2+pp1)/(conj(pp1)*pp2+1.);
  p2.set(dim-2,real(pp2));
  p2.set(dim-1,imag(pp2));

  Coordinates result = T.inverse().apply(p2); 
  return result;
}
\end{lstlisting}

\subsubsection{*Precision algebra/complex [\texttt{prec\_algebra/complex.\{h,cpp\}}]}
In the shown implementations of the chosen specializations of the isometries, several structures can been seen that we have not introduced yet. They present two different ways to approach the problems with numerical errors. \\ \\
The bundle consisting of \texttt{P\_Float},\texttt{P\_Coordinates} and \texttt{P\_Matrix} serves two purposes. The first one is to provide matrix and vector calculations. The second is to estimate the numerical error of its instances due to the finite precision of the nested representation of floats and the error propagation when using these error-prone representations in calculations. They do not improve the precision, but they provide at least a way to get the resulting error margins to take them into consideration when using the representations.\\ \\
The second approach is actually to substitute the floating point numbers with data type with much higher precision. This is of great importance when we are going to generate tessellations of the Poincare model as we will have to decide there, if two representations denote the same point. Near the origin we could still use the approach of estimating the error, but for nodes,  generated in a larger distance to the origin, the Euclidean distance between two neighboring nodes tends to zero while the error increases since more and more Isometries are combined to generate these points. We therefore quickly reach the limit where the error exceeds the quarter of the distance and therefore makes it impossible to decide if a new node already exists or not. To tackle this problem, we made use of two external C-libraries (\cite{mapm} and/or \cite{mapmx}) that already implemented high-precision numbers, so we have had to write a wrapper class for their functions to use them in our project.

\subsection{*Generating tessellations [\texttt{tesselation.\{h,cpp\}}]} \label{sec:src:tess}

When discussing tessellations in chapter \ref{ch:tess}, we have seen, that any kind of regular tiling can be generated easily by using one of its symmetry groups. This can be transferred one-to-one to an implementation. We will first create the symmetry group $G=\set{g_1,g_2,\dots}$ as proposed which is the group of involutions in the side centers. The following listing shows as an example the function to create the symmetry group, or rather, the generators for the Euclidean Cartesian Coordinates model.

\begin{lstlisting}[linerange={890-917}, caption=space\_factory.cpp - creating symmetry group]
const Tesselation<Complex,Euclidean_Cart_Coord_Space<Complex> >::Discrete_Symmetry_Group 
Tesselation<Complex,Euclidean_Cart_Coord_Space<Complex> >
::create_Tess_Group(const double& NN_dist, const uint schlaefli_p, const uint schlaefli_q)
{
  assert(NN_dist >0);
 
 Discrete_Symmetry_Group result;

 assert( //1./schlaefli_p + 1./schlaefli_q == 0.5
	   (schlaefli_p == 4 && schlaefli_q == 4) // cubic
	   || (schlaefli_p == 3 && schlaefli_q == 6) // triangular
	   || (schlaefli_p == 6 && schlaefli_q == 3) // hexagonal
	 ); // otherwise invalid
 Isometry<Complex, Euclidean_Cart_Coord_Space<Complex> > trans 
   =  Isometry<Complex, Euclidean_Cart_Coord_Space<Complex> >::Pure_Translation(NN_dist/2);
 Isometry<Complex, Euclidean_Cart_Coord_Space<Complex> > invol 
   = Isometry<Complex, Euclidean_Cart_Coord_Space<Complex> >::Pure_Rotation(M_PI);
 
 for (uint i=0; i< schlaefli_q; i++)
   {
     double angle = (double(i)/schlaefli_q)*2*M_PI;
     Isometry<Complex, Euclidean_Cart_Coord_Space<Complex> > rot 
       = Isometry<Complex, Euclidean_Cart_Coord_Space<Complex> >::Pure_Rotation(angle);
     result.push_back(rot.inverse()*trans.inverse()*invol*trans*rot);
   }
 
 return result;
}
\end{lstlisting}

According to ch.\ref{ch:tess} we can expand each node using the respective symmetry group. The following listing of the implementation shows the generation of the first \emph{ring} or \emph{layer} of nodes which is therefore the expansion of the origin (the expression ``ring''  originates from using triangular tailings, where each layer is ringlike connected. This holds not for arbitrary (regular) tesselations.)

\begin{lstlisting}[linerange={16-36}, caption=tesselation\_template.cpp - creating first layer]
export template<class PointType, class SpaceType>
vector<PointType> Tesselation_Factory<PointType,SpaceType>
::create_regular_Tesselation(const Discrete_Symmetry_Group& sym_grp, 
			     const function< bool(const PointType&, const double&) >& acceptance,	       
			     const uint& _max_level, const uint& _min_level,
			     const PointType& origin)
{
  vector<PointType> generated_points;
  vector<Node> nodes_last_level;
  vector<Node> nodes_current_level;
  vector<Node> nodes_next_level;

  // 1st level will generated by Symmetry_group starting at the origin
  for(typename Discrete_Symmetry_Group::const_iterator iter= sym_grp.begin();  iter != sym_grp.end(); iter++)
    {
      Symmetry_Trafo current_trafo =  *iter;
      double current_error;
      PointType current_point = current_trafo.apply(origin, &current_error);
      if(acceptance(current_point, current_error))
	nodes_next_level.push_back(Node(current_trafo, current_point, current_error));
    }
\end{lstlisting}

Each layer can now be obtained by expanding each node of the preceding one.

\begin{lstlisting}[linerange={40-70, 121-121,124-131}, caption=tesselation\_template.cpp - creating layers recursively]
  uint level = 2;
  uint max_level = _max_level;
  // if max_level is zero then use max level of implementation
  if(max_level==0) max_level = MAX_LEVEL;

  while (nodes_next_level.size() && level <= max_level)
    {
      //cout << "Tesselation: generating level " << level << " ...";

      // "old" last level not used any more. points are transferred to generated points if level before last level >= _min_level
      if (level >= _min_level+3)
	for ( typename vector<Node>::iterator it_nodes = nodes_last_level.begin(); 
	      it_nodes < nodes_last_level.end(); it_nodes++)
	  generated_points.push_back(it_nodes->point);

      nodes_last_level = nodes_current_level;
      nodes_current_level = nodes_next_level;
      nodes_next_level.clear();

      for (typename vector<Node>::iterator it_nodes = nodes_current_level.begin(); it_nodes < nodes_current_level.end(); it_nodes++)
	{
	  for (typename Discrete_Symmetry_Group::const_iterator it_sym = sym_grp.begin(); it_sym != sym_grp.end(); it_sym++)
	    {
	    Symmetry_Trafo current_trafo =  it_nodes->trafo * (*it_sym) ;
	    double current_error;
	    PointType current_point = current_trafo.apply(origin, &current_error);
	  
	    // check if accepted, i.e. e.g. uncertainty of generated point due to rounding error is small enough 
	    // compared e.g. to the expected distance between to nodes in this neighborhood
	    // or if generated point is still in the volume that shall be tesselated.
	    if(acceptance(current_point, current_error))
		    nodes_next_level.push_back(Node(current_trafo, current_point, current_error));
	      }
	    }
	  // cout << "Node totally expanded"<<endl;
	}

      level++;
      // cout << nodes_next_level.size() << " Points generated" << endl;
    }
\end{lstlisting}

As, in general, there exists many pathes in the tiling connecting the origin and $\nu$ by following the edges and each of this path represents one of the possible sequences such that \ref{eq:tess:gen_nu} holds, it is important to consider that the sequence is not unique for each point. That does not mean, by the way, that the compositions corresponding to these sequences are equal. It only implies that they act equally on the origin by mapping it to $\nu$. We have to consider this in order to avoid creating a vertex twice or even more times. Since each layer is fully expanded before we start to expand the first vertex of the following one, there are only three possibilities for these double creations to take place (assumed there are still none in the already expanded layers). Regarding an expansion of a vertex of the $n$-th layer:

\begin{itemize}
\item newly created vertex lies in the $(n-1)$-th layer
\item newly created vertex lies  in the $n$-th layer
\item newly created vertex lies  in the $(n+1)$-th layer
\end{itemize}

It cannot occur in lower layers since these have been already expanded and, thus, also their children vertices which lie at least a layer below the $n$-th one. Therefore the vertex that we are actually expanding would have to be a twin of a node in a layer below which contradicts our assumption that this has not occurred yet.
And it cannot take place in higher layers since this would obviously be a contradiction to the definition of the layers. Thus it is sufficient to check for each new vertex if it already belongs to one of the three layers. In the implementation, this is performed by the following code:

\begin{lstlisting}[linerange={16-36}, caption=tesselation\_template.cpp - implementation of the tessellation algorithm]
export template<class PointType, class SpaceType>
vector<PointType> Tesselation_Factory<PointType,SpaceType>
::create_regular_Tesselation(const Discrete_Symmetry_Group& sym_grp, 
			     const function< bool(const PointType&, const double&) >& acceptance,	       
			     const uint& _max_level, const uint& _min_level,
			     const PointType& origin)
{
  vector<PointType> generated_points;
  vector<Node> nodes_last_level;
  vector<Node> nodes_current_level;
  vector<Node> nodes_next_level;

  // 1st level will generated by Symmetry_group starting at the origin
  for(typename Discrete_Symmetry_Group::const_iterator iter= sym_grp.begin();  iter != sym_grp.end(); iter++)
    {
      Symmetry_Trafo current_trafo =  *iter;
      double current_error;
      PointType current_point = current_trafo.apply(origin, &current_error);
      if(acceptance(current_point, current_error))
	nodes_next_level.push_back(Node(current_trafo, current_point, current_error));
    }
\end{lstlisting} 

The truncation has been realized by, on the one hand, allowing to specify the maximal level of layers, and on the other hand, including an acceptance function, that allows to explicitly decide which node should be accepted in the map and which should be rejected due to e.g. a too high errors in the generation process or, in the case of the truncation in respect to given boundaries (as used to create e.g. the rectangular boundaries for the Euclidean maps), due to its position. Thus, in the case that we want to truncate in respect to the expansion level we simply pass this level when calling the generator and when we want to truncate in respect to boundaries, we pass a function object, that returns \texttt{false} for each passed node that lies outside these wanted boundaries and \texttt{true} otherwise (consideration of other reasons of rejection neglected). The maximal number of layers for the generation process has thereby to be set to infinity (The termination of the algorithm is then only ensured by the acceptance function).

\subsection{Generating sample sets [\texttt{stochastics.\{h,cpp\}}]} \label{sec:src:dist}

For the SOM configuration that we want to analyze numerically we had to implement all the distribution that we discussed in detail in Ch.\ref{ch:distr}. Besides the interface class \texttt{Distribution}, which works as a base class for all the distributions, the algorithm presented in Ch.\ref{sec:distr:generate} could be transfered one-to-one to the implementation. The \texttt{Uniform\_Hyperbolic\_Disk\_Distribution} and \texttt{Hyperbolic\_Disk\_Euclidean\_Cart\_} \texttt{Coord\_Distribution} thereby use the \texttt{Uniform\_Hyperbolic\_Spherical\_Distribution} and \texttt{Uniform} \texttt{\_Euclidean\_Cart\_Coord\_Distribution} as generators the corresponding random sample set. The class diagram in Fig/\ref{fig:uml:dist} in the appendix shows these aggregations.

\subsection{SOM Factory [\texttt{som\_factory.*}]} \label{sec:src:som_factory}

The SOM factory now uses all the structures that had been defined in the preceding sections to finally ensemble the SOMs for the numerical simulations. Similar to many generic classes above \texttt{SOM\_Factory} is thereby specialized for any needed combination of map and feature space. These specializations on their part provide functions to create all the variations of SOMs with the given spaces that we want to analyze. Since a listing of this implementation is not very instructive the interested reader searching for the concrete implementation may be therefore referred to the API.

\subsection{Evaluation methods and tools} \label{sec:src:eval}

The evaluation process, that has been used to analyze the stability features of the different som configurations, has been realized by dividing the problem into to subproblems. On the one hand we needed an automatization of the ``training'' of the SOM for a given set of parameters and sample distributions and, on the other side, we needed methods to analyze the SOM using various measurements. We will begin with having a look at the latter requirement.

\subsubsection{Analyze methods \& measurement [\texttt{evaluation.\{h,cpp\}}]}

As mentioned above, we need tools to analyze the SOM e.g. compute e.g. the Fourier transformation of its feature space or the mean value of each neuron taking multiple results of adaption steps into account. To meet this need, an abstract \texttt{Analyzer} class has been defined (cf. listing \ref{lst:src:analyzer}) and then a Fourier and a mean analyzer have been implemented.


\begin{lstlisting}[linerange={150-160,163-170}, label=lst:src:analyzer, caption=evaluation.h - Analyzer]
template<class MapPointType, class FeatPointType>
class Analyzer 
{
 public:

  virtual void add_sample(const FeatPointType& last_sample, const double* last_adapt_params,
			  const FeatPointType& last_winning_pos,
			  const vector<Neuron<MapPointType,FeatPointType> >& result_som_state)=0;

  virtual void reset()=0;

  virtual void registers(const Analyzable_SOM<MapPointType,FeatPointType>* parent,
			 const vector<Neuron<MapPointType,FeatPointType> >& initial_som_state 
			 // = vector<Neuron<MapPointType,FeatPointType> >() 
			 ) throw(Insufficient_Space_Exception,Max_Reg_Exception) =0; 
  /*! Function can be used e.g. to inform registrators if Analyzer is destroyed */
  virtual void unregisters(const Analyzable_SOM<MapPointType,FeatPointType>* parent)=0;
};
\end{lstlisting} 

Since these tools should work on our generalized maps, they can only use the properties and structures that are inherent to the spaces of the particular SOM. For example does the mean analyzer need to work on at least a geometric feature space to calculate the mean point of the sequence of results. Since this has been already a prerequisite of the SOM itself, it is always fulfilled. By contrast, this is not enough for the Fourier analyzer which has to work on a \setR-Vectorspace. It is obvious that there are geometric feature spaces that have not this needed structure (e.g. the models used for the hyperbolic space) and a Fourier analyzer would not be able to work on them. So, it has to be possible for the analyzers to get access to the spaces of the SOM and in return the SOM should be able to reject analyzers that would be able to work on it. These problems have been solved by extending the \texttt{SOM} class into the \texttt{Analyzable SOM} class that is shown below:

\begin{lstlisting}[linerange={36-87,89-89}, caption=evaluation.h - Analyzable SOM]
template<class MapPointType, class FeatPointType>
class Analyzable_SOM : public SOM<MapPointType, FeatPointType> {
  
 public:
	
	typedef Neuron<MapPointType,FeatPointType> neuron_type;
	// no pointers in vector to prevent illegal duplicates of neurons
	typedef vector<neuron_type> som_net_type;
	typedef function<double(const double&,const double&)> neighborhood_function;
	typedef Metric_Space<MapPointType> map_space_type;
	typedef Geometric_Space<FeatPointType> feat_space_type;

  Analyzable_SOM(som_net_type &neurons, 
		 const map_space_type& map_space, 
		 const feat_space_type& feat_space,
		 const neighborhood_function& h
		 );
  
  Analyzable_SOM(som_net_type &neurons, 
		 const typename metric<MapPointType>::f& dist_map, 
		 const typename metric<FeatPointType>::f& dist_feat,
		 const typename geodesic_line_segment<FeatPointType>::f& get_point_at_geodesic_feat_line,
		 const neighborhood_function& h
		 );

  virtual ~Analyzable_SOM();
  
  virtual Neuron<MapPointType,FeatPointType> adapt(const FeatPointType& input, 
						   const double& epsilon,
						   const double& sigma);
  
  void add_Analyzer(Analyzer<MapPointType,FeatPointType>* device,  
		    const uint& measure_interval = 1) 
    throw (Insufficient_Space_Exception,Max_Reg_Exception);
  void remove_Analyzer(Analyzer<MapPointType,FeatPointType>* device);
  void remove_All_Analyzers();

  const map_space_type& get_Map_Space() const;
  const feat_space_type& get_Feat_Space() const;
  
private:

  struct analyzer_registration {
    Analyzer<MapPointType,FeatPointType>* analyzer;
    unsigned long registered_since;
    int measure_interval;
  };
  
  typedef vector<analyzer_registration> reg_list;
  reg_list registered_Analyzers;
    
  unsigned long sample_counter; 
};
\end{lstlisting} 

It offers the possibility to add analyzers directly to the SOM without having to regard any of the prerequisites that have been discussed above. By adding an analyzer, the SOM itself will then provide both the map and the feature space for it. It will furthermore pass a copy of its internal state to allow the analyzer to use it as a reference or a first sample. In return the analyzer tests if the map and feature spaces meet its needs and accepts or rejects them, respectively. The rejection will be signaled by throwing an \texttt{Insufficient\_Space\_Exception}. If successfully added, the SOM will then automatically call the \texttt{add\_sample} method of the analyzer after a predefined amount of adaption steps (the size of this interval has been passed to the SOM when the analyzer has been added). This will be repeated until the analyzer has been removed from the SOM. In general, a SOM can handle many analyzers at the same time. It would also be possible to design an analyzer that can examine several SOM simultaneously. But in the case that the SOM or the analyzer have reached their maximal limit, this will be signaled by the \texttt{Max\_Reg\_Exception}.\\
\\
The so defined combination of extended SOM and analyzers thus allows to examine specific attributes. But, for our goal to find the stability limits in the parameter space, we have to define a kind of error function i.e. a mapping of a state or attribute of the SOM belonging to a given parameter set on a positive real number. This abstract class is defined as following:

\begin{lstlisting}[linerange={335-340}, caption=evaluation.h - Error Analyzer]
template<class MapPointType, class FeatPointType>
  class Error_Analyzer: public virtual Analyzer<MapPointType,FeatPointType> 
{
 public:
    virtual double get_error()= 0;
};
\end{lstlisting} 

The naive approach would be to measure the distance between the positions of the neurons in a reference state and in the actual state and calculate the error depending on this results. A problem thereby is, that the feature spaces may have many symmetries. In the Euclidean and hyperbolic models, these are certain translation and rotation invariances. Without breaking these symmetries (e.g. by using a (finite) bounded subspace or fixing a few neurons), the representation of map can rotate and in the case of a virtual unbounded space by using periodic boundary conditions even translate freely. That means, that there exists a whole continuum of equivalent som states. The calculated error should therefore depend only on the equivalence classes. That would be obviously not fulfilled by the naive approach above.
An implementation of this \texttt{Error\_Analyzer} is the \texttt{Mean\_Extra\_Dim\_Analyzer}. It extends the mean analyzer. It calculates the error by taking the mean som state and computing in the feature space the distance of each neuron to a corresponding neuron in a given reference som by regarding only a particular dimension. This is done by projecting both the mean neuron and the reference neuron onto a given plane and calculating the distance of the neurons and their projection. The error is then the sum or difference of these distances depending if both are on the same or on opposite sites of the projection plane. Thus this analyzer needs a feature space which has a projection mapping defined on it. We already briefly discussed in Ch.\ref{sec:src:space_fac} how our models of interest realize this projections and thus meet the requirement of this analyzer.\\
\\
Although we won't discuss this in detail, a few more analyzers and error analyzers had been implemented like the already mentioned Fourier analyzer and an additional error analyzer that computes the mean representation error of the neurons in respect to the presented sample set. In general, any (error) analyzer, which depends only on the available data that is passed by the Analyzable SOM (i.e. the current som state, the currently presented sample and the position of the representative of the winner neuron in the feature space), can be implemented.

\subsubsection{*Automatization [\texttt{test\_suite.\{h,cpp\}}]} \label{sec:src:auto}

At the beginning, we tried to implement an automatization that should place us in a position to run the stability analysis totally unsupervised, but the first attempts failed as sometimes the increasing fluctuations near the stability limit were by mistake identified as the limit itself and therefore the further sweeps of the search were performed at the wrong positions. We therefore decided to implement at least a set of functions that ease a supervised analysis. The class \texttt{Test} provides two simple functions that take a SOM, a generator for samples, a function that maps the current index of the adaption step to a parameter set for the adaption allowing to vary $\eps$ and $\sigma$ throughout the simulations. Then they run the adaption procedure on the given SOM for a given number of adaption steps.\\
\\
Since we want to analyze the behavior of a given SOM configuration in a particular one-dimensional parameter space, namely the size $s$ of the extra dimensions, we furthermore implemented the class \texttt{Stability\_Test} that inherits \texttt{Test}. It provides a function that automatizes the search at least for one sweep as it takes a discrete set of values of $s$, a set of the corresponding generators for the samples and additionally a list of error analyzers. The adaption process will then be performed for each $s$ and the error values, computed by the single analyzers, and the initial som state are stored in separate output files. If furthermore an error analyzer is also a mean analyzer (as it is the case in our simulations as we use the \texttt{Mean\_Extra\_Dim\_Analyzer}) the mean states of each search steps are also saved to allow further studies and visualizations of the results later on.

\section{Testers [\texttt{tester\_*.cc}]} \label{sec:src:tester}

Besides the two libraries the written software package comes with several tester programs. Most of them just have the purpose to test certain parts of the library. For example generates and visualize \texttt{tester\_stochastics} any of the pre-defined random distributions and \texttt{tester\_tesselation} any of the regular tessellations\footnote{The visualization requires an installed gnuplot}. \\
\\
But the most important for our purposes is \texttt{main\_test.cc}. In its compiled form it is the program we use for the numerical analysis. It takes several parameters that completely define the SOM configuration as well as the testing process\footnote{the specific syntax for the use is given Ch.\ref{sec:software:main_test} in the appendix.}. Given the specific parameters, it automatically generates the according GRiSOM and distribution as well as a \texttt{Mean\_Extra\_Dim\_Analyzer} and a \texttt{Mean\_Rep\_Error\_Analyzer}. Then it passes them to the \texttt{sweep} function of the \texttt{Test\_Suite}. Furthermore a new subfolder of folder \texttt{data/} in the working directory is created and the result will be written there for future analysis.

\part{Analysis}
\chapter{Stability Analysis of SOM / Analytic approach} \label{ch:analytic}

As mentioned in the motivation of this thesis we will introduce and extend two approaches formerly done by H. Ritter and K. Schulten in \cite{schulten}. The first one is the analytical approach, which we will present in this chapter by analyzing the dynamics of the (generalized) SOM and following closely the calculations of the paper. As an result we will obtain the Fokker-Planck Equation of this stochastic process. This result can then be used to determine analytically the stability under the special circumstance of having spatially uniform sample densities.\\
\\
The second approach that will be used is a numerical one, i.e. we will make use of \emph{Monte Carlo simulations} to search the parameter space for the limits of stability. It will be used to confirm the analytical results on the one hand and to compute results beyond these cases (e.g. for hyperbolic space) on the other hand (cf. chap. \ref{ch:numerical}).\\
\\
To keep this chapter as brief as possible, all longish, lesser instructive intermediate steps of the calculations are skipped. This includes also the geometric calculations needed for the concrete analysis of the regular maps.
For the interested reader they are put in the appendix to allow to quickly check the used results in detail. Nonetheless are the concrete calculations of the stability limits kept in this chapter, but have been partly shortened. The reader in a hurry, who is not interested in the calculations at all, may be referred to the last section, where all the results are briefly summarized.

\section{Definition of Stability}

At the beginning it should be clarified what is meant when we talk about ``stability of a SOM'' in the context of this thesis.
As mentioned in chapter \ref{ch:gsom}, the SOM face a dimension conflict if the number of the intrinsic dimensions of the sample set in the feature space exceed the dimension of the map space. The SOM, trying to detect the most significant dimensions, embed the map as best as possible in the given set of samples and therefore in the additional dimension, if the inputs are scattered deep enough, to minimize representation error as best as possible. But an additional dimension in e.g. sensory data, used as input, may not always result from interesting structures of an examined object, but instead from rather unwanted data like background noise. So, we are eager to know, how large the scattering of this noise can be, such that the SOM still regards the additional dimension as insignificant.\\
\\
To approach this question, we examine the behavior of a SOM with a two-dimensional map working in a three-dimensional feature space. The set of the input vectors is now bounded by the faces of a flat cuboid. The significant structure lies in the two directions, in which the cuboid is wider ($x$- and $y$-direction) , and the remaining ($z$-)direction is reserved for the noise. The case in which the cuboid is totally flat, thereby corresponds to the situation where the number of intrinsic dimensions of the input exactly matches the number of dimensions of the SOM, while it exceeds it by one dimension otherwise. 
\\
A straightforward definition of \emph{stable} is therefore, that the SOM despite the additional dimension in the feature space rest in its former equilibrium state of the flat cuboid. It may show, nevertheless, small distortions due to equilibrium fluctuations. In regard to our question above, we are thus interested in determining the limit of the edge length $2s^*$ of cuboid in the $z$-direction up to which the equilibrium remains a stable one. In this context, a certain SOM is then regarded as being ``more stable'' as another one, if the equilibrium remains stable for larger $s^*$.\\
\\
Fig.\ref{fig:analytic:states} shows a HSOM that is on the left side still in the starting equilibrium state.

\begin{figure}[!ht]
\begin{center}
\includegraphics[width=0.98\linewidth]{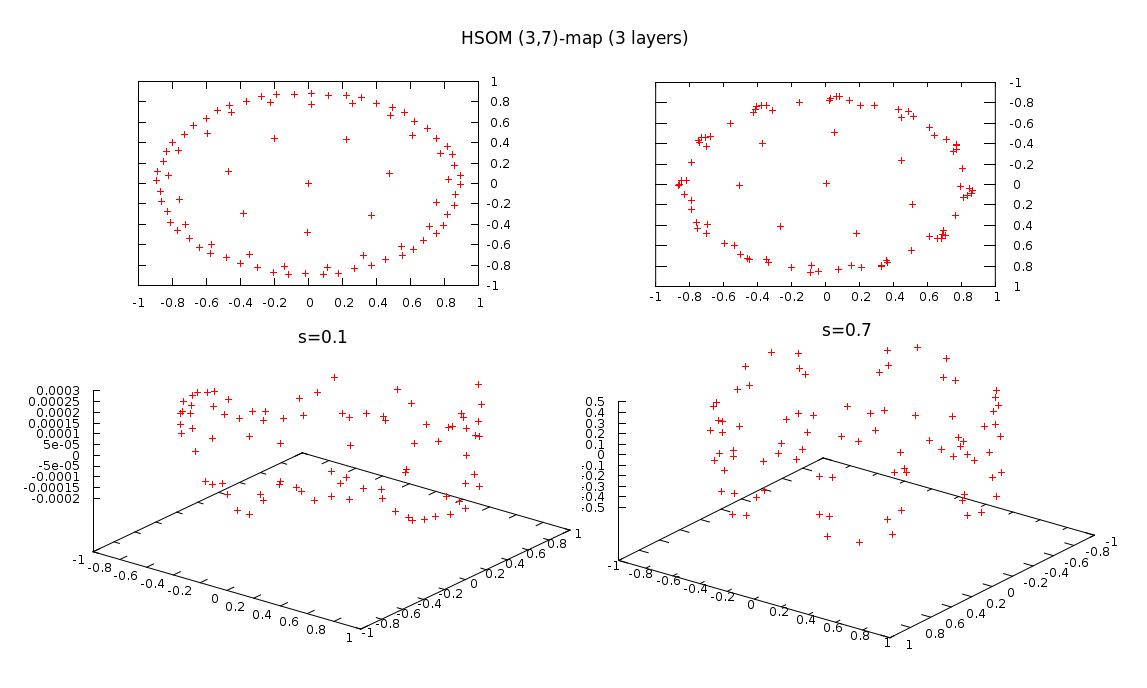}
\end{center}
\caption{HSOM with (3,7)-4 map (left) below stability limit (s=0.1) and (right) beneath (s=0.7)}
\label{fig:analytic:states}
\end{figure}

In the figures on the right side, the size of the scattering of the input towards the additional dimension already exceeds the stability limit. Having a SOM state with the representants only lying in or near the old, ``flat'' equilibrium would yield then a very poor representation for many samples. Therefore, the SOM leaves its former state and is drawn into a new one to embed itself better into the given set of samples and thereby minimizing the representation error. This can be easily observed by looking for significant lasting distortions in direction of the additional dimensions. This criterion will be used to determine the limit numerically in section \ref{ch:numerical}, but for the analytic study following below, we are instead having a look at the dynamics of the SOM, or to be more exact, at the direction of the ``drift'' of the map.

\section{*Fokker-Planck Equation for the SOM process}

We start our analytic approach by examining the dynamics of the ``classical'' SOM. It defines a learning system, i.e. a system that will reach its goal by a sequence of adaption steps defined by a learning algorithm. Each step results from the presentation of a input vector/sample $v$ and the interaction of each neuron with it (or with the cooperation with the winner neuron). This is represented by the transformation

\begin{equation}
w = T(w', v, \varepsilon)
\label{eq:stability:trafo}
\end{equation}
where $w'$ denote the state of the SOM before and $w$ after the transformation.\\
\\
 In general, a sample can be any observable in the environment like the measurement of an experiment, which are in most cases hardly predictable. However given a stationary environment, it is sensible to idealize the system by assuming that there is a stationary probability distribution $P(v)$ for the samples. Eq.\ref{eq:stability:trafo} does then not define a deterministic but a stochastic \emph{Markov process} where the state of the SOM is also only defined by a non-stationary distribution function $\tilde S(w,t)$ of the states where $t$ is the iteration time. Knowing this distribution at a certain time step, it is then possible to compute $\tilde S(w,t+1)$ using the transition probability in the next adaption step to get state $w'$ when starting in $w$.

\begin{equation} 
Q(w,w') = \underbrace{\sum_r \int_{F_r(w_r)} dv }_{\int dv} \delta(w-T(w',v,\eps)) P(v)
\label{eqn:stability:trans_prob}
\end{equation}

In order to analyze the dynamics of the stochastic process of the SOM we shall focus on how to derive the time development of the distribution function of the SOM. We will further assume, that the map is already in the vicinity of its equilibrium i.e. stationary state and that the step size $\eps$ is therefore sufficient small, i.e. we want to investigate the system only in the limit of small $\eps$. A mathematical approach is, to adapt the \emph{Fokker-Planck} equation, which describes the time evolution of probability functions. It was first used as a statistical description of Brownian motion of a particle in a fluid, but has been generalized to other observables as well. In our case, the respective observable will be the state of the map.\\
\\
Let us begin with a given array $A$ of neurons , labeled by their discrete positions $\vec r$ in the map space and let 
$w = (\vec w_{\vec r_1},\dots, \vec w_{\vec r_N})$ denote a state of them. Given that, we can then determine the \emph{Dirichlet} or \emph{Voronoi tesselation} $\set{F_{\vec r_1}, \dots F_{\vec r_n}}$ of the feature space $V$ in respect to $A$ by
\begin{equation}
 F_{\vec r_i}(w) = \setgen{\vec v \in V}{\norm{\vec v- \vec \vec w_{r_i}} \leq \norm {\vec v- \vec w_{r_j}} \forall \vec r_j \in A}
\end{equation}
We will call then a tile $F_{\vec r_i}(w)$ the \emph{feature set} or \emph{voronoi cell} of neuron $\vec r_i$.
Then the probability for a sample to belong to $F_{\vec r_i}(w)$ is given by
\begin{equation}
\hat P_{\vec r_i}(w) := \int_{F_{\vec r_i}(w)} dv P(v)  \label{eqn:stability:22}
\end{equation}
The expectation value of the samples restricted to the voronoi cell i.e. the \emph{centroid} is then defined by
\begin{equation}
 \bar \vec v_{\vec r_i} := \frac 1 {\hat P_{\vec r_i}(w)} \int_{F_{\vec r_i}(w)} d\vec v P(\vec v) \vec v \label{eqn:stability:23}
\end{equation}
According to the update rule in eq. \ref{eq:som:adapt}, for each selection of an input vector $\vec v$ the Transformation $T$ of the current state of the neurons $w$ is defined by:
\begin{equation}
 w_{\vec r_i} = (1-\varepsilon h^0_{\vec r_i,\vec s}) \vec w'_{\vec r_i} + \eps h^0_{\vec r_i,\vec s} \label{eqn:stability:5}
\end{equation}
where $\vec s$ is the winner neuron, i.e. $\vec v \in F_{\vec s}(w')$.
As mentioned in the introduction above, we are now interested in the dynamics of the SOM (in the limit of small step sizes \eps), i.e. the evolution of $w$, or to be exact, $\tilde S(w,t)$ in the time. Since we have a Markovian process, we will make use of the \emph{Chapman-Kolmogorov equation} (cf.e.g.\cite{kampen}). Inserting the transition probability (Eq.\ref{eqn:stability:trans_prob}) we get
\begin{equation}
\tilde S(w,t+1) = \int  d^N w' Q(w,w') \tilde S(w',t) = \sum_r \int d^Nw  \int_{F_r(w_r)}  dv P(v) \delta(w-T(w',v,\eps))\tilde S(w',t) \label{eqn:stability:10}
\end{equation}
which is the distribution of the states at the iteration time $t$.\\
\\
If we are in or rather very close to a stationary state, the probability to leave it should tend to be zero i.e. $\tilde S(w,t)$ is peaked around a $\bar w$ which is chosen such that
\begin{equation}
\int dv P(v) T(\bar w, v, \eps)- \bar w = 0 \label{eqn:stability:20}
\end{equation}
We can then substitute the distribution function by the distribution $S(u,t)$ of the deviations $u$ from the stationary expectation value $\bar w$.
\begin{equation}
S(u,t) := \tilde S (\bar w +u,t) \label{eqn:stability:21}
\end{equation}
After another few, barely instructive transformations, our version of the Fokker-Planck Equation is then finally obtained:
\begin{equation}
\boxed{ \begin{array}{rcl} \frac 1 \eps \partial_t S(u,t)& =& \sum_{\vec rm \vec r'n} \pdiff{}{u_{\vec rm}} \left(\underbrace{\pdiff{V_{\vec rm}}{w_{\vec r'n}}(\bar w)}_{=: B_{\vec rm \vec r'n}} u_{\vec r'n} S(u,t) \right) \\&+& \frac \eps 2 \sum_{\vec rm \vec r'n} D_{\vec rm \vec r'n}(\bar w) \frac {\partial^2 S(u,t)}{\partial u_{\vec rm} \partial u_{\vec r'n}}
\end{array}
 }
\end{equation}
where $B$ and $D$ are two matrices expressing the drift of the neurons in the feature space and the correlation of the drifts i.e. kind of diffusion, respectively. \\
\\
Given that equation we can now derive explicit expressions for the expectation value and the correlation matrix of the deviations $u_{\vec rm}$ namely $\bar u_{\vec rm}(t)$ and $C_{\vec rm \vec rs}(t)$ at any given point of time $t$:
\begin{eqnarray} 
 \bar u(t) &=& Y(t) \bar u(0) \label{eqn:stability:36} \\ 
C(t) &=& Y(t) \left[ C(0) + \int_0^t \eps(\tau)^2 Y(\tau)^{-1} D(Y(\tau)^{-1})^T d\tau \right] Y(t)^T  \label{eqn:stability:37}
\end{eqnarray}
where $Y(t)$ a matrix given by
\begin{equation*}
 Y(t) = \exp(-B \int_0^t \eps(\tau) d\tau)
\end{equation*}
A special case for Eq.\ref{eqn:stability:36} which we will need in the analysis of the regular maps, arises, when we assume that $B$ and $D$ commutes and $\eps$ is constant. Then the equation can be simplified and we get:
\begin{equation} \label{eqn:stability:39}
 C = \left(\expval{(u_{rm}-\bar u_{\vec rm})(u_{\vec sn}-\bar u_{\vec sn})}\right) = \eps (B + B^T)^{-1} D
\end{equation}

\section{*Analysis for Euclidean space with regular maps} \label{ch:analytic:regular_maps}

We are now going to determine the stability limit for the three possible regular Euclidean maps which we have found in chapter \ref{ch:tess}. Analogously to the approach in \cite{schulten} we shall thereby have a closer look at the case of spatially uniform probability densities $P(v)$ of the samples that we present to the SOM. Since the paper already discussed the case of having a square map, we will mainly analyze here the other two regular maps i.e. the triangular and hexagonal tiling and only extend the square case to cover the VQ neighborhood function and the Gaussian neighborhood function in the limit of small $\sigma$.\\
\\
We therefore assume that we have a three-dimensional parallelepiped $V$ bounded by 
\begin{eqnarray*}
&& 0 \leq x,y \leq N \qquad(\mathrm{square}) \\
&& 0\leq x\leq \frac 32 N, 0 \leq y\leq \frac{\sqrt 3} 2 N \qquad(\mathrm{hexagonal})\\
&& 0\leq x\leq N, 0 \leq y\leq \frac{\sqrt 3} 2 N \qquad(\mathrm{triangular})\\
\end{eqnarray*}
and $-s \leq z \leq s$ where $x$,$y$ and $z$ are cartesian coordinates of the Euclidean space in which $V$ is embedded. We then define a uniform distribution restricted to $V$ and thereby given by
\begin{eqnarray*}
P(v)=[2s N^2]^{-1} \qquad(\mathrm{square}) \\
P(v)=[\frac{3\sqrt 3} 4 s N^2]^{-1} \qquad(\mathrm{hexagonal})\\
P(v)=[\frac{\sqrt 3} 2 s N^2]^{-1} \qquad(\mathrm{triangular})\\
\end{eqnarray*}
The map is a two-dimensional $N \times N$ array (cf. section \ref{ch:tess}) of the particular form i.e. square, triangular or hexagonal. To avoid edge effects we furthermore impose periodic boundary conditions along the $x$- and $y$-directions in both map and feature space. Thus, obviously, an equilibrium state for the particular cases using the sample distribution as defined above is given by
\begin{eqnarray*}
\bar w_r &=& m \vec e_1 + n \vec e_2 \qquad(\mathrm{square}) \\
\bar w_r &=& \frac {\sqrt 3} 2 n \vec e_2 + \left\{ \begin{array}{ll} m \vec e_1 & \qquad \mathrm{if\ m\ and\ n\ even} \\
								(m+\frac 12) \vec e_1  & \qquad \mathrm{if\ m\ even\ and\ n\ odd}\\
								(m+2) \vec e_1 & \qquad \mathrm{if\ m\ odd\ and\ n\ even} \\
								(m+\frac 32) \vec e_1 & \qquad \mathrm{else}
						\end{array}\right. \qquad(\mathrm{hexagonal})\\
\bar w_r &=& \frac {\sqrt 3} 2 n \vec e_2 + \left\{ \begin{array}{ll} m \vec e_1 & \qquad \mathrm{if\ n\ even} \\
								(m+\frac 12) \vec e_1  & \qquad \mathrm{else} 
						\end{array}\right. \qquad(\mathrm{triangular})  
\end{eqnarray*}
since it fulfills the condition  of an equilibrium as in Eq.\ref{eqn:stability:20}:
\begin{equation} \label{eqn:stability:48}.
 \int dv P(v) [T(\bar w, v, \eps]_r - \bar w_r = 0
\end{equation}
We should remind us that, as already noted in section \ref{sec:src:eval}, this equilibrium state is not unique but represents a whole class of states closed under the symmetry transformations which are the rotations and translations in the $x$-$y$-plane. This particular choice of the representant here is just for the purpose to get a convenient selection of origin and orientation of the coordinate system.\\
\\
Reminding that we decided to work here with small and constant $\eps$ and if a distribution $S(\vec u,t)$ with finite variance and expectation value $\bar w$ exists, we can calculate the fluctuations about this stable equilibrium state by using Eq.\ref{eqn:stability:39}. $S(u) := \lim_{t\rightarrow \infty} S(\vec u,t)$ thereby denotes the corresponding stationary distribution function of the deviations in the equilibrium.\\
\\
One difficulty concerning the Fokker-Planck equation is that the deviations $u_{\vec r}$ are coupled. Fortunately, $D_{\vec rm \vec r'n}$ and $B_{\vec rm \vec r'n}$ are translational invariant i.e. depend only on the difference $\vec r- \vec r'$ and on $m$,$n$. Thus, we can easily switch to the Fourier space by Fourier transforming the deviations:
\begin{equation*}
 \hat u_{\vec k} = \frac 1 N \sum_r e^{i \vec k \vec r} u_{\vec r}
\end{equation*}
As each mode amplitude is distributed independently, we can decouple the distribution function and thus finally get a mutually independent stationary Fokker-Planck equations for each of the so-obtained individual mode distributions $\hat S_{\vec k}$:
\begin{equation} \label{eqn:stab:fokker_transformed}
\boxed{ 
\sum_{mn}\hat B(\vec k)_{mn} \frac{\partial}{\partial \hat u_m} \hat u_m \hat S_{\vec k}(\hat u) + \frac \eps 2 \sum_{mn} \hat D(\vec k)_{mn} \frac{\partial^2}{\partial \hat u_m \partial \hat u_n} \hat S_{\vec k}(\hat u) = 0 
}
\end{equation}
with
\begin{equation*}
\hat D(\vec k) = \sum_r e^{i\vec k(\vec r- \vec r')} D_{\vec r \vec r'} = \frac 1 N^2 [(\nabla_{\vec k} \hat h(\vec k))(\nabla_{\vec k} \hat h(\vec k))^T + M |\hat h(\vec k)|^2]
\end{equation*}
\begin{equation*}
\hat B(\vec k) = \frac {\hat h(0)}{N^2} \left[ 1- \frac {\hat h(\vec k)}{\hat h(0)} \hat a(\vec k) \right] - \frac 1 {N^2} (i \nabla_{\vec k} \hat h(\vec k))\hat b(\vec k)^T  
\end{equation*}
where
\begin{equation*}
M = \frac 1 {2s} \int_{F_{\vec r}(\bar w)} d\vec v (\vec v \vec v^T - \bar \vec v_r \bar \vec v_r^T) 
\end{equation*} 
is the covariance matrix of the input vectors $\vec v$ over a feature set $F_{\vec r}(\bar w)$, $\hat h(\vec k)$ is the Fourier transform of the (translation-invariant) neighborhood function $h_{\vec r \vec s}$ and $\hat a(\vec k)$ and $\hat b(\vec k)$ are the Fourier transform of the shift of the centroid of the voronoi cells $F_{\vec r}$ for small deviations of a node $\vec w_{\vec r'}$ in the feature space, given by the matrix $a_{\vec r \vec r'}$, and the corresponding relative change in the volume and therefore probability of the cells, respectively. The concrete matrices are hereby derived in the appendix \ref{ch:geometric_aspects} where we compute them in detail for the various types of maps.\\
\\
To simplify the further consideration, we postulate that in the case of regular maps $\hat B$ and $\hat D$ will have a diagonal block structure i.e. the components with index $(3,i)$ and $(i,3)$ vanish except for $i=3$. Thus we can easily determine the third eigenvector of both matrices. Furthermore do we postulate that the matrices commute\footnote{for the Gaussian neighborhood function at least for the approximation in the limit of large $\sigma$} (We will check both assumptions when we deal with particular maps below). The latter property now allows us to meet the prerequisites of the simplification done in Eq.\ref{eqn:stability:39} and can make use of it even if we have to deal here with the Fourier transformed version. The only detail, we have to modify, is, that instead of the transposed matrices we have to use the transposed and conjugate-complex counter-parts. We thus get
\begin{eqnarray*}
&&\expval{\hat u_n(\vec k)^2} = \frac {\eps \lambda_n^{\hat D}(\vec k)} {2 \Re(\lambda_n^{\hat B}(\vec k))}
\end{eqnarray*} 
An instability in the $z$-direction arises where 
\begin{equation} \label{eq:stab:exp_hat_u3}
\expval{\hat u_3(\vec k)^2} = \frac {\eps \hat D(\vec k)_{33}} {2 \Re(\hat B(\vec k)_{33})}
\end{equation}
exhibit a singularity i.e. the limit where the denominator becomes negative for certain choices of $\vec k$ and $s$. This corresponds to the point at which the equilibrium defined by our differential equation \ref{eqn:stab:fokker_transformed}
is no longer attractive as $\hat B$ is no longer positive definite. The smallest $s$ where this occurs is then the stability limit, denoted by $s^*$.
\\
As we are now only interested in the fluctuation along the $z$-direction, it is hereby sufficient to consider only the following components of $\hat B$,$\hat D$ and $M$ instead of the whole matrices:
\begin{eqnarray*}
\hat D_{33}(\vec k) &=&  \frac 1 N^2 [(\pdiff{}{k_3}\hat h(k))(\pdiff{}{k_3} \hat h(k)) + M_{33} \hat h(\vec k)^2]
\end{eqnarray*}
\begin{eqnarray}\label{eqn:stab:hatB33}
\hat B_{33}(\vec k) &=& \frac {\hat h(0)}{N^2} \left[ 1- \frac {\hat h(\vec k)}{\hat h(0)} \hat a_{33}(\vec k) \right] - \frac 1 {N^2} (i \pdiff{}{k_3} \hat h(\vec k))\hat b_{33}(\vec k) 
\end{eqnarray}
As in or near the equilibrium the distributions of the $z$-components for the input vectors in each voronoi cell are not correlated to the other components and the voronoi cells themselves are all identical, all but one values in the third column and row vanish. $M_{33}$ is the only non-vanishing one and is equal to the variance of the uniform distribution which is given by
\begin{equation}
 M_{33} = \frac  1 {12} (2s)^2 = \frac {s^2} 3
\end{equation}
Thus, we have finally derived everything we need for the following analytical analysis.

\subsection{*Square (4,4)-map}

We want to extend first the result of \cite{schulten} by studying the ultra-short ranged neighborhood function i.e. explicit vector quantization and the limit of $\sigma \rightarrow 0$ for the Gaussian neighborhood function.

\subsubsection{Long-ranged, Gaussian neighborhood function}

We start with the Gaussian neighborhood function as it has already been analyzed in the paper for a similar case (i.e. for $\sigma \gg 1$) and we can use some intermediate results. $\hat D$ and $\hat B$ (cf.\cite[Eq.64-69]{schulten}) satisfy the block-matrix condition and both commute in the limit of small $\sigma$ since we can then approximate $\hat D \approx \frac{4 \pi^2 \sigma^4}{N^2} M \exp(-k^2 \sigma^2)$ and, in addition, $B_{12}$ and $\hat B_{21}$ vanish in respect to the other components of $\hat B$ ($\hat B$ and $\hat D$ in this approximation are then both diagonal). We thus obtain for the fluctuation of the eigenmode amplitude of $\hat u_3$ for small $sigma$:
\begin{equation*}
 \expval{\hat u_3(\vec k)^2} = \eps \pi \sigma^2 \frac{s^2 \exp(-k^2 \sigma^2)}{3- 2s^2(2 - \cos k_1 - \cos k_2)}
\end{equation*}
As mentioned above we now have to find the smallest $s$ for which the denominator becomes negative i.e. has a null:
\begin{equation*}
 3- 2s^2 (2 - \cos k_1 - \cos k_2) \stackrel{!}= 0 \rightarrow s(\vec k) = \sqrt{\frac 3 { 4 - 2\cos k_1 - 2\cos k_2}}
\end{equation*}
Since $\sup_{\vec k}(4 - 2\cos k_1 - 2\cos k_2) = 8$, the stability limit $s^*$ is given by
\begin{equation}
\boxed{s^* =  \sqrt {\frac 3 8}}
\end{equation}

\subsubsection{Ultra-short-ranged VQ neighborhood function}

Next we will take a look at the explicit VQ neighborhood function, which was defined in Eq.\ref{eqn:gsom:vq}.
As $\hat a$,$\hat b$ and $M$ are already given, it just rests to determine the Fourier transform of $h$.
\begin{equation}\label{eqn:stability:hat_h_vq}
 \hat h(\vec k) = 1
\end{equation}
Since $\hat B$ and $\hat D$ obviously fulfill the block-matrix condition above,  we can directly determine the eigenvalues that we need:
\begin{eqnarray*}
\lambda^{\hat B}_3(\vec k) &=& \hat B_{33} = \frac 1 {N^2} \left( 1 - \frac {2 s^2} 3 (2-\kappa) \right)\\
\lambda^{\hat D}_3(\vec k) &=&\hat D_{33} = \frac {M_{33}} {N^2} = \frac {s^2}{3N^2}
\end{eqnarray*} where $\kappa(\vec k) = \cos(k_1) + \cos(k_2)$.
Thus we get:
\begin{equation*}
\expval{\hat u_3(\vec k)^2} = \frac{\eps s^2}{6 (1-\frac{2s^2}3(2-\kappa))}
\end{equation*}
The smallest $s$ for which the denominator vanishes is then determined by:
\begin{eqnarray*}
 && 1-\frac{2s^2}3(2-\kappa) = 0 \Rightarrow s^2(\vec k) = -\frac 3 {2 \kappa -4}
\end{eqnarray*}
The position of the minimum of $s(\vec k)$ coincides with the position of the minimum of $\kappa(\vec k)$ in respect to $\vec k$. The minimum of $\kappa$ is thereby $-2$. Thus the stability limit is
\begin{equation}
\boxed{s^* =  \sqrt {\frac 3 8}}
\end{equation}
Thus, we obtained the same result as for the limit of small neighborhoods in the Gaussian case above.

\subsection{*Hexagonal (3,6)-map}

We will now examine the first of the two other possible regular maps of the Euclidean plane, namely the (3,6)-map. Using the geometric calculations done in the appendix, we now obtain for small deviations of one node a shift of the centroids of the voronoi cells of each neighboring node in the feature space given by:

\begin{eqnarray*}
 a_{r,r'} &=& \delta(\vec r'-\vec r) \matrixdrei{\frac 5 9 & 0 & 0\\ 0 & * & 0 \\ 0 & 0& \frac 4 3 s^2} 
	- \delta(\vec r+\vec e_1-\vec r')  \matrixdrei{\frac 1 6  & 0 & 0 \\ - \frac 1 {3 \sqrt 3} & - \frac 1 {48} & 0 \\ 0 & 0 & - \frac 2 9 s^2} \\
	&-& \delta(\vec r+\vec e_1-\vec r')  \matrixdrei{\frac 1 6  & 0 & 0 \\ + \frac 1 {3 \sqrt 3} & - \frac 1 {48} & 0 \\ 0 & 0 & - \frac 2 9 s^2}\\ 
	&-& \delta(\vec r+\frac {\vec e_1}2 + \frac{\sqrt{3}}2 \vec e_2-\vec r')  \matrixdrei{\frac 1 {36}  & \frac 1 {9 \sqrt{3}} & 0 \\ \frac 5 {36 \sqrt 3} & * & 0 \\ 0 & 0 & - \frac 2 9 s^2} \\
	&-& \delta(\vec r+\frac {\vec e_1}2 - \frac{\sqrt{3}}2 \vec e_2-\vec r')  \matrixdrei{\frac 1 {36}  & - \frac 1 {9 \sqrt{3}} & 0 \\ - \frac 5 {36 \sqrt 3}  & * & 0 \\ 0 & 0 & - \frac 2 9 s^2} \\
	&-& \delta(\vec r-\frac {\vec e_1}2 + \frac{\sqrt{3}}2 \vec e_2-\vec r')  \matrixdrei{\frac 1 {36}  & - \frac 1 {9 \sqrt{3}} & 0 \\ - \frac 5 {36 \sqrt 3} & * & 0 \\ 0 & 0 & - \frac 2 9 s^2} \\
	&-& \delta(\vec r-\frac {\vec e_1}2 - \frac{\sqrt{3}}2 \vec e_2-\vec r')  \matrixdrei{\frac 1 {36}  & + \frac 1 {9 \sqrt{3}} & 0 \\ \frac 5 {36 \sqrt 3} & * & 0 \\ 0 & 0 & - \frac 2 9 s^2} 
\end{eqnarray*}

The three-dimensional Fourier transformation for $\hat a(\vec k)_{33}$ then yields:
\begin{eqnarray*}
 \hat a(\vec k)_{33} &=&  \frac {4 s^2} 3- (e^{i \vec e_1 \vec k} + e^{-i \vec e_1 \vec k}) \frac {2 s^2} 9 \\
&&-(e^{i (\frac {\vec e_1} 2 + \frac {\sqrt{3}} 2 \vec e_2) \vec k} +e^ {i (\frac {\vec e_1} 2 - \frac {\sqrt{3}} 2 \vec e_2) \vec k} +e^ {i (-\frac {\vec e_1} 2 + \frac {\sqrt{3}} 2 \vec e_2) \vec k} +e^ {-i (\frac {\vec e_1} 2 + \frac {\sqrt{3}} 2 \vec e_2) \vec k}  ) \frac {2 s^2} 9\\
&=&  \frac {4 s^2} 9 ( 3 - \cos(\vec k_1) - \cos(\frac {\vec k_1} 2 + \frac {\sqrt{3}} 2 \vec k_2)  -\cos(\frac {\vec k_1} 2 - \frac {\sqrt{3}} 2 \vec k_2) )\\
&=& \frac {4 s^2} 9 ( 3 - \tilde \kappa(\vec k))
\end{eqnarray*} where $\tilde \kappa$ is a substitution to simplify the notation.

Analogously we obtain for the $z$-component of $\hat b(\vec k)$:
\begin{eqnarray*}
\hat b(\vec k)_3 = 0
\end{eqnarray*}
Given these results we can again easily confirm that the matrices $\hat B$ and $\hat D$ have indeed such a diagonal block structure as we have claimed. We can thus easily calculate the stability limits for the three classes of neighborhood functions that we listed in \ref{sec:gsom:neighb_f}.

\subsubsection{Long-ranged, Gaussian neighborhood function}

We start by determining the Fourier of the Gaussian transform given in Eq.\ref{eqn:gsom:gauss}. We get
\begin{equation} \label{eqn:stability:63}
 \hat h(\vec k) = 2 \pi \sigma^2 \exp(- \frac {\sigma^2 k^2} 2) 
\end{equation}

\paragraph{Case $\sigma \gg 1$} $~~$\\

By inserting $a(\vec k)_{33}$,$b(\vec k)_3$ and $h(\vec k)$ in Eq.\ref{eqn:stab:hatB33} we obtain $\lambda^{\hat B}_3(k)$, but before, we will simplify $\tilde \kappa(\vec k)$ by observing that in the case of  large $\sigma$ either $\exp(-\sigma^2 k^2)$ is very small or $k_1$ and $k_2$ are sufficiently small and we can expand the trigonometric functions to leading order without causing a significant error:
\begin{eqnarray*}
\tilde \kappa(\vec k) &=& \cos(k_1) + \cos(\frac {k_1} 2 + \frac {\sqrt{3}} 2 k_2)  + \cos(\frac {k_1} 2 - \frac {\sqrt{3}} 2 k_2) \\ &\approx& (1- \frac{k_1^2} 2) + (1 - \frac {(\frac {k_1} 2) + \frac {\sqrt{3}} 2 k_2)^2} 2  + (1- \frac {(\frac {k_1} 2 - \frac {\sqrt{3}} 2 k_2)^2} 2) \\
&=& 3 - \frac {k^2_x} 2 - \frac {k^2_x} 4 - \frac {3k^2_y} 4 = 3 - \frac {3k^2} 4
\end{eqnarray*}
Hence we get for $\lambda^{\hat B}_3(k)$:
\begin{equation}
\lambda^{\hat B}_3(\vec k) =\hat B_{33} = 
\frac {2 \pi \sigma^2}{N^2} (1 - \frac {4 s^2} 9 ( 3 - \tilde \kappa ) \exp({- \frac{k^2 \sigma^2} 2}) )
= \frac {2 \pi \sigma^2}{N^2} (1 - \frac {s^2 k^2} 3  \exp({- \frac{k^2 \sigma^2} 2}) )
\label{eq:stab:3_6_Gauss_B}
\end{equation}
and analogously for $\lambda^{\hat D}_3(k)$
\begin{equation}
\lambda^{\hat D}_3(\vec k) = \hat D_{33} = \frac{4 \pi^2 \sigma^4}{N^2} M \exp(-k^2 \sigma^2)
= \frac{4 \pi^2 \sigma^4 s^2}{3 N^2} \exp(-k^2 \sigma^2) \label{eq:stab:3_6_Gauss_D}
\end{equation}
Using Eq.\ref{eq:stab:exp_hat_u3} we now get for the correlation of the corresponding eigenmode amplitude:
\begin{equation} \label{eq:stab:3_6_u}
\expval{\hat u_3(\vec k)^2} = \eps \frac{\pi \sigma^2 s^2 \exp(-k^2 \sigma^2)}{3 (1 - \frac {s^2 k^2} 3  \exp(- \frac{k^2 \sigma^2} 2) )} 
\end{equation}
Once more, we now have to find the smallest $s$ for which the denominator vanishes:
\begin{eqnarray*}
&& 6 \pi \sigma^2 (1 - \frac {s^2 k^2} 3  \exp({- \frac{k^2 \sigma^2} 2}) ) = 0 
\Leftrightarrow s^2 = \frac 3 {k^2} \exp({\frac{k^2 \sigma^2} 2})
\Rightarrow s(\vec k) = \frac{\sqrt{3}} k \exp({\frac{k^2 \sigma^2} 4})
\end{eqnarray*}
To get the minimum, the derivative of $s(\vec k)$ in respect to $\vec k$ has to vanish:
\begin{eqnarray*}
&\Rightarrow& s'(\vec k^*) = -\frac {\sqrt{3}} {2k^{*2}} \exp({\frac{k^{*2} \sigma^2} 4}) + \frac {\sqrt{3} \sigma^2} 4 \exp({\frac{k^{*2} \sigma^2} 4}) \stackrel{!}= 0\\
&\Rightarrow& \boxed{k^* = \frac {\sqrt 2} \sigma} \Rightarrow \boxed{s^* = \sigma \sqrt{\frac {3e} 2} \approx 2.02 \sigma}
\end{eqnarray*}

\paragraph{Case $\sigma \rightarrow 0$} $~~$\\

The case of small $\sigma$ can be tackled the same way as done for the (4,4)-map. We first notice, that $\hat B$ and $\hat D$ are blocked-shaped and diagonal and thus obtain now for the fluctuation (using approximation of Eq.\ref{eq:stab:3_6_Gauss_B}\&\ref{eq:stab:3_6_Gauss_D} for small $\sigma$):
\begin{equation*}
\expval{\hat u_3(\vec k)^2} =  \eps \pi \sigma^2 \frac{s^2 \exp(-k^2 \sigma^2)}{3 - \frac {4 s^2} 3 ( 3 - \tilde \kappa )}
\end{equation*}
And again, we have to find the null of denominator in respect to $s$ to determine $s^*$:
\begin{eqnarray*}
 && 3 - \frac {4 s^2} 3 ( 3 - \tilde \kappa ) \stackrel{!}=0 \Rightarrow s(\vec k) = \sqrt{\frac{9}{12 - 4\tilde \kappa}} 
\Rightarrow \boxed {s^* = \sqrt{\frac 1 2} \approx 0.707}
\end{eqnarray*} since by using global optimization methods we get $\inf_{\vec k}(\tilde \kappa(\vec k)) = -1.5$.

Thus we obtain a higher stability limit for this hexagonal map as above for the square case in the limit of small $\sigma$, while the results for the very large $\sigma$ are the same.

\subsubsection{Short-ranged NN neighborhood function}

Here we have to determine the Fourier transform of the NN-neighborhood function. Therefore we have to adapt Eq.\ref{eqn:gsom:NN} to fit in this case of six nearest neighbors at the (relative) coordinates $(\pm \frac 12, \pm \frac{\sqrt3}2)$,$(\pm \frac 12, \mp \frac{\sqrt3}2)$ and $(\pm 1,0)$. Thus the Fourier transformation result in:
\begin{eqnarray*}
 \hat h(\vec k) &=& \int e^{i \vec k \vec x} h(\vec x) d^3 x = 1 + ( \exp(i(\frac {k_1} 2+ \frac {\sqrt 3}2 k_2)) + \exp(-i(\frac {k_1} 2+ \frac {\sqrt 3}2 k_2))) \\ 
&& + ( \exp(i(\frac {k_1} 2- \frac {\sqrt 3}2 k_2)) + \exp(-i(\frac {k_1} 2- \frac {\sqrt 3}2 k_2))) + (\exp(ik_1)+\exp(-ik_1))\\
&=& 1 + 2 \cos(\frac {k_1} 2+ \frac {\sqrt 3}2 k_2) + 2 \cos(\frac {k_1} 2 - \frac {\sqrt 3}2 k_2) 
+ 2 \cos(k_1) = 1 + 2 \tilde \kappa(\vec k)
\end{eqnarray*} where $\tilde \kappa$ is the same substitution that we used above.
By Eq.\ref{eqn:stab:hatB33} we obtain for $\lambda^{\hat B}_3(k)$
\begin{eqnarray*}
\lambda^{\hat B}_3(\vec k) &=&\hat B_{33} = \frac 1 {N^2} \left[ 7 - (1 + 2 \tilde \kappa) \frac  4 9 s^2 (3-\tilde \kappa) \right]
\end{eqnarray*}
Analogously for $\lambda^{\hat D}_3(k)$:
\begin{equation*}
\lambda^{\hat D}_3(\vec k) =\hat D_{33} = \frac{s^2}{3 N^2} (1 + 2 \tilde \kappa)^2 
\end{equation*}
Thus we obtain in this case for $\expval{\hat u_3(k)^2}$:
\begin{equation*}
\expval{\hat u_3(\vec k)^2} = \eps \frac{s^2  (1 + 2 \tilde \kappa)^2 }{6 \left[ 7 - (1 + 2 \tilde \kappa) \frac  4 9 s^2 (3-\tilde \kappa) \right]}
\end{equation*}
Similar to the other cases we now get $s^*$ by determining the null of the denominator and find the minimum in respect to $\vec k$, or to be exact, to $\tilde \kappa$.
\begin{equation*}
 s^2(k) = - \frac {63}{8 \tilde \kappa^2 - 20 \tilde \kappa - 12} \Rightarrow \tilde \kappa^* = \frac 54 
\Rightarrow \boxed{s^* = \sqrt{\frac {18} 7} \approx 1.604}
\end{equation*}So the stability limit is slightly larger than for the square case as we expected.\footnote{Here, we don't need to determine the exact wave number $\vec k$. We just have to note that $\frac 54$ is within the range of $\tilde \kappa(\vec k)$ }.

\subsubsection{Ultra-short-ranged VQ neighborhood function}

For the final case we can again compute $s^*$ straightforward. We already have computed $\hat h$ for the square case above in Eq.\ref{eqn:stability:hat_h_vq} and since $M_{33}$ does not differ likewise, $\hat D$ has same third eigenvalue as in the square case:
\begin{equation*}
\lambda^{\hat D}_3(\vec k) = \frac {s^2}{3N^2}
\end{equation*}
For the corresponding eigenvalue of $\hat B$ we get:
\begin{equation*}
\lambda^{\hat B}_3(\vec k) = \hat B_{33} = \frac 1 {N^2} (1-\hat a_{33}) = \frac 1 {N^2}(1-\frac {4s^2}9(3-\tilde \kappa))
\end{equation*}
Thus we obtain for the fluctuations:
\begin{equation*}
\expval{\hat u_3(k)^2} = \eps \frac{s^2}{3 - s^2 ( 4 - \frac {4 \tilde \kappa}3)} \Rightarrow s^2(k) = \frac 3 {4- \frac 43 \kappa}
\end{equation*}
As in the square case, the minimum of the stability limits $s(\vec k)$ is determined by the minimum of $\tilde \kappa$ which is $-1.5$. Thus we get:
\begin{equation*}
\boxed{s^* = \sqrt{\frac 1 2} \approx 0.707 }
\end{equation*}
which is again the same result as in the Gaussian case for small $\sigma$.

\subsection{*Trigonal (6,3)-map}

Finally we will take a look at the last remaining possible regular map, namely the (6,3)-map. Since the geometric calculations are in this case due to some asymmetries quite annoying we restricted them to the needed results for the deviation and the shifts in z-direction. Thus we got:
\begin{eqnarray*}
 a_{\vec r,\vec r'} &=& \delta(\vec r'- \vec r) \matrixdrei{* & * & 0\\ * & * & 0 \\ 0 & 0 & \frac 4 3 s^2} 
	- \delta(\vec r+\vec e_1-\vec r')  \matrixdrei{* & * & 0 \\ * & * & 0 \\ 0 & 0 & \frac 4 9 s^2} \\
	&-& \delta(\vec r-\frac {\vec e_1}2 + \frac{\sqrt{3}}2 \vec e_2-\vec r')  \matrixdrei{* & * & 0 \\ * & * & 0 \\ 0 & 0 & \frac 4 9 s^2} \\
	&-& \delta(\vec r-\frac {\vec e_1}2 - \frac{\sqrt{3}}2 \vec e_2-\vec r')  \matrixdrei{ * & * & 0 \\ * & * & 0 \\ 0 & 0 &  \frac 4 9 s^2} 
\end{eqnarray*}

The three-dimensional Fourier transformation $\hat a(\vec k)_{33}$ therefor is:
\begin{eqnarray*}
 \hat a(\vec k)_{33} &=&  \frac {4 s^2} 3 - \frac {4 s^2} 9 \exp({i \vec e_1 \vec k}) - \frac {4 s^2} 9 \exp({i (-\frac {\vec e_1} 2 + \frac {\sqrt{3}} 2 \vec e_2) \vec k}) - \frac {4 s^2} 9 \exp({i (-\frac {\vec e_1} 2 - \frac {\sqrt{3}} 2 \vec e_2) \vec k})\\
&=& \frac {4 s^2} 9 ( 3 - \tilde {\tilde \kappa})
\end{eqnarray*} where $\tilde {\tilde \kappa}$ is once more a substitution to simplify the notation.

Furthermore we have for the $z$-component of $\hat b(\vec k)$:
\begin{eqnarray*}
\hat b(\vec k)_3 = 0
\end{eqnarray*}
As in the hexagonal case, we can easily confirm that the matrices $\hat B$ and $\hat D$ have indeed such a diagonal block structure and we can, again, calculate the stability limits for the three classes of neighborhood functions.
The only small difference, which we have to consider, is, that $\tilde {\tilde \kappa}$ is here complex-valued. Apart from that, most of the calculations can be shortened as they are the identical or at least very similar.

\subsubsection{Long-ranged, Gaussian neighborhood function}

\paragraph{Case $\sigma \gg 1$} $~~$\\

First, we simplify $\tilde {\tilde \kappa}(\vec k)$ or, to be more exact, in the real part $\Re(\tilde {\tilde \kappa}(\vec k))$ \footnote{since we are only interested in the denominator of $\expval{\hat u_3(k)^2}$.} for the limit of large $\sigma$: 
\begin{eqnarray*}
\Re(\tilde \kappa(\vec k)) &=&  \Re(\exp({i \vec e_1 \vec k}) + \exp({i (-\frac {\vec e_1} 2 + \frac {\sqrt{3}} 2 \vec e_2) \vec k}) + \exp({i (-\frac {\vec e_1} 2 - \frac {\sqrt{3}} 2 \vec e_2) \vec k}) ) \\
&=& \cos(\vec e_1 \vec k) + \cos({ (-\frac {\vec e_1} 2 + \frac {\sqrt{3}} 2 \vec e_2) \vec k}) + \cos({(-\frac {\vec e_1} 2 - \frac {\sqrt{3}} 2 \vec e_2) \vec k}) \\
&=& 3 - \frac {k^2_x} 2 - \frac {k^2_x} 4 - \frac {3k^2_y} 4 = 3 - \frac {3k^2} 4
\end{eqnarray*}
Since the eigenvalues $\lambda^{\hat B}_3(k)$ and $\lambda^{\hat D}_3(k)$ are now given by
\begin{eqnarray*}
\lambda^{\hat B}_3(\vec k) &=& \hat B_{33} = 
\frac {2 \pi \sigma^2}{N^2} (1 - \frac {4 s^2} 9 ( 3 - \tilde{ \tilde \kappa} ) \exp({- \frac{k^2 \sigma^2} 2}) )
\\
\lambda^{\hat D}_3(\vec k) &=& \hat D_{33} = \frac{4 \pi^2 \sigma^4}{N^2} M \exp(-k^2 \sigma^2)
= \frac{4 \pi^2 \sigma^4 s^2}{3 N^2} \exp(-k^2 \sigma^2)
\end{eqnarray*}
we get here for the fluctuation
\begin{eqnarray*}
\expval{\hat u_3(k)^2} &=& \eps \pi \sigma^2 \frac{s^2 \exp(-k^2 s^2)}{3 \cdot \Re(1- \frac {4s^2} 9 (3- \tilde{\tilde \kappa})\exp({- \frac{k^2 \sigma^2} 2})} = \eps \pi \sigma^2 \frac{s^2 \exp(-k^2 s^2)}{3- \frac {4s^2} 3 (3- \Re(\tilde{\tilde \kappa}))\exp({- \frac{k^2 \sigma^2} 2})}\\
&=& \eps \pi \sigma^2 \frac{s^2 \exp(-k^2 s^2)}{3 - \frac{s^2 k^2}3 \exp({- \frac{k^2 \sigma^2} 2})}
\end{eqnarray*}
This result is identical to Eq.\ref{eq:stab:3_6_u} of the (3,6)-map. Thus the stability limit $s^*$ has to be also the same as for the other two maps:
\begin{equation*}
\boxed{s^* = \sigma \sqrt{\frac {3e} 2} \approx 2.02 \sigma}
\end{equation*}

\paragraph{Case $\sigma \rightarrow 0$} $~~$\\
 Since $\tilde{kappa} = \Re(\tilde{\tilde \kappa})$, both $\lambda^{\hat D}_3(\vec k)$ and the real part of $\lambda^{\hat B}_3(\vec k)$ are identical to their counterparts for the (3,6)-map. Thus, the stability limit for small $\sigma$ is again given by:
\begin{equation*}
\boxed{s^* = \sqrt{\frac 1 2} }
\end{equation*}

\subsubsection{Short-ranged NN neighborhood function}

The Fourier transformation of the NN-neighborhood function adapted for the triangular map is:
\begin{equation*}
\hat h(\vec k) = 1 + \exp({i \vec e_1 \vec k}) + \exp({i (-\frac {\vec e_1} 2 + \frac {\sqrt{3}} 2 \vec e_2) \vec k}) + \exp({i (-\frac {\vec e_1} 2 - \frac {\sqrt{3}} 2 \vec e_2) \vec k}) = 1+ \tilde{\tilde \kappa}
\end{equation*}
Thus we get:
\begin{eqnarray*}
 \hat B_{33} &=& \frac 4 {N^2} \left[ 1 - (1+ \tilde{\tilde \kappa}) \frac {s^2} 9 (3 - \tilde{\tilde \kappa}) \right] 
= \frac 4 {9 N^2} (9 - 3 s^2 - 2 s^2 \tilde{\tilde \kappa} + s^2 \tilde{\tilde \kappa}^2)\\
 \hat D_{33} &=& \frac {s^2} {3 N^2} |1 + \hat B_{33}|^2
\end{eqnarray*}
For the fluctuation of the eigenmode amplitude we obtain:
\begin{eqnarray*}
 \expval{\hat u_3(k)^2} &=& \frac{\eps s^2|1 + \hat B_{33}|^2}{\frac 8 3 (9 - 3 s^2 - 2 s^2 \Re(\tilde{\tilde \kappa}) + s^2 \Re(\tilde{\tilde \kappa})^2)} \Rightarrow s^2(k) = - \frac 9 {\Re(\tilde{\tilde \kappa})^2 - 2 \Re(\tilde{\tilde \kappa}) -3}
\end{eqnarray*}
As $s^2(\vec k)$ is minimal for $\Re(\tilde{\tilde \kappa}^*) = 1$, the wanted stability limit $s^*$ finally is:
\begin{equation*}
\boxed{s^* = \frac 3 2 = 1.5} 
\end{equation*}
So this is the smallest stability limit for the NN case of all three regular maps.

\subsubsection{Ultra-short-ranged VQ neighborhood function}

At last, we will briefly determine the stability limit for the case of the vector quantization. Since
\begin{equation*}
 \hat B_{33} = \frac 1 {N^2} ( 1 - \frac {4 s^2} 9 ( 3 - \tilde {\tilde \kappa}))
\end{equation*}
and
\begin{equation*}
 \hat D_{33} = \frac {s^2}{3 N^2}
\end{equation*}
are very similar to the corresponding components in the hexagonal case (even $\tilde K$ and $\Re(\tilde {\tilde \kappa})$ are identical), we get by the same calculations the same result:
\begin{equation*} 
\boxed{s^* = \sqrt{\frac 1 2} }
\end{equation*}

\section{Summary of the analytic results}

Tab.\ref{tab:stab:result} summarize all the results that we have obtained in the preceding section in detail (including resulting stability limits of the square map for NN and Gaussian neighborhood, which have been derived in \cite[sec.5]{schulten}).

\begin{table}[!ht]
\begin{tabular}{|l||c|c|c|c|}\hline
map & Gauss (large $\sigma$) & Gauss (small $\sigma$) & NN & VQ \\ \hline \hline
(4,4) &$\sigma \sqrt{\frac {3e}2}\approx 2.02 \sigma $&$\sqrt {\frac 3 8} \approx 0.612 $&$\sqrt{\frac {12}5} \approx 1.549 $&$\sqrt {\frac 3 8} \approx 0.612 $\\ \hline
(3,6) &$\sigma \sqrt{\frac {3e}2}\approx 2.02 \sigma $&$\sqrt {\frac 1 2} \approx 0.707$&$\sqrt{\frac {18}7} \approx 1.604 $&$\sqrt {\frac 1 2} \approx 0.707 $\\ \hline
(6,3) &$\sigma \sqrt{\frac {3e}2}\approx 2.02 \sigma $&$\sqrt {\frac 1 2} \approx 0.707 $&$\frac {3}2 = 1.5 $&$\sqrt {\frac 1 2} \approx 0.707 $\\ \hline
\end{tabular}
\caption{Results of the analytic stability analysis}
\label{tab:stab:result}
\end{table}

It can be noticed that for the three choices of maps, the results for the long-ranged Gaussian with $\sigma \gg 1$ are the same. This may result from the fact that, given a winner node, the neighboring nodes of it are all adapted in the same way as the neighborhood function hardly differs there and in large distances where the neighborhood finally decreases significantly, the differences of the lattices from the distant viewpoint of the winner node become blurred as the differences between the distances from these outer.nodes to the winner are all relatively small.\\
\\
The stability limits for the Gaussian function with very small $\sigma$ are, as may have been expected, identical to the one obtained by Vector Quantization. The (4,4)-map thereby exhibit a smaller limit than the other two maps.
\\
The use of nearest neighborhood functions results in stability limits that grow with the number of nearest neighbors, i.e. a SOM using the hexagonal map with 6 nearest neighbors at each node has the highest stability while the triangular map with half the number of nearest neighbors yields to lesser stable SOMs. A possible explanation may be that the ``visibility range'' of a node in respect to inputs is enlarged by the size of the voronoi cells of its neighbors.
 Thus, the ``height'' of the input space becomes smaller compared to the spread in the other directions. It is the same effect that we can notice in the case of the Gaussian neighborhood where the stability rises proportional to $\sigma$.\\
\\
In the following chapter \ref{ch:numerical}, we are now going to verify these result using a numerical approach.

\chapter{Numerical Analysis / Monte Carlo Simulations} \label{ch:numerical}

In chapter \ref{ch:analytic} we have analytically determined the stability limit of the case of both an Euclidean map and feature space where the feature space exceeds the two-dimensional map space by one dimension. In the following we will not only try to verify these analytic results, but furthermore extend this analysis to the case where we have a hyperbolic map space (i.e. HSOM) and, in addition, an hyperbolic feature space (i.e. GRiSOM). Again, the readers in a hurry may be referred to the last section of this chapter, where the results of this analysis are very briefly summarized.\\
\\
Before starting to discuss the simulation in detail, it has to be noted that the used software library (cf. chapter. \ref{ch:src}) has been optimized to support an easy intuitive implementation of map and feature spaces. Thus, the SOM algorithm is implemented in the most general way and could therefore not be specially optimized for particular choices of spaces, maps or neighborhood functions. That means, that, for each adaption step, it has to explore and adapt the whole map which may take a lot of time depending on the number of neurons in it and, since time is a limited resource, this drastically restrains the configurations of the SOM that we are able to examine. In first testing runs it has been observed that each of the available processor cores (single cores of Intel(R) Core(TM)2 Duo CPU E6550  @ 2.33GHz) was capable of adapting about one billion neurons per hour using an Euclidean map and feature space and a long-range neighborhood function\footnote{The use of NN and VQ neighborhood functions would, of course, boost the speed}. By using an hyperbolic map space the speed does not differ significantly as we still use Euclidean geodesics, but in the case of using hyperbolic map and feature spaces this decreases even down to only 25 millions per hour. 
This led to the dilemma that in order to avoid possible boundary effects the map had to be as large as possible but at the same time the sample set has to be large enough so that there are enough samples in each Voronoi cell of the neurons to ensure that the map has enough ``time'' to adapt to the distribution. Therefore we had to choose a trade-off that bounded the computing time of the result for one parameter set at maximal 2-3 hours. Fortunately, the whole exploration of the parameter space could be easily parallelized, because the result seldomly depended on each other. Thus the whole search could be rolled out onto the about available 20 workstations (i.e. 40 processor cores). That reduced the duration of the whole numerical work to ``merely'' several weeks instead of many months.

\section{Test procedure}

To determine the stability limits of the SOM configurations numerically \emph{Monte Carlo simulations} of the SOM algorithm are carried out. That is be done by generating samples by using the according uniform distribution as defined in chapter.\ref{ch:distr} which is then used to adapt the neurons using the SOM algorithm. By doing so, the SOM will adapt itself to the given sample set embedding itself into it.\\
\\
In order to determine a particular stability, we have to consider all the possible parameters on which the result can depend. In general, these are:
\begin{itemize}
\item type of map space
\item type of feature space
\item type of regular map i.e. tessellation, which is subdivided into:
\begin{itemize}
\item type of polygon (number of vertices)
\item number of neighbors at each vertex
\item edge length
\end{itemize}
\item neighborhood function (here: Gaussian, NN, VQ)
\item adaption parameters $\eps$ and $\sigma$
\item distribution of samples, or to be more exact
\begin{itemize}
\item type of distribution
\item boundaries is map dimension
\item parameter $s$ for boundaries is extra dimension
\end{itemize}
\item size of sample set
\end{itemize}

It has to be remarked, that several of these parameters depend on each other. So determines the type of the map and the feature space not only the possible choices for the tessellations (cf. Ch.\ref{ch:tess}) but also the type of distribution of the samples that is used as listed in Ch.\ref{ch:distr}. Furthermore is the size of the sample set restricted by the available resources (i.e. computing time and storage space) as mentioned above. The used values for the various maps are listed in tables \ref{tab:numerical:sample_size_SOM}, \ref{tab:numerical:sample_size_HSOM} and \ref{tab:numerical:sample_size_GRiSOM}.
\\ \\
The most important parameter is, of course, $s$. By fixing all other parameters and varying $s$ we now determine the stability limit $s^*$ i.e. the smallest value for $s$ where the equilibrium state of the SOM that lies in a plane in the map dimensions is no longer stable. In the analytic approach we examined the fluctuations in the extra dimension to determine, if the equilibrium is still attractive. Here, we use a different method based on the \texttt{Mean\_Extra\_Dim\_Analyzer} introduced in the evaluation section of chapter \ref{ch:src}. In other words, we measure the mean deviation $\bar u_3$ or, to be more exact, the normalized deviation $\bar u_3/s$ from the initial equilibrium state, which is the equilibrium for ($s=0$) in direction of the extra dimension\footnote{see Ch.\ref{sec:src:eval} for more details about exact definition for non-Euclidean spaces}. As long as the initial equilibrium is attractive, the analytic value of $\bar u_3$ vanishes. Since we work here numerically with a finite sample set, the mean state will show distortions around the equilibrium, but these fluctuations and thus the measured errors are very small and can be (easily) identified and reduced as they depend on $\eps$. By contrast, above the stability limit, where the SOM has moved into a new equilibrium state which also uses the extra dimension, the measured error is significantly larger and is only slightly affected by a variation of $\eps$. As an example, Fig.\ref{fig:numerical:phase_transitions} shows a graph of the error in respect to $s$.

\begin{figure}[!ht]
\begin{center}
\includegraphics[width=0.98\linewidth]{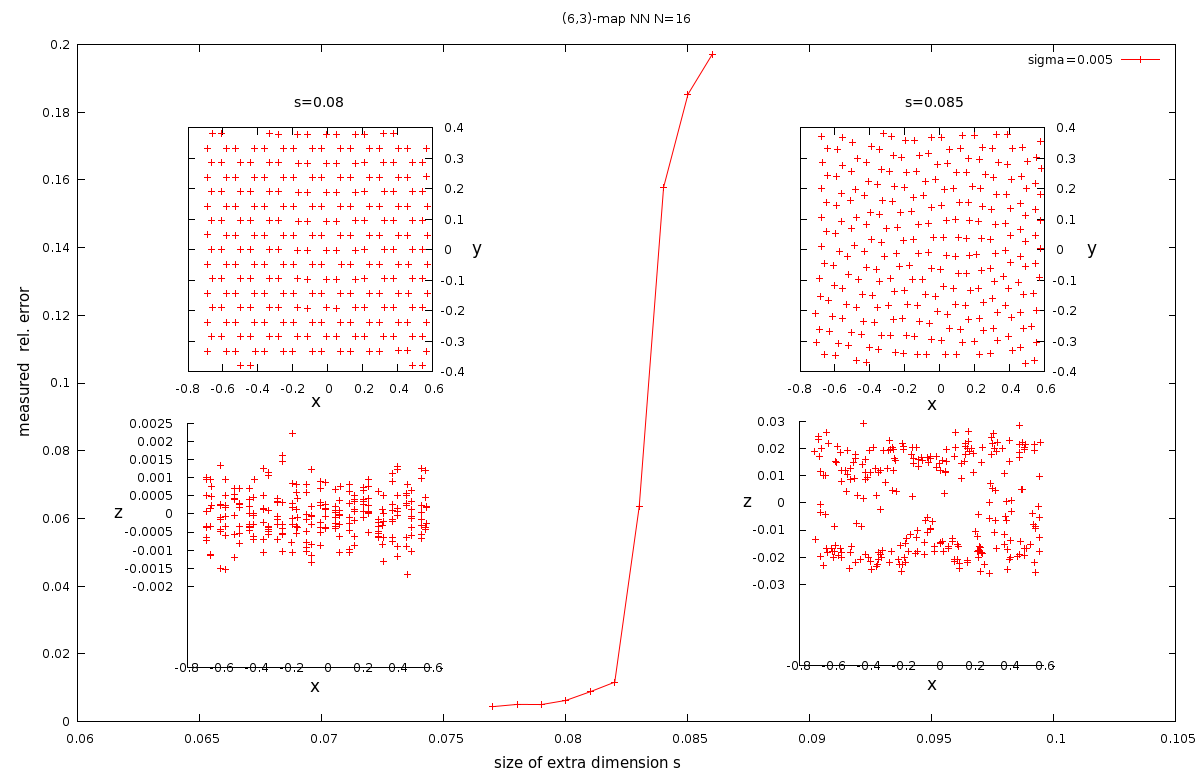}
\end{center}
\caption{Transition between former and new equilibrium at the stability limit}
\label{fig:numerical:phase_transitions}
\end{figure}

To find $s^*$ for a given SOM configuration, we thus have to do a first search run to roughly locate it by looking out for significantly large errors. Once we narrow the limit down to a interval, we can iteratively refine the search by using smaller step sizes for the sweeps until we will have pinpointed the stability limit up to a sufficiently small error margin. In Tab.\ref{tab:result:search_runs}, the search runs for a (6,3)-map with N=16 are listed as an example. Fig.\ref{fig:numerical:search_runs} shows then the corresponding graph.

\begin{figure}[!ht]
\begin{center}
\includegraphics[width=0.98\linewidth]{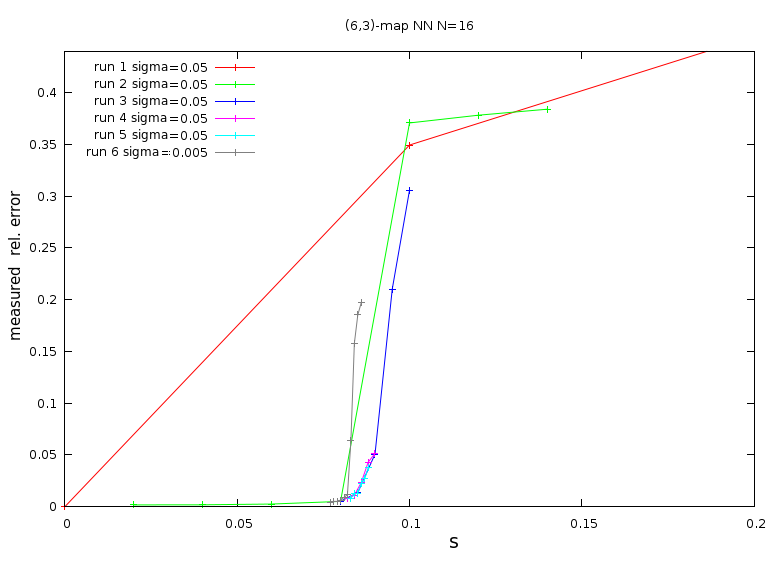}
\end{center}
\caption{}
\label{fig:numerical:search_runs}
\end{figure}

For the first runs $\epsilon=0.05$ was used. In the last run $\epsilon$ was then reduced to 0.005, which result in a sharper edge as can be seen in the graph. It can also be noticed that we did not use ideal nested intervals for the search of $s^*$. In fact, we used about. 10 steps for each sweep. This may have been slower, but allows a better overview about the neighborhood around the position of the stability limit. This is very important for the cases where the transition between the equilibriums is not such a sharp increase but instead ``blurred'' by accompanying fluctuations. Another benefit of the refinement by a factor 5 or 10 at each run is that the search needs less supervision as the sweeps have then a runtime of 12-24 hours. With 30-40 searches running parallel and the need to adapt the parameters after each search run, the required time for the supervision of the whole simulation is thus reduced to a feasible amount. The search is now stopped when the stability limit is determined within a wanted accuracy. For most of our search runs, this accuracy limit was set to 1-5\%. Only for very small values that were obtained by using the HSOM or for very time-consuming searches we already stopped the search earlier which resulted in higher error margins. The stability limits thus obtained are listed in detail in the appendix (cf. Tab.\ref{tab:results:eucl_3_6},\ref{tab:results:eucl_4_4},\ref{tab:results:eucl_6_3}). They are the basis for all the evaluations that are carried out in the following sections.

\section{Verification of the analytical results in the Euclidean case}

As mentioned, we want to verify the analytical results obtained in chapter \ref{ch:numerical}. In order to do so, we generate the three types of maps i.e. the triangular, square and hexagonal maps of the Euclidean plane. The shape of the maps is a rectangular grid of $N\times N$ neurons as seen in Fig.\ref{fig:tesselation:periodic_bound}. Instead of working with a constant length of the edges as we did in the analytical approach, we fix the area of the 2-dimensional map to 1. Thus we use the following normalized edge lengths:
\begin{equation} \label{eqn:analysis:EuclNNdist}
 d_{NN}^{(3,6)} = \frac 2 {\sqrt[4]{3} N} ,\qquad d_{NN}^{(4,4)} = \frac 1 N ,\qquad d_{NN}^{(6,3)} = \frac 2 {\sqrt[4]{27} N}
\end{equation}
In the case of the (3,6)-map, the periodic boundaries (if used) are therefore given by 
$ - \frac 2 {\sqrt[4]{3} N} \pm \frac {\sqrt[4]{3}}2$ in the $x$-direction and $ \pm \frac 1 {2 \sqrt[4]{3}}$ in the $y$-direction. Analogously for the (4,4)-map, we use $\pm \frac 12$ (in both directions) and for the (6,3)-map $\pm \frac 1 {\sqrt[4] 3}$ and $\pm \frac {\sqrt[4] 3} 2$. As we now embed the map one-to-one into the 3-dimensional Euclidean feature space, the same boundaries are used for the uniform sample set, which we use to examine the stability, in two of the three dimensions and, if used, the periodic boundaries in the feature space.\\
\\
\begin{table}[!ht]
\begin{center}
 \begin{tabular}{|l|c|c|c|c|c|c|}
\hline 
&\multicolumn{6}{|c|}{Euclidean case - lattice size} \\  
& 6x6 & 10x10 & 12x12 & 16x16 & 24x24 & 32x32  \\ \hline \hline
number of nodes & 36 & 100 & 144 & 256 & 576 & 1024 \\ \hline
number of meas. per param. & 10000 & 10000 & 10000 & 5000 & 2500 & 2500 \\ \hline
number of adapt.steps & 20 Mio & 20 Mio & 20 Mio & 10 Mio & 5 Mio &  5 Mio  \\ \hline
\end{tabular}
\end{center}
\caption{Choice of the size of the sample sets for Euclidean maps}
\label{tab:numerical:sample_size_SOM}
\end{table}

\FloatBarrier

\subsection{Short-ranged NN neighborhood}

Since the distance $d_{NN}$ depends on the choice of the map type as well as on its ``size'' $N$, the range of the NN neighborhood function differs for most configurations. So, in order to make the ``raw'' results, that we obtained by the Monte Carlo simulation, comparable, we have to renormalize them in respect to $d_{NN}$ i.e. dividing them by the corresponding distance between nearest neighbors. These ``normalized'' stability limits are listed in Tab.\ref{tab:numerical:NN_final} and visualized in Fig.\ref{fig:numerical:NN_final}\footnote{The corresponding errors have hereby been calculated by standard Gaussian error propagation.}.

\begin{table}[!ht]
\begin{center}
 \begin{tabular}{|c|c|c|c|}
\hline 
&\multicolumn{3}{c|}{map type} \\  
lattice size& (3,6)&(4,4)&(6,3)  \\ \hline \hline
6x6&1.627\err{0.008}&1.548\err{0.006}&1.525\err{0.007}\\ \hline
10x10&1.592\err{0.007}&1.540\err{0.010}&1.504\err{0.011}\\ \hline
12x12&1.595\err{0.008}&1.536\err{0.012}&1.504\err{0.014}\\ \hline
16x16&1.600\err{0.011}&1.536\err{0.016}&1.514\err{0.018}\\ \hline
24x24&1.595\err{0.016}&1.536\err{0.024}&1.504\err{0.027}\\ \hline
32x32&1.600\err{0.021}&1.536\err{0.032}&1.495\err{0.036}\\ \hline
mean &1.602\err{0.01}&1.540\err{0.01}&1.508\err{0.01}\\ \hline
analytical&1.604&1.549&1.500 \\ \hline
\end{tabular}
\end{center}
\caption{Normalized stability limits for Euclidean maps using the NN neighborhood function}
\label{tab:numerical:NN_final}
\end{table}

\begin{figure}[!ht]
\begin{center}
\includegraphics[width=0.98\linewidth]{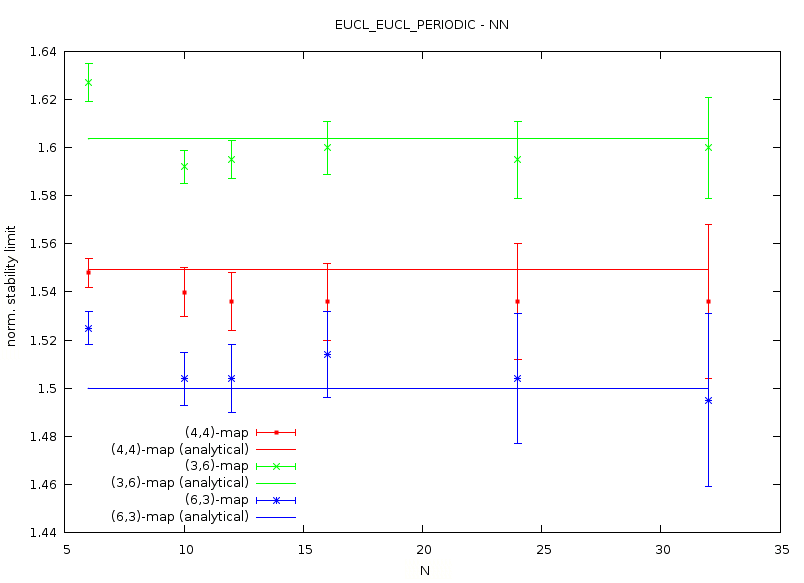}
\end{center}
\caption{Stability limits of the three Euclidean maps for the NN neighborhood function}
\label{fig:numerical:NN_final}
\end{figure}

Except a few outliers in the case of smaller grid sizes, these results verify nicely our analytical calculations and confirm the conjecture that the stability limit rises with number of nearest neighbors (at least in the case of regular Euclidean maps). The outliers may thereby result from the periodic boundary conditions, since they put constraints onto the whole map and may thereby have an effect on the stability of the map.

\FloatBarrier

\subsection{Long-ranged Gaussian neighborhood} \label{sec:numerical:Eucl_Gauss}

In the case of the Gaussian neighborhood function, we want to verify the results we obtained analytically for the cases of either very small or very large $\sigma$ and furthermore analyze at least qualitatively the stability behavior in between.\\
\\
As noted above, the distances between nearest neighbors strictly depend on the type of the map and the size of the grid. So, we again have to renormalize the determined results. This includes $\sigma$, which becomes the relative $\tilde \sigma = \sigma / d_{NN}$,and ,analogously, is $s^*$ normalized in the same way. Fig.\ref{fig:numerical:Gauss_final} now plots the results by using $\tilde \sigma$ as $x$-values and the quotient of the normalized stability limit and $\sigma$ for the second axis.

\begin{figure}[!ht]
\begin{center}
\includegraphics[width=0.98\linewidth]{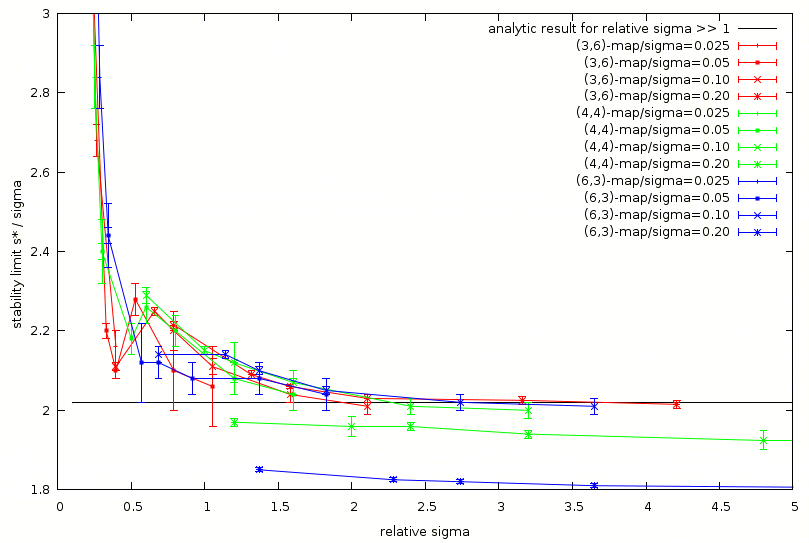}
\end{center}
\caption{Stability limits of the three Euclidean maps for the Gaussian neighborhood function}
\label{fig:numerical:Gauss_final}
\end{figure}

We can easily identify the two subdomains we will discuss below. For small $\tilde \sigma$ the graphs are similar to hyperboles, while for larger $\tilde \sigma$, they tend to approximate a constant function for rather large $\tilde \sigma$. As these are the domains that we already analyzed in the analytic approach, we will examine them now more closely. After that, we will also briefly discuss the domain in between at least qualitatively.

\subsubsection{Case $\tilde \sigma \gg 1$}

In chapter \ref{ch:analytic} we have seen, that for very large $\tilde \sigma$ the stability limit increases approximately proportionally to $\tilde \sigma$. Thus, we examine here the quotient of the normalized stability limit and the normalized range of the Gaussian neighborhood. The results thus obtained have been already shown in Fig.\ref{fig:numerical:Gauss_final}. As we may have expected, the approximation error, i.e. the difference between a point on the curves and the constant function at the stability limit for $\tilde \sigma \gg 1$, decreases exponentially in the domain of large $\tilde \sigma$. This can be easily checked by plotting the logarithm of this error, which has been done in Fig.\ref{fig:numerical:Gauss_3_6_log} for the hexagonal map.

\begin{figure}[!ht]
\begin{center}
\includegraphics[width=0.98\linewidth]{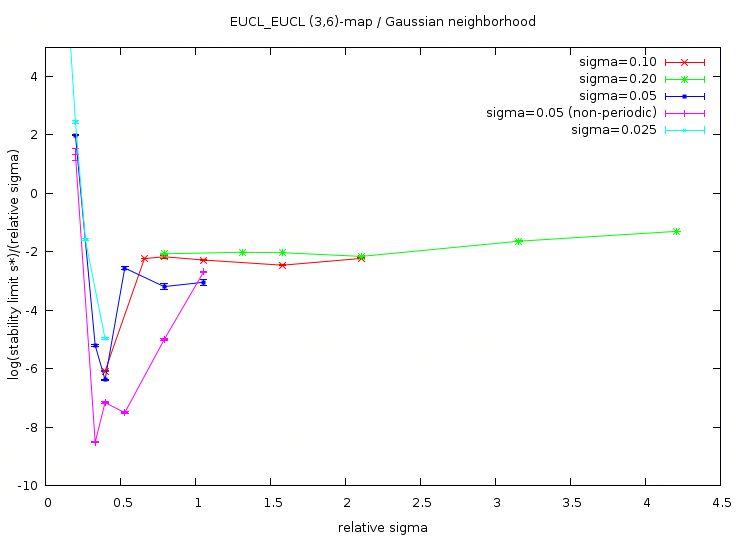}
\end{center}
\caption{Logarithmic plot of the approximation error for a hexagonal (3,6)-map}
\label{fig:numerical:Gauss_3_6_log}
\end{figure}

Thus, we can extrapolate the stability in the limit $\tilde \sigma \gg 1$ by fitting the curves in the domain of large $\tilde \sigma$ in respect to an exponentional function $a \cdot e^{bx} +c$. The values for $c$ thus obtained are the desired results, which are now listed in Tab.\ref{tab:numerical:Gauss_large_final} \footnote{we hereby only consider the parts of the curves for $\sigma=0.1$ and $\sigma=0.2$, which satisfy $\tilde \sigma\geq 1$.}. 

\begin{table}[!ht]
\begin{center}
 \begin{tabular}{|c|c|c|c|}
\hline 
& \multicolumn{3}{c|}{map type} \\
$\sigma$ & (3,6) & (4,4) & (6,3)\\ \hline \hline
\multicolumn{4}{c}{periodic boundaries} \\ \hline
0.1 & 2.01\err{0.01}& 1.99\err{0.01}& 1.99\err{0.02} \\ \hline
0.2 & 2.00\err{0.01}& 1.92\err{0.01}& 1.78\err{0.01} \\ \hline
\multicolumn{4}{c}{non-periodic boundaries} \\ \hline
0.05 & 1.98 \err{0.02}& 1.75\err{0.10}& 1.82\err{0.09} \\ \hline
0.2 & --- & 1.50 \err{0.02}& --- \\ \hline
analytical&\multicolumn{3}{c|}{2.02} \\ \hline
\end{tabular}
\end{center}
\caption{Extrapolated stability limits for Euclidean maps using the Gaussian neighborhood function with large $\sigma$}
\label{tab:numerical:Gauss_large_final}
\end{table}

We see that the analytic results are well verified for the hexagonal maps with and without periodic boundaries and for the square and the triangular map in the case of periodic boundaries and $\sigma=1$. For the other cases we observe some significant deviations in respect to the analytic results. The lower stability limits in the case of missing periodic boundary conditions may be caused by the fact that the neurons at the ``loose'' edge have far fewer neighbors. As we have discussed in the summary of the analytic results, the size of the neighborhood does have a direct impact on the stability of a SOM. The lower stability limits here seem to provide further evidence for this assumption. The reason for the lower stability of the SOM configurations using $\sigma = 0.2$ may be first not so evident, but a possible explanation may be found by taking into consideration, that $\sigma$ is very large in respect to the size of the whole feature space, which has, as is known, a base area of 1. We actually imposed boundary conditions to avoid edge effects. But this does work only locally i.e. the winner neurons only cooperate with one of the infinite many representatives, which are a result of the periodic condition. Thus the cooperation always only works on a volume which is identical to the bounded space i.e. the neighborhood cannot be larger than the space itself. Thus, if the range of the neighborhood is increased beyond this limit, the size of the neighborhood and therefore, as we assume, the stability
remains the same and thus the quotient that we examined above becomes significantly smaller. Thus, also this observed effect may back up our assumption. To be more certain, it would be necessary to perform further studies of SOMs with such large-ranged neighborhood functions.

\FloatBarrier

\subsubsection{Case $\tilde \sigma \rightarrow 0$}

Due to results of the analytic analysis in the preceding chapter, we expect, that, for the case of small (relative) $\sigma$, we observe a stability limit which, unlike to the other limit, does not depend on $\sigma$. To get sufficient small $\tilde \sigma$ we restrict our evaluation of the numerical results to the two smallest choices of $\sigma$ and the lattices with a rather large edge length. According to Eq.\ref{eqn:analysis:EuclNNdist}, this are those with fewer nodes. Fig.\ref{fig:numerical:Gauss_small_final} and Tab.\ref{tab:numerical:Gauss_small_final} show these results.

\begin{table}[!ht]
\begin{center}
 \begin{tabular}{|c|c|c|c|c|c|c|c|}
\hline 
&&\multicolumn{6}{c|}{map type} \\
lattice size & $\sigma$ & \multicolumn{2}{c|}{(3,6)} & \multicolumn{2}{c|}{(4,4)}&\multicolumn{2}{c|}{(6,3)}\\
&& $\tilde \sigma$&norm. $s^*$&$\tilde \sigma$&norm. $s^*$&$\tilde \sigma$&norm. $s^*$  \\ \hline 
\multicolumn{8}{c}{periodic boundaries} \\ \hline
6x6&0.025&0.099&0.703\err{0.004}&0.150&0.738\err{0.012}&0.171&0.827\err{0.014}\\ \hline
10x10&0.025&0.165&0.711\err{0.007}&0.250&0.710\err{0.020}&0.285&0.809\err{0.023}\\ \hline
12x12&0.025&0.197&0.719\err{0.016}&0.300&0.720\err{0.024}&0.342&0.834\err{0.027}\\ \hline
16x16&0.025&0.263&0.705\err{0.011}&0.400&---&0.456&--- \\ \hline
6x6&0.05&0.197&0.691\err{0.008}&0.300&0.720\err{0.006}&0.342&0.834\err{0.007}\\ \hline
10x10&0.05&0.329&0.724\err{0.007}&0.500&1.090\err{0.020}&0.570&1.208\err{0.057}\\ \hline
12x12&0.05&0.395&0.829\err{0.008}&0.600&1.356\err{0.024}&0.684&1.450\err{0.027}\\ \hline
\multicolumn{8}{c}{non-periodic boundaries} \\ \hline
6x6&0.05&0.197&0.655\err{0.039}&0.300&0.696\err{0.006}&0.342&0.827\err{0.007}\\ \hline
10x10&0.05&0.329&0.684\err{0.007}&0.500&1.070\err{0.010}&0.570&1.106\err{0.023}\\ \hline
12x12&0.05&0.395&0.774\err{0.016}&0.600&1.284\err{0.024}&0.684&1.299\err{0.027}\\ \hline
\multicolumn{2}{|c|}{constant fit}&\multicolumn{2}{c|}{0.703\err{0.008}}&\multicolumn{2}{c|}{0.726\err{0.006}}&\multicolumn{2}{c|}{0.826\err{0.003}}\\ \hline
\multicolumn{2}{|c|}{analytical}&\multicolumn{2}{c|}{0.707}&\multicolumn{2}{c|}{0.612}&\multicolumn{2}{c|}{0.707} \\ \hline
\end{tabular}
\end{center}
\caption{Normalized stability limits for Euclidean maps using the Gaussian neighborhood function with small $\sigma$}
\label{tab:numerical:Gauss_small_final}
\end{table}

\begin{figure}[!ht]
\begin{center}
\includegraphics[width=0.98\linewidth]{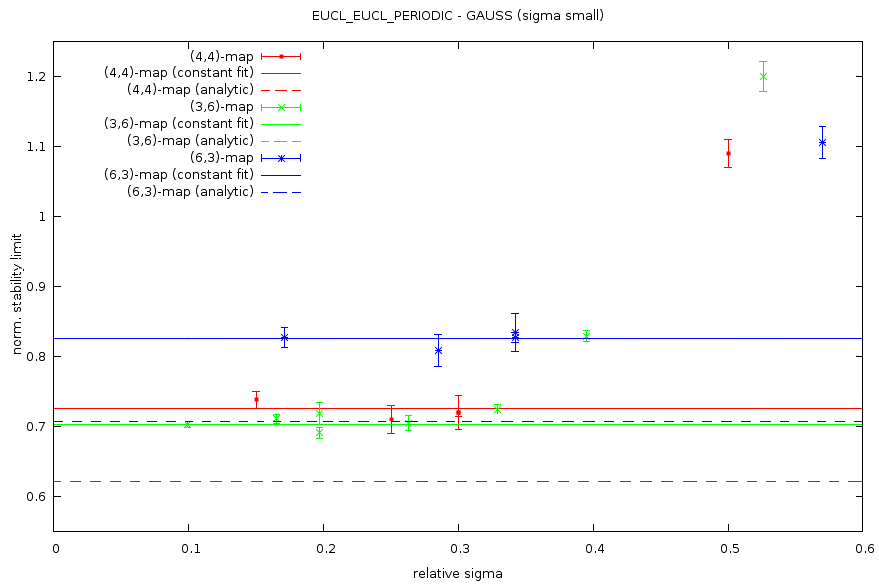}
\end{center}
\caption{Stability limits of the three Euclidean maps for the Gaussian neighborhood function with small $\sigma$ (only for periodic boundaries)}
\label{fig:numerical:Gauss_small_final}
\end{figure}

We now furthermore fitted the data points, which have a $\tilde \sigma$ lesser than 0.35, using constant functions and taking the error margins into account. The results thus obtained have also been added to the table and the graph. Remembering the analytic results, we have expected, that the stability limit is $s^* \approx 0.612$ for the square map and $s^* \approx 0.707$ for the other two maps. But instead, the numerically determined stability limits for all but one are a good deal higher. Only the hexagonal (3,6)-map shows the stability limit that we calculated before. The analytic value lies even inside the error margins of the computed constant fit. As we can not ad hoc pinpoint the source of this very large systematic error for the other two cases, we will delay the discussion until the results of the VQ are evaluated to see if this deviation only occurs when using the Gaussian neighborhood function or even when vector quantization is used.

\FloatBarrier

\subsubsection{Case $\tilde \sigma$ ``between the limits''}

Even if we have omitted a detailed discussion of the $\sigma$ between the limits, we can explain, at least, a basic aspect of the shape of the curves shown in Fig.\ref{fig:numerical:Gauss_final}. As we neglected all higher terms of $k_1^2$ and $k_2^2$ in the approximation for large Gaussian neighborhood ranges, these terms would have decreased the $s^2 k^2$ term in the equation of $\lambda_3^{\hat B}$. Thus, the resulting stability limit increases then. This is the reason why the curves are converging to the proportional factor of the limit from above.

\subsection{Ultra-short-ranged VQ neighborhood}

The last analytic results, which we still have to verify, are those for the ultra-short-ranged neighborhood i.e. vector quantization. In the analytic approach we have determined the same stability limit for the three types of maps (using periodic boundaries) as for the Gaussian neighborhood function in the limit of small $\sigma$. Now, by using the simulation we obtain the corresponding numerical results shown in Tab.\ref{tab:numerical:VQ_final} and Fig.\ref{fig:numerical:VQ_final} (The renormalization is thereby the same as in the NN case).

\begin{table}[!ht]
\begin{center}
 \begin{tabular}{|c|c|c|c|}
\hline 
&\multicolumn{3}{|c|}{map type} \\  
lattice size& (3,6)&(4,4)&(6,3)  \\ \hline \hline
6x6&0.691\err{0.004}&0.690\err{0.030}&0.821\err{0.007}\\ \hline
10x10&0.691\err{0.007}&0.680\err{0.010}&0.775\err{0.011}\\ \hline
12x12&0.687\err{0.016}&0.696\err{0.024}&0.793\err{0.027}\\ \hline
16x16&0.684\err{0.011}&0.672\err{0.032}&0.766\err{0.018}\\ \hline
24x24&0.695\err{0.032}&0.696\err{0.072}&0.766\err{0.055}\\ \hline
32x32&0.632\err{0.063}&0.704\err{0.064}&0.802\err{0.073}\\ \hline
constant fit&0.690\err{0.005}&0.683\err{0.003}&0.803\err{0.010}\\ \hline
analytical&0.707&0.612&0.707 \\ \hline
\end{tabular}
\end{center}
\caption{Normalized stability limits for Euclidean maps using the VQ neighborhood function}
\label{tab:numerical:VQ_final}
\end{table}

\begin{figure}[!ht]
\begin{center}
\includegraphics[width=0.98\linewidth]{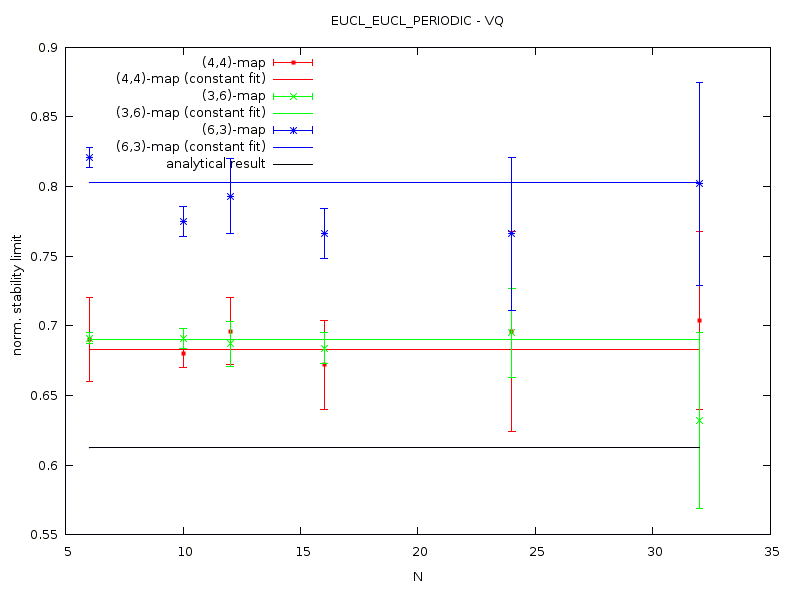}
\end{center}
\caption{Stability limits of the three Euclidean maps for the VQ neighborhood function}
\label{fig:numerical:VQ_final}
\end{figure}

The first thing, that can be noticed, is, that these results show the same anomaly in respect to the analytic limits as for the Gaussian neighborhood function, i.e. the stability limits of both the square and the triangular lattice exceed the analytic value significantly. These discrepancies thereby do not show any sign of dependence on the grid size, which let us assume that they are not caused by any boundary effect. Indeed, the numerical values match the ones obtained in the Gaussian case quite well. The same is almost true for the comparison of the numerical and analytic result for the hexagonal case, which may lead to the assumption that either we have a numerical error that only shows up for the lesser symmetric grids (we have seen such a difference in the behavior of the hexagonal map for $\sigma \gg 1$ as well) or, also possible, we missed in the analytic approach a constraint of $\vec k$, which may then increase the infimum of $\kappa$,$\tilde \kappa$ and $\tilde {\tilde \kappa}$ and thus the corresponding stability limits for the VQ and the Gaussian neighborhood function with small $\sigma$.

\FloatBarrier

\section{Analysis of configuration with hyperbolic map space (HSOM)}

After having a look at the ''Euclidean-Euclidean`` case above we want now investigate the stability properties of SOMs consisting of non-Euclidean map or feature spaces. More precisely, we will start by examining the case where we use a HSOM i.e. having an hyperbolic map space and continue with the case of having additionally a hyperbolic feature space in the next section.\\
\\
In chapter \ref{ch:tess}, we have discussed two possible truncations for the hyperbolic map. Due to lack of time as the implementation of software took longer as expected, we, unfortunately, had only the time to examine one of them. We decided to use the equi-hierarchical method, i.e. we created the map by restricting the depth of the hierarchical structure, because this is the method used in founded papers (e.g. \cite{ritter},\cite{ontrup2}) and we therefore obtain results which are comparable and may e.g. help to estimate the sensibility of the used HSOM in respect to non-hierarchical noise in a hierarchical sample set.\\
Since we nonetheless use a disk-like spread of the sample set (cf. section \ref{sec:distr:def_HE}) with the radius of $1.5 \cdot d_{NN}$, this will result in certain deformations of the embedding of the map in the feature space, as the non-circular map will try to fill the whole disk as best as possible. These deformations may then cause a difference between the so determined stability limits and the ones we would gain if we would use the other truncation method which would lead to a more circular map and therefore in lesser distorsions.\\
\\
Nevertheless, we try to determine certain stability properties of the HSOM.

\subsection{Stability vs. Size}

One of the most important features of this SOM, which is gained by using a hyperbolic map, is the capacity to easily adapt to higher-dimensional Euclidean space as the number of nodes at the edge growth exponentially in respect to the number of layers/radius. Tab.\ref{tab:numerical:sample_size_HSOM} lists, besides the number of measurements and size of sample sets we used, the number of nodes in the map and of the nodes in the outermost layer.
\begin{table}[!ht]
\begin{center}
 \begin{tabular}{|l|c|c|c|}
\hline 
&\multicolumn{3}{c|}{HSOM case - maps} \\  
&(3,7)-\{3,4,5\} layers&(3,9)-\{3/4\} layers &(4,5)-\{3/4/5\} layers\\ \hline \hline
\#nodes&85,232,617&271,1306&61,166,441 \\ \hline
\#nodes in outer layer&56,147,385&216,1035&40,105,275\\ \hline 
\# meas./param.& 2500 & 2500 & 2500\\ \hline
\# adapt.steps&3.75 Mio&3.75 Mio&3.75 Mio\\ \hline
\end{tabular}\\
\vspace{0.05 \linewidth}
\begin{tabular}{|l|c|c|} \hline
&\multicolumn{2}{c|}{HSOM case - maps} \\  
&(6,4)-\{3/4/5/6\} layers&(7,3)-\{3/4/5/6\} layers\\ \hline \hline
\# nodes&49,133,353,929&22,40,70,115\\ \hline
\#nodes in outer layer&32,84,220,576&12,18,30,45\\ \hline
\# meas./param.&2500 & 2500\\ \hline
\# adapt.steps&3.75 Mio&3.75 Mio\\ \hline
\end{tabular}
\end{center}
\caption{Size of the sample sets/maps/layers for HSOM}
\label{tab:numerical:sample_size_HSOM}
\end{table}
We want to start  by examining the dependency of the stability limit on the size of the map, or, to be more exact, the number of nodes at the edge. We therefor calculate the ratio between the stability limits of maps which differ only in respect of the number of layers and compare it to the corresponding reciprocal ratio of the number of outer map nodes (cf. Fig.\ref{fig:numerical:HSOM_ratio}).

\begin{figure}[!ht]
\begin{center}
\includegraphics[width=0.98\linewidth]{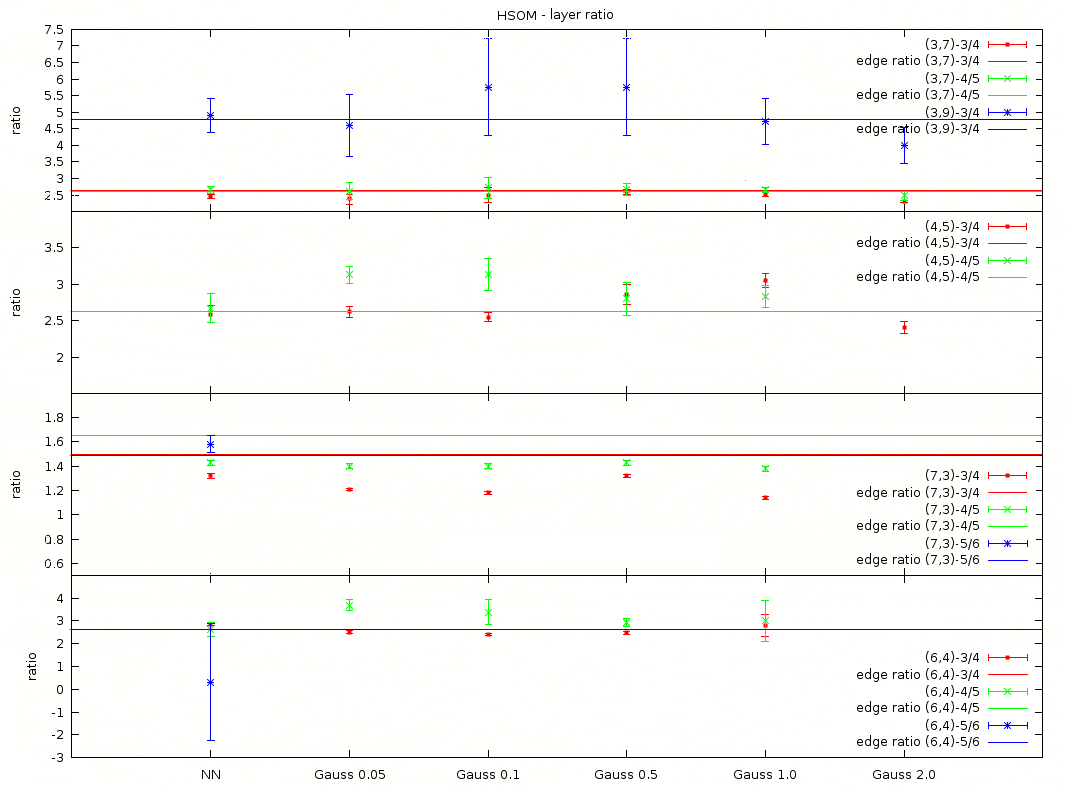}
\end{center}
\caption{Ratio of the stability limits between HSOM maps with different sizes}
\label{fig:numerical:HSOM_ratio}
\end{figure}

Even if the ratio of the number of edge nodes for some of the results does not lie within the calculated error margins, the relation between the size of the outermost layer and the stability is quite clear i.e. the stability limit appear to be inversely proportional to the number of nodes in the outermost ring. This corresponds to the fact that the HSOM can solve or, at least, reduce the dimensional conflict due to its much larger growth rate in respect to the Euclidean extra dimension of the input.

\subsection{Short-ranged NN neighborhood}

It has been verified above, that the stability for the classic SOMs using NN neighborhood functions depends on the type of map. We suggested that the reasons are the higher number of nearest neighbors. We want to check, if this conjecture also holds when a HSOM is used. To be able to compare the different maps and number of layers we use the knowledge about the dependence of the stability on the edge size and thus normalize the stability limit in respect to the number of edge nodes. So we get rid of the already known influence of the different edge size and get the results plotted in Fig.\ref{fig:numerical:HSOM_NN}.

\begin{figure}[!ht]
\begin{center}
\includegraphics[width=0.98\linewidth]{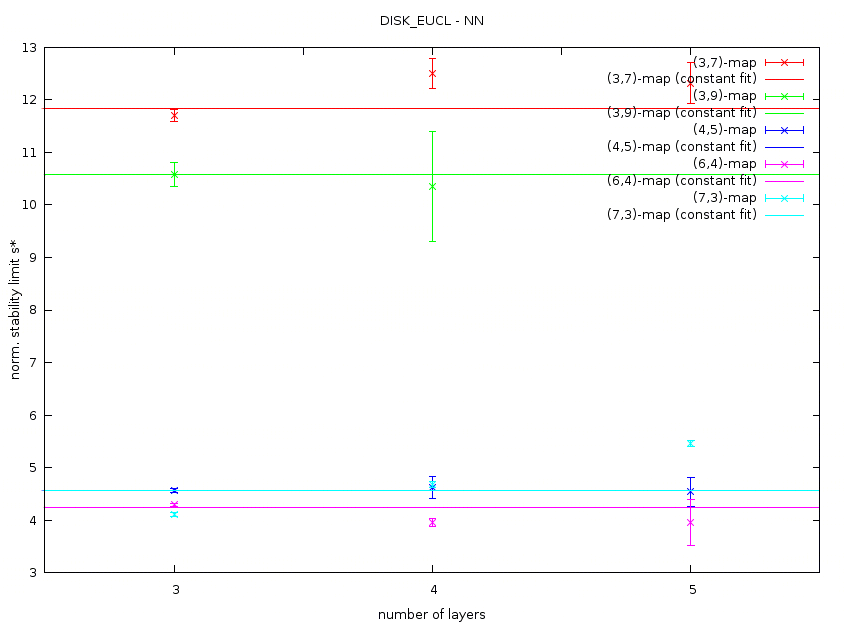}
\end{center}
\caption{Stability limits of the five Hyperbolic maps for the NN neighborhood function}
\label{fig:numerical:HSOM_NN}
\end{figure}

The (3,7)- and (3,9)-maps seem to have the highest stability and there is a gap between them and the others. It is thereby difficult to determine how large the influence of the deformations exactly is (see above). For the (3,6)-map with 3 layers, it may certainly be negligible as this map is mostly circular, but the others will be more effected, which may, for example, cause that the (3,9)-map has a slightly lower stability limit than the (3,6)-map or the lower limit of the (6,4)-map compared to the (7,3)-map.

\subsection{Long-ranged Gaussian neighborhood} \label{sec:numerical:HSOM_Gauss}

Now, we want to check if the HSOM shows similar dependencies on the size of $\sigma$ i.e. range of the Gaussian neighborhood as the classic SOM did above. We therefore determine again the quotient of the stability limit and $\sigma$. By normalizing again the result by the size of the edge, we obtain the results shown in Fig.\ref{fig:numerical:HSOM_gauss_ratio}.

\begin{figure}[!ht]
\begin{center}
\includegraphics[width=0.98\linewidth]{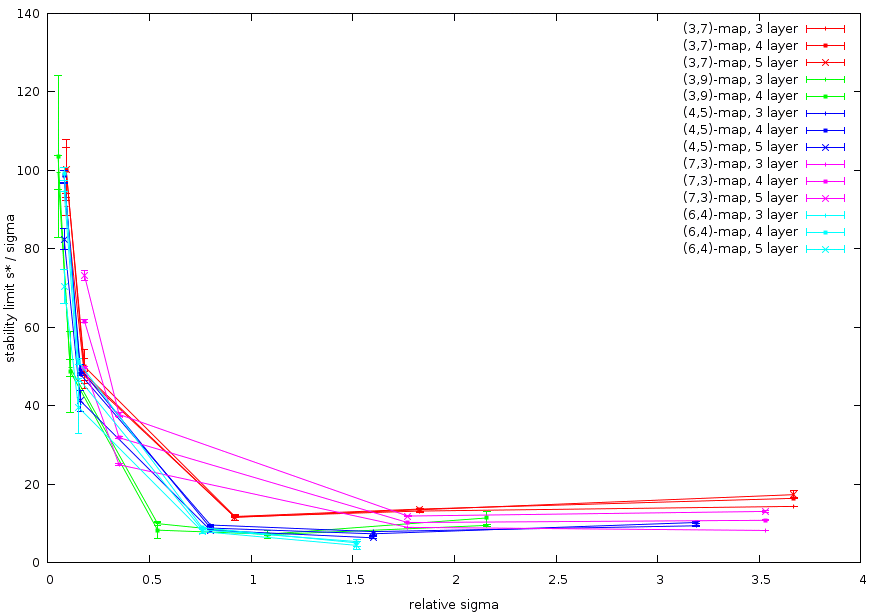}
\end{center}
\caption{Stability limits of the five Hyperbolic maps for the Gaussian neighborhood function}
\label{fig:numerical:HSOM_gauss_ratio}
\end{figure}

The similarity to the corresponding graph in the classic case (cf. Fig.\ref{fig:numerical:Gauss_final}) is quite obvious. We can again identify the three domains, i.e. the domain of very small $\tilde \sigma$, the one with $\tilde \sigma \gg 1$ and the third one in between. We want to take a brief look at the first domain.

\subsubsection{Case $\tilde \sigma \rightarrow 0$}

For the classic SOM, we had observed that in the limit of small $\sigma$ the stability does not depend on $\sigma$ i.e. was constant. To check this for the HSOM, we now have a look at the normalized stability limits for $\tilde \sigma < 0.5$ in Fig.\ref{fig:numerical:HSOM_gauss_small}.

\begin{figure}[!ht]
\begin{center}
\includegraphics[width=0.98\linewidth]{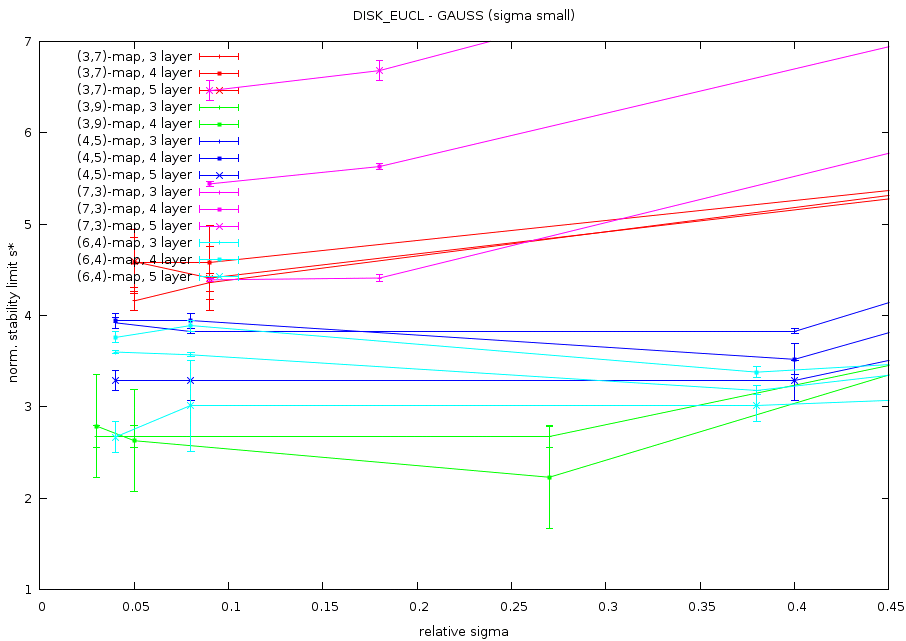}
\end{center}
\caption{Stability limits of the five Hyperbolic maps for the Gaussian neighborhood function with very small $\tilde \sigma$}
\label{fig:numerical:HSOM_gauss_small}
\end{figure}

It can be seen, that the stability limits are indeed nearly constant. Furthermore can we again notice the influence of the deformations as even the stability limits for different sizes of the same map differ significantly. While the (3,6)-map is again less effected, the split-up for the other maps is more intense. This also makes a further comparison of the different maps futile.







\FloatBarrier 

\section{Analysis of configuration with hyperbolic map and feature space (GRiSOM)}

At last, we are going to analyze a SOM with a non-Euclidean feature space. As mentioned in the introduction of this chapter, the adaption steps for this configuration are much more time-consuming as the program has to compute rather complex geodesics instead of simple vector differences. Thus, we only had the time to study one class of maps. We chose the (7,3)-map as it 
has the slowest growth rate. This allowed us despite the lack of time to find at least the stability limits for up to 5 layers and 4 different choices of neighborhoods.

\begin{table}[!ht]
\begin{center}
 \begin{tabular}{|l|c|c|c|} \hline
&\multicolumn{3}{c|}{GRiSOM case - (7,3)-maps} \\  
\# layers& 3& 4& 5 \\ \hline \hline
\# nodes& 22 & 40 & 70\\ \hline
\# meas./param.& 1500 & 1500 & 1500\\ \hline
\# adapt.steps&1.5 Mio&1.5 Mio&1.5 Mio\\ \hline
\end{tabular}
\end{center}
\caption{Choice of the size of the sample sets for the Hyperbolic-Hyperbolic GRiSOM}
\label{tab:numerical:sample_size_GRiSOM}
\end{table}

As there are still too few results to perform a meaningful quantitative analysis, we will just briefly focus on the most significant qualitative results, we can obtain from the stability limits listed in Tab.\ref{eqn:numerical:result_grisom}.

\begin{table}[!ht]
\begin{center}
\begin{tabular}{|l||c|c|c|} \hline
$d_{NN} \approx 0.566$&\multicolumn{3}{|c|}{(7,3)}\\ \hline
neighborhood & 3 layers & 4 layers & 5 layers \\ \hline \hline
NN&1.10\err{0.05}&1.15\err{0.05}&1.180\err{0.05}\\ \hline
Gauss $\sigma$=0.05&0.63\err{0.05}&0.61\err{0.02}&0.68\err{0.04}\\ \hline
Gauss $\sigma$=0.1&0.65\err{0.02}&0.62\err{0.02}& ---\\ \hline
Gauss $\sigma$=0.5&1.25\err{0.01}&1.34\err{0.01}& 1.44\err{0.05}\\ \hline
\end{tabular}
\end{center}
\caption{Results for (7,3)-hyperbolic map in Hyperbolic Poincare disk space}
\label{eqn:numerical:result_grisom}
\end{table}

First we notice, that the stability limit is no longer influenced by the number of layers as it has been for the HSOM. This is evident since the geometry of the input is the same as the map, but its dimension exceeds the map dimension. Thus the SOM face the same kind of dimensionality conflict as the classic SOM. Another similarity to the classic SOM is the observed height of the stability limit. If we normalize, for example, the stability limit for the NN neighborhood function in respect to the edge length of the map, we get approximately $s^*(\mathrm{normalized}) = 2.0$ which lies still relatively close to the corresponding limits of the classic case. Even more obvious are the similarities for the Gaussian neighborhood function. As the edge length is rather large in respect to $\sigma$, the relative range $\tilde \sigma$ is very small ($\approx 0.1$). The relative range for the two smaller choices of Gaussian neighborhoods is thereby lower than the limit of 0.35, below which we observed the constant stability limits in the classic case (cf. section \ref{sec:numerical:Eucl_Gauss}). The same constancy can be found for the Hyperbolic-Hyperbolic case, even if the stability exceeds the found classic ones by a factor of 2. As even $\sigma=0.5$ does not result in a large relative range, but are not able to study the other limit. This and as well the VQ case would be therefore an interesting subject for further future studies.

\section{Summary of numerical results}

To provide a final overview, the results, which have been obtained by evaluating the numerical simulations, have been listed in Tab.\ref{tab:numerical:result}.

\begin{table}[!ht]
\begin{tabular}{|l||c|c|c|c|}\hline
\multicolumn{5}{c}{Classic SOM}\\ \hline
map & Gauss (large $\sigma$) & Gauss (small $\sigma$) & NN & VQ \\ \hline \hline
(3,6) & \err{0.00}&0.703\err{0.008}&1.602\err{0.01}&0.690\err{0.005}\\ \hline
(4,4) & \err{0.00}&0.726\err{0.006}&1.540\err{0.01}&0.683\err{0.003}\\ \hline
(6,3) & \err{0.00}&0.826\err{0.003}&1.508\err{0.01}&0.803\err{0.010}\\ \hline
\multicolumn{5}{c}{Hyperbolic-Hyperbolic GRiSOM}\\ \hline
(7,3) &$\approx 1.13$ & --- & $\approx 2.01$ & --- \\ \hline
\end{tabular}
\caption{Results of the analytic stability analysis ($N_{outer}$=number of edge nodes)}
\label{tab:numerical:result}
\end{table}
For the classic SOM we could therefore verify most of the analytic results as most of the discrepancies between the numeric and corresponding analytic results could be explained. Only stability behavior for the Vector Quantization  still remain to be studied further(cf. Conclusion).\\
\\
The main problem of the evaluation of the HSOM was the underlying effect of the deformations caused by the choice of the truncation method. Nevertheless we could show that the stability limit depend on the number of edge nodes and that all maps exhibit the same stability behavior for the Gaussian neighborhood function that we have seen for the classic SOM before.\\
\\
The evaluation of the Hyperbolic-Hyperbolic GRiSOM finally does not provide much informations as we only had determined the stability limits for just one map, but we at least could deduce that the stability limit does not depend on the size of the map as one would expect since the map and feature space have the same hyperbolic geometry.

\FloatBarrier

\section{''travelling ruler problems``}

We will conclude this chapter about the numerical analysis by returning to the three ''travelling ruler problems`` which we have presented in the motivation. After defining and implementing the GRiSOM we have finally the means to solve all of them using the numerical simulations. We therefore use a map space that is homeomorphic to the $\mathbb S^1$. In this space we embed $N=50$ neurons regularly. The feature space is chosen such that it meets the requirements of the particular problem. Furthermore, we use a Gaussian neighborhood function and the adaption parameters following parameters in all three cases
\begin{equation*}
 \eps = 0.1 \qquad \sigma = 10 \cdot \mathrm{dist.neurons} \cdot 0.02^{\frac m {\mathrm{\#adapt.steps}}}
\end{equation*} 
where $m$ is the index of the current adaption step and $\mathrm{\#adapt.steps}=10000$ is the number of adaption steps. The large $\sigma$ in the beginning allows intense deformations of the initial path to form the rough shape of the final path and as the range of the Gaussian neighborhood function decreasing over time, the path becomes more and more detailed. The positions of the cities are thereby used as the input samples randomly picking at each step one of them. To give a better visualization, the results were finally plotted into the maps that we used to motivate these problems.

\subsubsection*{Emperor Frederick I Barbarossa}

In this first case  we work with an Euclidean feature space with no boundaries. As mentioned, the samples, that we use to train the SOM, are the positions of the cities are given by the coordinates in the map (i.e. the coordinates of the respective pixels in the image of the map) as shown in Tab.\ref{tab:numeric:medieval_cities}.

\begin{table}[!ht]
\begin{center}
 \begin{tabular}{|l|c|c|}
\hline 
&\multicolumn{2}{|c|}{position} \\  
city&x&y\\ \hline \hline
Frankfurt&201&-298\\ \hline
Roma&334&-708\\ \hline
Jerusalem&1100&-1095\\ \hline
Gallipoli&790&-772\\ \hline
Buda&540&-445\\ \hline
\end{tabular}
\end{center}
\caption{cities and their positions in the Barbarossa TRP}
\label{tab:numeric:medieval_cities}
\end{table}

\begin{figure}[!ht]
\begin{center}
\includegraphics[width=0.46\linewidth]{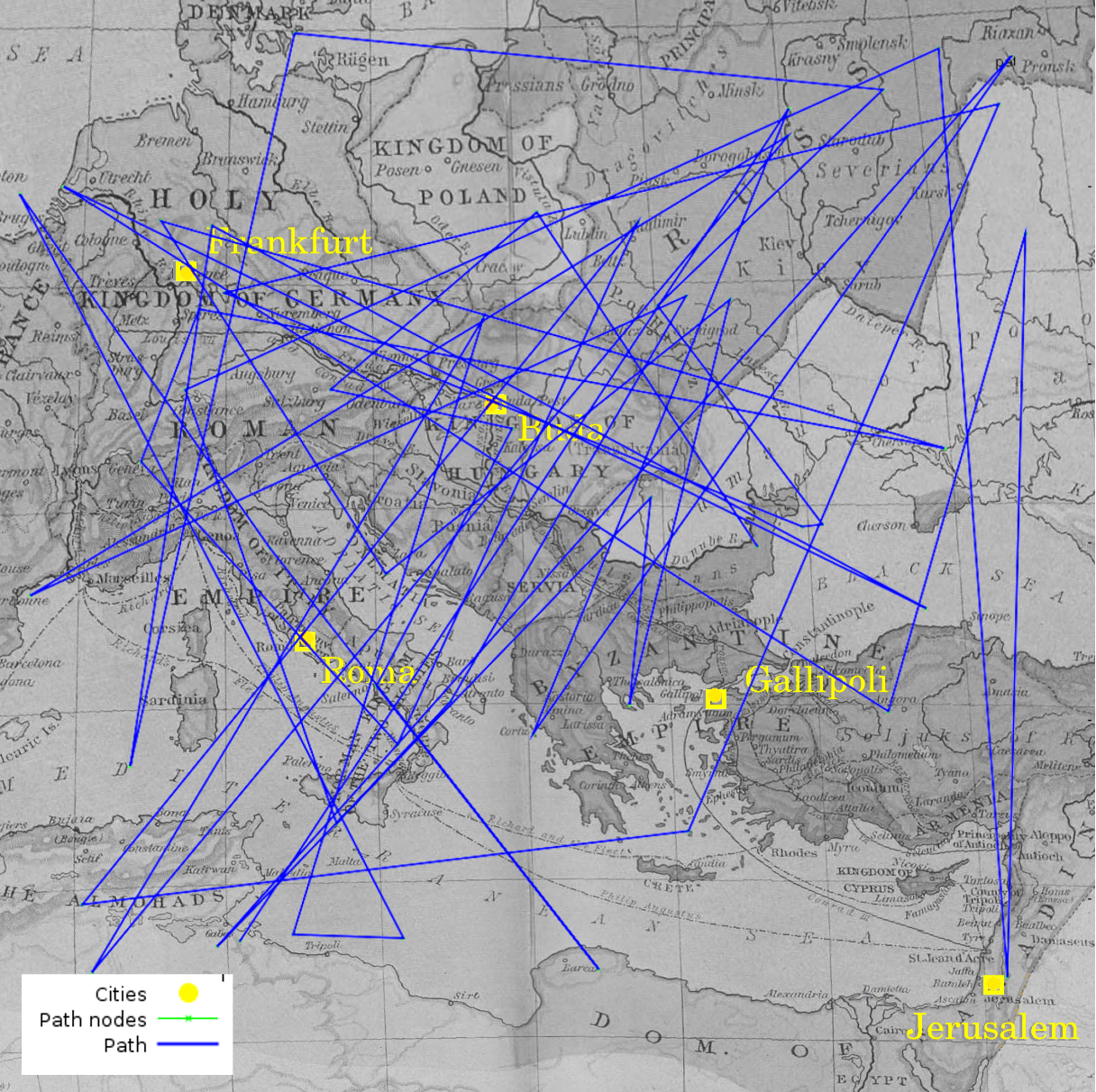}
\hspace{0.06\linewidth}
\includegraphics[width=0.46\linewidth]{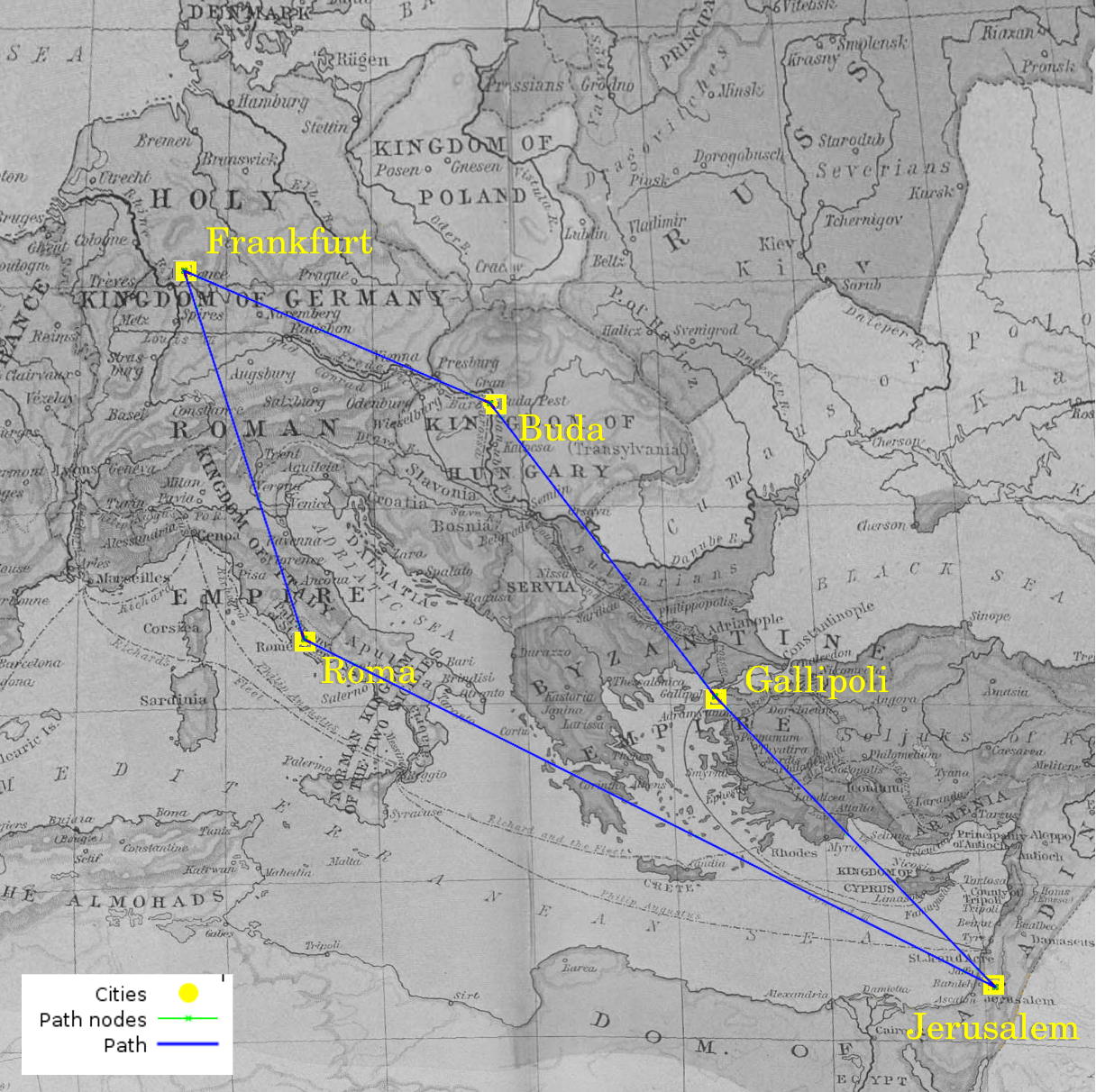}
\end{center}
\caption{Traveling ruler problem: (left) init (right) final path}
\end{figure}
\FloatBarrier
\subsubsection*{Pres. Barack Obama} 

For the case of the TRP of Pres. Obama,  we now have an non-Euclidean feature space i.e. sphere. Again, the samples, that we use to train the SOM, are the positions of the cities, which  are this time given by the geographic coordinates on earth as shown in Tab.\ref{tab:numeric:obama_cities}.

\begin{table}[!ht]
\begin{center}
 \begin{tabular}{|l|c|c|}
\hline 
&\multicolumn{2}{c|}{position} \\  
city&longitude&latitude\\ \hline \hline
Washington D.C.&38N&77W\\ \hline
Berlin&52N&13E\\ \hline
Jerusalem&31N&35E\\ \hline
Moskow&55N&37E\\ \hline
Nairobi&1S&36E\\ \hline
Beijing&39N&116E\\ \hline
\end{tabular}
\end{center}
\caption{cities and their positions in the Obama TRP}
\label{tab:numeric:obama_cities}
\end{table}

Using the GRiSOM with a Spherical feature space, gives us the shortest path shown in Fig.\ref{fig:numerical:trp_obama}.

\begin{figure}[!ht]
\begin{center}
\includegraphics[width=0.48\linewidth]{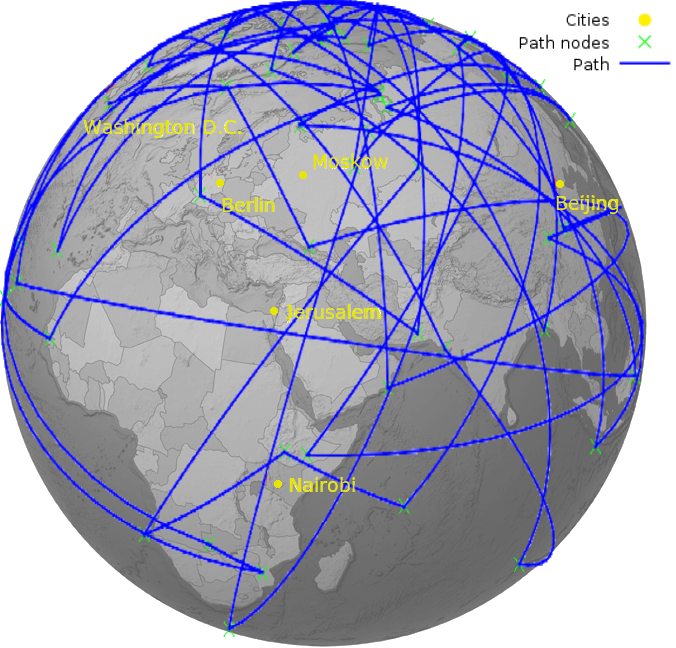}
\includegraphics[width=0.48\linewidth]{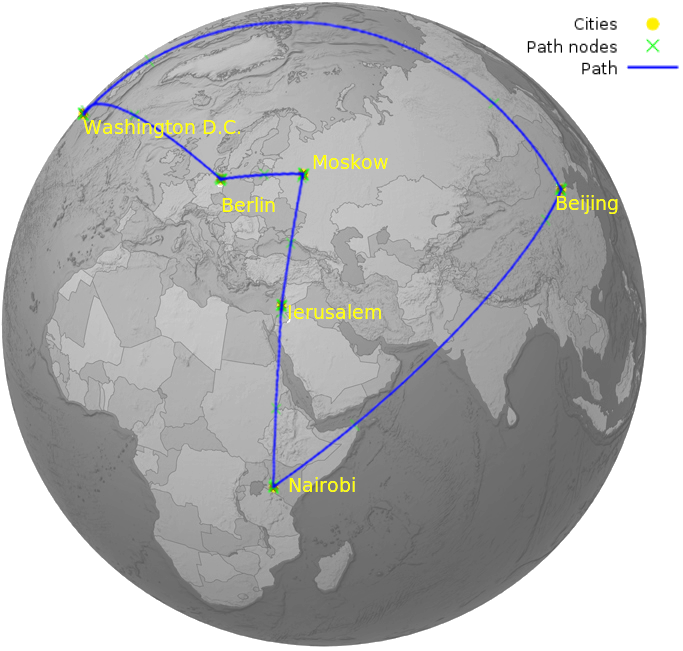}
\end{center}
\caption{Traveling ruler problem: (left) init (right) final path}
\label{fig:numerical:trp_obama}
\end{figure}

It can be noticed that all the neurons lie perfectly on the geodesics that connect the cities. The minimal flight distance needed to visit these 5 cities by following the linear map, is 35 232,3 km.
\FloatBarrier
\subsubsection*{President Kse'nu} 

The last of the three problems was situated in the hyperbolically-curved space, where we have to minimize the flight path of Kse'nu pan-galactic cruiser on its tour to visit several galaxies located at the (fictional) Poincare disk coordinates listed in Tab.\ref{tab:numeric:galaxies}.

\begin{table}[!ht]
\begin{center}
 \begin{tabular}{|l|c|c|}
\hline 
&\multicolumn{2}{|c|}{position} \\  
galaxy&x&y\\ \hline \hline
Milky Way&0&0\\ \hline
M31&-0.4&0.8\\ \hline
NGC 6822&0.1&0.65\\ \hline
Sombrero Galaxy&-0.8&-0.4\\ \hline
\end{tabular}
\end{center}
\caption{galaxies and their (fictional) positions in the Kse'nu TRP}
\label{tab:numeric:galaxies}
\end{table}

By training the GRiSOM for this data, we obtain the shortest route plotted in Fig.\ref{fig:numerical:galaxies}.

\begin{figure}[!ht]
\begin{center}
\includegraphics[width=0.48\linewidth]{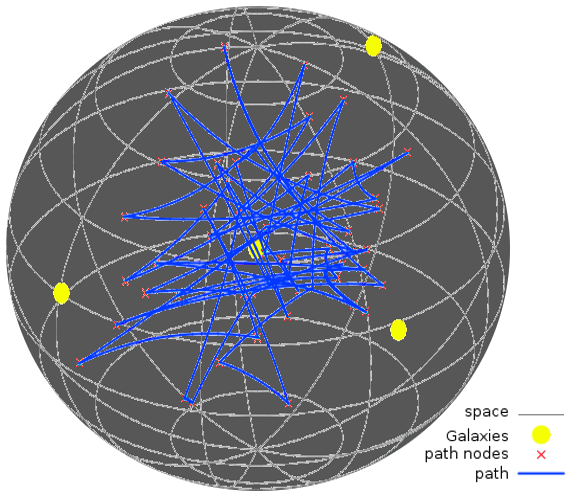}
\includegraphics[width=0.48\linewidth]{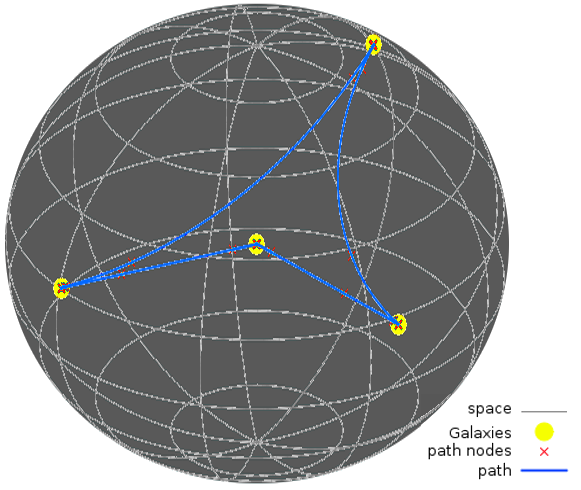}
\end{center}
\caption{Traveling ruler problem: (left) init (right) final path}
\label{fig:numerical:galaxies}
\end{figure}

Thus, we have also finally achieved our primary goal to solve the three travelling ruler problems despite the non-Euclidean feature spaces.

\part{Epilogue}
\chapter{Conclusion and Outlook}

Motivated by the three ``travelling ruler problems'', we have defined in this thesis the General Riemannian Self-Organizing Maps as a modification of the classic SOM in order to be able to use more general map and input spaces. In addition, we have provided a concrete implementation of this GRiSOM for at least Hyperbolic and Spherical map and feature spaces, which enabled us to finally solve the TRPs in their underlying spaces.\\
\\
Our second goal was also achieved as we studied both the triangular and hexagonal Euclidean maps and hence finally got an overall view on the stability properties of SOM with arbitrary regular Euclidean maps in regard to the three used classes of neighborhood functions. We were also able to obtain numerical results concerning the stability of non-Euclidean SOMs as we studied the HSOM and moreover got a glimpse of the stability behavior of a Hyperbolic-Hyperbolic GRiSOM configuration which was not possible with any other formerly defined SOM.\\
\\
All things considered, this thesis raises many new and interesting questions and subjects to be studied in future works. First of all, the gap between the numerical and analytic stability limit for the square and hexagonal Euclidean maps when using very small neighborhoods has to be further analyzed to find the specific reason for it. Secondly, it would be very interesting to extend the stability analysis of the HSOM and GRiSOM by using e.g. the distant-based map truncation to get rid of probable deformation effects or even perform an analytic approach. One of the most interesting subjects can the extension of this generalization of map and input spaces on related techniques like the vector quantization. First tests on spherical input spaces have, for example, shown that in the case of only few quantization vectors the representation errors is often smaller if we use the inherent space instead of mapping the space onto an Euclidean space and using then the classic Euclidean VQ technique. While this can be credited to the better suited competition process, the following short example shows, how VQ can also benefit from the modified update rule using moves along geodesics. 

\subsubsection*{``N-valley problem''}

We assume that we have given a landscape with $n$ concentric circular valleys, pairwise separated by high mountains. The metric is then given by the cost to travel there. While is quite easy (i.e. cheap) to reach any point in the same valley, it is quite expensive to visit any other point outside. To simplify the definition of the geodesics, we assume that the travelling costs inside a valley are thereby very small compared to the cost for travelling the mountains and that the valleys are also very narrow. Thus, the corresponding line element is approximated by
\begin{equation*}
 ds^2 \approx dr^2 + r^2 \left(1-\sum_i \delta(r-R_i)\right) d\Theta + r^2 \eps \sum_i \delta(r-R_i) d\Theta
\end{equation*}
where $r$ and $\Theta$ are the polar coordinates with the origin in the common center of the circular valleys, $R_i$ is the radius of the $i$th valley and $\eps$ the ratio between the cost of moving in the valley and in the mountains. The geodesics therefor exists and are (almost) everywhere unique\footnote{Only for points which are antipodal in the same valley or lie in different valleys the geodesics are not unique.}. They are simply combinations of segments of the circles forming the valleys and almost straight lines perpendicular to them. \\
Our goal is now, to quantize samples that are distributed only in the valleys. The top left plot in Fig.\ref{fig:numerical:VQtest} shows the initial positions of the quantization vector which have been randomly drawn from the area bounded by the outermost valley. First, we used the classic Euclidean VQ. The result can be found in the plot in the top right showing the positions of the quantization vector after a training phase with 10000 samples and the Voronoi tessellation in respect to the Euclidean metric. Although the result is stationary, it is not a good quantization for intrinsic structures of the valleys. Now, we modified the competition process to take the metric, we defined above into account. This leads to the result plotted in the bottom left of Fig.\ref{fig:numerical:VQtest} where the Voronoi cells are now defined by the metric of the valleys. While one outermost vector represents nearly all the valleys, the rest of them is located near the origin. This state is not even stationary for finite learning rates as it happens that the outermost vector, by following an Euclidean line, gets closer to the origin than the outermost of the rest of the vectors. In this case, the roles between both are swapped. Finally, both the competition and adaption process have been modified as described in for the GRiSOM, which will finally provide the desired quantization. It is note-worthy, that, given the same number of vectors as valleys initialized as described above and a sufficient long training phase (or large enough learn rate), this quantization is always obtained. If a vector quantize more than one ring at a given state, it is certain due to the ergodicity of the stochastic VQ process, that the vector will be at one point in time further away from the innermost valley represented by it than a vector lying closer to the origin. This will allow the latter one to take over this inner valley. Thus the vectors ``trickle'' towards
the outer valleys until each of them represents one valley.
\begin{figure}[!ht]
\begin{center}
\includegraphics[width=0.48\linewidth]{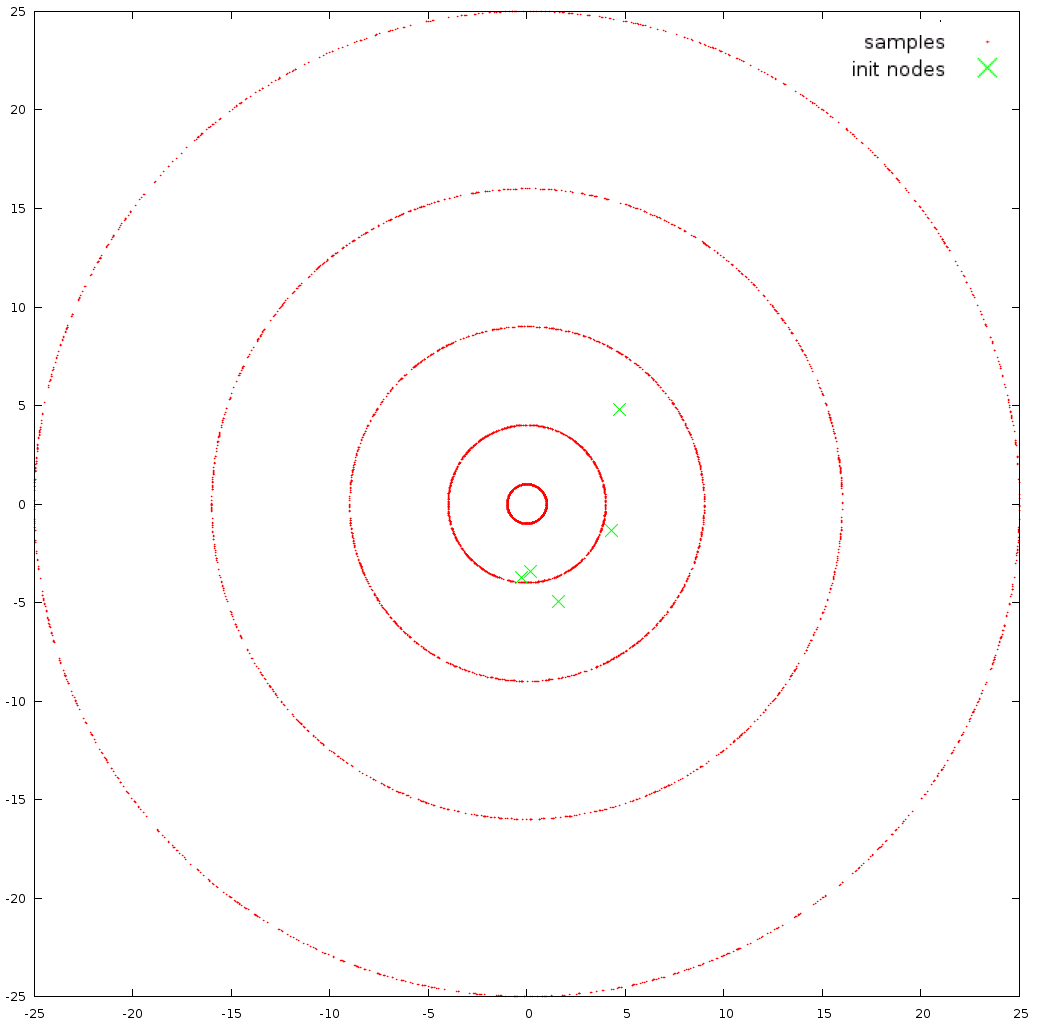}
\includegraphics[width=0.48\linewidth]{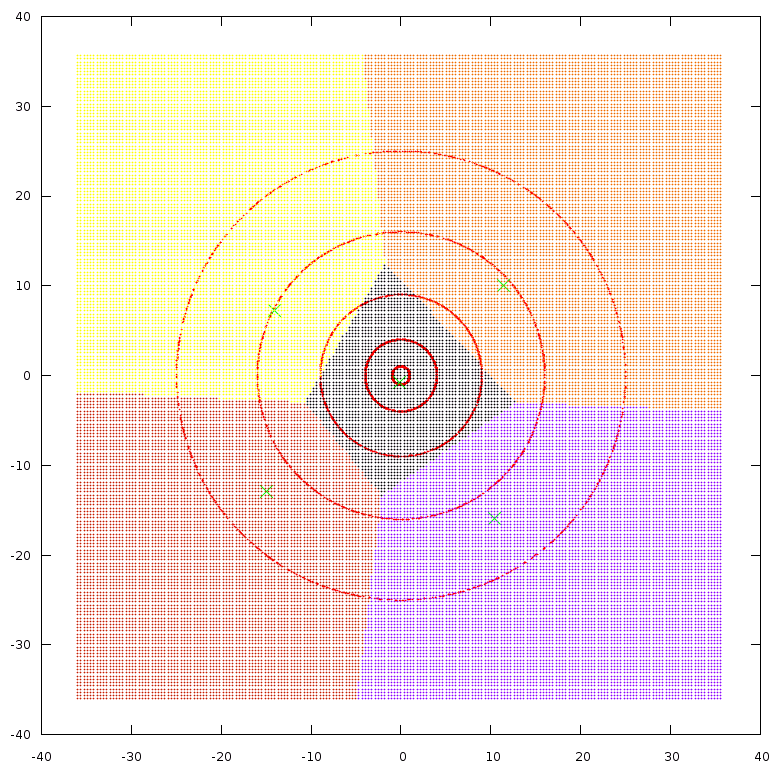}\\
\includegraphics[width=0.48\linewidth]{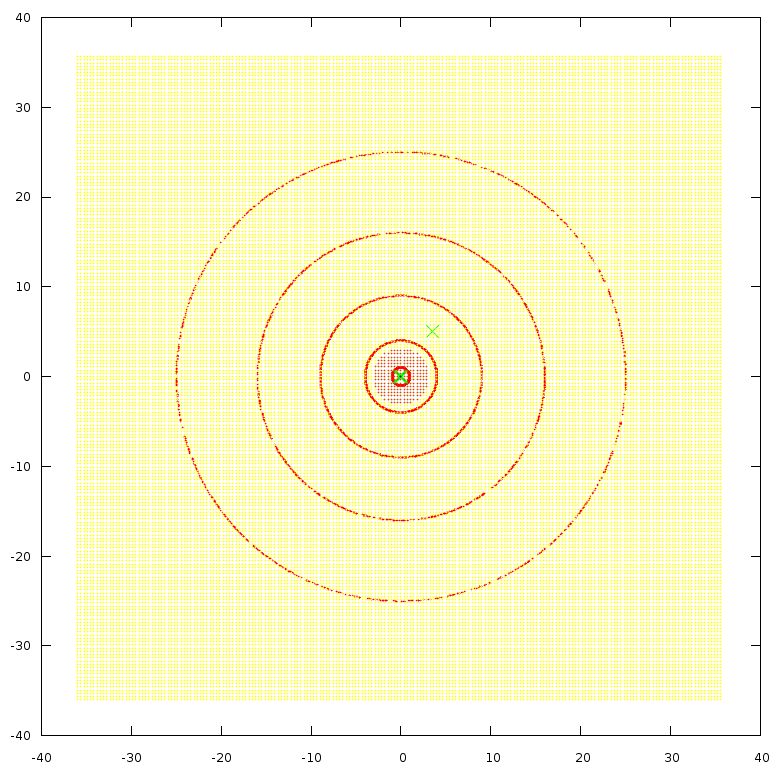}
\includegraphics[width=0.48\linewidth]{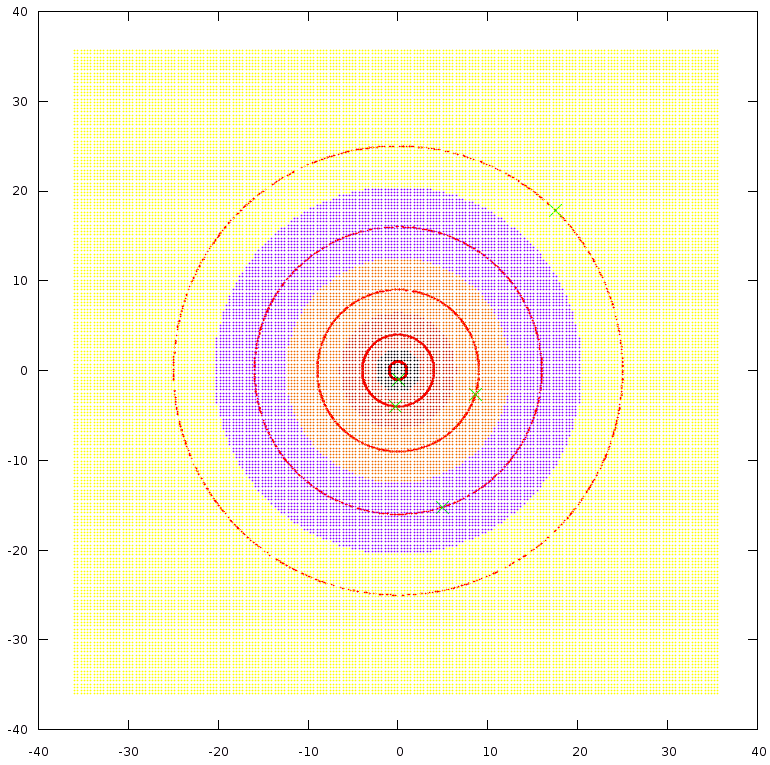}
\end{center}
\caption{VQ in ``N-valley problem'' (top left) initial state (top right) classic VQ (bottom left) VQ with modified competition and classic update process (bottom right) VQ with modified competition and update process}
\label{fig:numerical:VQtest}
\end{figure}
\\ \\
Even if the conditions that have specified are very particular, this example serves well to show the possibilities that are inherent to our proposed modification and may motivate future work on this field of study.

\tocsection{Bibliography}
\nocite{*}
\bibliographystyle{alpha}
\bibliography{diplom}

\tocsection{List of Figures}
\listoffigures

\part{Appendix}
\appendix

\chapter{Calculations/Proofs}

This chapter contains the most important (longish) calculation that have been skipped or shortened in chapter \ref{ch:analytic} due to lack of space. We start with closing the gap in the derivation of the Fokker-Planck equation, and list then derivations and geometric calculations which had been skipped in chapter \ref{ch:analytic}, but are not too obvious. The proofs of the theorems in chapter \ref{ch:distr} are finally concluding this chapter.

\section{Derivation of the Fokker-Planck-Equation}

By inserting in the \emph{Chapman-Kolmogorov equation} the transition probability (Eq.\ref{eqn:stability:trans_prob}) we get
\begin{equation*}
\tilde S(w,t+1) = \int  d^N w' Q(w,w') \tilde S(w',t) = \sum_r \int d^Nw  \int_{F_r(w_r)}  dv P(v) \delta(w-T(w',v,\eps))\tilde S(w',t) \label{eqn:stability:10}
\end{equation*}
To simplify the integral we'll now substitute $w:= T(w',v,\eps)$. Therefore we need the reciprocal Jacobian determinant
 \begin{eqnarray*}
\pdiff {T_{rm}}{w_{r'n}} & = & \pdiff {(w_{rm} + \eps \ho (v_m - w_{rm})}{w_{r'n}} = \delta_{rr'}(1- \eps  \ho) \\
\Rightarrow J(\eps) &=& \left[ \det (\delta_{rr'} ( 1- \eps \ho)^d) \right]^{-1} \\
&\ind{\mathrm{diagonal\ matrix}}& \left[ \prod_r (1- \eps \ho) \right]^{-d} 
\end{eqnarray*} 
where $s$ is the winning neuron for sample $v$ in the state $w'$. Thus we can now easily perform the integration by $w$ and resubstitute $w$.
\begin{eqnarray*}
             \tilde S(w, t+1) & \ind{10}&  \sum_r \int J(\eps) dz  \int_{F_r(T^{-1}(z,v,\eps)} dv P(v) \delta(w-z) S(T^{-1}(z,v,\eps)) \quad \footnotemark  \\
	&\ind{\delta fct.}& \sum_r J(\eps) \underbrace{\int_{F_r(T^{-1}(w,v,\eps))} dv}_{\int dv \chi_r(T^{-1}(w,v,\eps),v)} P(v) S(T^{-1}(w,v,\eps))  
             \end{eqnarray*} \footnotetext{Subst. $T(w',v,\eps)=z \Rightarrow dw' = dz J(\eps)$}
where $\chi_r(w,v)$ is the characteristic function of the Voronoi cell of neuron $r$:
\begin{equation*}
 \chi_r(w,v) = \left\{\begin{array}{l@{,\quad}l} 1 & \mathrm{if\ } v \in F_r(w) \\  0 & \mathrm{else\ } \end{array} \right.
\end{equation*}
Thus we obtain
\begin{eqnarray*} 
\Rightarrow w_r &=& \left[ T(w',v,\eps) \right]_r = w'_r + \eps \ho (v-w'_r) = w'_r (1- \eps \ho) + \eps \ho v\\
\Rightarrow w'_r &=& \left[ T^{-1}(w,v,\eps)\right]_r = \frac{w_r -  \eps \ho v}{1- \eps \ho} = \frac {w_r}{1- \eps \ho} - \eps \underbrace{\frac {\ho }{1-\eps \ho}}_{=:h_{rs}} v \approx w_r  - \eps h_{rs} v
 \end{eqnarray*}
where we have introduced  the function $h_{rs} := \frac {\ho }{1-\eps \ho}$ which difference to $\ho$ is only of an order of \eps.
Since $ \eps$ is small, we can expand $\tilde S(w',t)$ and likewise $J(\eps)$ and obtain
\begin{eqnarray*}
\tilde S(T^{-1}(w,v,\eps),t) &\stackrel{\mathrm{Taylor\ at\ w}}=& \tilde S(w,t) + \sum_{r,m} (\left[T^{-1}(w,v,\eps)\right]_{rm} - w_{rm} ) \pdiff{\tilde S}{w_{rm}}(w,t) \\ 
&+& \frac 12 \sum_{rm} \sum_{r'n} (\left[T^{-1}(w,v,\eps)\right]_{rm} - w_{rm} )(\left[T^{-1}(w,v,\eps)\right]_{r'n}\\ & -& w_{r'n} )\frac{\partial^2 \tilde S}{\partial w_{rm} \partial w_{r'n}}(w,t) + \Landau(\eps^3) \\ 
&=& \tilde S(w,t) + \sum_{rm} (w_{rm} + \eps h_{rs} (w_{rm}- v_m) - w_{rm}) \pdiff{\tilde S}{w_{rm}}(w,t)\\ &+& \frac 12 \sum_{rm} \sum_{r'n} \eps^2 h_{rs}(w_{rm}-v_m)h_{r's}(w_{r'n} - v_n) \frac{\partial^2 \tilde S}{\partial w_{rm} \partial w_{r'n}}(w,t) + \Landau(\eps^3)  
\end{eqnarray*}
and
\begin{eqnarray*}
 J(\eps) &=& 1 + \eps \pdiff{J}{\eps}|_{\eps=0} + \Landau(\eps^2) \qquad (\mathrm{Taylor\ of\ \eps\ at\ 0}) \\
&=& 1 + \eps \frac {(-d)\sum_r (-\ho)\sum_{r' \neq r}(1-\tilde \eps \hop)}{\left[\prod_r (1- \tilde \eps \ho)\right]^{d+1}}|_{\tilde \eps =0} + \Landau(\eps^2) \\
&=& 1 + \eps d \sum_r \underbrace{\frac \ho {(1- \tilde \eps \ho)}}_{h_{rs}} \frac{\prod_{r' \neq r}  (1- \tilde \eps \hop)}{\prod_{r' \neq r}  (1- \tilde \eps \hop)(\prod_{t} (1- \tilde \eps h_{ts}^0)^d  }|_{\tilde \eps =0} + \Landau(\eps^2) \\
& = & 1 + \eps d \sum_r h_{rs} + \Landau(\eps^2) \ind{\mathrm{transl.inv.}} 1+ \eps \underbrace{d \sum_r h_{r0}}_{=:J_1} + \Landau(\eps^2) 
\end{eqnarray*}
$h_{rs}$ could be replaced by $h_{r0}$ since we sum in $J_1$ over the whole set of neurons.
Combining these result with the result of the integral, we get
\begin{eqnarray}
 && \frac 1 \eps ( \tilde S(w, t+1) - \tilde S(w,t))\ind{13} \frac 1 \eps  \left[\sum_r J(\eps) \int dv \chi_r(T^{-1}(w,v,\eps),v) P(v) S(T^{-1}(w,v,\eps)) \right. \\ && \left. - \tilde S (w,t) \right)\nonumber  \\ 
&=& \frac 1 \eps  \left(\sum_r (1+\eps J_1 + \Landau(\eps^2)) \int dv \chi_r(T^{-1}(w,v,\eps),v) P(v) (\tilde S(w,t)  \right.  \nonumber \\  
&+& \left. \sum_{rm} (w_{rm} + \eps h_{rs} (w_{rm}- v_m) - w_{rm}) \pdiff{\tilde S}{w_{rm}}(w,t) \right. \nonumber \\ 
&+&  \left. \frac 12 \sum_{rm} \sum_{r'n} \eps^2 h_{rs}(w_{rm}-v_m)h_{r's}(w_{r'n} - v_n) \frac{\partial^2 \tilde S}{\partial w_{rm} \partial w_{r'n}}(w,t) + \Landau(\eps^3)) - \tilde S (w,t) \right] \nonumber \\
& =& \overbrace{\frac 1 \eps ((\underbrace{\sum_s \int dv  \chi_r(T^{-1}(w,v,\eps),v) P(v)}_{=1}) \tilde S(w,t) - \tilde S(w,t)) }^{=0} \nonumber \\
&+& \sum_s \int_{F_s(w)} dv P(v) \sum_{rm} h_{rs} (w_{rm} - v_m) \pdiff{\tilde S}{w_{rm}}(w,t) + J_1 \underbrace{\sum_s \int_{F_s(w)} dv P(v)}_{=1} \nonumber \\
&+& \frac \eps 2 \sum_s \int_{F_s(w)} dv P(v)  \sum_{rm} \sum_{r'n} \eps^2 h_{rs}(w_{rm}-v_m)h_{r's}(w_{r'n} - v_n) \frac{\partial^2 \tilde S}{\partial w_{rm} \partial w_{r'n}}(w,t) \nonumber \\ 
&+& \Landau(\eps^3) + \frac \eps 2 J_2 \tilde S(w,t)     \label{eqn:stability:19}
\end{eqnarray}
If we are in or rather very close to a stationary state, the probability to leave it should be zero i.e. $\tilde S(w,t)$ is peaked around a $\bar w$ which is chosen such that
\begin{equation*}
\int dv P(v) T(\bar w, v, \eps)- \bar w = 0 
\end{equation*}
Thus, we define 
\begin{equation*}
S(u,t) := \tilde S (\bar w +u,t)
\end{equation*}
For the following it is convenient to have following quantities defined:
\begin{equation}
V_{rm}(w) := \sum_s (w_{rm} - \bar v_{sm}) h_{rs} \hat P_s (v) \Rightarrow V_r(w) = \sum_s (w_r - \bar v_s ) h_{rs} \hat P_s(w) 
\end{equation}
\begin{eqnarray*}
&&D_{rmr'n}(w) := \sum_s h_{rs} h_{r's} \left[ (w_{rm} - \bar v_{sm}) (w_{r'n} - \bar v_{sn}) \hat P_s(w)+ \int_{F_s(w)} (v_m -v_n - \bar v_{sm} \bar v_{sn} P(v) dv \right] \\
&\Rightarrow& D_{rr'}(w) = \sum_s h_{rs} h_{r's} \left[ (w_r-\bar v_s)(w_{r'} - \bar v_s)^T \hat P_s(w) + \int_{F_s(w)} (v v^T - \bar v_s \bar v_s^T) P(v) dv \right] 
\end{eqnarray*}
 In the limit of small \eps, we may obtain then a first Fokker-Planck equation for our problem:
\begin{eqnarray*}
&& S(u,t+1) - S(u,t) \stackrel{\delta t\rightarrow 0} \rightarrow \partial_t S(u,t) \\
&\stackrel{(\ref{eqn:stability:19})}\Rightarrow& \frac 1 \eps S(u,t) = J_1 S(u,t) + \sum_{rm} v_{rm}(\bar w + u) \pdiff {S(u,t)}{u_{rm}} \\
&& + \frac \eps 2 \sum_{rmr'n} D_{rmr'n}(\bar w) \frac {\partial^2 S(u,t)}{\partial u_{rm} \partial u_{r'n}} + \frac \eps 2 J_2 S(u,t)
\end{eqnarray*}
We are now going to find more convenient notation of it by finding other expression for the first two terms on the RHS:\\
\\
As the first order term represents the restoring force, is has to vanish at $u=0$ which yields
\begin{eqnarray*}
\sum_{rm} V_{rm}(\bar w +u) \pdiff{S(u,t)}{u_{rm}} &\ind{\mathrm{Taylor}}& \underbrace{-\sum_{rm} \pdiff{V_{rm}}{w_{rm}} S(u,t) +\sum_{rmr'n} \pdiff{V_{rm}}{w_{rm}} S(u,t) \delta_{rr'}\delta_{mn}}_{=0}\\ &+& \sum_{rm}(\underbrace{V_{rm}(\bar w)}_{=0\footnotemark} + \sum_{r'n} \pdiff{V_{rm}}{w_{r'n}}(\bar w) u_{r'n} + \Landau(u^2) ) \pdiff{S(u,t)}{u_{rm}}\\
&\ind{\mathrm{chain rule}}&  -\sum_{rm} \pdiff{V_{rm}}{w_{rm}} S(u,t)  + \sum_{rmr'n} \pdiff{}{u_{rm}} \left(\pdiff{V_{rm}}{w_{r'n}}(\bar w) u_{r'n} S(u,t) \right) 
\end{eqnarray*} \footnotetext{Indeed, the restoring force vanishes in the equilibrium state}

if $\eps \ll 1$ 
\begin{eqnarray*}
V_r(w) &\ind{24}& \sum_s (w_r - \bar v_s) h_{rs} \hat P_s(w) = \sum_s (w_r - \frac 1 {\hat P_s(w)} ) \int_{F_s(w)} dv P(v)v) h_{rs} \\&=& \sum_s (\hat P_s(w) w_r - \int_{F_s(w)} dv P(v) v) h_{rs} =\sum_s \int_{F_s(w)} dv P(v)(w_r - v) h_{rs} \\
T(w,v,\eps)_r &=& w_r + \eps \ho (w_r-v)\\
\Leftrightarrow h_{rs}(w-v) &=& \frac {w_r - T(w,v,\eps)_r}{\eps (1-\eps \ho)} \stackrel{\eps \ll 1}\approx \frac 1 \eps (w_r - T(w,v,\eps)_r)\\
\Rightarrow V_r(w) & =& \sum_s \int_{F_s(w)} dv P(v) \frac 1 \eps (w_r - T(w,v,\eps)_r) = \frac 1 \eps \int dv P(v) (w_r - T(w,v,\eps)_r) 
 \end{eqnarray*}
Thus we obtain:\\
 $\sum_{rm} \pdiff{V_{rm}}{w_{rm}} \ind{28} \frac 1 \eps \sum_{rm} \int dv P(v) (1- \pdiff{T_{rm}}{w_{rm}}) = \frac 1 \eps \int dv P(v) \underbrace{\sum_{rm} (\delta_{rm} - \pdiff{T_{rm}}{w_{rm}})}_{=\Tr(1 - \pdiff{T}{w})} $\\ 
And by using the fact that $det(1 + X) \approx 1 +X$ if $X$ is small, we get:\\
$$\Rightarrow J(\eps) = (1- \eps A) + \Landau(\eps^2) = 1- \eps Tr A + \Landau(\eps^2) $$
Comparison with coefficients. of Eq.\ref{eqn:stability:19} yields\\
 $J(\eps) = 1+ \eps J_1 + \Landau(\eps^2) $\\ 
$\Rightarrow J_1 = - \Tr A = - \frac 1 \eps  \Tr(\eps A) = \frac 1 \eps \Tr(1-\pdiff{T}w)$ 
$\stackrel{(\ref{eqn:stability:19})}\Rightarrow \sum_{rm} \pdiff{V_{rm}}{w_{rm}} = \frac 1 \eps \underbrace { \int dv P(v)}_{=1} \eps J_1 = J_1$ \\
 Thus we obtain the final form of our version of the Fokker-Planck Equation:
\begin{eqnarray*}
&& \Rightarrow \frac 1 \eps \partial_t S(u,t) = J_1 S(u,t) + \left[- \sum_{rm} \pdiff{V_{rm}}{w_{rm}} S(u,t) \right.\\
&& \left. + \sum_{rmr'n} \pdiff{}{u_{rm}} \left( \pdiff{V_{rm}}{w_{r'n}}(\bar w) u_{r'n} S(u,t) \right) \right] \\
&& + \frac \eps 2 \sum_{rmr'n} D_{rmr'n}(\bar w) \frac {\partial^2 S(u,t)}{\partial u_{rm} \partial u_{r'n}} \\
&&= \underbrace{J_1 S -J_1 S}_{=0} + \sum_{rmr'n} \pdiff{}{u_{rm}} \left(\underbrace{\pdiff{V_{rm}}{w_{r'n}}(\bar w)}_{=: B_{rmr'n}} u_{r'n} S(u,t) \right) \\&+& \frac \eps 2 \sum_{rmr'n} D_{rmr'n}(\bar w) \frac {\partial^2 S(u,t)}{\partial u_{rm} \partial u_{r'n}} 
\end{eqnarray*}

\section{Various calculations}

\begin{proof}[Proof of Eq. \ref{eqn:stability:36}] Initial cond.: $ S(u,0) = \delta^N (u-u_0)$ \\
\begin{eqnarray*}
&& C_{rmsn}(0) = \int du S(u,0) (u_{rm}(0)- \bar u_{rm}(0))(u_{sn}(0)- \bar u_{sn}(0)) \\
&& = (u_{rm,0}- u_{rm,0})(u_{sn,0}- u_{sn,0})=0\\
&\Rightarrow& \partial_t \bar u_{sl}(t) = -\eps \sum_{r'n} B_{slr'n} \bar u_{r'n} \Rightarrow \partial_t \bar u(t) = -B \cdot \bar u(t) 
\end{eqnarray*}
The solution of this differential equation is given by:
\begin{equation*}
\Rightarrow \bar u_0 \exp\left( -B \int_0^t \eps(\tau) d\tau \right) = Y(t) \bar u_0 
\end{equation*}
The 2nd moment can then be deducted by:
\begin{eqnarray*}
&\Rightarrow& \partial_t \expval{u_{sl}(t)\cdot u_{s'l'}} = -\eps \left( \sum_{rm} B_{slrm}  \expval{u_{rm}(t)\cdot u_{s'l'}} + \sum_{r'n} B_{s'l'r'n}  \expval{u_{sl}(t)\cdot u_{r'n}} \right) +  \eps^2 D_{sls'l'} \\
 &\Rightarrow& \partial_t C_{sls'l'} = \partial_t \left( \expval{u_{sl} \cdot u_{s'l'} } - \expval{u_{sl}} \expval{u_{s'l'}} \right) =  \partial_t \expval{u_{sl} \cdot u_{s'l'} } - (\partial_t \expval{u_{sl}}) \expval{u_{s'l'}} - \expval{u_{sl}}(\partial_t \expval{u_{s'l'}})\\
&&= -\eps \left( \sum_{rm} B_{slrm}  \expval{u_{rm}(t)\cdot u_{s'l'}} + \sum_{r'n} B_{s'l'r'n}  \expval{u_{sl}(t)\cdot u_{r'n}} \right) +  \eps^2 D_{sls'l'}\\
&& - (- \eps) \sum_{rm} B_{slrm} \bar u_{rm}  \bar u_{s'l'} - (-\eps) \sum_{r'n} B_{s'l'r'n} \bar u_{r'n}  \bar u_{sl} \\
&&=  -\eps \left( \sum_{rm} B_{slrm}  (\expval{u_{rm}(t)\cdot u_{s'l'}}- \expval{u_{rm}(t) \cdot u_{s'l'}})
 + \sum_{r'n} B_{s'l'r'n}  (\expval{u_{sl}(t)\cdot u_{r'n}}  -   \bar u_{sl} \bar u_{r'n} )\right) \\
&& +  \eps^2 D_{sls'l'} =  -\eps \left( \sum_{rm} B_{slrm} C_{rms'l'}  + \sum_{r'n}  C_{slr'n}  B^T_{r'ns'l'} \right) +  \eps^2 D_{sls'l'} \\
&&=  -\eps \left( (BC)_{sls'l'}  +  (CB^T)_{sls'l'} \right) +  \eps^2 D_{sls'l'} \\
&\Rightarrow& \partial_t C(t) = -\eps(t) (BC + CB^T) + \eps^2(t) D ()\\
&&C(t)=:S Y(t) C^*(t) Y^T(t) = e^{ -B \int_0^t \eps(\tau) d\tau} C^*(t) e^{ -B^T \int_0^t \eps(\tau) d\tau} \qquad \small{\mathrm {(interaction\ picture)}}
\end{eqnarray*}
\end{proof}

\begin{proof}[Proof of Eq. \ref{eqn:stability:37}]
\begin{eqnarray*}
&\stackrel{\ref{eqn:stability:36}}\Rightarrow&  \partial_t C(t) = -B \eps e^{ -B \int_0^t \eps(\tau) d\tau} C^*(t) e^{ -B^T \int_0^t \eps(\tau) d\tau} \\ &+&  e^{ -B \int_0^t \eps(\tau) d\tau } (\partial_t C^*(t)) e^{ -B^T \int_0^t \eps(\tau) d\tau} +  e^{ -B \int_0^t \eps(\tau) d\tau} C^*(t) (-B^T \eps) e^{ -B^T \int_0^t \eps(\tau) d\tau} \\
&\land&\  \partial_t C(t) =  -\eps(t) (B  (e^{ -B \int_0^t \eps(\tau) d\tau} C^*(t) e^{ -B^T \int_0^t \eps(\tau) d\tau}) \\&+& (e^{ -B \int_0^t \eps(\tau) d\tau } C^*(t) e^{ -B \int_0^t \eps(\tau) d\tau})  B^T) + \eps^2(t) D \\
&\Rightarrow&  e^{ -B \int_0^t \eps(\tau) d\tau} (\partial_t C^*(t)) e^{ -B^T \int_0^t \eps(\tau) d\tau} = \eps^2(t) D \\
&\Rightarrow&  \partial_t C^*(t) = \eps^2(t) e^{B \int_0^t \eps(\tau) d\tau} D  e^{B^T \int_0^t \eps(\tau) d\tau}\\
&\Rightarrow& C^*(t) = C^*(0) + \int_0^t d\tau \eps^2(\tau) e^{B \int_0^\tau \eps(\tilde \tau) d\tilde \tau} D  e^{B^T \int_0^\tau \eps(\tilde \tau) d\tilde \tau} \\
&\Rightarrow& C(t) = Y(t)\left(\underbrace{C^*(0)}_{=0} +\int_0^t d\tau \eps^2(\tau)  \underbrace{e^{B \int_0^\tau \eps(\tilde \tau) d\tilde \tau}}_{Y^{-1}(t)} D   \underbrace{e^{B^T \int_0^\tau \eps(\tilde \tau) d\tilde \tau}}_{(Y^T(t))^{-1}} \right) Y^T(t) 
\end{eqnarray*}
\end{proof}

\begin{proof}[Proof of Eq. \ref{eqn:stability:39}] now $\eps$ constant, $B,D$ (and therefore $D,Y$) commute:
\begin{eqnarray*}
\Rightarrow C &\ind{\ref{eqn:stability:37}}&  Y(t)\left(\int_0^t d\tau \eps^2 e^{B \int_0^\tau \eps(\tilde \tau) d\tilde \tau} D  e^{B^T \int_0^\tau \eps(\tilde \tau) d\tilde \tau} \right) Y^T(t)\\
&=& \eps^2 Y(t) \int_0^t d\tau  e^{(B+B^T) \int_0^\tau \eps(\tilde \tau) d\tilde \tau}  Y^T(t) D \\
&=& \eps^2 Y(t) \frac {(B + B^T)^{-1}} \eps \underbrace{e^{(B+B^T) \int_0^t \eps(\tilde \tau) d\tilde \tau}}_{=Y^{-1}(t)(Y^T(t))^{-1}}  Y^T(t) D = \eps(B + B^T)^{-1} D 
\end{eqnarray*}
\end{proof}

\begin{proof}[Proof of Eq. \ref{eqn:stability:48}] L.H.S. of (\ref{eqn:stability:48}):
\begin{eqnarray*}
\int dv P(v) [T(\bar w, v, \eps]_r - \bar w_r &=& -\bar w_r +\underbrace{\int dv P(v)}_{=1} \bar w_r + \int dv P(v) \eps h^0_{rs} (v-\bar w_r) \\
&\stackrel{\eps \ll 0} \approx& \eps \sum_s h_{rs} \int_{F_s(\bar w)} dv P(v)(v-\bar w_r) \\ &=&  \eps \sum_s h_{rs} \int_{F_s(\bar w)} dv P(v)v - \underbrace{\int_{F_s(\bar w)} dv P(v)}_{=N^{-2}} \bar w_r 
\\& =&  \eps \sum_s h_{rs} N^{-2} (\bar w_s - \bar w_r) =\footnotemark 0 
\end{eqnarray*}\footnotetext{follows since $h_{rs}$ depends only on Distance between $r$ and $s$ and periodic bounding condition}
\end{proof}

\section{Geometric Aspects of regular Euclidean maps} \label{ch:geometric_aspects}

In this section we will briefly determine some geometric properties that we need in ch.\ref{ch:analytic}. 
Given a ``center node'' and its nearest neighbors of a triangular and hexagonal map\footnote{The analogous calculations for the square map was already covered by \cite{ritter_schulten}.}, we deduce the the shift of the centroids and the volume change of the corresponding voronoi cells i.e. feature sets for small deviations of the center node. Thus we obtain the matrices $a$ and $b$ which are are then used in section \ref{ch:analytic:regular_maps} to analytically pinpoint the stability limit.

\subsection{Calculating shift of centroids}


To determine the shift of the centroids, we move the center node in the $x$-,$y$- and $z$-direction of the embedding three-dimensional Euclidean space by the infinitesimal distance $d$. We then calculate the centroids of the voronoi cells in respect to the changed cell borders by integrating the slices of the cell perpendicular to the respective direction (with area $A(h)$) weighted by the coordinate $h$ of the slice in the examined dimension and then normalizing the result by the volume of the cell. As $d$ is very small, we restrict our calculations to the linear response i.e. we will ignore all the terms that contain powers of $d$ that are larger than 1.

\begin{figure}[!ht]
 \begin{center}
\includegraphics[height=0.3\linewidth]{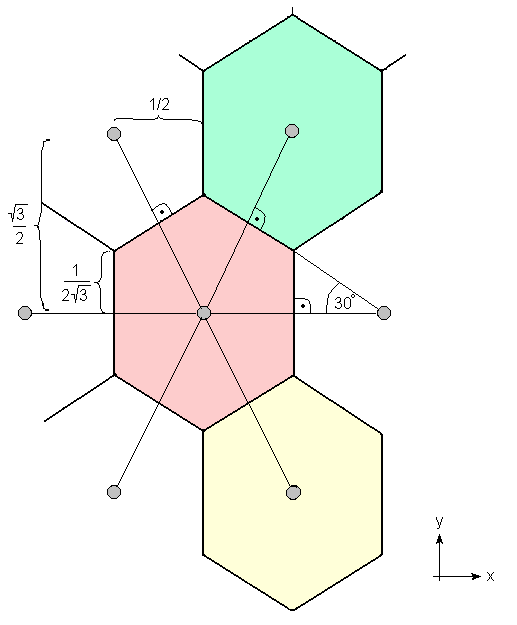}
\hspace{0.04\linewidth}
\includegraphics[height=0.3\linewidth]{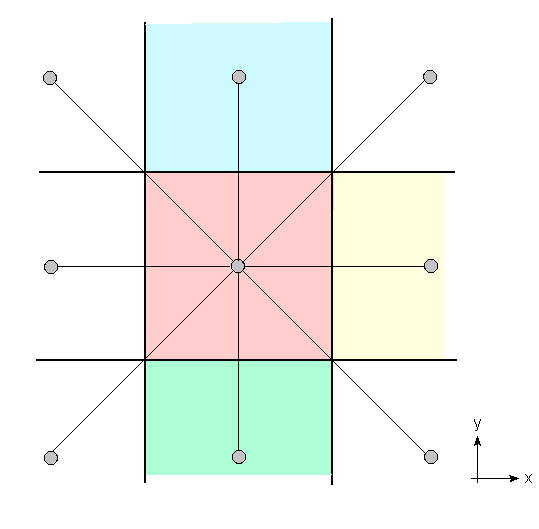}
\hspace{0.04\linewidth}
\includegraphics[height=0.3\linewidth]{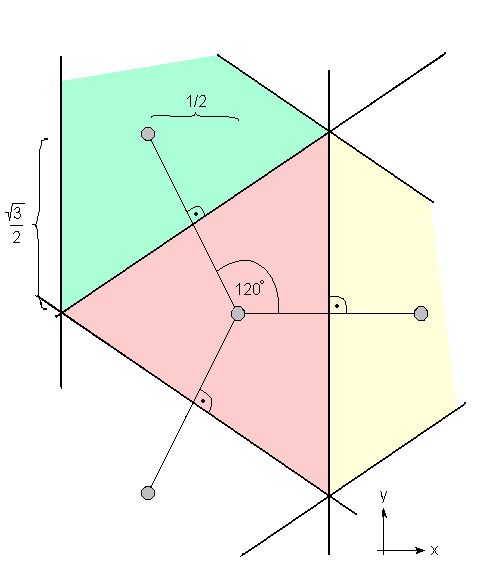}
\end{center}
\caption{center node and neighboring voronoi cells in the (left)(3,6)-map, (center)(4,4)-map and (right)(6,3)-map}
\label{fig:app_geometry:maps}
\end{figure}

\subsubsection{(3,6)-Tesselation}

As we have seen in chapter \ref{ch:tess}, each node of the (3,6)-map has six nearest neighbors. Fig.\ref{fig:app_geometry:maps} shows hereby the ``center node'' and its neighboring nodes.

\subsubsection*{x-direction}

The first case that we want to consider is when the center node is moved slightly in the x-direction (cf.Fig.\ref{fig:app_geometry:3_6_x}).

\begin{figure}[!ht]
 \begin{center}
\includegraphics[height=0.4\linewidth]{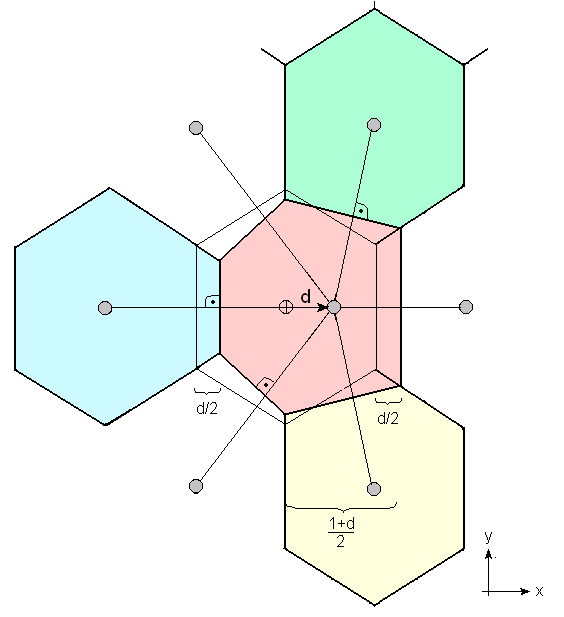}
\end{center}
\caption{Move of center node into x-direction}
\label{fig:app_geometry:3_6_x}
\end{figure}

Center in x-direction:\\
\begin{minipage}{0.2\linewidth}
\begin{center}
\includegraphics[width=\linewidth]{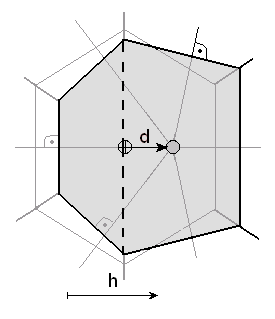}
\end{center}
\end{minipage}
\begin{minipage}{0.8\linewidth}
\begin{enumerate}[(1)]
\item 
$A(h) = \frac 2 {\sqrt 3} \left( \frac{1-d}2 \right) + 2 h \frac{1 + 2d}{\sqrt{3}} =  \frac {1-d}{\sqrt 3} + h \frac{2+4d} {\sqrt 3}$
\item 
$A(h) = \frac{1-d}{\sqrt 3} + \frac{1-d}{\sqrt 3} + \frac {1-d} 2 \frac  {2+4d}{\sqrt 3} - 2 (h - \frac {1-d} 2 ) \frac {1-2d}{\sqrt 3} $\\
$= \frac {1 -d +1 -d +2d}{\sqrt 3}  + \frac{(1-d)(1-2d)}{\sqrt 3} - \sqrt{2 - 4d}{\sqrt 3} h$ \\ 
$= \frac{2+1-d-2d}{\sqrt 3} - \frac{2 - 4d}{\sqrt 3}h = \frac {3 - 3d}{\sqrt 3} - \frac {2 - 4d}{\sqrt 3} h$
\end{enumerate}
\end{minipage}\\ \\
 \begin{eqnarray*}
\bar x  
&=& \frac {\int h A(h) dh} {\int A(h) dh} = \frac { \left[ (1-d) \frac {h^2} 2+ \frac 2 3 h^3 (1+2d)\right]_0^{\frac {1-d}2} + \left[ \frac {3 - 3d} 2 h^2 -  \frac 2 3 h^3 (1-2d) \right]_{\frac {1-d} 2}^1}
{ \left[ (1-d) h + h^2 (1 +2d) \right]_0^{\frac {1-d}2}+ \left[ (3-3d) h - h^2 (1-2d) \right]_{\frac {1-d} 2}^1 }\\
&=& \frac{ \frac {1-3d}8 + \frac 1 {12} (1-3d)(1+2d) + \frac{3-3d}{2} - \frac{2 - 4d}{3} - \frac 3 8 (1-3d) + \frac 1{12} (1-3d)(1-2d) +  \Landau{d^2} }{\frac {1 - 2d} 2 + \frac 1 4 (1-2d)(1+2d) + 3 - 3d - (1-2d) - \frac 32 (1-2d) + \frac 14 (1-2d)(1-2d) + \Landau{d^2}} \\
&=& \frac {\frac 34 + \frac 1 {12} d + \Landau{d^2}} {\frac 32 + \Landau{d^2}} = \frac 12 + \frac 1 {18} d + \Landau{d^2}
 \end{eqnarray*}
\begin{equation*}
 \Rightarrow \boxed{\Delta \bar x = \frac 5 9 d}
\end{equation*}

Center in y-direction:
\begin{equation*}
 \Rightarrow \boxed{\Delta \bar y = 0}
\end{equation*}

Lower right in x-direction:\\
\begin{minipage}{0.2\linewidth}
\begin{center}
\includegraphics[width=\linewidth]{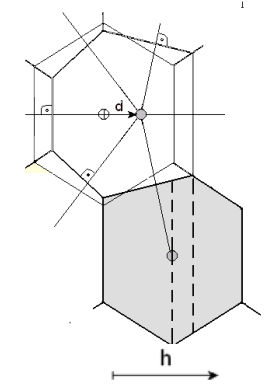}
\end{center}
\end{minipage}
\begin{minipage}{0.8\linewidth}
\begin{enumerate}[(1)]
 \item 
$A(h)=\frac 1{\sqrt 3} + h ( \frac {1 - 2d} {\sqrt 3} + \frac 1{\sqrt 3} ) = \frac 1 {\sqrt 3} + 2 h ( \frac {1-d}{\sqrt 3})$
\item 
$A(h)= \frac 1{\sqrt 3} + \frac {1-d}{\sqrt 3} + (h- \frac 12 )( \frac {1 - 2d}{\sqrt 3} - \frac 1{\sqrt 3} ) = \frac {2 - d+ d}{\sqrt 3} - \frac {2 hd}{\sqrt 3}$\\$ = \frac 2{\sqrt 3} - \frac {2 hd}{\sqrt 3}$
\item 
$A(h)=  \frac 2{\sqrt 3} - \frac {2d}{\sqrt 3}\frac {1+d} 2  - \frac 2 {\sqrt 3}(h-\frac {1+d} 2) = \frac { 2- d +1 +d}{\sqrt 3} - \frac 2 {\sqrt 3} h$\\$ ={\sqrt 3} - \frac 2{\sqrt 3} h$
\end{enumerate}
\end{minipage}\\ \\
\begin{eqnarray*}
 \bar x
&=& \frac {\int h A(h) dh} {\int A(h) dh} =\frac {\left[ \frac {h^2}{2\sqrt 3}  \frac {2 h^3} 3 \frac {1-d}{\sqrt 3} \right]_0^{\frac 12} + \left[ \frac {h^2}{\sqrt3} - \frac 2 3 h^3 \frac d {\sqrt 3} \right]^{1+\frac d2}_\frac 12 + \left[\frac {\sqrt 3} 2 h^2 - \frac 2 {3{\sqrt 3}} h^3 \right]_{\frac {1+d}2 }^1}
 { \left[ \frac h{\sqrt 3} + h^2 \frac {1-d} {\sqrt 3} \right]^{\frac 12}_0 + \left[ \frac {2h} {\sqrt 3} - \frac {h^2 d}{ \sqrt 3} \right]^{\frac {1+d}2}_{\frac 12} + \left[ {\sqrt 3} - \frac{ h^2}{\sqrt 3} \right]^1_{\frac {1+d}2}}\\
&=& \frac {\frac 1 8 + \frac{1-d} {12} + \frac d 2 + \frac  32 - \frac 23 - \frac 3 8 (1+2d) + \frac 12 (1+3d)}{\frac 1 2 + \frac {1-d} 4 + d  + 3 -1 -\frac{3 + 3d} 2 + \frac {1+2d}4} = \frac {\frac 34 - \frac d {12}} { \frac 32 - \frac d4} \\
&=& \frac  12 + \frac  1 {36} d
\end{eqnarray*}
\begin{equation*}
 \Rightarrow \boxed{\Delta \bar x = \frac 1 {36} d}
\end{equation*}

Lower right in y-direction:\\
\begin{minipage}{0.2\linewidth}
\begin{center}
\includegraphics[width=\linewidth]{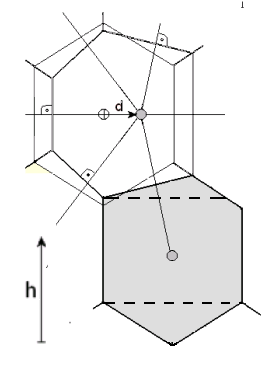}
\end{center}
\end{minipage}
\begin{minipage}{0.8\linewidth}
\begin{enumerate}[(1)]
 \item 
$A(h)= 2 h{\sqrt 3}$
\item 
$A(h)=1$
\item 
$A(h)= 1 - (h - \frac 3 {2{\sqrt 3}})(\sqrt 3 + \frac{\sqrt 3} {1-2d}) = 1 + \frac 32 (\frac {2 -2d}{1-2d}) -{\sqrt 3} h (\frac {2 -2 d}{1 -2d})$ \\
$= \frac {4-5d}{1-2d} -{\sqrt 3} h (\frac {2 -2 d}{1 -2d})$
\end{enumerate}
\end{minipage}\\ \\
\begin{eqnarray*}
 \bar y &=& \frac {\int h A(h) dh} {\int A(h) dh} = \frac {\left[ \frac {2h^3} {\sqrt 3} \right]_0^{\frac 1 {2 {\sqrt 3}}} + \left[\frac  {h^2} 2\right]_{\frac 1 {2 {\sqrt 3}}}^{\frac 3 {2 {\sqrt 3}}} + \left[ \frac {4-5d}{1-2d} \frac {h^2} 2 - \frac {h^3} {\sqrt 3} (\frac {2 -2 d}{1 -2d}) \right]_ {\frac 3 {2 {\sqrt 3}}}^{\frac {4-d} {2 {\sqrt 3}}} }
{\left[{\sqrt 3} h^2 \right]_0^{\frac 1 {2 \sqrt 3}} + \left[h\right]_{\frac 1 {2 \sqrt 3}}^{\frac 3 {2 \sqrt 3}} + \left[ \frac {4-5d}{1-2d} h - \frac {\sqrt 3 h^2} 2 (\frac {2 -2 d}{1 -2d}) \right]_ {\frac 3 {2 \sqrt 3}}^{\frac {4-d} {2\sqrt 3}} } \\ 
&=& \frac { \frac 2 {24} + \frac 9 {24} - \frac 1 {24} + \frac {(16-8)(4-5d)}{2 \cdot 12 (1-2d)} - \frac {(4-d)^2(2-2d)}{24 \cdot 3 (1-2d)} - \frac 9 {24} \frac {4 -5d}{1-2d} + \frac {27 (2-2d)}{24 \cdot 3 (1-2d)}}
{\frac 1 {4\sqrt 3} + \frac 1 {\sqrt 3}+ \frac {(4-d)(4-5d)}{2\sqrt 3 (1-2d)} - \frac{(16-8d)\sqrt 3 (2-2d)}{24 (1-2d)} - \frac 3 {2{\sqrt 3}} \frac {4-5d}{1-2d} + \frac {9{\sqrt 3}} {24} \frac {2-2d}{1-2d} } \\
&=& \frac {\frac 12 - \frac {83} {72} d} { \frac{\sqrt 3} 2 - \frac {13}{4\sqrt 3} d} = \sqrt 1{\sqrt 3} - \frac 1 {9{\sqrt 3}} d
\end{eqnarray*}
\begin{equation*}
 \Rightarrow \boxed{\Delta \bar y = \frac 1 {9\sqrt 3} d}
\end{equation*}

Analogously:

Lower left:
\begin{equation*}
 \Rightarrow \boxed{\Delta \bar x = \frac 1 {36} d}
\land \boxed{\Delta \bar y = + \frac 1 {9\sqrt 3} d}
\end{equation*}

Upper left:
\begin{equation*}
 \Rightarrow \boxed{\Delta \bar x =  \frac 1 {36} d}
 \land \boxed{\Delta \bar y = - \frac 1 {9\sqrt 3} d}
\end{equation*}

Upper right:
\begin{equation*}
 \Rightarrow \boxed{\Delta \bar x =  \frac 1 {36} d}
\land  \boxed{\Delta \bar y = + \frac 1 {9\sqrt 3} d}
\end{equation*}

Left:\\
\begin{minipage}{0.25\linewidth}
\begin{center}
\includegraphics[width=\linewidth]{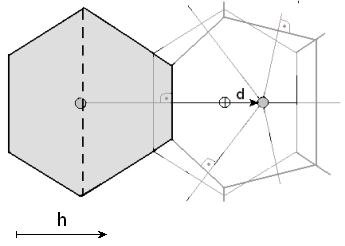}
\end{center}
\end{minipage}
\begin{minipage}{0.75\linewidth}
\begin{enumerate} [(1)]
 \item 
$2 h ( \frac { \sqrt 3 / 2}{1/2} )^{-1} = \frac {2h} {\sqrt 3} + \frac 1 {\sqrt 3}$
\item 
$\frac 1 {\sqrt 3} + 2 \cdot \frac 1 2 \cdot \frac 2 {\sqrt 3} - (h-\frac 12) \frac  2  {\sqrt 3} = \sqrt 3 - \frac 2 {\sqrt 3} h$
\end{enumerate}
\end{minipage}\\ \\
\begin{eqnarray*}
 \bar x
&=& \frac {\int h A(h) dh} {\int A(h) dh} = \frac { \left[ \frac {h^2} 2 + \frac 2 3 h^3 \right]_0^{\frac 12} + \left[ \frac 32 h^2 - \frac 23 h^3 \right]_{\frac 12}^{1-\frac d2} }
{\left[h + h^2\right]_0^{\frac 12} + \left[ 3 h - h^2 \right]_{\frac 12}^{1 - \frac d2}} \\
&=& \frac { \frac  18 + \frac 1 {12} + \frac 32 (1-\frac d2)^2 - \frac 23 (1 - \frac d2)^3 - \frac 38 + \frac 1 {12}}
{\frac 12 + \frac 14 + 3 - \frac 32 d - (1-\frac d2)^2 - \frac 3 2 + \frac 14}\\
&=& \frac 1 2 - \frac 1 6 d
\end{eqnarray*}
\begin{equation*}
 \Rightarrow \boxed{\Delta \bar x =  \frac 1 6 d}
\land  \boxed{\Delta \bar y = 0}
\end{equation*}

Right:
\begin{equation*}
 \Rightarrow \boxed{\Delta \bar x =  \frac 1 6 d}
\land  \boxed{\Delta \bar y = 0}
\end{equation*}

\subsubsection*{y-direction}

Now we want to see what we get when the center node is moved infinitesimally into the y-direction (cf.Fig.\ref{fig:app_geometry:3_6_y}).
\begin{figure}[!ht]
 \begin{center}
\includegraphics[height=0.5\linewidth]{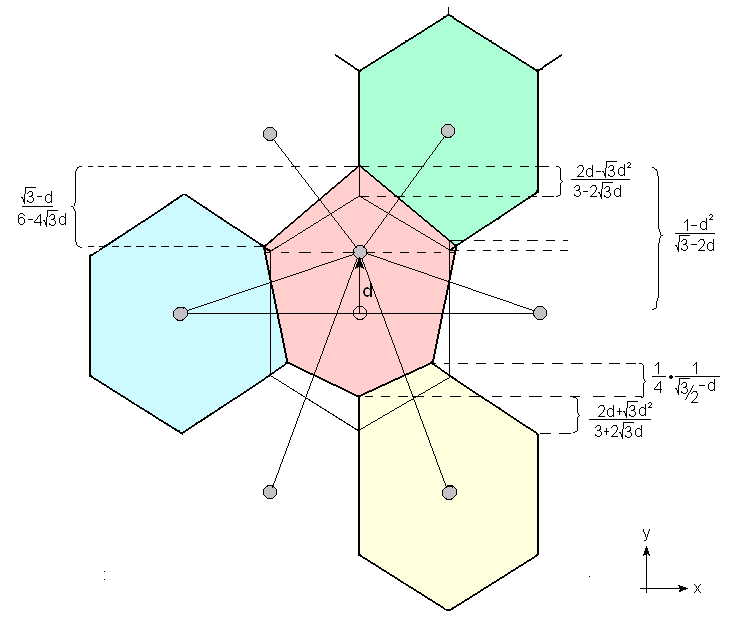}
\end{center}
\caption{Move of center node into y-direction}
\label{fig:app_geometry:3_6_y}
\end{figure}

%

Left in x-direction:\\
\begin{minipage}{0.2\linewidth}
\begin{center}
\includegraphics[width=\linewidth]{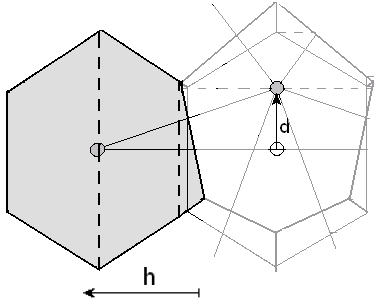}
\end{center}
\end{minipage}
\begin{minipage}{0.8\linewidth}
\begin{enumerate}[(1)]
\item 
$A(h) = \frac 1 {\sqrt 3} + \frac {2 h} {\sqrt 3}$ 
\item 
$A(h) = \frac 2 {\sqrt 3} - \frac 2 {\sqrt 3} (h-\frac 12) = \sqrt{3} - \frac {2h}{\sqrt 3}$
\item 
$A(h) = \frac 3 {\sqrt 3} - \frac 2 {\sqrt 3} \frac {\sqrt 3-d}{2 \sqrt 3} - \left(h-(\frac 1 2 + \frac {\sqrt 3-d}{2 \sqrt 3})\right)(d+ \frac 1 {2 \sqrt 3})$
$= \frac{8 \sqrt 3 + 15 d + 1 - 2 \sqrt 3 d^2}{12} - (d+\frac 1 {2 \sqrt 3})h$
\end{enumerate}
\end{minipage}\\ \\
 \begin{eqnarray*}
\bar x  
&=& \frac {\int h A(h) dh} {\int A(h) dh} \\
&=& \frac {\left[ \frac {h^2} {2\sqrt 3} + \frac {2 h^3} {3 \sqrt 3} \right]_0^{\frac 12}
+\left[\frac {\sqrt{3}h^2}2 - \frac {2h^3}{3\sqrt 3} \right]_{\frac 12}^{1-\frac d {2\sqrt 3}} 
+\left[ \frac{8 \sqrt 3 + 15 d + 1 - 2 \sqrt 3 d^2}{12} \cdot \frac {h^2}2 - (d+\frac 1 {2 \sqrt 3})\frac {h^3}3 \right]_{1-\frac d {2\sqrt 3}}^{1+\frac d {2\sqrt 3}} }
{\left[\frac h {\sqrt 3} + \frac {h^2} {\sqrt 3} \right]_0^{\frac 12}
+\left[ \sqrt{3}h - \frac {2^2}{\sqrt 3}\right]_{\frac 12}^{1-\frac d {2\sqrt 3}} 
+\left[ \frac{8 \sqrt 3 + 15 d + 1 - 2 \sqrt 3 d^2}{12}h - (d+\frac 1 {2 \sqrt 3})\frac {h^2}2 \right]_{1-\frac d {2\sqrt 3}}^{1+\frac d {2\sqrt 3}}}\\
&=& \frac {\frac 3{4\sqrt 3} + \frac 13 d + \Landau(d^2) }{\frac 6{4\sqrt 3} + \frac 13 d + \Landau(d^2)} =\frac 12 + \frac d {3 \sqrt 3}  + \Landau(d^2)
 \end{eqnarray*}
\begin{equation*}
 \Rightarrow \boxed{\Delta \bar x = \frac d {3 \sqrt 3}}
\end{equation*}


Lower right in x-direction:\\
\begin{minipage}{0.2\linewidth}
\begin{center}
\includegraphics[width=\linewidth]{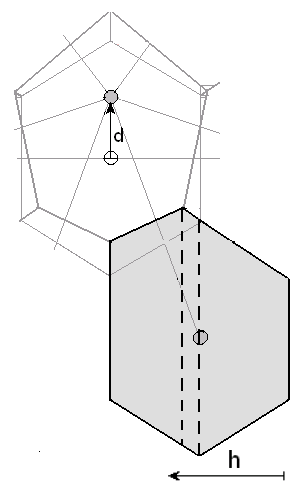}
\end{center}
\end{minipage}
\begin{minipage}{0.8\linewidth}
\begin{enumerate}[(1)]
\item 
$A(h) = \frac 1 {\sqrt 3} + \frac 2 {\sqrt 3} h$ 
\item 
$A(h) = \frac 2 {\sqrt 3}$
\item 
$A(h) = \frac 2 {\sqrt 3} - (\frac 1 {\sqrt 3} + \frac 1 {\sqrt 3 + 2d} )(h = \frac 12 - \frac d {2\sqrt 3})$
$= \frac 3 {\sqrt 3} - h \frac {2(\sqrt 3 +d)}{3 + 2 \sqrt 3 d}$
\end{enumerate}
\end{minipage}\\ \\
 \begin{eqnarray*}
\bar x  
&=& \frac {\int h A(h) dh} {\int A(h) dh} 
= \frac {\left[ \frac {h^2} {2 \sqrt 3} + \frac 2 {3\sqrt 3} h^3 \right]_0^{\frac 12}
+\left[ \frac {h^2} {\sqrt 3} \right]_{\frac 12}^{\frac 12 + \frac d {2 \sqrt 3}} 
+\left[ \frac {3h^2} {2\sqrt 3} - \frac {h^3}3 \frac {2(\sqrt 3 +d)}{3 + 2 \sqrt 3 d} \right]_{\frac 12 + \frac d {2 \sqrt 3}}^{1}}
{\left[ \frac h {\sqrt 3} + \frac {h^2} {\sqrt 3} \right]_0^{\frac 12}
+\left[ \frac {2h} {\sqrt 3}  \right]_{\frac 12}^{\frac 12 + \frac d {2 \sqrt 3}} 
+\left[ \frac 3 {\sqrt 3} - h^2 \frac {\sqrt 3 +d}{3 + 2 \sqrt 3 d} \right]_{\frac 12 + \frac d {2 \sqrt 3}}^{1}}\\
&=& \frac{(\frac 4 3 - \frac d {4 \sqrt 3})(\sqrt 3 + 2d) - \frac 2 {\sqrt 3} - \frac 2 3 d + \frac {\sqrt 3}{12} + \frac d 4 + \frac d {12}}
{(\frac 9 4 - \frac d {2 \sqrt 3})(\sqrt 3 + 2d) - \sqrt 3 - d + \frac {\sqrt 3} 4 + \frac 3 4 d} = \frac {\frac 9 {12} \sqrt 3 + \frac {25}{12} d}{\frac 6 {4} \sqrt 3 + \frac {15}{4} d} \\
&=& \frac 12 + \frac 5 {36 \sqrt 3}d + \Landau(d^2)
 \end{eqnarray*}
\begin{equation*}
 \Rightarrow \boxed{\Delta \bar x = \frac 5 {36 \sqrt 3}d }
\end{equation*}

Lower right in y-direction:\\
\begin{minipage}{0.2\linewidth}
\begin{center}
\includegraphics[width=\linewidth]{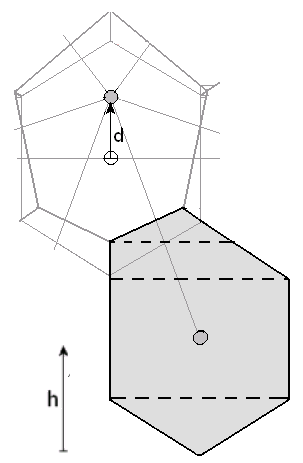}
\end{center}
\end{minipage}
\begin{minipage}{0.8\linewidth}
\begin{enumerate}[(1)]
\item 
$A(h) = 2 h \sqrt 3$ 
\item 
$A(h) = 1 $
\item 
$A(h) = 1 - \left(h- \frac{\sqrt 3 -10d}{2 \sqrt 3(\sqrt 3 - 2d)}\right)(\sqrt 3 - 2d)$
$= \frac 52 - \frac 5 {\sqrt 3} d - h(\sqrt 3 - 2d)$
\item 
$A(h) = \frac 52 - \frac 5 {\sqrt 3} d - h(\sqrt 3 - 2d) - (h - \frac {\sqrt 3} 2)(2 \sqrt 3 - 2d) $
$= 4 - \frac 5 {\sqrt 3} d - 2 h (\sqrt 3 - d)$
\end{enumerate}
\end{minipage}\\ \\
 \begin{eqnarray*}
\bar y  
&=& \frac {\int h A(h) dh} {\int A(h) dh} 
= \frac {
\left[  \right]_0^{\frac 1 {2\sqrt 3}}
+\left[ \right]_{\frac 1 {2 \sqrt 3}}^{\frac{\sqrt 3 -10d}{2 \sqrt 3(\sqrt 3 - 2d)}} 
+\left[ \right]_{\frac{\sqrt 3 -10d}{2 \sqrt 3(\sqrt 3 - 2d)}}^{ \frac {\sqrt 3} 2 }
+\left[ \right]_{\frac {\sqrt 3} 2}^{\frac 2 {\sqrt 3}}
}{
\left[  \right]_0^{\frac 1 {2 \sqrt 3}}
+\left[ \right]_{\frac 1 {2 \sqrt 3}}^{\frac{\sqrt 3 -10d}{2 \sqrt 3(\sqrt 3 - 2d)}} 
+\left[ \right]_{\frac{\sqrt 3 -10d}{2 \sqrt 3(\sqrt 3 - 2d)}}^{}
+\left[ \right]_{\frac {\sqrt 3} 2}^{\frac 2 {\sqrt 3}}
}
  \end{eqnarray*}
\begin{equation*}
 \Rightarrow \boxed{\Delta \bar y =}
\end{equation*}

\subsubsection*{z-direction}

\begin{figure}[!ht]
 \begin{center}
\includegraphics[height=0.5\linewidth]{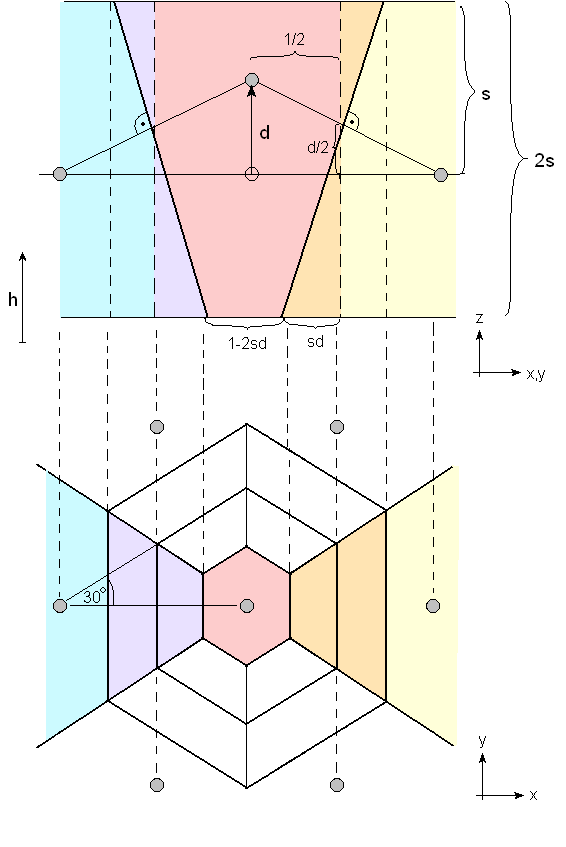}
\end{center}
\caption{Move of center node into z-direction}
\label{fig:appendix_geometry:3_6_z}
\end{figure}

Center:\\
$ 
A(h) =  2 \left[ \sqrt 3 h^2 \right]_0^{\frac w {2 \sqrt 3}} + \left[w h \right]_{\frac w {2 \sqrt 3}}^{\frac {3w} {2 \sqrt 3}} = \frac {\sqrt 3} 2 (1 - 2 ds + 2 dh) = \frac {\sqrt 3} 2 - 2 \sqrt 3 d(s-h)
$
\begin{eqnarray*}
 \bar z 
&=& \frac {\int h A(h) dh} {\int A(h) dh} = \frac{ \left[ \frac {\sqrt 3} 4 h^2  - \sqrt 3 d h^2 s + \frac {2 h^3} {\sqrt 3} d \right]_0^{2s} }
{ \left[ \frac {\sqrt 3} 2 h - 2 \sqrt 3 d h s + h^2 d \sqrt 3 \right]_0^{2s}}\\
&=&  \frac{\sqrt 3 s^2 - 4 \sqrt 3 d s^3 + \frac {16}{\sqrt 3} d s^3 }
{\sqrt 3 s - 4 \sqrt 3 d s^2 + 4 \sqrt 3 d s^2} = s + \frac 4 3 d s^2
\end{eqnarray*}
\begin{equation*}
 \Rightarrow \boxed{\Delta \bar z =  \frac 4 3 s^2 d}
\end{equation*}

the other voronoi cells yield:
\begin{eqnarray*}
 A(h) &=& \left[ \frac h {\sqrt 3} + \frac {h^2}{\sqrt 3} \right]_0^{\frac 12} + \left[ \frac {3h} {\sqrt 3} - \frac {h^2}{\sqrt 3} \right]_{\frac 12}^1 + \left[ \frac {3 h} {\sqrt 3} - \frac{h^2}{\sqrt 3} \right]_1^{\tilde v}\\
&=& \frac 1 {2 \sqrt 3} + \frac 1 {4 \sqrt 3} \sqrt 3 \tilde v - \frac {\tilde v^2} {\sqrt 3} -\frac 3 {2 \sqrt 3} + \frac 1 {4 \sqrt 3} = - \frac 1 {2 \sqrt 3} + \sqrt 3 \\
&&- d(h-s) \sqrt 3- \frac 1 {\sqrt 3} + \frac {2 d (h-s)}{\sqrt 3} = \frac{\sqrt 3} 2 - d(h-s) \frac 1 {\sqrt 3}
\end{eqnarray*} where $\tilde v = 1 - \frac v 2 + \frac 12  = 1 - d (h-s)$
\begin{eqnarray*}
 \bar z 
&=& \frac {\int h A(h) dh} {\int A(h) dh}
 = \frac{ \left[ \frac 3 4 h^2 - d (\frac{h^3} 3 - \frac {sh^2} 2) \right]_0^{2s}} { \left[\frac 32 h - d (\frac {h^2} 2 - sh) \right]_0^{2s}} 
= \frac{3s^2 -d (\frac 83 s^3 - 2 s^3)}{3 s - d (2s^2 - 2s^2)} 
= s - \frac 2 9 s^2 d
\end{eqnarray*}
\begin{equation*}
 \Rightarrow \boxed{\Delta \bar z =  -\frac 2 9 s^2 d}
\end{equation*}

\subsubsection{(6,3)-Tesselation}

For the (6,3)-map we restrict the calculations of the shifts to those resulting by a move of the central node in the $z$-direction (cf.Fig.\ref{fig:appendix_geometry:6_3_z})

\subsubsection*{z-direction}

\begin{figure}[!ht]
 \begin{center}
\includegraphics[height=0.5\linewidth]{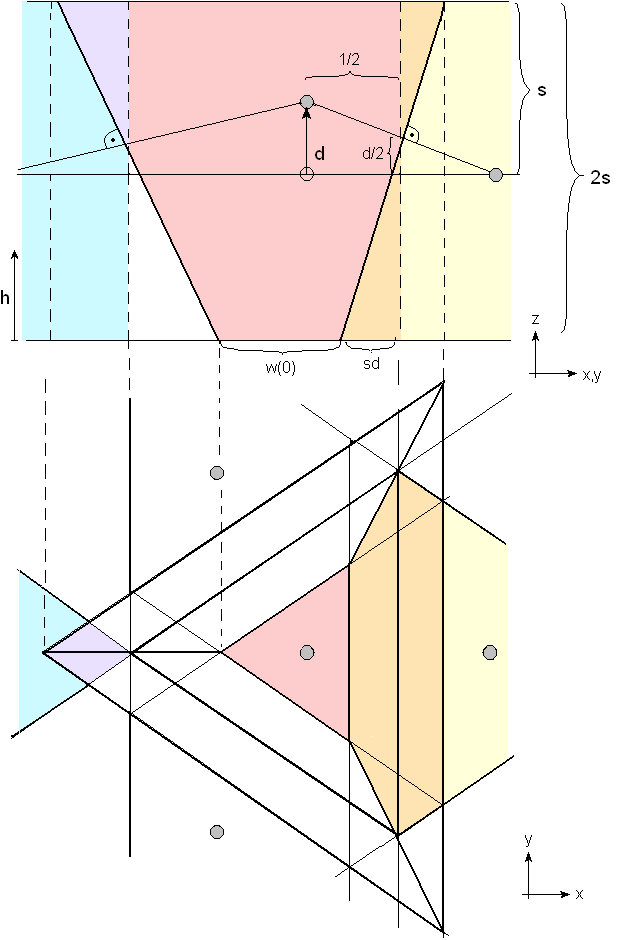}
\end{center}
\caption{Move of center node into z-direction}
\label{fig:appendix_geometry:6_3_z}
\end{figure}

Center:
\begin{equation*}
 A(h) = w(h) \cdot \frac {w(h)} {\sqrt 3} = 3 \sqrt 3 (\frac 12 + d(h-s))^2 = 3 \sqrt 3 (\frac 14 + d(h-s)) + \Landau(d^2)
\end{equation*}
where $w(h)$ is the height of the triangle of the slice at height $h$.
\begin{eqnarray*}
 \bar z 
&=& \frac {\int h A(h) dh} {\int A(h) dh}
= \frac{\left[ \frac {h^2} 8 + \frac {dh^3}3 -\frac{dsh^2}2 \right]_0^{2s} }
{ \left[ \frac h 4 + \frac {dh^2} 2 - hds \right]_0^{2s}}=  \frac{\frac {s^2}2 + \frac{8ds^3}3 - 2 ds^3 }{\frac s2 + 2 ds^2 - 2 ds^2} 
= s + \frac 4 3 d s^2
\end{eqnarray*}
\begin{equation*}
 \Rightarrow \boxed{\Delta \bar z =  \frac 4 3 s^2 d}
\end{equation*}

For the centroids of the three neighboring voronoi cells, we get:
\begin{eqnarray*}
 A(h) &=& \frac {\int h A(h) dh} {\int A(h) dh}
= \sqrt 3 ( \frac 3 4 - dh +ds)
\end{eqnarray*} where $\tilde v = 1 - \frac v 2 + \frac 12  = 1 - d (h-s)$
\begin{eqnarray*}
 \bar z 
&=& \frac {\int h A(h) dh} {\int A(h) dh}
 = \frac{\left[ \frac {3h^2} 8 - \frac {dh^3}3 +\frac{dsh^2}2 \right]_0^{2s} }
{ \left[\frac {3h} 4 - \frac{dh^2}2 + dsh \right]_0^{2s}}
= \frac{\frac {3s^2}2 + \frac 23 d s^3}{ \frac {3s}2 - 2 ds^2 + 2 ds^2} = s - \frac 4 9 s^2 d
\end{eqnarray*}
\begin{equation*}
 \Rightarrow \boxed{\Delta \bar z =  -\frac 4 9 s^2 d}
\end{equation*}
All the centroids of the other cells like the blue one in the figure are only shifted by an order of $d^2$ or higher, which can be neglected as mentioned.

\subsection{Calculation the volume change}

To shorten the discussion here, we will again focus only on the case of a move of the center node in $z$-direction.\\
\\
\subsubsection*{z-direction}

By looking closely to Fig.\ref{fig:appendix_geometry:3_6_z} and \ref{fig:appendix_geometry:6_3_z}, we can easily conclude that for each of both maps, the volume of the neighboring voronoi cells only changes by a term of $\Landau(d^2)$. Indeed, this can be seen if consider that the deformation is point symmetric about the point ($\frac 12$,$\frac d2$). Therefore an upper bound for the volume growth can be given by $sd\cdot \frac d2 \sim \Landau(d^2)$. Since this holds for all the finite neighboring voronoi cells, the volume deficit of the center voronoi cell is of the same order and can be therefore neglected.




%


\section{Proofs} \label{ch:app_distr}

\begin{proof}[Proof of Thm. \ref{thm:distributions:inversion}]
\begin{enumerate}[1)]
 \item For all $x\in \setR$:
\begin{equation*}
 P(F^{-1}(U) \leq x) = P(\inf \setgen{y\in \setR}{F(y)=U} \leq x) = P(U \leq F(x)) = F(x)
\end{equation*}
\item
\begin{equation*}
P(F(X) \leq u) = P(X \leq F^{-1}(u)) = F(F^{-1}(u)) = u
\end{equation*}
\end{enumerate}
Hence this theorem is completely proven.
\end{proof}

\begin{proof}[Proof of Thm. \ref{thm:distributions:rejection_1}]

 \begin{enumerate}[1)]
 \item Take a Borel set $B \subseteq$ A and let $B_x$ be the section of $B$ at $x$. By Tonelli's theorem follows then
$$P((X,cUf(X))\in B) = \int \int_{B_x} \frac 1 {c f(x)} du f(x) dx = \frac 1 c \int_B du dx$$
and since the area of $A$ is c, the first part has been proven.
\item It's sufficient to show that  $P(X\in B) = \int_B f(x) dx$ holds:
\begin{eqnarray*}
 P(X\in B) &=& P((X,U)\in B_1 = \setgen{(x,u)}{x\in B, 0 \leq u \leq cf(x)}) \\
&=& \frac {\int \int_{B_1} du dx}{\int \int_{A} du dx} = \frac 1c \int_B c f(x) dx = \int_B f(x) dx
\end{eqnarray*}
\end{enumerate}
\end{proof}

\begin{proof}[Proof of Thm. \ref{thm:distributions:rejection_2}]
Let $B$ be an arbitrary Borel set. Then we obtain what we want by calculation
\begin{eqnarray*}
 P(Y \in B) &=& \sum_{i=1}^\infty P (X_1 \notin A, \dots , X_{i-1} \notin A,X_i \in B \bigcap A) \\
&=& \sum_{i=1}^\infty (1-p)^{i-1} P(X_1 \in A \bigcap B) = \frac 1 {1-(1-p)} P(X_1 \in A \bigcap B) \\
&=& \frac { P(X_1 \in A \bigcap B) }p
\end{eqnarray*}
If $X_1$is uniformly distributed in $A_0$ then we get
\begin{equation*}
 P(Y \in B) = \frac { P(X_1 \in A \bigcap B) }p = \frac { P(X_1 \in A \bigcap B) }{P(X\in A_1)} = \frac {\int_{A_0 \bigcap A \bigcap B} dx }{\int_{A_0} dx} \cdot \frac {\int_{A_0 } dx }{\int_{A_0\bigcap A} dx} = \frac {\int_{A \bigcap B} dx }{\int_{A} dx} 
\end{equation*}
which concludes this proof.
\end{proof}


\chapter{Software Manual} \label{ch:software}

This manual describes how to install and use the software, that was implemented in the course of this thesis and is included as source code on the CD-ROM which is attached to the printouts.

\section{System requirements}

To install and use the GRiSOM software library, the following requirements have to been fulfilled

\begin{itemize}
 \item ISO conform C++ compiler (e.g. g++ of the GNU compiler collection)
 \item Boost C++ Libraries \cite{boost} installed
 \item MAPM \cite{mapm} (version $\ge$4.9.5) and MAPMX \cite{mapmx} Libraries installed (see below)
 \item \texttt{make} utility to install it using the Makefile
 \item (for visualization:) gnuplot installed
 \item (for creation of API:) doxygen installed \cite{doxygen}
\end{itemize}

\section{Installation manual}

Before starting to install the GRiSOM software library, make sure that the system requirements above are met. Both the MAPM \cite{mapm} and the MAPMX \cite{mapmx} Libraries are thereby included on the CD-ROM. They can be found in the subfolder \texttt{<CD-directory>/GRiSOM/libs}. Just unpack the two zipped tar files and follow the instructions in the \texttt{README} and \texttt{INSTALL} files. To install now the GRiSOM libraries, just perform the following steps

\begin{enumerate}[1.]
\item Unpack the sources (\texttt{GRiSOM.tar.gz}) located in the subfolder \texttt{<CD-directory>/GRiSOM}
\item Call $\texttt{make libs}$ to compile the sources. This will create the static libraries \texttt{libgrisom\_core.a} and \texttt{libgrisom\_aux.a} as well as the corresponding shared libraries \texttt{libgrisom\_core.so.*} and \texttt{libgrisom\_aux.so.*} 
\item \texttt{make install} will try to copy the static library files and headers into the directories specified in the Makefile. By default, this is \texttt{/usr/local/lib} for libraries and \texttt{/usr/local/include/grisom} for header files.\footnote{For further detail be referred to the commentaries in the Makefiles.} 
Alternatively, \texttt{make shared\_install} can be used to copy and install the shared libraries instead of the static ones (needs root permissions!).
\end{enumerate}

Now the libraries are installed and ready to be used e.g. for compiling the simulation program \texttt{main\_test} used as the simulation program in the numerical stability analysis as well as any testing program of the software package (cf. section \ref{sec:src:tester}).

\section{Programs}

The following programs are provided by the GRiSOM package. 

\subsection{Simulation program \texttt{main\_test}} \label{sec:software:main_test}

The simulation program \texttt{main\_test} is a tool to study the stability limit of a SOM configuration. It can be compiled by simply using the command
\begin{center}
\texttt{make simulation}
 \end{center}
It is then executed by passing a set of 12 arguments (separated by whitespaces) to it to specify the configuration of the SOM and the search procedure:
\begin{itemize}
 \item[-] <type of spaces> (available: EUCL\_EUCL, EUCL\_DISK, DISK\_DISK)
 \item[-] <first Schlaefli symbol, defining shape of cells in the map> (integer)
 \item[-] <second Schlaefli symbol, defining number of neighbors at each vertex in the map>  (integer)
 \item[-] <type of neighborhood function> (available: GAUSS, NN, VQ)
 \item[-] <size of epsilon> (double)
 \item[-] <size of sigma> (double) [is ignored if NN,VQ is used]
 \item[-] <start value for search> (double)
 \item[-] <step size of search> (double)
 \item[-] <number of steps> (integer)
 \item[-] <number of measurements per step> (integer)
 \item[-] <number of adaption processes between two measurements> (integer)
 \item[-] <euclidean map: number of nodes along one edge, hyperbolic: number of layers> (integer)
\end{itemize}

Fig.\ref{fig:man:konsole} shows such an execution of \texttt{main\_test}.

\begin{figure}[!ht]
 \begin{center}
\includegraphics[width=\linewidth]{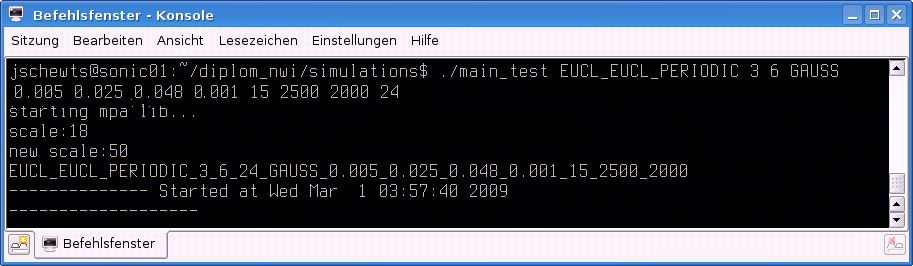}
\end{center}
\caption{execution of \texttt{main\_test}}
\label{fig:man:konsole}
\end{figure}

According to the passend arguments, it simulates a GRiSOM with both Euclidean periodic map and feature space, a hexagonal map with a 24x24 grid and a Gaussian neighborhood function with $\sigma=0.025$. The learn rate $\eps$ of the adaption is thereby set to $0.005$. The whole search run consists of $15$ steps starting at $s=0.048$ with a step size of $0.001$. In each of this steps $2500$ snapshots, which lie $2000$ adaption steps apart, are presented to the \texttt{Mean\_Extra\_Dim\_Analyzer}, which then calculates the equilibrium state by averaging the states in the snapshots and computes then the error according to the deviation of the equilibrium state thus obtained and the reference equilibrium at $s=0$. The equilibrium states for each search step and the results for the error are then saved in files in the subdirectory of \texttt{./data/} specified by the string in the line above the starting time.

\subsection{Testing programs [\texttt{tester\_*}]}
The testing programs like \texttt{tester\_tess} and \texttt{tester\_stochastics} including their source code are provided here as instructive examples how to use several features of the software library. They can be furthermore easily adapted to test extensions like new distributions or spaces. Just make sure, that you added the location of your library extension in the \texttt{INC\_PATH} (if necessary) and the library name in the \texttt{EXTERN\_LIBS} variable in the Makefile and call \texttt{make tester\_<name>} to compile the testing program.

\section{API}

An already generated pdf version of the API can be found in the \texttt{GRiSOM/API}-subdirectory, but if you want to get a HTML version or maybe only the API of the core library, just use \texttt{make full\_api} and \texttt{make core\_api}, respectively, to create them.

\chapter{Simulation data/results}

\begin{table}[!ht]
\begin{tabular}{|l||c|c|c|c|c|c|}\hline
neighb.fct. & 6x6 & 10x10 & 12x12 & 16x16 & 24x24 & 32x32  \\ \hline \hline
\multicolumn{7}{c}{periodic boundary conditions} \\ \hline
NN &0.412\err{0.002}&0.242\err{0.001}&0.202\err{0.001}&0.152\err{0.001}&0.101\err{0.001}&0.076\err{0.001}\\ \hline
G$\sigma$=0.025&0.178\err{0.001}&0.108\err{0.001}&0.091\err{0.002}&0.067\err{0.001}&0.054\err{0.001}& --- \\ \hline
G$\sigma$=0.05&0.175\err{0.001}&0.110\err{0.001}&0.105\err{0.001}&0.114\err{0.002}&0.105\err{0.005}&0.102\err{0.005}\\ \hline
G$\sigma$=0.10&0.211\err{0.001}&0.225\err{0.001}&0.220\err{0.005}&0.211\err{0.002}&0.204\err{0.002}&0.201\err{0.002}\\ \hline
G$\sigma$=0.12& --- &0.264\err{0.005}& --- & --- & --- & ---  \\ \hline
G$\sigma$=0.20&0.443\err{0.002}&0.418\err{0.002}&0.412\err{0.001}&0.406\err{0.002}&0.405\err{0.002}&0.403\err{0.002}\\ \hline
VQ &0.175\err{0.001}&0.105\err{0.001}&0.087\err{0.002}&0.065\err{0.001}&0.044\err{0.002}&0.030\err{0.003}\\ \hline
\multicolumn{7}{c}{no boundary conditions} \\ \hline
G$\sigma$=0.05&0.166\err{0.001}&0.104\err{0.001}&0.098\err{0.002}&0.100\err{0.002}&0.100\err{0.002}&0.098\err{0.001} \\ \hline
\end{tabular}
\caption{Results for (3,6)-map}
\label{tab:results:eucl_3_6}
\end{table}

\begin{table}[!ht]
\begin{tabular}{|l||c|c|c|c|c|c|} \hline
neighb.fct. & 6x6 & 10x10 & 12x12 & 16x16 & 24x24 & 32x32  \\ \hline \hline
\multicolumn{7}{c}{periodic boundary conditions} \\ \hline
NN &0.258\err{0.001}&0.154\err{0.001}&0.128\err{0.001}&0.096\err{0.001}&0.064\err{0.001}&0.048\err{0.001}\\ \hline
G$\sigma$=0.025&0.123\err{0.002} & 0.071\err{0.002}&0.060\err{0.002}&---&---&---\\ \hline
G$\sigma$=0.05&0.120\err{0.001} & 0.109\err{0.002}&0.113\err{0.002}&0.110\err{0.002}&0.104\err{0.002 }&0.102\err{0.002}\\ \hline
G$\sigma$=0.10&0.229\err{0.002}&0.215\err{0.001}&0.212\err{0.005}&0.207\err{0.003}&0.201\err{0.002}&0.200\err{0.002 }\\ \hline
G$\sigma$=0.12& --- &0.250\err{0.005}& --- & --- & --- & ---  \\ \hline
G$\sigma$=0.20&0.394\err{0.002}&0.392\err{0.005}&0.392\err{0.002}&0.388\err{0.002}&0.385\err{0.005}&0.385\err{0.002}\\ \hline
VQ &0.115\err{0.005}&0.068\err{0.001}&0.058\err{0.001}&0.042\err{0.002}&0.029\err{0.003}&0.022\err{0.002}\\ \hline
\multicolumn{7}{c}{no boundary conditions} \\ \hline
G$\sigma=0.05$ &0.116\err{0.001}&0.107\err{0.001}&0.107\err{0.002}&0.100\err{0.002}&0.095\err{0.002}&0.093\err{0.001}\\ \hline
G$\sigma=0.2$ &0.336\err{0.001}&0.320\err{0.002}&0.318\err{0.002}&0.313\err{0.002}&0.307\err{0.001}&0.305\err{0.001}\\ \hline
\end{tabular}
\caption{Results for (4,4)-map}
\label{tab:results:eucl_4_4}
\end{table}

\begin{table}[!ht]
\begin{tabular}{|l||c|c|c|c|c|c|}\hline
neighb.fct. & 6x6 & 10x10 & 12x12 & 16x16 & 24x24 & 32x32  \\ \hline \hline
\multicolumn{7}{c}{periodic boundary conditions} \\ \hline
NN &0.223\err{0.001}&0.132\err{0.001}&0.110\err{0.001}&0.083\err{0.001}&0.055\err{0.001}&0.041\err{0.001}\\ \hline
G$\sigma$=0.025&0.121\err{0.002}&0.071\err{0.002}&0.061\err{0.002}&---&---&---\\ \hline
G$\sigma$=0.05&0.122\err{0.001}&0.106\err{0.005}&0.106\err{0.002}&0.104\err{0.002}&0.104\err{0.002}&0.102\err{0.002}\\ \hline
G$\sigma$=0.10&0.214\err{0.002}&0.214\err{0.001}&0.210\err{0.001}&0.205\err{0.001}&0.202\err{0.002}&0.201\err{0.002}\\ \hline
G$\sigma$=0.12& --- &0.250\err{0.005}& --- & --- & --- & ---  \\ \hline
G$\sigma$=0.20&0.370\err{0.001}&0.365\err{0.001}&0.364\err{0.001}&0.362\err{0.001}&0.361\err{0.001}&0.359\err{0.002}\\ \hline
VQ &0.120\err{0.001}&0.068\err{0.001}&0.058\err{0.002}&0.042\err{0.001}&0.028\err{0.002}&0.022\err{0.002}\\ \hline
\multicolumn{7}{c}{no boundary conditions} \\ \hline
G$\sigma$=0.05&0.121\err{0.001}&0.097\err{0.002}&0.095\err{0.002}&0.094\err{0.001}&0.094\err{0.001}&0.091\err{0.001}\\ \hline
\end{tabular}
\caption{Results for (6,3)-map}
\label{tab:results:eucl_6_3}
\end{table}

\begin{table}[!ht]
\begin{center}
\begin{tabular}{|l||c|c|c||c|c|} \hline
& \multicolumn{3}{c||}{(3,7)} & \multicolumn{2}{c|}{(3,8)}  \\ \hline
neighbourhood & 3 layers & 4 layers & 5 layers &  3 layers & 4 layers  \\ \hline \hline
NN &0.209\err{0.002}&0.085\err{0.002}&0.032\err{0.001}&0.049\err{0.001}&0.010\err{0.001}\\ \hline
Gauss $\sigma$=0.05&0.081\err{0.001}&0.034\err{0.003}&0.013\err{0.001}&0.023\err{0.001}&0.004\err{0.001}\\ \hline
Gauss $\sigma$=0.1&0.085\err{0.002}&0.034\err{0.002}&0.013\err{0.001}&0.023\err{0.001}&0.004\err{0.001}\\ \hline
Gauss $\sigma$=0.5&0.104\err{0.001}&0.040\err{0.001}&0.015\err{0.001}&0.023\err{0.001}&0.004\err{0.001}\\ \hline
Gauss $\sigma$=1.0&0.235\err{0.005}&0.093\err{0.001}&0.035\err{0.001}&0.033\err{0.001}&0.007\err{0.001}\\ \hline
Gauss $\sigma$=2.0&0.512\err{0.003}&0.223\err{0.002}&0.090\err{0.005}&0.088\err{0.002}&0.022\err{0.003}\\ \hline
\end{tabular}
\end{center}
\caption{Results for triangular hyperbolic map in euclidean cartesian space}
\end{table}

\begin{table}[!ht]
\begin{center}
\begin{tabular}{|l||c|c|c|} \hline
& \multicolumn{3}{c|}{(4,5)}  \\ \hline
neighbourhood & 3 layers & 4 layers & 5 layers  \\ \hline \hline
NN &0.114\err{0.001}&0.044\err{0.002}&0.0165\err{0.001}\\ \hline
Gauss $\sigma$=0.05&0.123\err{0.002}&0.047\err{0.001}&0.015\err{0.001}\\ \hline
Gauss $\sigma$=0.1&0.120\err{0.001}&0.047\err{0.001}&0.015\err{0.001}\\ \hline
Gauss $\sigma$=0.5&0.120\err{0.001}&0.042\err{0.002}&0.015\err{0.001}\\ \hline
Gauss $\sigma$=1.0&0.198\err{0.002}&0.070\err{0.005}&0.023\err{0.001}\\ \hline
Gauss $\sigma$=2.0&0.470\err{0.010}&0.195\err{0.005}&---\\ \hline
\end{tabular}
\end{center}
\caption{Results for square hyperbolic map in euclidean cartesian space}
\end{table}

\begin{table}[!ht]
\begin{center}
\begin{tabular}{|l||c|c|c|c|} \hline
& \multicolumn{4}{c|}{(7,3)}  \\ \hline
neighbourhood & 3 layers & 4 layers & 5 layers & 6 layers \\ \hline \hline
NN&0.342\err{0.003}&0.260\err{0.003}&0.182\err{0.002}&0.115\err{0.005}\\ \hline
Gauss $\sigma$=0.05&0.207\err{0.001}&0.171\err{0.001}&0.122\err{0.002}&---\\ \hline
Gauss $\sigma$=0.1&0.208\err{0.002}&0.177\err{0.001}&0.126\err{0.002}&---\\ \hline
Gauss $\sigma$=0.5&0.375\err{0.001}&0.284\err{0.002}&0.198\err{0.002}&---\\ \hline
Gauss $\sigma$=1.0&0.684\err{0.002}&0.600\err{0.005}&0.435\err{0.005}&---\\ \hline
\end{tabular}
\end{center}
\caption{Results for (7,3)-hyperbolic map in euclidean cartesian space}
\end{table}

\begin{table}[!ht]
\begin{center}
\begin{tabular}{|l||c|c|c|c|} \hline
& \multicolumn{4}{c|}{(6,4)}  \\ \hline
neighbourhood & 3 layers & 4 layers & 5 layers & 6 layers  \\ \hline \hline
NN&0.134\err{0.001}&0.047\err{0.001}&0.018\err{0.002}&0.007\err{0.002}\\ \hline
Gauss $\sigma$=0.05&0.148\err{0.001}&0.059\err{0.001}&0.016\err{0.001}&---\\ \hline
Gauss $\sigma$=0.1&0.147\err{0.001}&0.061\err{0.001}&0.018\err{0.003}&---\\ \hline
Gauss $\sigma$=0.5&0.131\err{0.002}&0.053\err{0.001}&0.018\err{0.001}&---\\ \hline
Gauss $\sigma$=1.0&0.168\err{0.002}&0.060\err{0.002}&0.020\err{0.005}\\ \hline
Gauss $\sigma$=2.0&0.440\err{0.010}&0.165\err{0.010}&---\\ \hline
\end{tabular}
\end{center}
\caption{Results for (6,4)-hyperbolic map in euclidean cartesian space}
\end{table}

\begin{table}[!ht]
\begin{center}
\begin{tabular}{|l||c|c|} \hline
$s$ & mean deviation $\bar u_3$ & normalized mean deviation $\frac {\bar u_3} s$  \\ \hline \hline
\multicolumn{3}{c}{run 1}  \\ \hline
0.00000&0.00000&0.00000  \\ \hline
0.10000&0.03494&0.34943  \\ \hline
0.20000&0.09088&0.45438  \\ \hline
0.30000&0.13948&0.46495  \\ \hline
0.40000&0.19054&0.47635  \\ \hline
0.50000&0.23466&0.46933  \\ \hline
\multicolumn{3}{c}{run 2}  \\ \hline
0.02000&0.00003&0.00166  \\ \hline
0.04000&0.00007&0.00173  \\ \hline
0.06000&0.00015&0.00248  \\ \hline
0.08000&0.00040&0.00504  \\ \hline
0.10000&0.03708&0.37082  \\ \hline
0.12000&0.04539&0.37827  \\ \hline
0.14000&0.05378&0.38411  \\ \hline
\multicolumn{3}{c}{run 3}  \\ \hline
0.08000&0.00041&0.00515  \\ \hline
0.08500&0.00115&0.01357  \\ \hline
0.09000&0.00456&0.05071  \\ \hline
0.09500&0.01997&0.21020  \\ \hline
0.10000&0.03057&0.30570  \\ \hline
\multicolumn{3}{c}{run 4}  \\ \hline
0.08200&0.00062&0.00760  \\ \hline
0.08400&0.00092&0.01097  \\ \hline
0.08600&0.00200&0.02324  \\ \hline
0.08800&0.00377&0.04285  \\ \hline
0.09000&0.00463&0.05145  \\ \hline
\multicolumn{3}{c}{run 5}  \\ \hline
0.08300&0.00067&0.00805  \\ \hline
0.08400&0.00106&0.01258  \\ \hline
0.08500&0.00111&0.01301  \\ \hline
0.08600&0.00192&0.02227  \\ \hline
0.08700&0.00234&0.02694  \\ \hline
0.08800&0.00329&0.03737  \\ \hline
\multicolumn{3}{c}{final run}  \\ \hline
0.07700&0.00034&0.00442  \\ \hline
0.07800&0.00040&0.00508  \\ \hline
0.07900&0.00040&0.00503  \\ \hline
0.08000&0.00050&0.00626  \\ \hline
0.08100&0.00072&0.00883  \\ \hline
0.08200&0.00096&0.01165  \\ \hline
0.08300&0.00527&0.06344  \\ \hline
0.08400&0.01325&0.15780  \\ \hline
0.08500&0.01575&0.18526  \\ \hline
0.08600&0.01695&0.19712  \\ \hline
\end{tabular}
\end{center}
\caption{Search runs for (6,3)-map with N=16}
\label{tab:result:search_runs}
\end{table}

\chapter{UML Diagrams}
 \begin{figure}[!ht]
\begin{center}
\includegraphics[width=\linewidth]{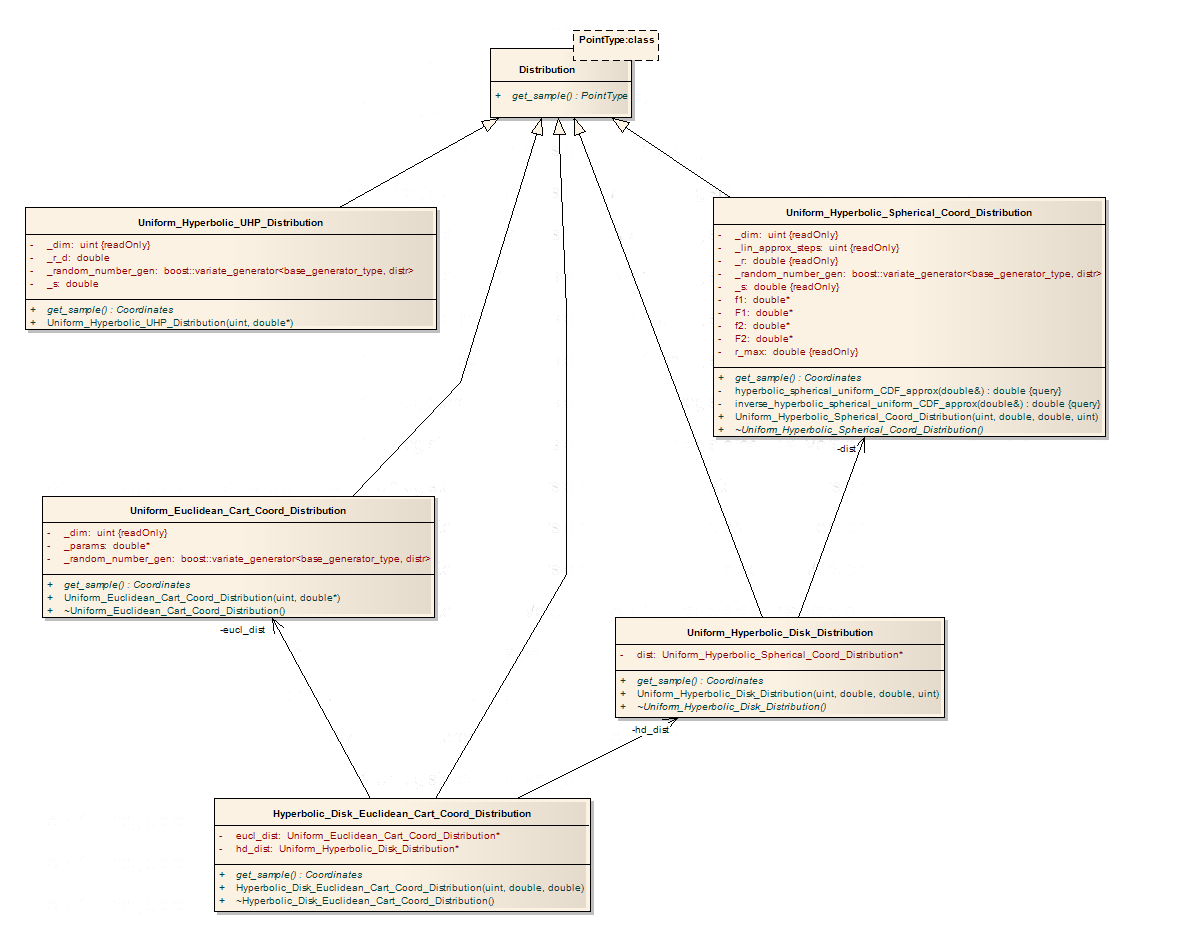}
\end{center}
\caption{Class diagram of distributions}
\label{fig:uml:dist}
\end{figure}
\begin{landscape}
 \begin{figure}[!ht]
\begin{center}
\includegraphics[width=\linewidth]{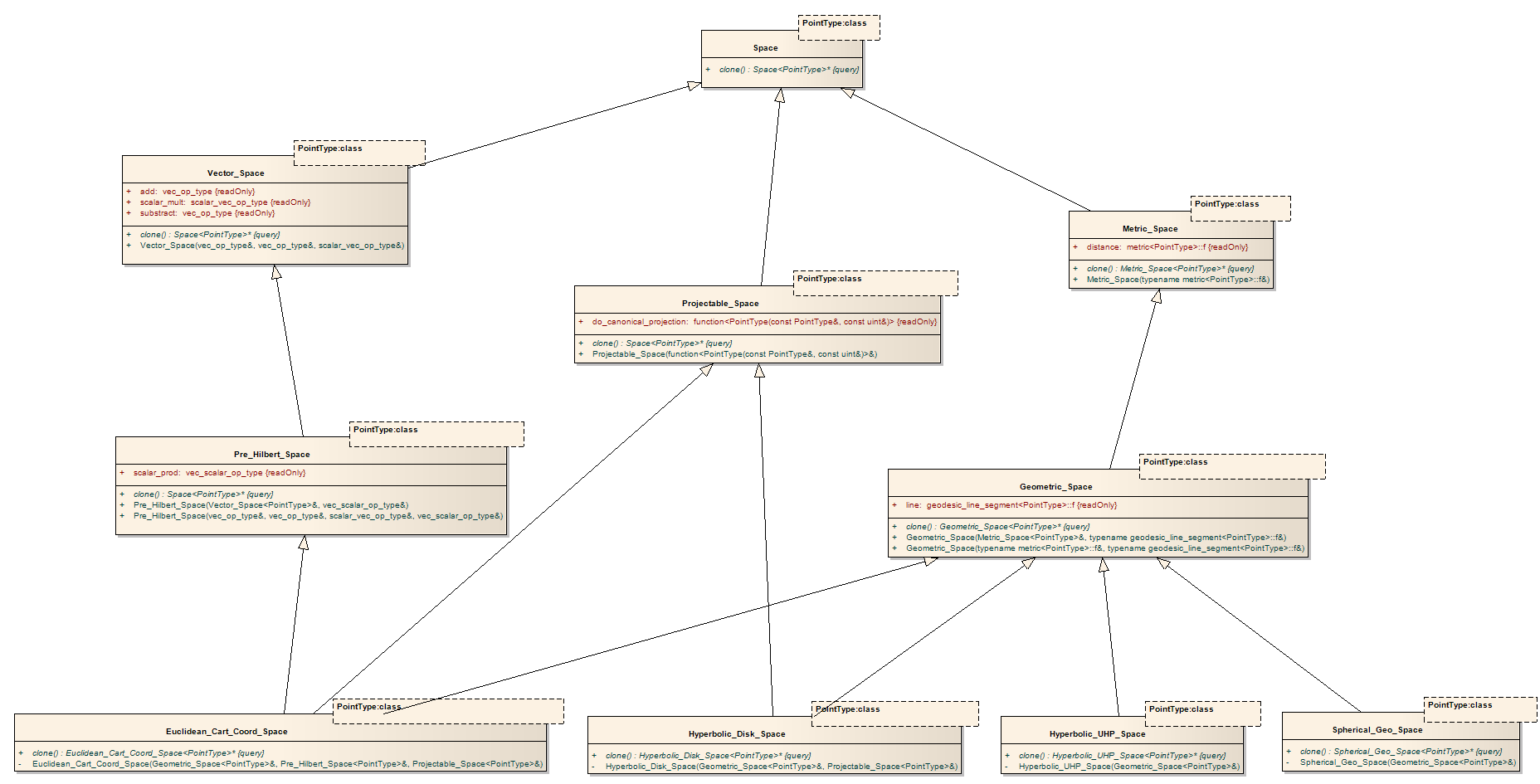}
\end{center}
\caption{Class diagram of spaces}
\label{fig:uml:space}
\end{figure}
\newpage
 \begin{figure}[!ht]
\begin{center}
\includegraphics[width=\linewidth]{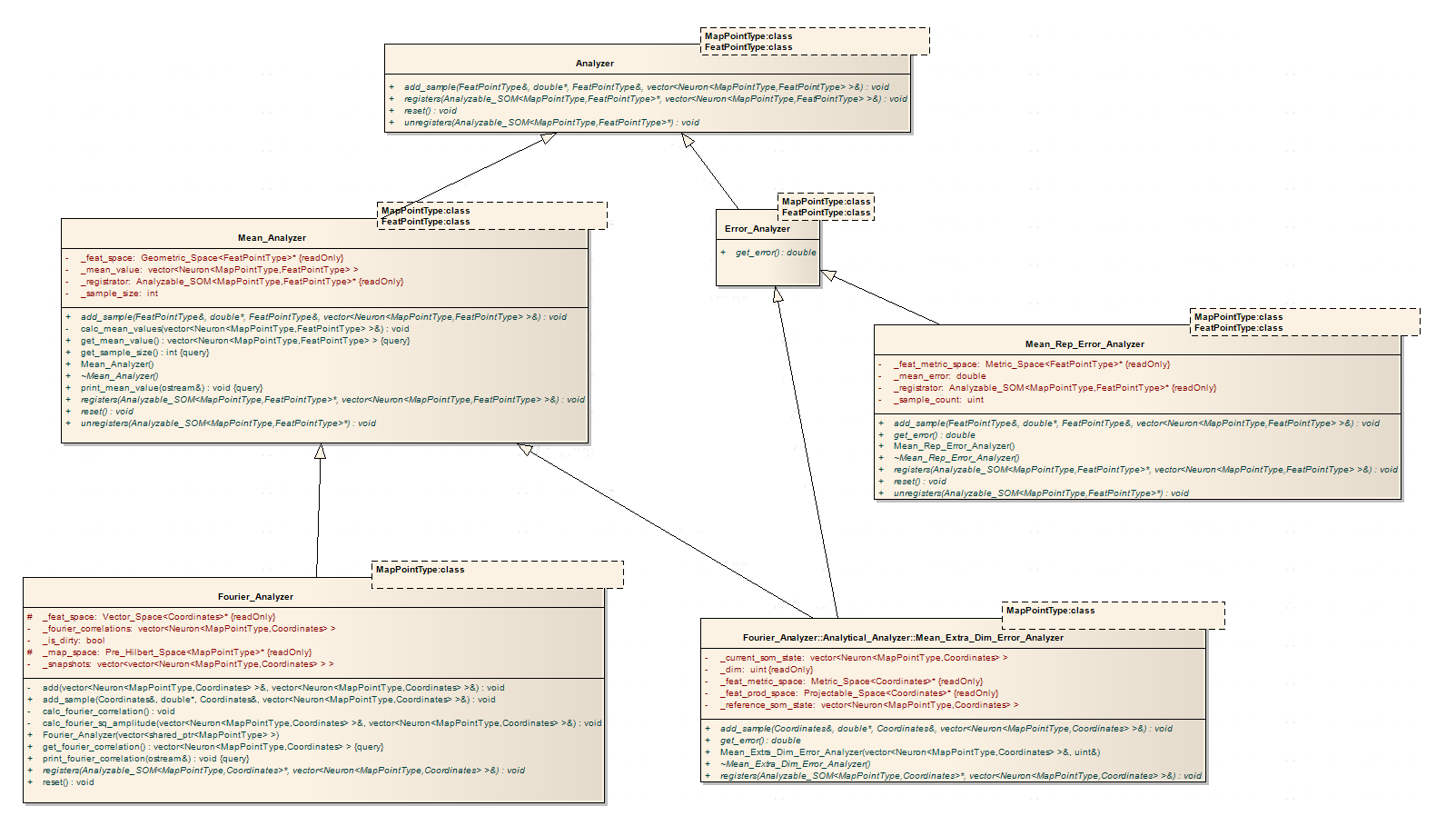}
\end{center}
\caption{Class diagram of analyzers}
\label{fig:uml:analyzers}
\end{figure}
\newpage
 \begin{figure}[!ht]
\begin{center}
\includegraphics[width=\linewidth]{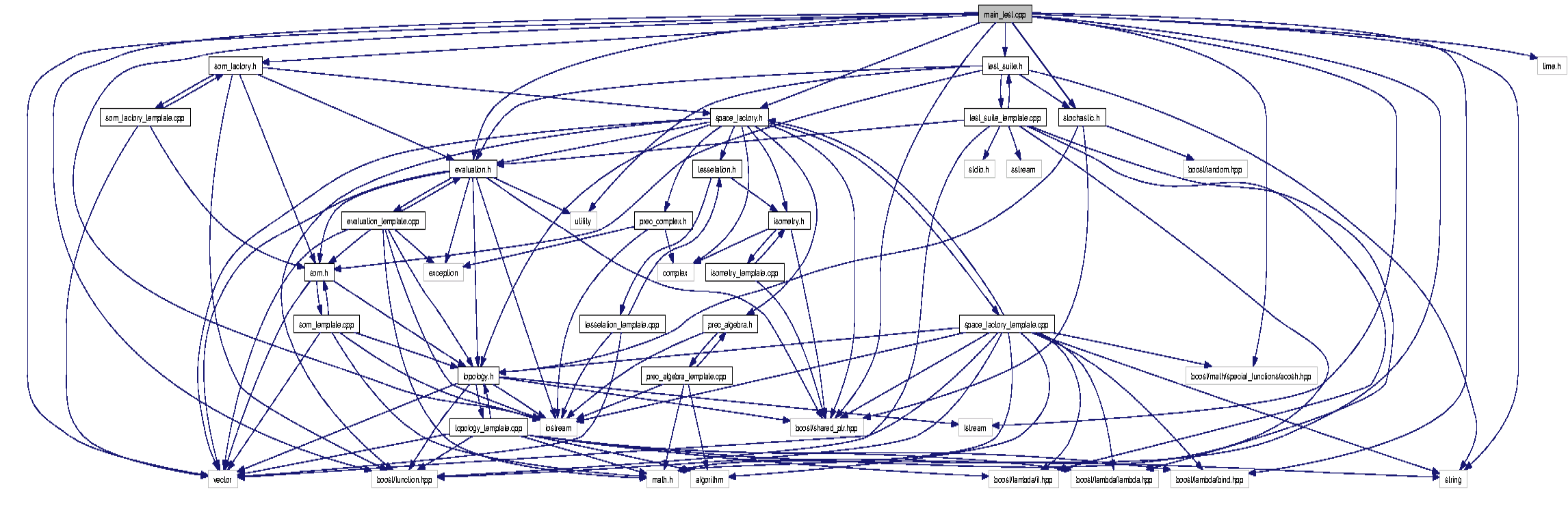}
\end{center}
\caption{Dependency diagram of the projects (header files only)}
\label{fig:uml:dep}
\end{figure}
\end{landscape}

\abschliessendeerklaerung{\author}

\end{document}